\documentclass{article}

\usepackage{preprint_style,times}

\usepackage[utf8]{inputenc}
\usepackage[T1]{fontenc}
\usepackage{amsmath}
\usepackage{amssymb}
\usepackage{amsthm}
\usepackage{amsfonts}
\usepackage{microtype}
\usepackage{graphicx}
\usepackage{enumitem}
\usepackage{booktabs}
\usepackage{array}
\usepackage{tabularx}
\usepackage{makecell}
\usepackage{wrapfig}
\usepackage{bm}
\usepackage{xcolor}
\usepackage{refcount}
\usepackage{url}
\usepackage{subcaption}
\usepackage{needspace}
\newcommand{\keepwithfloat}[1][0.8]{\Needspace{#1\textheight}}
\usepackage{hyperref}
\usepackage{titletoc}
\usepackage{cleveref}

\usepackage{xcolor}
\definecolor{paperteal}{HTML}{006a98}
\usepackage{hyperref}
\hypersetup{
  colorlinks=true,
  linkcolor=paperteal,
  citecolor=paperteal,
  urlcolor=paperteal
}

\DeclareRobustCommand{\linktoappendixtoc}[1]{{%
  \hypersetup{hidelinks}%
  \hyperlink{appendixcontents}{#1}%
}}
\DeclareRobustCommand{\appendixtocmark}{%
  \nobreak\hspace{0.45em}%
  \hyperlink{appendixcontents}{%
    {\normalfont\scriptsize\itshape\textcolor{black!45}{$\uparrow$\,toc}}%
  }%
}

\makeatletter
\renewcommand\paragraph{\@startsection{paragraph}{4}{\z@}{0pt}{-1em}{\normalsize\bf}}
\makeatother

\newcommand{\refsep}{\nobreak\hskip.125em\relax}
\newcommand{\defrefformats}[8]{%
  \crefformat{#2}{#3#1##2#7##1#8##3}%
  \Crefformat{#2}{#5#1##2#7##1#8##3}%
  \crefrangeformat{#2}{#4#1##3#7##1#8##4 to\nobreakspace##5#7##2#8##6}%
  \Crefrangeformat{#2}{#6#1##3#7##1#8##4 to\nobreakspace##5#7##2#8##6}%
  \crefmultiformat{#2}{#4#1##2#7##1#8##3}{ and\nobreakspace##2#7##1#8##3}{, ##2#7##1#8##3}{ and\nobreakspace##2#7##1#8##3}%
  \Crefmultiformat{#2}{#6#1##2#7##1#8##3}{ and\nobreakspace##2#7##1#8##3}{, ##2#7##1#8##3}{ and\nobreakspace##2#7##1#8##3}%
}
\defrefformats{\refsep}{figure}{Fig.}{Figs.}{Figure}{Figures}{}{}
\defrefformats{\refsep}{equation}{Eq.}{Eqs.}{Eq.}{Eqs.}{(}{)}
\defrefformats{\refsep}{section}{Sec.}{Secs.}{Sec.}{Secs.}{}{}
\defrefformats{\refsep}{subsection}{Sec.}{Secs.}{Sec.}{Secs.}{}{}
\defrefformats{\refsep}{subsubsection}{Sec.}{Secs.}{Sec.}{Secs.}{}{}
\defrefformats{\refsep}{table}{Tab.}{Tabs.}{Tab.}{Tabs.}{}{}
\defrefformats{\refsep}{appendix}{App.}{Apps.}{App.}{Apps.}{}{}
\defrefformats{\refsep}{subappendix}{App.}{Apps.}{App.}{Apps.}{}{}
\defrefformats{\refsep}{subsubappendix}{App.}{Apps.}{App.}{Apps.}{}{}
\defrefformats{\nobreakspace}{proposition}{Proposition}{Propositions}{Proposition}{Propositions}{}{}
\newcommand{\appref}[1]{App.\refsep\hyperref[#1]{\ref*{#1}}}
\newcommand{\apprefpair}[2]{Apps.\refsep\hyperref[#1]{\ref*{#1}} and~\hyperref[#2]{\ref*{#2}}}

\newcommand{\bad}[1]{\textcolor{red}{#1}}
\newcommand{\superbad}[1]{\textcolor{red!80!black}{\textbf{#1}}}

\newcommand{\oursLong}{\underline{G}rassmann \underline{R}estricted \underline{O}ccupied-\underline{O}rbital \underline{T}raining}
\newcommand{\ours}{GROOT}
\newcommand{\ourLossLong}{\underline{R}esidual \underline{O}ptimal-gauge \underline{C}onditioning-aware \underline{K}S-\underline{E}q.\,\underline{T}raining}
\newcommand{\ourLoss}{ROCKET}
\newcommand{\oursOrthonormalLong}{\underline{O}rtho\underline{N}ormal-\underline{Loss}}
\newcommand{\oursOrthonormal}{ON-Loss}

\newcommand{\model}{GSM}

\newcommand{\relsens}[1]{\mathcal{S}_{\mathrm{rel},#1}}

\newcommand{\smashsqrt}[1]{\raisebox{-0.4pt}[0pt][0pt]{$\sqrt{#1}$}}

\newtheoremstyle{propositioninline}
  {\topsep}
  {\topsep}
  {\itshape}
  {}
  {\bfseries}
  {.}
  { }
  {}
\newtheoremstyle{propositionbreak}
  {\topsep}
  {\topsep}
  {\itshape}
  {}
  {\bfseries}
  {.}
  {\newline}
  {}

\theoremstyle{propositioninline}
\newtheorem{proposition}{Proposition}

\theoremstyle{propositionbreak}
\newtheorem{appendixproposition}[proposition]{Proposition}



\newcommand{\proofref}[1]{%
  \ \textnormal{(\hyperref[proof:#1]{proof in App.~\ref*{proof:#1}})}}

\newif\ifpropdisplay
\propdisplayfalse
\newcommand{\propeq}[1]{%
  \ifpropdisplay
    \begin{equation}#1\end{equation}%
  \else
    $#1$%
  \fi}

\newcommand{\propP}{\ifpropdisplay\widetilde{\mathbf{P}}\else\mathbf{P}\fi}
\newcommand{\propR}{\ifpropdisplay\widetilde{\mathbf{R}}\else\mathbf{R}\fi}
\newcommand{\propW}{\ifpropdisplay\widetilde{\mathbf{W}}\else\mathbf{W}\fi}

\newcommand{\propReachabilityBody}{%
Let $\propP_0$ and $\propP_1$ be physical densities with the same rank in the orthonormal basis, $\propP_0,\propP_1 \in \mathcal{M}_{(N_\mathrm{elec}, B)}$.
Then there exists a generator
$\mathbf{X}\in\mathfrak{p}$
such that
\propeq{
\propP_1
=
\exp\!\left(\mathbf{X}\right)
\propP_0
\exp\!\left(-\mathbf{X}\right).
}
Consequently,
$\mathbf{U}:=\exp(\mathbf{X}) \in\mathrm{SO}(B)$,
and $\mathbf{U}$ transports the occupied subspace of
$\propP_0$
to that of
$\propP_1$.%
}

\newcommand{\propGeneralRotationBody}{%
Let $\propW \in \mathrm{SO}(B)$ satisfy
\propeq{
\propP_1
=
\propW
\propP_0
\propW^{\mathsf{T}}.
}
Then
\propeq{
\propW
=
\exp\!\left(\mathbf{X}\right)
\mathbf{V}
}
with
$\mathbf{X}\in\mathfrak{p}$
and a block-diagonal orbital rotation
$\mathbf{V}\in\mathrm{SO}(B)$
that commutes with the reflection,
$\mathbf{V}\propR_0=\propR_0\mathbf{V}$.%
}

\newcommand{\propGrootValidityBody}{%
For an arbitrary matrix
$\hat{\mathbf{M}}\in\mathbb{R}^{B\times B}$,
the generator
$\hat{\mathbf{X}}_\mathrm{\ours}(\hat{\mathbf{M}})$
of \cref{eq:groot_map}
lies in
$\mathfrak{p}$
for every input.%
}

\newcommand{\propReadoutExpressivityBody}{%
Every generator
$\mathbf{X}\in\mathfrak{p}$
can be reached by the \ours{} map,
since
$\hat{\mathbf{X}}_\mathrm{\ours}(\mathbf{X})=\mathbf{X}$.%
}

\title{Physics-Aligned Electronic Ground-State Learning Improves Generalization}

\author{
\textbf{Eike S.~Eberhard}$^{*,1,2,3}$,
\hspace{0.6cm}
\textbf{Xaver Kainz}$^{*,1,2}$,
\\[0.6em]
\hspace{0.1cm}\textbf{Viktor Kotsev}$^{1}$,
\hspace{0.6cm}
\textbf{Abdulrahman Aldossary}$^{4}$,
\hspace{0.6cm}
\textbf{Stephan G\"unnemann}$^{1,2,3}$
\\[0.4em]
\normalfont
\{\texttt{e.eberhard}, \,\texttt{x.kainz}\}\texttt{@tum.de} \\[0.3em]
$^1$Technical University of Munich (TUM) \quad
$^2$Munich Data Science Institute (MDSI) \\
$^3$Munich Center for Machine Learning (MCML) \quad
$^4$NVIDIA \qquad
$^*$Equal contribution \\
}

\hypersetup{
    pdftitle={Physics-Aligned Electronic Ground-State Learning Improves Generalization},
    pdfauthor={Eike S. Eberhard, Xaver Kainz, Viktor Kotsev, Abdulrahman Aldossary, Stephan G\"unnemann},
}

\iclrfinalcopy

\begin{document}
\maketitle
\lhead{Preprint}  %

\begin{abstract}
    \looseness=-1
    Machine-learned interatomic potentials (MLIPs) excel at in-distribution tasks, accelerating drug and material development, yet they struggle to generalize out-of-distribution.
    We propose to push the cost-accuracy Pareto frontier by designing observable-agnostic electronic ground-state descriptor models (\model s) with computational costs situated between MLIPs and Kohn-Sham density functional theory (KS-DFT).
    We align the learning objectives and architectures of \model s with the governing equations of KS-DFT by enforcing physical constraints and removing optimization pressure on unphysical or irrelevant degrees of freedom.
    In our size-extrapolation experiments from QM9 to QM40, our combined contributions \oursOrthonormalLong{} (\oursOrthonormal) and \oursLong{} (\ours) reach a 79.1\% energy and 83.4\% force mean absolute error (MAE) reduction over previous state-of-the-art density \model s.
    For Hamiltonian \model s, \oursOrthonormal{} and \ourLossLong{} (\ourLoss) together reduce the energy and force MAEs of the strongest baseline by 99.8\% and 95.9\%, respectively.
    Using a self-consistency rejection criterion, we filter out extrapolation errors on QMugs, rejecting fewer than 0.4\% of predictions while reaching an energy MAE of 0.07~mHa.
    Finally, we demonstrate the efficiency of label-free self-consistency fine-tuning, and transfer \model s to reactive chemistry in Transition1x, reaching energy errors below chemical accuracy.
\end{abstract}

\begin{figure}[!h]
    \vspace{-0.4cm}
    \includegraphics[width=\linewidth]{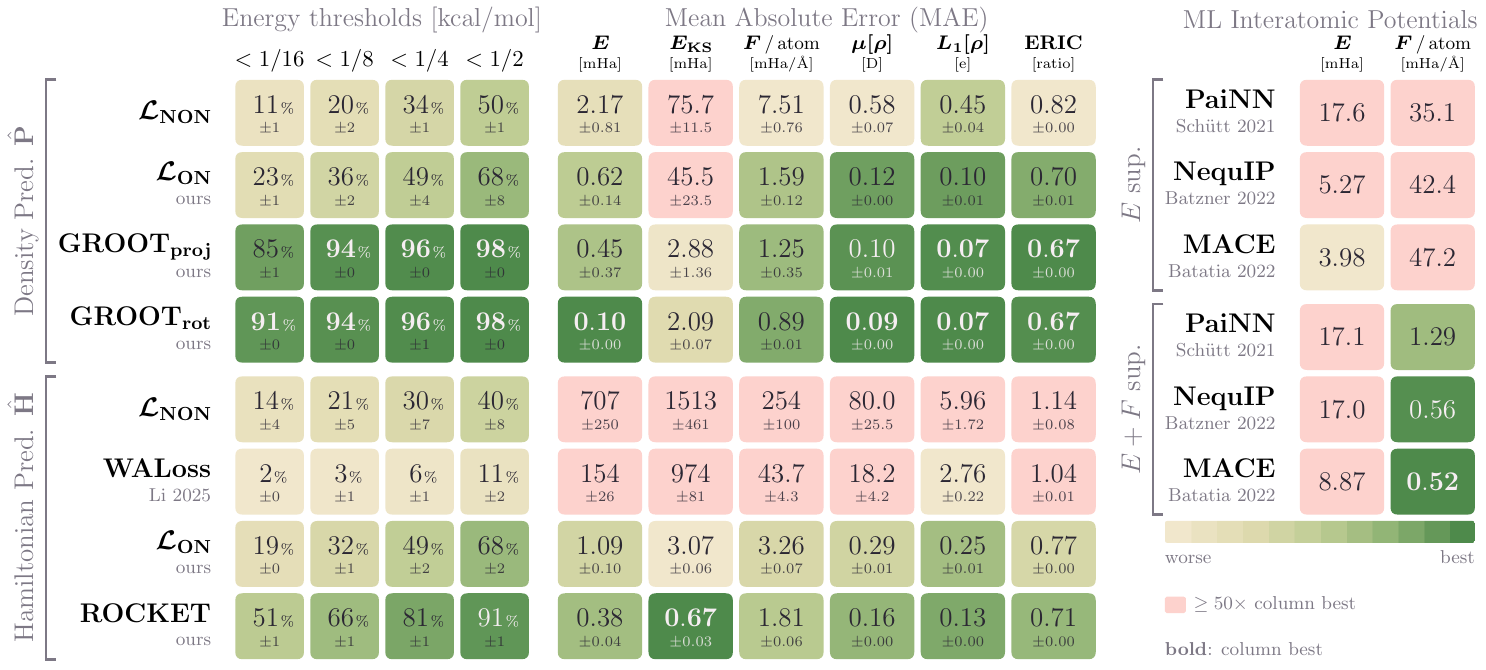}
    \vspace{-0.6cm}
    \centering
    \caption{
    \textbf{\model{} size extrapolation from QM9 to molecules up to $4\times$ larger QM40 molecules.}
    The \ours{} models generalize robustly, \textbf{reaching errors below half chemical accuracy} ($\sim0.8$~mHa) \textbf{on 98\% of the tested molecules.}
    All observables are computed at a budget of exactly one Fock build averaged over 3 training seeds each. \textbf{The \ours{} variants are the best performing models for almost all observables} except for the Kohn--Sham (self-consistency) energy where \ourLoss{} is best. NequIP and MACE trained on forces still outperform even the best \model s on forces, but physically aligned \model s achieve an order of magnitude better energies compared to all trained MLIPs.}
    \label{fig:qm40_size_extrapolation}
    \vspace{-0.3cm}
\end{figure}

\section{Introduction}

\looseness=-1
Many computational chemistry applications aim to accurately approximate the electronic ground state of molecules and materials.
The rapid developments in this field are driven by the abundance of applications enabling the \textit{in silico} prescreening of potential materials and drugs before moving into resource- and time-intensive laboratory experiments or clinical trials.
Since the exact computation of the ground state is intractable for almost all systems of interest, the art of computational chemistry is to find the correct trade-off between accuracy and computational cost for a given application~\citep{szaboModernQuantumChemistry1996}.
Contemporary method development pushes this Pareto frontier as far as possible.
For many applications, KS-DFT \citep{kohnSelfConsistentEquations1965} is the method of choice \citep{antypasNerscWorkloadAnalysis2014, haunschildComprehensiveAnalysisHistory2019}.
At a high level, its computational efficiency can be attributed to its lower-dimensional representation of the ground state based on the electron density $\rho:\mathbb{R}^3\to[0,\infty)$,
compared to wavefunction methods $\psi:\mathbb{R}^{N_\mathrm{elec} \times 3}\to\mathbb{C}$~\citep{cramerEssentialsComputationalChemistry2014}.

\textbf{Ground-state descriptor models (\model s) predicting the Hamiltonian or density matrix} have been investigated as a means to accelerate~\citep{liuTransferableAccelerationMethod2026,eberhardTransferableScfAcceleration2026} or completely circumvent~\citep{schuttUnifyingMachineLearning2019}
the computationally demanding iterative solvers which drive KS-DFT calculations~\citep{lehtolaOverviewSelfConsistent2020}.
These solvers use finite basis sets to discretize the governing equations of KS-DFT,
which yields a generalized eigenvalue problem (gEVP)~\citep{roothaanNewDevelopmentsMolecular1951}.
Present \textbf{state-of-the-art \model s are not well-aligned with the goal of learning physical ground states}, motivating the following contributions:
\begin{enumerate}[leftmargin=2em]
    \looseness=-1
    \item Current literature has converged on supervising \model s on the matrix representations of the gEVP (\cref{par:non-orthonormal-basis})
    measuring matrix errors in a non-orthonormal (NON) basis, whose spectral weighting is ill-aligned with physical observables. We show mathematically and empirically how \textbf{\oursOrthonormalLong{} (\oursOrthonormal)} cures this by measuring errors in an orthonormal (ON) basis.

    \item On the problem side of the gEVP, entire families of Hamiltonian matrices yield the same ground state.
    The established loss formulation views these physically equivalent Hamiltonians as different targets. We systematically treat this invariance w.r.t. changes in the Hamiltonian by introducing \textbf{\ourLossLong{} (\ourLoss{})}.

    \looseness=-1
    \item On the solution side, physically valid density matrices are constrained to a manifold (the \emph{Grassmannian}), but existing readout heads predict in the entire ambient space of symmetric matrices.
    We introduce \textbf{\oursLong{} (\ours{})}, a matrix readout for $SO(3)$-equivariant graph neural networks (GNNs) that restricts predictions to valid densities, removing the leading order energy error term compared to unconstrained baselines.

    \looseness=-1
    \item We develop a \textbf{practical framework} under a \textbf{strict budget of one Fock build}, the cost-dominant operation of KS-DFT.
    For this, we adopt a computationally cheap \textbf{first-order energy correction}, identify a \textbf{\model-native uncertainty measure} for rejecting unreliable predictions, and generalize self-consistency training~\citep{zhangSelfConsistencyTraining2024} for \textbf{label-free fine-tuning of foundational \model s to new chemical domains} such as transition states. These constitute key advantages of \model s over MLIPs that might justify their increased evaluation costs for some applications.
\end{enumerate}
\looseness=-1
At this single Fock-build budget, the evaluation costs of \model s are in between MLIPs and machine-learned (ML) approaches that require multiple Fock builds~\citep{dickHighlyAccurateConstrained2021,
gaoLearningEquivariantNon2024,eberhardDerivativeInformedLearning2026,luiseAccurateScalableExchangecorrelation2026,kimHighOrderEquivariant2025,
kimMachineLearningHamiltonians2026}.
In our size extrapolation experiments on QMugs, we train on molecules with at most 150 electrons and evaluate on up to 660 electrons (\cref{fig:qmugs_ood_extrapolation}).
Here, \oursOrthonormal{} alone reduces energy, force, dipole and density errors by 87--99\% for density \model s and by 93--99.5\% for Hamiltonian \model s.
\ours{}\textsubscript{proj} and \ourLoss{} \textbf{compound} on top of that, reaching 99.6\% and 99.9\% lower energy errors than the NON baselines.
At 0.21 and 0.56~mHa (without rejections), their \emph{mean} energy errors lie below chemical accuracy and are 14 and 5 times lower than the best energy-trained MLIP, while they reject fewer than 0.4\% of predictions.
Using the predicted ground states as self-consistent field (SCF) initial guesses reduces the number of solver iterations needed to converge.
For the QM9-trained models evaluated on QM40, \oursOrthonormal{} lowers the density \model's Effective Relative Iteration Count (ERIC)~\citep{eberhardTransferableScfAcceleration2026} from 0.82 to 0.70, and \ours{} lowers it further to 0.67, a third fewer iterations than the traditional MINAO initial guess.
After label-free self-consistency fine-tuning on Transition1x~\citep{schreinerTransition1xDatasetBuilding2022}, the QMugs-trained \ours{}\textsubscript{proj} and \ourLoss{} reach energy errors of 0.66 and 0.74~mHa, below chemical accuracy, while rejecting fewer than 1\% of predictions.

\section{Background}
\label{sec:background}

The governing equation of KS-DFT in finite basis sets $\{\chi_\mu(\mathbf{r})\}_{\mu=1}^{B}$ is a gEVP of the Hamiltonian $\mathbf{H}\in\mathbb{R}^{B\times B}$ and \textit{overlap} matrix $S_{\mu\nu}=\int_{\mathbb{R}^3}\chi_\mu(\mathbf{r})\chi_\nu(\mathbf{r})\,\mathrm{d}\mathbf{r}$,
with solutions (eigenvectors) $\mathbf{C}\in\mathbb{R}^{B\times B}$
\begin{equation}
    \mathbf{H}\mathbf{C}
    =
    \mathbf{S}\mathbf{C}\operatorname{diag}(\epsilon_i)\:.
    \label{eq:generalized_eigenvalue_problem}
\end{equation}
In this context, the solutions are also known as \emph{molecular orbitals} $\mathbf{C}$.
For further details and the connection of \cref{eq:generalized_eigenvalue_problem} to the underlying continuous theory, we refer to \appref{app:finite-basis-ks} and \citet{lehtolaOverviewSelfConsistent2020}.
By convention, the eigenvalues $\epsilon_i$ (\emph{orbital energies}) are enumerated ascendingly.
The $O=N_\mathrm{elec}/2$ lowest eigenvectors (columns of $\mathbf{C}$) form the \emph{occupied block} $\mathbf{C}_O\in\mathbb{R}^{B\times O}$, while the remaining $V=B-O$ are \emph{virtual}.
The density matrix %
\begin{equation}
    P_{\mu\nu}=(\mathbf{C}_O\mathbf{C}_O^{\mathsf{T}})_{\mu\nu} = \textstyle{\sum_{i=1}^{O}} \: C_{\mu i} C_{\nu i}
    \label{eq:P_from_C}
\end{equation}
\looseness=-1
encodes the electron density in the finite basis.
The Hamiltonian $\mathbf{H}$ depends on the density matrix $\mathbf{P}$, which in turn depends on the coefficients $\mathbf{C}$, so \cref{eq:generalized_eigenvalue_problem} is non-linear.
It can be linearized by iteratively solving for $\mathbf{C}_t$ given $\mathbf H_{t-1} := \mathbf{H}(\mathbf{P}(\mathbf{C}_{t-1}))$, which reduces the original problem to the \textbf{iterative solution of the gEVP} (\cref{eq:generalized_eigenvalue_problem})~\citep{roothaanNewDevelopmentsMolecular1951}.
This approach \textbf{requires an initial guess $\mathbf{P}^{(0)}$} to start the iteration.
The Superposition of Atomic \emph{Densities} (SAD) method and its minimal-basis variant MINAO construct such a guess in density space by combining spherically averaged atomic density matrices~\citep{almlofPrinciplesDirectScf1982,vanlentheStartingScfCalculations2006}.
The ground state $\mathbf{C}^*$ is a fixed point of the iterative solver, and all three matrices $\mathbf{C}^*$, $\mathbf{P}^*$, and $\mathbf{H}^*$ are complete descriptors allowing access to all observables of the ground state.
These matrices are commonly represented in atom-centered basis sets, which are \textbf{non-orthonormal (NON)} $\mathbf{S}\neq\mathbf{I}$.
We can perform a change of basis
$\widetilde{\chi}_{\mu}(\mathbf{r})
=
\sum_{\nu=1}^{B}\chi_{\nu}(\mathbf{r})(\mathbf{S}^{-1/2})_{\nu\mu}$,
to obtain \textbf{orthonormal (ON)} basis representations
\begin{equation}
    \smash{\widetilde{\mathbf{C}}}
    =
    \mathbf{S}^{1/2}\mathbf{C},
    \qquad
    \smash{\widetilde{\mathbf{P}}}
    =
    \mathbf{S}^{1/2}\mathbf{P}\mathbf{S}^{1/2},
    \qquad
    \smash{\widetilde{\mathbf{H}}}
    =
    \mathbf{S}^{-1/2}\mathbf{H}\mathbf{S}^{-1/2},
    \qquad
    \smash{\widetilde{\mathbf{S}}}
    =
    \mathbf{I},
    \label{eq:lowdin_transform}
\end{equation}
which reduce the gEVP to the regular eigenvalue problem $\smash{\widetilde{\mathbf{H}}\widetilde{\mathbf{C}}=\widetilde{\mathbf{C}}\operatorname{diag}(\epsilon_i)}$~\citep{lowdinNonOrthogonalityProblemConnected1950}.

\paragraph{Infinitely many representations $(\mathbf{H}, \mathbf{C})$ encode the same physical ground state.}
Only the density matrix $\mathbf{P}$ is in one-to-one correspondence with this state.
On the problem side of \cref{eq:generalized_eigenvalue_problem}, for Hamiltonians $\mathbf{H}$, one can identify constant shifts
\begin{equation}
    \smash{\widetilde{\mathbf{H}}}
    \mapsto
    \smash{\widetilde{\mathbf{H}}}+c\mathbf{I},
    \qquad
    c \in \mathbb{R}
    \label{eq:hamiltonian_gauge}
\end{equation}
as a degree of freedom, leaving the eigenvectors $\mathbf{C}$ invariant.
Similarly, so-called \emph{virtual level shifts}
$\smash{\widetilde{\mathbf{H}}_\xi
:=
\widetilde{\mathbf{H}}^*
+
\xi(\mathbf{I}-\widetilde{\mathbf{P}}^*)}$,
for an arbitrary $\xi \geq 0$
only increase the eigenvalues of the virtual block $\smash{\widetilde{\mathbf{C}}_V}$ by $\xi$.
Both of these operations leave $\mathbf{C}$ and hence $\mathbf{P}$ unchanged.
On the solution side of \cref{eq:generalized_eigenvalue_problem}, we note that the solution space $\mathbf{C}$ carries considerably more degrees of freedom than the density matrix $\mathbf{P}$ it yields.
In the mapping $\mathbf C \mapsto \mathbf{P}$ in \cref{eq:P_from_C}, the virtual block $\mathbf{C}_V$ never enters and is discarded right away. Additionally, the occupied columns $\smash{\widetilde{\mathbf{C}}_O}$ themselves are not unique and can be multiplied with an arbitrary orthogonal matrix $\mathbf{U}_{OO}\in \mathrm{O}(O)$
without changing the density
\begin{equation}
\widetilde{\mathbf{P}}' = \widetilde{\mathbf{C}}'_O (\widetilde{\mathbf{C}}'_O)^{\mathsf{T}}
\equiv
\big(\widetilde{\mathbf{C}}_O\mathbf{U}_{OO}\big)
\big(\widetilde{\mathbf{C}}_O\mathbf{U}_{OO}\big)^{\mathsf{T}}
=
\widetilde{\mathbf{C}}_O
\mathbf{U}_{OO}\mathbf{U}_{OO}^{\mathsf{T}}
\widetilde{\mathbf{C}}_O^{\mathsf{T}}
=
\widetilde{\mathbf{P}},
\label{eq:gauge_invariance}
\end{equation}
and the same holds for the virtual slice $\smash{\widetilde{\mathbf{C}}_{\mathrm{V}}}$.
\textbf{Whole families of Hamiltonian $\mathbf{H}$ and coefficient matrices $\mathbf C$ collapse onto the same physical density $\mathbf{P}$}.
What distinguishes one ground state from another is the $O$-dimensional subspace spanned by their occupied columns $\mathrm{span}\{\mathbf{c}_i | i=1,2,\ldots,N_\mathrm{elec}/2\}$.
The smallest problem instance $B=3$ and $O=2$ is intuitively accessible. In this case,
the density is a plane through the origin of $\mathbb{R}^{3}$.
Any two eigenvectors $c_1, c_2$ lying in that plane span it equally well.
Predicting a density means predicting this subspace.

To constrain a model, we need to transfer this subspace criterion to the space of density matrices $\mathbf{P}$.
The orthonormality of the occupied orbitals,
$\smash{\widetilde{\mathbf{C}}_O^{\mathsf{T}}
\widetilde{\mathbf{C}}_O=\mathbf{I}_{O}}$, inserted
into \cref{eq:P_from_C}, gives
$
    \widetilde{\mathbf{P}}^{2}
    =
    \widetilde{\mathbf{C}}_O
    (
    \widetilde{\mathbf{C}}_O^{\mathsf{T}}
    \widetilde{\mathbf{C}}_O
    )
    \widetilde{\mathbf{C}}_O^{\mathsf{T}}
    =
    \widetilde{\mathbf{P}},
    \label{eq:orthogonal_physicality}
$
so $\smash{\widetilde{\mathbf{P}}}$ is idempotent with its eigenvalues either 0 or 1.
The trace condition $\mathrm{Tr}(\widetilde{\mathbf{P}})=O$ fixes their multiplicities to $O$ ones and $V$ zeros.
Idempotency is not only a consequence of the construction but a complete test of it, because any \textbf{symmetric idempotent matrix} of the right trace factorizes back into an orthonormal occupied block~\citep{mcweenySomeRecentAdvances1960}.
Thus, the set of valid density matrices is exactly
\begin{equation}
    \mathcal{M}_{(N_\textrm{elec},B)}
    \equiv
    \left\{
    \widetilde{\mathbf{P}}\in\mathbb{R}^{B\times B}
    \ \middle|\
    \widetilde{\mathbf{P}}^{\mathsf{T}}=\widetilde{\mathbf{P}},
    \ \
    \widetilde{\mathbf{P}}^{2}=\widetilde{\mathbf{P}},
    \ \
    \mathrm{Tr}(\widetilde{\mathbf{P}})=N_{\mathrm{elec}}/2
    \right\},
    \label{eq:manifold}
\end{equation}
This set is the \textit{Grassmann manifold}.
Its dimension is the $BO$ entries of $\smash{\widetilde{\mathbf{C}}_O}$, minus $O(O+1)/2$
orthonormality conditions, minus $O(O-1)/2$ mixing directions~\citep{edelmanGeometryAlgorithmsOrthogonality1998,absilOptimizationAlgorithmsMatrix2008}.
Hence, learning valid density matrices $\mathbf P \in \mathbb{R}^{B \times B}$ is an $OV$-dimensional prediction task.

\section{Related Work}
\label{sec:related-work}
\paragraph{ML-surrogates for ground-state observables prediction}
\looseness=-1
based on equivariant GNNs have been a key ingredient for recent advances in computational chemistry~\citep{kulichenkoDataGenerationMachine2024,woodUMAFamilyUniversal2025}.
MLIPs are the most widely adopted of these surrogates. They predict the energy as a function of the molecular geometry $G = \{(Z_a, r_a)\}_{a=1}^{N_\text{atoms}}$, $G\mapsto E(G) \in \mathbb{R}$ in linear time with respect to system size~\citep{schuttSchNetDeepLearning2018}.
These methods have demonstrated high accuracy and generalizability for a wide range of molecular systems~\citep{kovacsEvaluationMACEForce2023,jacobsPracticalGuideMachine2025b,burgerDEQuifyForceField2025,burgerShootHipHessian2025,kokerTrainingFoundationModel2025} and are typically trained on KS-DFT reference data~\citep{kulichenkoDataGenerationMachine2024}.
Another modeling ansatz focuses on \textbf{predicting the fundamental descriptors of KS-DFT},
aiming either at the Hamiltonian $\mathbf H \in \mathbb{R}^{B \times B}$
or the density matrix $\mathbf P \in \mathbb{R}^{B \times B}$, i.e., $G\mapsto \mathbf H(G), \mathbf P(G)$~\citep{schuttUnifyingMachineLearning2019,yuEfficientEquivariantGraph2023}.
\textbf{From these descriptors, the computation of almost all observables,
such as energy, forces, and dipole moments, requires at most one Fock build.} In particular, these models can be used to accelerate SCF calculations used to solve \cref{eq:generalized_eigenvalue_problem}~\citep{liuTransferableAccelerationMethod2026,eberhardTransferableScfAcceleration2026}.
A substantial body of work has built on QH9~\citep{yuQh9QuantumHamiltonian2023}, a QM9-based~\citep{ramakrishnanQuantumChemistryStructures2014} benchmark for molecular Hamiltonian prediction. However, with recent models reporting energy errors more than two orders of magnitude below chemical accuracy~\citep{kimMachineLearningHamiltonians2026} the benchmark has reached saturation. QH9's restriction to very small molecules limits its ability to assess size extrapolation to larger molecules, a shortcoming highlighted by recent studies~\citep{liEnhancingScalabilityApplicability2025,eberhardTransferableScfAcceleration2026}.
\citet{kaniselvanLearningElectronicStructure2026} have connected \model s to MLIPs, using Hamiltonians for pretraining, and finetuning with an energy readout head.
\citet{zhangSymmetrypreservingTransferableRepresentation2026} propose to evaluate observables directly from the predicted density.
\citet{kimMachineLearningHamiltonians2026} do the same for the Hamiltonian, matching the force accuracy of MLIPs on in-distribution MD-trajectories albeit at higher computational cost, requiring two Fock builds, one for $\mathbf H(\mathbf P_{\mathrm{MINAO}})$ $\Delta$-learning~\citep{eberhardTransferableScfAcceleration2026} and a second for the observable evaluation.
\textbf{We compare the observables of all \model s at a one Fock-build budget} starting from the molecular geometry.

\paragraph{Predicting and supervising the electronic ground state in the NON basis representation is the de-facto
standard approach of the field}
\label{par:non-orthonormal-basis}
\looseness=-1
~\citep{schuttUnifyingMachineLearning2019,unkeSE3equivariantPredictionMolecular2021,liDeepLearningDensity2022,liEnhancingScalabilityApplicability2025,shmilovichOrbitalMixerUsing2022,gongGeneralFrameworkE2023,yuQh9QuantumHamiltonian2023,yuEfficientEquivariantGraph2023,zhongTransferableEquivariantGraph2023,hazraPredictingOneParticle2024,wangUniversalMaterialsModel2024,zhangSelfConsistencyTraining2024,febrerGraph2matUniversalGraph2025,kimHighOrderEquivariant2025,kimMachineLearningHamiltonians2026,luoEfficientScalableDensity2025,zhouyinLearningLocalEquivariant2025,eberhardTransferableScfAcceleration2026,kaniselvanLearningElectronicStructure2026,liuTransferableAccelerationMethod2026,qianEquivariantElectronicHamiltonian2026}.
Predicting in the NON basis is well-motivated by its spatial localization relative to the ON alternative in \cref{eq:lowdin_transform}~\citep{liDeepLearningDensity2022,kaniselvanLearningElectronicStructure2026}, although this argument is mainly valid for absolute Hamiltonian predictions, less so for $\Delta$-learning, and not at all for densities  (See \appref{app:matrix-localization}).
Supervising model outputs in the NON basis lacks a similar justification,
and \citet{kaniselvanLearningElectronicStructure2026} report that matrix-element errors correlate only weakly with resulting eigenvalue errors in the NON basis, and we observe the same for downstream observables (\appref{app:frobenius-vs-observables}).
In the few works that \emph{do} supervise in the ON basis~\citep{nigamEquivariantRepresentationsMolecular2022,cignoniElectronicExcitedStates2024,zhangSymmetrypreservingTransferableRepresentation2026}, the implications of ON supervision for physical observables remain poorly understood~\citep{nigamEquivariantRepresentationsMolecular2022}, which might explain the lack of adoption.
More recently, \citet{liEnhancingScalabilityApplicability2025} introduce an ON Hamiltonian loss-term,
but perturb it (\appref{app:waloss}), and combine it with conventional NON-supervision.
We aim to close this gap by directly comparing \oursOrthonormal{} and the NON-baseline \textbf{analytically via worst-case observable error bounds} (\cref{sec:loewdin-loss}) and \textbf{empirically via correlations between loss residuals and observable errors} (\cref{sec:exp_alignment}),
providing conclusive evidence in favor of ON-basis supervision.

\looseness=-1
Unrestricted density prediction in the space of symmetric matrices provides no guarantee for the physicality of $\smash{\hat{\mathbf{P}}} \notin \mathcal{M}_{(N_\textrm{elec},B)}$.
One way to restore the validity is to purify $\smash{\hat{\mathbf{P}}}$ post-hoc~\citep{lehtolaAssessmentInitialGuesses2019}, either by projecting onto $\mathcal{M}_{(N_\textrm{elec}, B)}$~\citep{zhangSymmetrypreservingTransferableRepresentation2026} or by charge rescaling (charge conservation) and McWeeny purification~\citep{yescasramosAcceleratedScfWorkflows2026}, which only works within a basin of attraction~\citep{mcweenySomeRecentAdvances1960}.
\textbf{Neither incorporates the constraint into the model}, correcting only at evaluation time~\citep{zhangSymmetrypreservingTransferableRepresentation2026} or adding trace and idempotency loss terms~\citep{yescasramosAcceleratedScfWorkflows2026}.
In contrast, \textbf{\ours{}\textsubscript{rot} learns to transport an initial guess to the ground state $\smash{\hat{\mathbf{P}}^* \in \mathcal{M}_{(N_\textrm{elec}, B)}}$ end-to-end} (\cref{sec:methods}), while \textbf{\ours{}\textsubscript{proj}
includes the manifold projection in its parameter gradient computations}.
The Hamiltonian-side degree of freedom of the gEVP~\eqref{eq:generalized_eigenvalue_problem} has been addressed for periodic structures~\citep{wangUniversalMaterialsModel2024,yinAdvancingUniversalDeep2025}.
The molecular Hamiltonian losses in previous work leave this degree of freedom untreated, such that the model allocates expressivity to fit a component that no derived observable constrains.
\textbf{\ourLoss{} combines $\Delta$-learning with a per-sample shift-invariant loss for molecular Hamiltonians} (\cref{sec:rocket}).

\section{Methods}
  \label{sec:methods}
  \paragraph{\oursOrthonormal{}.}
\label{sec:loewdin-loss}
\begin{wraptable}[23]{R}{0.37\textwidth}
    \vspace{-13pt}
    \centering
    \footnotesize
    \setlength{\tabcolsep}{3pt}
    \caption{\textbf{Worst-case leading-order observable errors in terms of matrix errors.}
    \(\lambda_{\min}\) and \(\lambda_{\max}\) are the extreme eigenvalues of the overlap matrix $\mathbf{S}$ for a given molecule, and \(\Delta\epsilon_{\mathrm{HL}}\) its HOMO-LUMO gap.
    The overlap factors are the worst-case conversion factors from non-orthonormal to orthonormal matrix errors, attained along the extreme overlap eigenvectors.
    Derivations are in \appref{app:loewdin-observable-bounds}.}
    \label{tab:loewdin-observable-scaling}
    \vspace{-5pt}
    \begin{tabular}{@{}l@{\hspace{8pt}}c@{}r@{}}
    \toprule
    \textbf{Observable}
    &
    \multicolumn{2}{r@{}}{\textbf{Leading order in $\|\cdot\|_F$}}
    \\
    \midrule
    \multicolumn{3}{@{}l}{\textbf{Unconstrained $\Delta \mathbf{P}$}} \\
    $E_\text{KS}$ & \(\propto\) & \(\lambda_{\max}\|\Delta\mathbf{P}\|_F\) \\
    \(\mu_\rho/L_\rho^1/L_\rho^2\) & \(\propto\) & \(\lambda_{\max}\|\Delta\mathbf{P}\|_F\) \\
    \addlinespace
    \multicolumn{3}{@{}l}{\textbf{Manifold Constrained $\Delta \mathbf{P}$}} \\
    $E_\text{KS}[\rho]$ & \(\propto\) & \(\lambda_{\max}^2\|\Delta\mathbf{P}\|_F^2\) \\
    \(\mu_\rho/L_\rho^1/L_\rho^2\) & \(\propto\) & \(\lambda_{\max}\|\Delta\mathbf{P}\|_F\) \\
    \addlinespace
    \multicolumn{3}{@{}l}{\textbf{Unconstrained $\Delta \mathbf{H}$}} \\
    $E_\text{KS}$ & \(\propto\) & \(\Delta\epsilon_{\mathrm{HL}}^{-2}\lambda_{\min}^{-2}\|\Delta\mathbf{H}\|_F^2\) \\
    \(\mu_\rho/L_\rho^1/L_\rho^2\) & \(\propto\) & \(\Delta\epsilon_{\mathrm{HL}}^{-1}\lambda_{\min}^{-1}\|\Delta\mathbf{H}\|_F\) \\
    \bottomrule
    \end{tabular}
\end{wraptable}
When learning ground-state descriptors \(\mathbf{X}\in\{\mathbf{H},\mathbf{P}\}\), we aim to train on a loss that closely tracks the observable errors in the vicinity of the ground state.
The tightest target-specific supervision is achieved by directly training on the observables. Although suitable for task-specific readouts used in MLIPs (\cref{sec:related-work}), this ansatz deviates from the goal of learning general descriptors that support observables that are unspecified at training time.
Instead, we propose to \textbf{minimize the Frobenius norm of matrix errors \(\Delta\mathbf{X}=\smash{\hat{\mathbf{X}}-\mathbf{X}^*}\) in their ON-basis representation} (\cref{eq:lowdin_transform})
\vspace{-0.1cm}
\begin{equation}
    \mathcal{L}_\mathrm{ON}(\Delta\mathbf{X}, \mathbf{S}) \equiv \mathcal{L}_\mathrm{ON}(\Delta\widetilde{\mathbf{X}})
    :=\frac{\|\Delta\widetilde{\mathbf{X}}\|_F}{B} \:.
\label{eq:loss}
\vspace{-0.1cm}
\end{equation}
This contrasts the standard approach of prior work, which directly supervises $\Delta\mathbf{X}$ in the non-orthonormal \textbf{NON}-basis (\cref{sec:related-work}).
The overlap matrix $\mathbf{S}$ mediates between matrix-errors $\Delta X$ and observable-errors (\cref{tab:loewdin-observable-scaling}).
Its eigenvalues satisfy \(0 < \lambda_{\min}\leq1\leq\lambda_{\max}\) (\cref{eq:overlap_matrix_properties}) and act as a proxy for the relative importance of each direction w.r.t. physical observables, such that the
\textbf{spectrum $\lambda(\mathbf{S})$ emphasizes some matrix-error directions more than others}.
These spectral factors only vanish in ON-bases $\mathbf{S} = \mathbf{I}$ where the loss$\rightarrow$observable error-propagation becomes approximately isotropic, which we also confirm empirically in \cref{sec:experiments}.

  \paragraph{Constraining density \model s to the manifold}
\looseness=-1
$\smash{\mathcal{M}_{(N_\mathrm{elec},B)}}$ guarantees the physical validity of their predictions $\smash{\hat{\mathbf P}}$ and further tightens the observable error bounds.
While the dipole, \(L_1\) and \(L_2\) real-space density errors remain linear, the total-energy error becomes quadratic in $\smash{\|\Delta\widetilde{\mathbf{P}}\|_F'}$ (\cref{tab:loewdin-observable-scaling}).
A priori, neither the initial guess nor the learned predictions are restricted to the Grassmannian.
We therefore introduce \oursLong{} (\ours{}) to predict densities that lie on the Grassmannian manifold by construction.
Following \citet{lehtolaAssessmentInitialGuesses2019}, we define an idempotent projection map $\mathcal{R}: \mathbf{P} \mapsto \mathbf{C}^{(\mathbf{P})}_O \mathbf{C}^{(\mathbf{P})\mathsf{T}}_O \in \mathcal{M}_{(N_\textrm{elec},B)}$ that purifies any candidate density matrix.
Given an arbitrary trial density $\mathbf{P}$, we diagonalize it in the orthonormal basis to extract its eigenvectors $\mathbf{C}^{(\mathbf{P})}$, retaining the $O$ eigenvectors $\mathbf{C}^{(\mathbf{P})}_O$ corresponding to the largest eigenvalues.
This yields the valid density matrix closest to $\mathbf{P}$ in the Frobenius norm \citep{eckartApproximationOneMatrix1936}. \textbf{We evaluate two approaches for predicting valid density matrices (\ours{}\textsubscript{rot}, \ours{}\textsubscript{proj})}, using $\mathcal{R}$ either as a pre- or postprocessing step.
\textbf{\ours{}\textsubscript{rot}} is a one-step flow map~\citep{boffiFlowMapMatching2025} that first projects the initial guess $\mathbf{P}_0 \equiv \mathcal{R}(\mathbf{P}_\textrm{MINAO})$ and then learns the constant velocity $\hat{\mathbf{X}}$, \cref{eq:groot_map}, of the flow
\begin{equation}
    \frac{\mathrm{d}}{\mathrm{d}t}\mathbf{P}(t)
    =
    \left[\hat{\mathbf{X}}, \mathbf{P}(t)\right]
    =
    \hat{\mathbf{X}}\mathbf{P}(t) - \mathbf{P}(t)\hat{\mathbf{X}},
    \quad
    \mathbf{P}(0) = \mathbf{P}_0.\label{eq:groot_flow}
\end{equation}
This flow equation has a closed-form solution $\smash{\mathbf{P}(t) = \exp(t\hat{\mathbf{X}})\,\mathbf{P}_0\exp(-t\hat{\mathbf{X}})}$ that transports the initial projection $\mathbf{P}_0$ to the predicted density $\mathbf{P}(1) = \smash{\hat{\mathbf{P}}}$ \citep{sakuraiModernQuantumMechanics2021}. The underlying transformation of the initial guess corresponds to a \textbf{rotation} $\smash{\hat{\mathbf{U}}(t) = \exp(t\hat{\mathbf{X}}) \in SO(B)}$ and preserves the Grassmannian constraints of \cref{eq:manifold}. These transformations $\mathbf{U}$ are also referred to as \textit{orbital rotations} in this context.
\citet{wongDifferentialGeometryGrassmann1967} showed that any two densities on $\smash{\mathcal{M}_{(N_\textrm{elec},B)}}$ are connected by pure mixing rotations. At the initial guess $\mathbf{P}_0$ these are generated by skew-symmetric matrices $\smash{\hat{\mathbf{X}}^{\mathsf{T}} = -\hat{\mathbf{X}}}$ that \emph{anti}commute with the reflection $\mathbf{R}_0=2\mathbf{P}_0-\mathbf{I}$, i.e. $\smash{\hat{\mathbf{X}}\mathbf{R}_0=-\mathbf{R}_0\hat{\mathbf{X}}}$. We map any equivariant readout-head output $\smash{\hat{\mathbf{M}}}$ to satisfy these conditions by applying
$
\hat{\mathbf{X}}\big(\hat{\mathbf{M}}\big)
    =
    \mathbf{P}_0
    \hat{\mathbf{M}}
    \mathbf{\Pi}_0
    -
    \mathbf{\Pi}_0
    \hat{\mathbf{M}}^{\mathsf{T}}
    \mathbf{P}_0
$, where $\mathbf{\Pi}_0 \equiv \mathbf{I}-\mathbf{P}_0$ is the projector onto the virtual orbitals. Since the unit-time map $\mathbf{P}(1)$ is available in closed form, we can bypass numerical ODE integration by directly supervising the prediction $\smash{\hat{\mathbf{P}}}$ against the target density.
(For details, see \appref{app:rotational_groot}.)
\textbf{\ours{}\textsubscript{proj}} enforces validity by projecting the symmetrized matrix readout $\smash{\hat{\mathbf{M}}}$ to the Grassmannian manifold $
\smash{
    \hat{\mathbf{P}} = \mathcal{R}(\mathbf{P}_\textrm{MINAO}
    + \tfrac{1}{2}(\hat{\mathbf{M}} + \hat{\mathbf{M}}^{\mathsf{T}}))
}
$
with the \textbf{parameter gradients flowing through $\mathcal{R}$}.

  \paragraph{\ourLoss{}}
\label{sec:rocket} is a loss function for Hamiltonian learning.
Unlike \ours{} which \emph{enforces} additional \emph{physical} constraints, \ourLoss{} \emph{lifts}
optimization pressure from \emph{unphysical or irrelevant} degrees of freedom.
\textbf{\ourLoss{} systematically treats \emph{gauge freedoms} of $\smash{\hat{\mathbf{H}}}$},
meaning transformations of $\smash{\hat{\mathbf{H}}}$ which leave the coefficients
$\smash{\hat{\mathbf{C}} := \mathbf{C}(\hat{\mathbf{H}})}$ unchanged.
In the ON-basis, one such transformation is $\smash{\tilde{\mathbf{H}}\mapsto\tilde{\mathbf{H}}+c\mathbf{I}}$ which only shifts the \emph{eigenvalues} $\epsilon_i$ but does not affect the eigenvectors $\smash{\tilde{\mathbf{C}}}$.
We transfer this invariance to \oursOrthonormal{} (\cref{eq:loss}) using a closed-form minimizer
$c^*=\sum_{\mu}^{B} \Delta \widetilde{H}_{\mu\mu} / B$ and obtain
\begin{equation}
    \label{eq:rocket_stage_1}
    \mathcal{L}^\textrm{(1st stage)}_\textrm{\ourLoss}(\Delta\widetilde{\mathbf{H}})
    :=
    \min_{c}\,
    \mathcal{L}_\mathrm{ON}(\Delta\widetilde{\mathbf{H}}-c\mathbf{I})\:.
\end{equation}
This leaves a flat $c\,\mathbf{S}$ loss-direction for every molecule.
A pure data-normalization ansatz with a training-set-optimal gauge $c_\textrm{t}$ and a standard matrix loss would strictly penalize gauge variation between molecules, while \cref{eq:rocket_stage_1} permits arbitrary gauge variation. In between these extremes
\begin{align}
    \label{eq:rocket_stage_2}
    \mathcal{L}^\textrm{(2nd stage)}_\textrm{\ourLoss}(\Delta\widetilde{\mathbf{H}})
    :=
    \mathcal{L}^\textrm{(1st stage)}_\textrm{\ourLoss}(\Delta\widetilde{\mathbf{H}})
    +
    \alpha_{\textrm{\ourLoss}}\,\left(c^* - \bar{c}\right)^2
\end{align}
\looseness=-1
\textbf{softly restrains cross-molecule variation along the gauge direction}, which performs best on our validation ablation (\appref{app:rocket-ablations}). The reference gauge $\bar{c}$ in \cref{eq:rocket_stage_2} is an exponential moving average (EMA) of the per-molecule optimal shifts $c^*$.
Additionally, \textbf{we remove supervision from Hamiltonian directions that do not affect the lowest $O+k$ eigenvectors $\smash{\mathbf{C}_{O+k}(\hat{\mathbf{H}})}$, leaving $\smash{\mathbf{P}(\mathbf{C}(\hat{\mathbf{H}}}))$ unchanged.}
We define the matrix $\widetilde{\mathbf{Q}}_{O+k}$ that projects onto the lowest $O+k$ reference eigenvectors as
$
    \widetilde{\mathbf{Q}}_{O+k}
    :=
    \widetilde{\mathbf{C}}^*_{O+k}
    \widetilde{\mathbf{C}}_{O+k}^{*\top}
$,
and mask out errors entirely within its orthogonal complement
\begin{equation}
    \smash{
    \Delta\widetilde{\mathbf{H}}_{O+k}
    :=
    \Delta\widetilde{\mathbf{H}}
    -
    (\mathbf{I}-\widetilde{\mathbf{Q}}_{O+k})
    \Delta\widetilde{\mathbf{H}}
    (\mathbf{I}-\widetilde{\mathbf{Q}}_{O+k}).
    }
\end{equation}
Plugging this into \cref{eq:rocket_stage_2} yields the 3rd and final stage of \ourLoss{}
\begin{equation}
    \label{eq:rocket_stage_3}
    \mathcal{L}^\textrm{(3rd stage)}_\textrm{\ourLoss}(\Delta\widetilde{\mathbf{H}})
    :=
    \mathcal{L}^\textrm{(2nd stage)}_\textrm{\ourLoss}(\Delta\widetilde{\mathbf{H}}_{O+k})\:.
\end{equation}
The lowest $O$ eigenvectors determine the density matrix, while \textbf{retaining an additional $k$ eigenvectors increases the robustness to prediction errors by separating the occupied states from the first unsupervised eigenstate.}
This allows for a controlled tradeoff between learning a smooth but higher-dimensional Hamiltonian target vs learning a discontinuous but lower-dimensional density target. We use $k=1$ by default and ablate larger $k$ in \appref{app:rocket-ablations}.

\section{Experiments}
\label{sec:experiments}
We assess all models on the same equivariant GNN backbone (\appref{app:model-implementation}), varying only the matrix readout.
For the Hamiltonian ansatz, we obtain the density by solving the gEVP (\cref{eq:generalized_eigenvalue_problem}), which is computationally cheap compared to the Fock-build~\citep{eberhardTransferableScfAcceleration2026}.
\textbf{To evaluate size extrapolation, we benchmark our models across two complementary dataset regimes.}
In the first regime, we train on equilibrium geometries from \textbf{QM9}~\citep{ramakrishnanQuantumChemistryStructures2014}, restricted to molecules with at most 20 atoms and elements H, C, N, O, and F (\cref{tab:train-dataset}).
We then evaluate these models on the CHNOF subset of \textbf{QM40}~\citep{madushankaQm40RealisticQuantum2024} with up to 37 heavy atoms (\cref{tab:eval-dataset}), with molecules up to four times larger than those of the training distribution.
In the second regime, we train on QMugs~\citep{isertQmugsQuantumMechanical2021}, testing extrapolation on conformer diversity searched at $300\,\mathrm{K}$ across a wider chemical space (\cref{tab:qmugs-dataset}).
Here, we \textbf{train on all QMugs molecules with up to 150 electrons and evaluate on systems with up to 660 electrons} (\cref{tab:qmugs-dataset}).
Across all datasets, we generate reference density and Hamiltonian matrices with restricted KS-DFT at the B3LYP/def2-SVP level of theory~\citep{stephensAbInitioCalculation1994,weigendBalancedBasisSets2005} using GPU-accelerated \texttt{GPU4PySCF}~\citep{liIntroducingGpuAcceleration2025,sunRecentDevelopmentsPyscf2020} (\appref{app:dft-reference}).
\textbf{To obtain the strongest possible baselines} and accurately characterize relative performances, we perform extensive ablations of both absolute and $\Delta$-learning approaches~(\appref{app:loss-vs-observables}) in our initial training on QM9, and select learning rates and weight decay for each method separately based on a 75-molecule QM40 validation set~(\appref{app:training-optimizer}).
All other hyperparameters are kept constant, and we observe no difference in training cost across all ablated ground-state predictors and loss combinations.
We avoid the hidden Fock-build $\mathbf H (\mathbf P_\textrm{MINAO})$ for Hamiltonian $\Delta$-learning ~\citep{eberhardTransferableScfAcceleration2026} by using the superposition of atomic potentials (SAP) ansatz $\mathbf{H}_\textrm{SAP}$~\citep{lehtolaAssessmentInitialGuesses2019}, and additionally ablate training-set-optimal gauge shifts ($c_t \mathbf S$) for the Hamiltonian baseline.
For $\Delta$-learning in density space, we use $\mathbf{P}_\textrm{MINAO}$~\citep{lehtolaAssessmentInitialGuesses2019}.
When evaluating the energies from the predicted $\hat{\mathbf{P}}$ or $\mathbf{P} (\hat{\mathbf{H}})$ density matrix, we get the \textbf{first order energy correction in the self-consistency error} $\Delta^\text{SC} \mathbf{P} := \mathbf{P}(\mathbf{H}(\hat{\mathbf{P}}))-\hat{\mathbf{P}}$ to the ground-state energy expression of KS-DFT $E_{KS}$ essentially for free,
such that our predictions are given by $\hat E = E_\textrm{KS} (\hat P) + \sum_{\mu \nu} H_{\mu \nu}(\hat P)\Delta^\textrm{SC} P_{\mu \nu}$~\citep{harrisSimplifiedMethodCalculating1985,foulkesTightbindingModelsDensityfunctional1989}.

\begin{figure}[t]
    \vspace{-0.1cm}
    \centering
    \includegraphics[width=\linewidth]{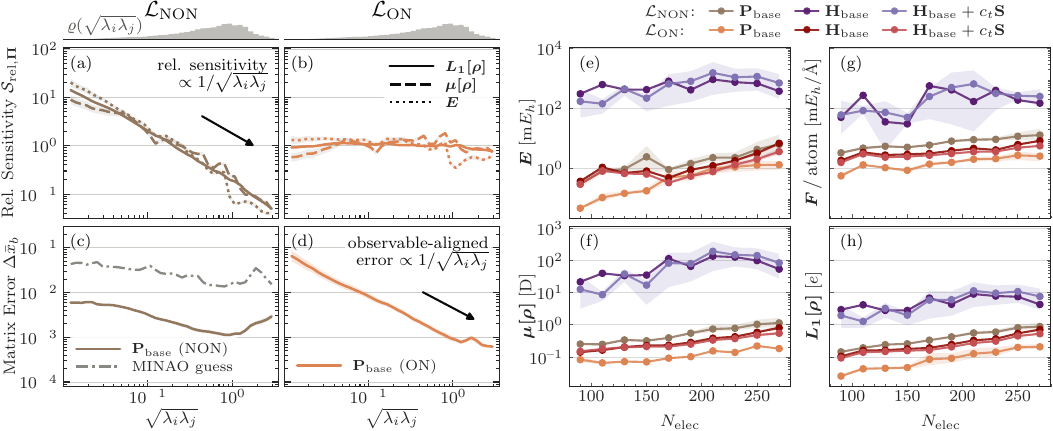}
    \vspace{-0.5cm}
    \caption{
    \textbf{\oursOrthonormal{} is "observable-isotropic" and improves size extrapolation.}
    \textbf{(a,b,c,d)} Values log-binned by $\smashsqrt{\lambda_i\lambda_j}$, matching the directions $\boldsymbol{\Pi}_{ij}$ of QM40 molecules (\cref{tab:eval-dataset}); shading indicates the standard deviation, histograms show the direction-density $\varrho$.
    \textbf{(a,b)} Mean-normalized relative sensitivity for the $L^1$ density, dipole $\mu$, and energy $E$ errors (line styles): \textbf{(a)} $\mathcal{L}_\mathrm{NON}$ falls as $1/\smashsqrt{\lambda_i\lambda_j}$, whereas \textbf{(b)} $\mathcal{L}_\mathrm{ON}$ stays approximately constant for all three observables.
    \textbf{(c,d)} Matrix error $\Delta\bar{x}_b \equiv \sqrt{\langle\mathrm{tr}(\boldsymbol{\Pi}_{ij}\Delta\mathbf{X})^2\rangle_b}$ of the trained network, with $P_\textrm{MINAO}$ for reference: \textbf{(c)} $\mathcal{L}_\mathrm{NON}$ leaves error that follows $\varrho$ rather than observable impact, while \textbf{(d)} $\mathcal{L}_\mathrm{ON}$ concentrates it in the high-impact directions.
    \textbf{(e--h)} Size extrapolation experiments from \textbf{QM9 to QM40}, featuring the same observables MAEs
    plus the force $F$ acting on the atoms plotted against the electron count $N_\mathrm{elec}$. We compare \model s under both $\mathcal{L}_\mathrm{NON}$ and $\mathcal{L}_\mathrm{ON}$ loss functions, as well as Hamiltonian variants with a training-set-optimal constant gauge shift $c_t$.
    Markers are means and bands the minimum and maximum across identical runs; \textbf{every $\mathcal{L}_\mathrm{ON}$ model beats its $\mathcal{L}_\mathrm{NON}$ counterpart.}}
    \label{fig:loss_anisotropy}
    \vspace{-0.6cm}
\end{figure}
\paragraph{Loss Observable Alignment.}
\label{sec:exp_alignment}
\looseness=-1
We compare the alignment of $\mathcal{L}_\mathrm{ON}$ and $\mathcal{L}_\mathrm{NON}$ by
investigating their agreement with key observables in the vicinity of the target density $\mathbf P^*$.
To that end, we define the \emph{sensitivity} $\mathcal{S}_{\boldsymbol{\Pi}}(\mathcal{O})\in \mathbb{R}$ of an observable $\mathcal{O}$ (or Loss $\mathcal{L}$) w.r.t. a perturbation direction $\boldsymbol{\Pi} \in \mathbb{R}^{B \times B}$ as the absolute directional derivative (\appref{app:absolute-directional-derivative}).
The \textbf{\textit{relative directional sensitivity}}~(\appref{app:loewdin-sensitivity})
$
    \relsens{\boldsymbol{\Pi}}(\mathcal{L},\mathcal{O}) := \mathcal{S}_{\boldsymbol{\Pi}}(\mathcal{L})\:/\:\mathcal{S}_{\boldsymbol{\Pi}}(\mathcal{O})
$
captures if the loss changes more (or less) rapidly than the actual observable along a given direction.
We analyze the sensitivities in a basis of symmetric matrices built from the eigenvectors $\mathbf{u}_k$ of the overlap $\smash{\mathbf{S} = \sum_k \mathbf{u}_k \mathbf{u}^{\mathsf{T}} \lambda_k}$, with normalized elements $\smash{\boldsymbol{\Pi}_{ij} \propto \mathbf{u}_i \mathbf{u}_j^{\mathsf{T}} + \mathbf{u}_j \mathbf{u}_i^{\mathsf{T}}}$.
Our experiments show that the optimization process prioritizes directions according to their loss sensitivity $\mathcal{S}_{\boldsymbol{\Pi}}(\mathcal{L}_\mathrm{NON}) = 1 / B$ and $\mathcal{S}_{\boldsymbol{\Pi}}(\mathcal{L}_\mathrm{ON}) = \smashsqrt{\lambda_i\lambda_j}/B$ for density matrices.
We demonstrate this by comparing the matrix error components in the overlap basis, $\Delta\bar{x}_b := \sqrt{\langle\mathrm{tr}(\boldsymbol{\Pi}_{ij}\Delta\mathbf{X})^2\rangle_b}$, binned by $\smashsqrt{\lambda_i\lambda_j}$ (\cref{fig:loss_anisotropy}~(c,d), and \appref{app:loewdin-sensitivity}).
An ideal loss function should have direction-independent (isotropic) relative sensitivities $\relsens{\boldsymbol{\Pi}}(\mathcal{L},\mathcal{O}) \approx c^{(\mathcal{O})} \: \forall \:\boldsymbol{\Pi}\,\forall \: \mathcal{O}$ across a diverse set of observables.
\Cref{fig:loss_anisotropy}~(a,b) shows that this is approximately achieved by \oursOrthonormal{} and violated by the NON baseline.
Training on \oursOrthonormal{} strictly outperforms this baseline on QM40  across model configurations on all tested observables (see \cref{fig:qm40_size_extrapolation} and \cref{fig:loss_anisotropy}~(e-h)).
Across all evaluated models, we observe higher mutual information between the \oursOrthonormal{} residual and downstream observable errors (\cref{tab:frobenius_panel_stats}), independent of the training loss that produced the model weights.

\begin{figure}[t]
    \vspace{-0.1cm}
    \centering
    \includegraphics[width=\linewidth]{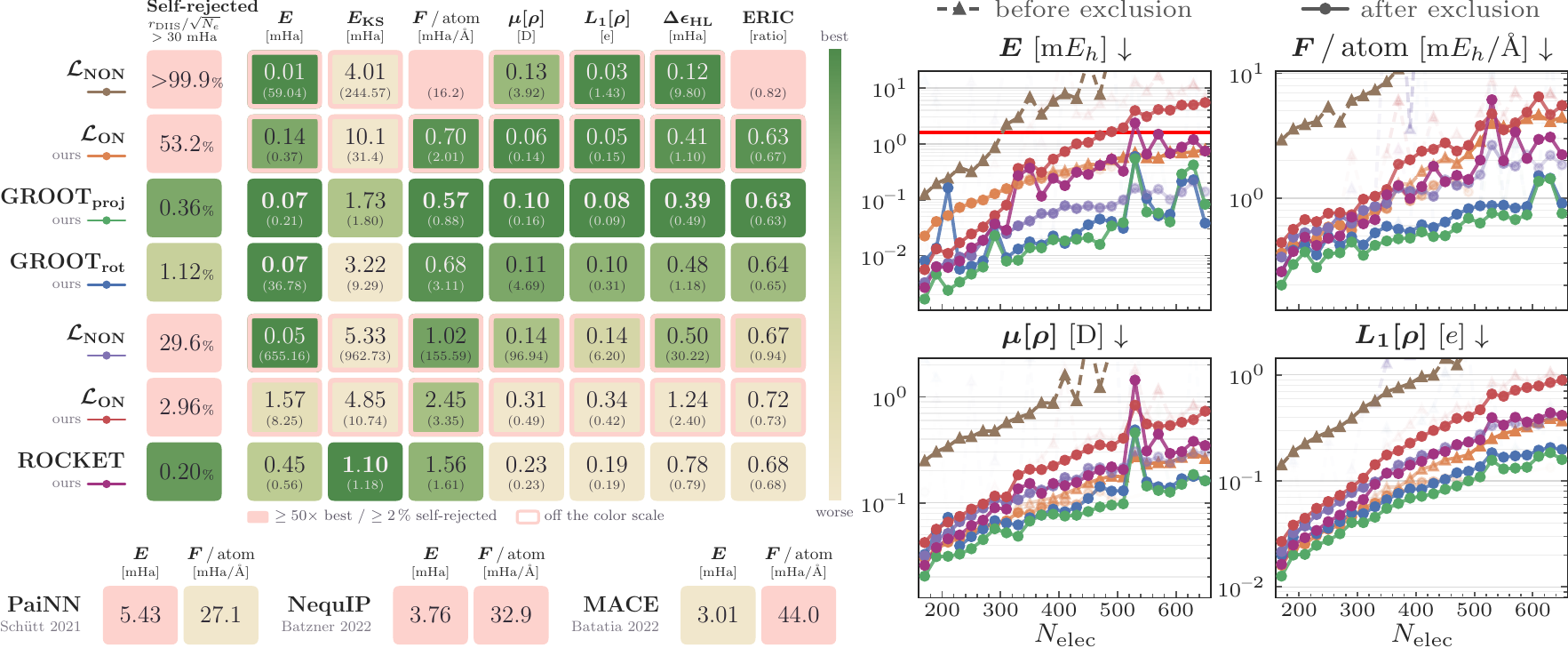}
    \vspace{-0.4cm}
    \caption{
    \looseness=-1
    \textbf{\ours{} and \ourLoss{} show robust extrapolation, and can be equipped with \model s-specific error indicators} (\cref{tab:ablation_qmugs_ood160_r100}).
    \model s and MLIPs trained on QMugs molecules with at most 150 electrons and evaluated on systems with up to 660 electrons.
    \textbf{The table} shows the fraction of molecules rejected by the self-consistency criterion and the observable errors of the remaining molecules, with the MAE including the rejected predictions in brackets below.
    Compared to linear scaling MLIPs like PaiNN~\citep{schuttEquivariantMessagePassing2021}, NequIP~\citep{batznerE3equivariantGraphNeural2022}, and MACE~\citep{batatiaMACEHigherOrder2022} trained on energies only (meaning at equal reference compute budget), \ours{}\textsubscript{proj} achieves an order of magnitude better energies, rejecting less than 1\% of its own predictions, characterizing the achievable cost-accuracy trade-off of \model s compared to MLIPs.
    The ERIC~\citep{eberhardTransferableScfAcceleration2026} relative to MINAO shows that our density models yield additional SCF acceleration over the baselines.
    \textbf{Observable MAEs vs. electron count.} \ours{} and \ourLoss{} stay orders of magnitude below the $\mathcal{L}_\mathrm{NON}$ baselines with the margin widening as the system size increases.}
    \label{fig:qmugs_ood_extrapolation}
    \vspace{-0.7cm}
\end{figure}

\paragraph{Size-Extrapolation.}
\label{sec:exp_size_extrapolation}
Reference calculations for KS-DFT ground states scale with the 3rd or 4th power with system size, depending on the level of theory.
This restricts training data for ML surrogates to small to mid-sized molecules.
Their practical utility, therefore, depends in part on their ability to accurately extrapolate to larger systems.
In our extrapolation experiments on equilibrium structures from QM9 and QM40 (\cref{fig:qm40_size_extrapolation}, \cref{fig:loss_anisotropy}~(e-h), \appref{app:training-on-qm9}),
the $\mathcal{L}_\mathrm{NON}$ baselines generalize poorly to larger molecules, whereas \oursOrthonormal{} consistently reduces  prediction errors across both density and Hamiltonian representations (\cref{tab:ablation_qm40_r100}).
\ourLoss{} and \ours{} compound on top of these improvements:
Although \ourLoss{} already outperforms \emph{all} tested Hamiltonian \model s including the new \oursOrthonormal{} on \emph{all} tested observables,
it is improved upon by \ours{}\textsubscript{rot} on all observables except the Kohn--Sham energy, while \ours{}\textsubscript{proj} performs worse than \ourLoss{} on both energies but surpasses it on the remaining observables (see \cref{fig:qm40_size_extrapolation}).
Excluding the manifold projection from the training gradients and applying it only post hoc to the \oursOrthonormal{} baseline rejects $66\%$ of the molecules as unreliable, compared to $0.6\%$ for \ours{}\textsubscript{proj} (\cref{tab:ablation_qm40_r100}).

Both versions of \ours{} achieve energy errors of less than half \emph{chemical accuracy} (1 kcal/mol) for $98\%$ of the QM40 molecules in the test set
with the vast majority of errors passing the even stricter threshold of $1 / 16 \,\mathrm{kcal/mol}$.
This motivates the search for an \textbf{uncertainty measure that quantifies prediction confidence}.
We identify the residual term
$
    r = \sqrt{8}\, \lVert \tilde{H}\tilde{P} - \tilde{P}\tilde{H} \rVert_F
$,
commonly used in convergence acceleration methods for KS-DFT calculations~\citep{pulayConvergenceAccelerationIterative1980,pulayImprovedScfConvergence1982}, as a suitable measure to equip \model s with the approximately \textbf{size-intensive rejection criterion}
$
    \lVert r \rVert_F / \sqrt{N_{\mathrm{elec}}} > \varepsilon_\text{SC}
$.
Ideally, the self-consistency threshold $\varepsilon_\text{SC}$ could be calibrated on a large validation set per model.
For our experiments, we settled on uniform thresholds for all models, which we estimated based on the error distribution of a small large-molecule validation set to $\varepsilon_\text{SC} = 30\,\mathrm{mE_h}$.
\textbf{Besides size extrapolation} we also isolated the conformational generalization capabilities of \model s by training on QMugs and evaluating on the 3BPA~\citep{kovacsLinearAtomicCluster2021} torsional surface (\cref{fig:3bpa_and_transition1x}, \cref{tab:ablation_3bpa_dih}).
We observe excellent performance for all our \model s: compared to PaiNN, NequIP, and MACE, \ourLoss{} reduces energy errors by roughly an order of magnitude ($8$--$12\times$) and force errors by $5$--$7\times$, while the \oursOrthonormal{}-trained density \model{} $\mathbf{P}_\mathrm{base}$ (ON) reduces them by $148$--$234\times$ and $20$--$29\times$, respectively.

\paragraph{Towards foundational \model s.}
\looseness=-1
Our evaluations suggest that \emph{well-designed} \model s can achieve significantly improved energies compared to MLIPs, both trained on energies only and trained on energies and forces, while force-trained MLIPs still lead in force accuracy~(\cref{fig:qm40_size_extrapolation}, \cref{fig:qmugs_ood_extrapolation}).
Despite their higher inference time costs compared to MLIPs, viewed from an \textbf{equal reference compute perspective}, \model s offer an interesting trade-off well suited for foundation models.
Not only do \model s outperform energy-only trained MLIPs in terms of force accuracy by $\sim \times 30-50$, \cref{fig:qm40_size_extrapolation}, but they also offer access to observables undefined at training time.
\textbf{Pretrained \model s are conveniently adaptable to new chemistry by combining the idea of label-free Self-Consistency Training (SCT)} introduced by \citet{zhangSelfConsistencyTraining2024} with the occupied--virtual (OV) gradient objective of \citet{eberhardTransferableScfAcceleration2026} (\appref{app:sct}).
Depending on the application the resulting amortization cost may be favorable for research applications with low total evaluation counts since it incurs a much lower up front fine-tuning cost.
We finetune our QMugs-trained models on the Transition1x dataset~\citep{schreinerTransition1xDatasetBuilding2022}, containing molecular geometries along reaction pathways.
Models pretrained with the NON loss fail under the OV objective, where the Hamiltonian model diverges and $\mathbf{P}_\mathrm{base}$ (NON) rejects all predictions.
While stationarity keeps the first-order corrected energy error of $\mathbf{P}_\mathrm{base}$ (ON) at $0.52$~mHa, it rejects $25\%$ of test frames and incurs a self-consistent KS energy error of $21.1$~mHa on all frames because unconstrained densities leave the Grassmannian~(\cref{tab:loewdin-observable-scaling}), which \ours{} avoids by construction.
In contrast, \ours{} and \ourLoss{} adapt reliably and reach energy errors well below chemical accuracy on all accepted predictions (\cref{fig:3bpa_and_transition1x}, \cref{tab:ablation_transition1x_sct_test}).

\begin{figure}[t]
    \vspace{-0.1cm}
    \centering
    \includegraphics[width=\linewidth]{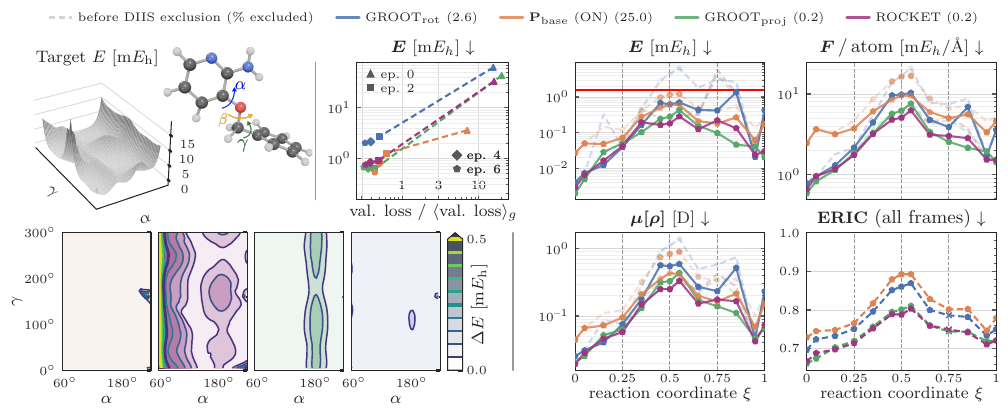}
    \vspace{-0.6cm}
    \caption{
        \looseness=-1
        \textbf{Beyond size extrapolation: generalization across conformations and label-free adaptation to new
        chemistry.}
        QMugs-trained models generalize to the $\beta=120^\circ$ cut of the 3BPA backbone-dihedral energy surface
        (left panel, \cref{tab:ablation_3bpa_dih}): the ground truth $(\alpha,\gamma)$ energy-surface, and the
        energy deviations of \model s, with median (MAE-optimal) constant offset removed per model.
        \textbf{Label-free self-consistency training (SCT) on Transition1x.}
        The upper-right panel shows the validation loss against the self-consistency residual over 0, 2,
        4, and 6 epochs of cosine-annealed SCT, suggesting that shorter training budgets might be feasible.
        The observable MAEs are resolved along the reaction coordinate $\xi$, with the transition state at $\xi=0.5$.
        The energies of the finetuned \model s consistently achieve chemical accuracy (solid red line) on the accepted predictions, with \ours{}\textsubscript{proj} and \ourLoss{}
        rejecting the fewest predictions ($0.2\%$, \cref{tab:ablation_transition1x_sct_test}).
    }
    \label{fig:3bpa_and_transition1x}
    \vspace{-0.3cm}
\end{figure}

\section{Discussion}
\label{sec:discussion}
\looseness=-1
We have aligned \model s with the governing equations of Kohn-Sham Density Functional Theory (KS-DFT) through \oursOrthonormal{}, \ours{}, and \ourLoss{}.
To this end, we performed a theoretical and empirical analysis of representational choices of $\mathbf{P},\mathbf{H}$ for \model{} training
and their effects on downstream observables, finding strong evidence in favor of orthonormal (ON) basis representations (\oursOrthonormal{}).
Unconstrained density readouts can leave the manifold of valid density matrices, which can violate charge conservation and lead to fractional occupations, a regime where KS-DFT is known to misbehave~\citep{perdewDensityFunctionalTheory1982,morisanchezLocalizationDelocalizationErrors2008}.
\ours{} enforces these constraints by construction and thus predicts physical density matrices.
The densities obtained from Hamiltonian \model s are inherently physical, but current loss functions overconstrain them.
\ourLoss{} reduces this excess supervision.
In our QM9 to QM40 \textbf{size extrapolation} experiments, these contributions compound relative to state-of-the-art \model s (\cref{fig:qm40_size_extrapolation}).
Alongside improvements across other observables, \oursOrthonormal{} and \ours{} reduce the energy error of the previous state-of-the-art density \model{} by $79.1\%$, and \oursOrthonormal{} and \ourLoss{} reduce that of the strongest Hamiltonian baseline by $99.8\%$.
Under \ours{}, $98\%$ of molecules achieve energy errors below $0.5\times$ chemical accuracy.
On QMugs systems up to 660 electrons (\cref{tab:ablation_qmugs_ood160_r100}), \ours{}\textsubscript{proj} and \ourLoss{} reduce energy errors by $99.6\%$ and $99.9\%$, respectively, relative to the corresponding NON baselines.
Evaluating \model s on energy observables requires a Fock build incurring
significantly higher inference cost than MLIPs, but it yields an \textbf{inherent uncertainty measure} through the self-consistency residual.
This self-rejection can filter large extrapolation errors while maintaining high yield.
Our \model{} framework adopts a generalization of \textbf{label-free fine-tuning} via Self-Consistency Training~\citep{zhangSelfConsistencyTraining2024} to density \model s.
When transferring QMugs-pretrained models to the reactive chemistry of Transition1x~\citep{schreinerTransition1xDatasetBuilding2022}, this reaches energy errors well below chemical accuracy on the accepted predictions across the entire reaction coordinate (\cref{fig:3bpa_and_transition1x}, \cref{tab:ablation_transition1x_sct_test}).

\paragraph{Limitations \& Future Work.}
\looseness=-1
Our evaluations are restricted to closed-shell Kohn-Sham. Although a generalization to unrestricted systems is straightforward, this has not been covered by our evaluations.
All experiments have been performed at the B3LYP/def2-SVP level of theory. Transfer to other functionals and basis sets remains untested.
The unique capabilities of \model s could enable active-learning pipelines for MLIP reference data generation on large molecules~\citep{kulichenkoDataGenerationMachine2024}.
For each molecule, \ours{} can predict the ground state at the cost of a single Fock build.
If the self-consistency residual falls below the rejection threshold, the pipeline accepts the one-iteration observables.
When the residual exceeds the threshold, the calculation falls back to iterative SCF, where the prediction can serve as an initial guess that accelerates convergence.
The resulting solutions can continually update the \model{}.
The combination of this ansatz with ML exchange-correlation functionals~\citep{gaoLearningEquivariantNon2024,luiseAccurateScalableExchangecorrelation2026,eberhardDerivativeInformedLearning2026} could provide exceptional cost-accuracy trade-offs.

\subsection*{Reproducibility statement}

We provide proofs for all theoretical claims in the appendix, including the reachability, geometric validity, and expressivity of \ours{} on the Grassmann manifold (\appref{app:rotational_groot}) as well as the observable error bounds and quadratic energy scaling for valid density projectors (\appref{app:loewdin-observable-bounds}).
Across all benchmarks, we document dataset sources, reference calculation protocols, and evaluation splits in \appref{app:experimental-configuration}, and specify model architectures and training hyperparameters in \apprefpair{app:model-architectures}{app:training-optimizer}.

The code will be released with the final publication at \href{https://github.com/ESEberhard/ML4DFT}{github.com/ESEberhard/ML4DFT}.

\ificlrfinal
\subsubsection*{Acknowledgments}
We thank Luca A.~Thiede for insightful discussions.
We thank the Vector Institute as well as the Matter Lab (Al{\'a}n Aspuru-Guzik's group)
for co-hosting Eike S.~Eberhard during a research stay where part of this work was conducted.
\fi

\bibliographystyle{preprint_style}
\bibliography{references}

\newpage
\appendix
\captionsetup[table]{font=footnotesize,skip=\dimexpr\abovecaptionskip-0.2cm\relax}
\captionsetup[figure]{font=footnotesize,skip=\dimexpr\abovecaptionskip-0.2cm\relax}

\startcontents[appendix]
\section*{\hypertarget{appendixcontents}{Appendix Contents}}
\begingroup
\linespread{0.85}\selectfont
\renewcommand{\addvspace}[1]{\vspace{0.15em}}
\printcontents[appendix]{}{1}{\setcounter{tocdepth}{2}}
\endgroup
\makeatletter
\let\gr@typeset@sect\H@old@sect
\long\def\H@old@sect#1#2#3#4#5#6[#7]#8{%
  \ifnum#2<4\relax
    \gr@typeset@sect{#1}{#2}{#3}{#4}{#5}{#6}[{#7}]{\linktoappendixtoc{#8}\appendixtocmark}%
  \else
    \gr@typeset@sect{#1}{#2}{#3}{#4}{#5}{#6}[{#7}]{#8}%
  \fi
}
\renewcommand{\@seccntformat}[1]{%
  \protect\linktoappendixtoc{\csname the#1\endcsname}\hskip 1em
}
\makeatother
\newpage

\section{Kohn--Sham Theory in a Finite Basis}
\label{app:finite-basis-ks}

We provide a brief overview of the mathematical fundamentals underlying KS-DFT, aimed at a broader audience at the intersection of quantum chemistry and machine learning.
Hence, we refrain from any explicit references to electronic spin degrees of freedom.
This spin-free treatment is sufficient for so-called closed-shell systems where $\rho^{(\text{up})} = \rho^{(\text{down})}$, requiring an even number of electrons. For further details and generalization to \textit{open-shell} systems, we refer to \citet{szaboModernQuantumChemistry1996} and \citet{lehtolaOverviewSelfConsistent2020}.

In their seminal work, \citet{kohnSelfConsistentEquations1965} showed that the ground state electron density $\rho(\mathbf{r})$ can be constructed from eigenfunctions $\varphi_i(\mathbf{r}): \mathbb{R}^3\to\mathbb{R}$ which are solutions of the Kohn-Sham equations
\begin{equation}
    \left(-\frac{1}{2}\nabla^2 - \sum_{A=1}^{N_\textrm{nuc}} \frac{Z_A}{|\mathbf{r}-\mathbf{R}_A|} + \int_{\mathbb{R}^3}\frac{\rho(\mathbf{r'})}{|\mathbf{r}-\mathbf{r'}|} \,\mathrm{d}\mathbf{r'} + v_\textrm{xc}[\rho](\mathbf{r})\right)\varphi_i(\mathbf{r}) = \epsilon_i\varphi_i(\mathbf{r}),
    \label{eq:kohn_sham_equations}
\end{equation}
which are the governing equations of KS-DFT.
The first term is the kinetic energy operator, the second term the potential due to the atomic nuclei of the system, the third term the electron-electron potential, the fourth term the so-called \textit{exchange-correlation} potential, and $\epsilon_i$ is the eigenvalue of $\varphi_i(\mathbf{r})$.
The eigenfunctions $\varphi_i(\mathbf{r})$ are commonly referred to as molecular orbitals (MOs) and relate to the density via
\begin{equation}
    \rho(\mathbf{r}) = 2 \sum_{i}^{O} |\varphi_i(\mathbf{r})|^2 \:,
    \label{eq:density_from_orbitals}
\end{equation}
where $O = N_\mathrm{elec} / 2$ and $V = B - O$ are the numbers of \textit{occupied} and \textit{virtual orbitals}, respectively, and the eigenfunctions are indexed lowest to highest with respect to their eigenvalues.
By construction, the $\varphi_i(\mathbf{r})$ are orthonormal
$
    \int_{\mathbb{R}^3} \varphi_i(\mathbf{r}) \varphi_j(\mathbf{r}) \,\mathrm{d}\mathbf{r} = \delta_{ij}
$
and hence
$
    \int_{\mathbb{R}^3} \rho(\mathbf{r}) \,\mathrm{d}\mathbf{r} = N_\mathrm{elec}
$.
\cref{eq:kohn_sham_equations} can be discretized by representing the $\varphi_i(\mathbf{r})$ in a \textbf{finite basis}
\begin{equation}
    \left\{\chi_\mu(\mathbf{r}) \Big| \chi_\mu: \mathbb{R}^3 \mapsto \mathbb{R} \: \text{and} \: \int_{\mathbb{R}^3}\chi_\mu(\mathbf{r})^2\,\mathrm{d}\mathbf{r}=1\right\}_{\mu=1}^{B}
\end{equation}
with coefficients $\mathbf C \in \mathbb{R}^{B \times B}$ such that $\varphi_i(\mathbf{r}) = \sum_{\mu=1}^{B} C_{\mu i} \chi_\mu(\mathbf{r})$~\citep{lehtolaOverviewSelfConsistent2020}. Inserting this into \cref{eq:density_from_orbitals} yields
\begin{equation}
    \rho(\mathbf{r}) = 2 \sum_{i}^{O} \sum_{\mu \nu}^{B} C_{\mu i} C_{\nu i} \chi_\mu(\mathbf{r}) \chi_\nu(\mathbf{r}) = 2 \sum_{\mu \nu}^{B} \sum_{i}^{O} C_{\mu i} C_{\nu i} \chi_\mu(\mathbf{r}) \chi_\nu(\mathbf{r}) =: 2 \sum_{\mu\nu}^{B} P_{\mu\nu} \chi_\mu(\mathbf{r}) \chi_\nu(\mathbf{r}),
    \label{eq:density_matrix_def}
\end{equation}
which defines the \textbf{density matrix} $\mathbf P \in \mathbb{R}^{B \times B}$. For molecules, basis functions are commonly constructed from atom-centered Gaussians aligned with the Cartesian axes, whereas periodic chemistry codes often use plane waves. This allows for an efficient computation of many integrals involved in the computation, for example, the integral over the density in \cref{eq:density_from_orbitals} becomes
\begin{equation}
   N_\mathrm{elec} =\int_{\mathbb{R}^3} \rho(\mathbf{r}) \,\mathrm{d}\mathbf{r} = 2 \sum_{\mu \nu}^{B} P_{\mu\nu} \int_{\mathbb{R}^3} \chi_\mu(\mathbf{r}) \chi_\nu(\mathbf{r}) \,\mathrm{d}\mathbf{r} =: 2 \sum_{\mu \nu}^{B} P_{\mu\nu} S_{\mu\nu}
    \label{eq:charge_conservation}
\end{equation}
with the overlap matrix $\mathbf S\in\mathbb{R}^{B \times B}$.
Notably, this integral tensor $S_{\mu\nu}$ is a constant for a given basis set such that it only needs to be computed once per molecule.
This is especially convenient since it can be shown that in finite basis sets \cref{eq:kohn_sham_equations} is equivalent to solving \cref{eq:generalized_eigenvalue_problem}.
Notably, the overlap matrix $\mathbf{S}$ is a \textbf{positive-definite Gram matrix} of the atomic orbital (AO) basis functions, defined as
\begin{equation}
    S_{\mu\nu}
    :=
    \int_{\mathbb{R}^3}
    \chi_\mu(\mathbf{r})
    \chi_\nu(\mathbf{r})
    \,\mathrm{d}\mathbf{r}.
\end{equation}
The inner product is symmetric,
so $S_{\mu\nu}=S_{\nu\mu}$ and $\mathbf{S}=\mathbf{S}^{\mathsf{T}}$.
For any nonzero $\mathbf{x}\in\mathbb{R}^B$,
$\mathbf{x}^{\mathsf{T}}\mathbf{S}\mathbf{x}
=
\sum_{\mu,\nu}x_\mu x_\nu S_{\mu\nu}
=
\bigl\|\sum_\mu x_\mu\chi_\mu\bigr\|^2$.
Linear independence of the AO basis implies
$\sum_\mu x_\mu\chi_\mu\neq 0$ for all $\mathbf{x}\neq 0$,
and therefore $\mathbf{x}^{\mathsf{T}}\mathbf{S}\mathbf{x}>0$.
The overlap matrix is thus symmetric positive definite,
$\mathbf{S}\succ 0$.
Hence $\mathbf{S}$ is invertible.
Normalization of the individual AO basis functions implies $S_{\mu\mu}=1$,
and therefore $\operatorname{tr}(\mathbf{S})=\sum_{\mu=1}^B S_{\mu\mu}=B$.
The $B$ eigenvalues of $\mathbf{S}$ are strictly positive and sum to $B$
s.t.
\begin{equation}
    0
    <
    \lambda_{\min}(\mathbf{S})
    \leq
    1
    \leq
    \lambda_{\max}(\mathbf{S})
    \leq
    B,
\label{eq:overlap_matrix_properties}
\end{equation}
such that \textbf{if any eigenvalue differs from $1$, the remaining eigenvalues compensate with the opposite total deviation.}

\section{Rotational \ours{}: Transport on the Grassmannian}
\label{app:rotational_groot}
\label{app:readout_expressivity}
\label{app:decomposition}

This appendix gives the full construction of \ours{}\textsubscript{rot}: why an
occupied--virtual rotation suffices to reach any target density, how the generators of
such rotations are characterized, and how an unconstrained decoder output is mapped into
them.

This section proves that \ours{} can connect any two physical densities with the same number of occupied orbitals.
The argument has two logically separate parts.
First,
we study the geometry:
does there exist an occupied--virtual rotation that transports the initial density to the target,
does its generator follow a Grassmann geodesic,
and how does it relate to nonminimal rotations that produce the same density?
Second,
we study the readout:
does every model output produce such a valid generator,
and can the readout reach every generator required by the geometry?

We answer these questions in four propositions.
The first two concern only rotations on the Grassmannian and are independent of the neural network.
The third proves that the \ours{} head is valid for every input matrix.
The fourth proves surjectivity,
which is the statement that the head itself introduces no expressivity bottleneck.

\ours{}\textsubscript{rot} is a one-step flow map~\citep{boffiFlowMapMatching2025}, which transports a computationally cheap initial density $\mathbf{P}_0\in \mathcal{M}_{(N_\mathrm{elec},B)}$ to a final density $\widehat{\mathbf{P}}$.\footnote{
    All matrices in this appendix are in the orthonormal basis introduced in \cref{eq:lowdin_transform}. We omit tildes to avoid cluttering the equations.
}
Rotating with an orthogonal matrix $\mathbf{U} \in SO(B)$ preserves these Grassmannian constraints and is therefore the transformation of choice
\begin{equation}
    \widehat{\mathbf{P}} := \mathbf{U} \mathbf{P}_0 \mathbf{U}^{\mathsf{T}}, \qquad
    \mathrm{tr}(\widehat{\mathbf{P}}) = \mathrm{tr}(\mathbf{P}_0 \mathbf{U}^{\mathsf{T}} \mathbf{U}) = \mathrm{tr}(\mathbf{P}_0), \qquad
    \widehat{\mathbf{P}}^2 = \mathbf{U} \mathbf{P}_0^2 \mathbf{U}^{\mathsf{T}} = \widehat{\mathbf{P}}.
\end{equation}
After dropping all orbital rotations that leave $\mathbf{P}$ unchanged,
cf. \cref{fig:coefficient_matrix_gauge_freedom} and \cref{eq:gauge_invariance},
any two valid densities are connected by a matrix
$\mathbf{U}\in SO(B)$
that mixes the occupied and virtual subspaces.
We formalize this by defining the projector into the virtual subspace $\mathbf{\Pi} := \mathbf{I}-\mathbf{P}$ and the reflection matrix $\mathbf{R}:=\mathbf{P}-\mathbf{\Pi}$.
When evaluated at a specific density matrix, such as $\mathbf{P}_0$, we denote the corresponding operators with a subscript ($\mathbf{\Pi}_0, \mathbf{R}_0$).
The reflection $\mathbf{R}$ flips the virtual orbitals around the origin while not affecting the occupied orbitals, i.e. $\mathbf{R}\mathbf{C}_\mathrm{V} = -\mathbf{C}_\mathrm{V}$, $\mathbf{R}\mathbf{C}_\mathrm{O} = \mathbf{C}_\mathrm{O}$.
Rotations that commute with the reflection, $\mathbf{R}\mathbf{U}=\mathbf{U}\mathbf{R}$, are those that are block-diagonal in the molecular orbitals,
mixing occupied with occupied and virtual with virtual, and thus span the subspace of redundant orbital-rotation directions that leave $\mathbf{P}$ unchanged, as shown in \cref{eq:gauge_invariance}.
Complementarily, there also exists a canonical purely mixing rotation between the occupied and the virtual orbitals $\mathbf{U} \equiv \exp(\mathbf{X})\in SO(B)$ generated by a skew-symmetric matrix $\mathbf{X} \in \mathfrak{p}$ that \emph{anti}commutes with the reflection,
\begin{figure}[t]
    \centering
    \includegraphics[width=\linewidth]{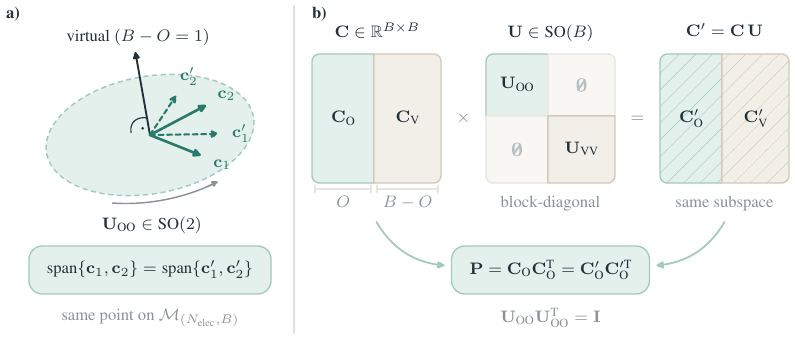}
    \caption{\textbf{Many coefficient matrices, one density.}
    \textbf{(a)}
    The smallest case, $B=3$ and $O=2$: the occupied orbitals $c_1, c_2$ span a plane which determines the density.
    Rotating the arrows within the plane by $\mathbf{U}_{OO}$ changes the orbitals,
    but not the plane.
    \textbf{(b)}
    For arbitrary $O, B \in \mathbb{N}$ with $O < B$.
    The coefficients $\mathbf{C}\in\mathbb{R}^{B\times B}$ split into an
    occupied block $\mathbf{C}_O\in\mathbb{R}^{B\times O}$ and a virtual
    block $\mathbf{C}_{\mathrm{V}}\in\mathbb{R}^{B\times V}$. Only the occupied
    block determines the density matrix as seen in \cref{eq:P_from_C}.
    Multiplication from the right by a block-diagonal orthogonal matrix $\mathbf{U}$ does not change the density matrix, see \cref{eq:gauge_invariance}.
    }
    \label{fig:coefficient_matrix_gauge_freedom}
\end{figure}
\begin{equation}
\mathfrak{p}
\equiv
\left\{
\mathbf{X}\in\mathbb{R}^{B\times B}
\ \middle|\
\mathbf{X}^{\mathsf{T}}=-\mathbf{X},
\ \
\mathbf{X}\mathbf{R}
=
-\mathbf{R}\mathbf{X}
\right\}, \quad
\text{with } \quad \mathbf{R} = 2\mathbf{P}-\mathbf{I}.
\label{eq:grassmann_tangent_generators}
\end{equation}
Throughout this section all matrices are represented in the orthonormal basis (\cref{eq:lowdin_transform}).
Let
$\widetilde{\mathbf{P}}_0$
and
$\widetilde{\mathbf{P}}_1$
be physical densities with the same rank,
and define
\begin{equation}
\widetilde{\mathbf{\Pi}}_a
=
\mathbf{I}-\widetilde{\mathbf{P}}_a,
\qquad
\widetilde{\mathbf{R}}_a
=
\widetilde{\mathbf{P}}_a
-
\widetilde{\mathbf{\Pi}}_a,
\qquad
a\in\{0,1\}.
\label{eq:oao_endpoint_reflections}
\end{equation}
Each reflection is symmetric and involutory,
\begin{equation}
\left(\widetilde{\mathbf{R}}_a\right)^{\mathsf{T}}
=
\widetilde{\mathbf{R}}_a,
\qquad
\left(\widetilde{\mathbf{R}}_a\right)^2
=
\mathbf{I},
\end{equation}
and acts as
$+1$
on the occupied subspace and as
$-1$
on the virtual subspace.
\citet{wongDifferentialGeometryGrassmann1967} showed that these pure mixing rotations suffice to connect any two densities on $\mathcal{M}_{(N_\textrm{elec},B)}$.
\begin{proposition}[Reachability by an occupied--virtual rotation]
\label{prop:decomposition}
\propReachabilityBody
\end{proposition}

\begin{proof}
If \(O>V\), consider instead the complementary projectors
\(\widetilde{\mathbf{\Pi}}_a=\mathbf{I}-\widetilde{\mathbf{P}}_a\),
which have rank \(V\) and codimension \(O\).
Their reflections are
\(-\widetilde{\mathbf{R}}_a\), so they define the same anticommuting generator space
\(\mathfrak{p}\).
Moreover, a rotation transports
\(\widetilde{\mathbf{\Pi}}_0\) to \(\widetilde{\mathbf{\Pi}}_1\) if and only if it transports
\(\widetilde{\mathbf{P}}_0\) to \(\widetilde{\mathbf{P}}_1\).
It therefore suffices to prove the case \(O\le V\).

Since the physical densities in the orthonormal basis are themselves orthogonal projectors, we can always find
\(\widetilde{\mathbf{O}}_0,\widetilde{\mathbf{O}}_1\in O(B)\) such
that
\(
\widetilde{\mathbf{P}}_a
=\widetilde{\mathbf{O}}_a
\widetilde{\mathbf{P}}_\mathrm{ref}
\widetilde{\mathbf{O}}_a^{\mathsf{T}}
\),
with
\(a\in\{0,1\}\) and
\(\widetilde{\mathbf{P}}_\mathrm{ref}
=\operatorname{diag}\!\left(
\mathbf{I}_{O},
\mathbf{0}_{V}
\right)\).
Express the second MO frame in the
first through
\begin{equation}
\widetilde{\mathbf{Q}}
=\widetilde{\mathbf{O}}_0^{\mathsf{T}}\widetilde{\mathbf{O}}_1,
\qquad
\widetilde{\mathbf{O}}_1
=\widetilde{\mathbf{O}}_0\widetilde{\mathbf{Q}}.
\label{eq:relative_mo_rotation}
\end{equation}
Consequently, in the \(\widetilde{\mathbf{O}}_0\)-MO basis,
\begin{equation}
\widetilde{\mathbf{O}}_0^{\mathsf{T}}
\widetilde{\mathbf{P}}_0
\widetilde{\mathbf{O}}_0
=\widetilde{\mathbf{P}}_\mathrm{ref},
\qquad
\widetilde{\mathbf{O}}_0^{\mathsf{T}}
\widetilde{\mathbf{P}}_1
\widetilde{\mathbf{O}}_0
=\widetilde{\mathbf{Q}}
\widetilde{\mathbf{P}}_\mathrm{ref}
\widetilde{\mathbf{Q}}^{\mathsf{T}}.
\label{eq:projectors_mo1_basis}
\end{equation}

The general cosine-sine (CS) decomposition states that for every orthogonal
\(\widetilde{\mathbf{D}}\in\mathrm{SO}(B)\),
there exist block-diagonal orthogonal matrices
\(
\widetilde{\mathbf{L}}
=\operatorname{diag}(
\widetilde{\mathbf{L}}_\mathrm{o},
\widetilde{\mathbf{L}}_\mathrm{v}
)
\)
and
\(
\widetilde{\mathbf{S}}
=\operatorname{diag}(
\widetilde{\mathbf{S}}_\mathrm{o},
\widetilde{\mathbf{S}}_\mathrm{v}
)
\),
as well as principal angles
\(0\le\theta_i\le\frac{\pi}{2}\),
such that
\(
\widetilde{\mathbf{D}}
=\widetilde{\mathbf{L}}
\widetilde{\mathbf{G}}
\widetilde{\mathbf{S}}^{\mathsf{T}}
\),
where, in the
\(O+O+(V-O)\)
partition,
\begin{equation}
\widetilde{\mathbf{G}}
=
\begin{bmatrix}
\cos\bm{\Theta}&\sin\bm{\Theta}&\mathbf{0}\\
-\sin\bm{\Theta}&\cos\bm{\Theta}&\mathbf{0}\\
\mathbf{0}&\mathbf{0}&\mathbf{I}_{V-O},
\end{bmatrix}
\qquad
\text{with}
\qquad
\bm{\Theta}
=\operatorname{diag}(\theta_1,\ldots,\theta_{O}).
\label{eq:cs-decomp}
\end{equation}

We proceed to apply the general CS decomposition to \(\widetilde{\mathbf{Q}}\).
Because \(\widetilde{\mathbf{L}}\) and \(\widetilde{\mathbf{S}}\) preserve the occupied--virtual splitting,
they are block-diagonal orbital rotations that leave
\(\widetilde{\mathbf{P}}_\mathrm{ref}\)
unchanged and can be freely exchanged for one another in
\begin{equation}
\widetilde{\mathbf{L}}^{\mathsf{T}}
\widetilde{\mathbf{P}}_\mathrm{ref}
\widetilde{\mathbf{L}}
=\widetilde{\mathbf{P}}_\mathrm{ref}
=\widetilde{\mathbf{S}}^{\mathsf{T}}
\widetilde{\mathbf{P}}_\mathrm{ref}
\widetilde{\mathbf{S}}.
\label{eq:cs_gauge_invariance}
\end{equation}

Still working in the \(\widetilde{\mathbf{O}}_0\)-MO basis, we obtain
\begin{equation}
\widetilde{\mathbf{O}}_0^{\mathsf{T}}
\widetilde{\mathbf{P}}_1
\widetilde{\mathbf{O}}_0
=
(
\widetilde{\mathbf{L}}
\widetilde{\mathbf{G}}
\widetilde{\mathbf{L}}^{\mathsf{T}}
)
\widetilde{\mathbf{P}}_\mathrm{ref}
(
\widetilde{\mathbf{L}}
\widetilde{\mathbf{G}}^{\mathsf{T}}
\widetilde{\mathbf{L}}^{\mathsf{T}}
).
\label{eq:active_transport_mo_basis}
\end{equation}
Returning to the orthonormal basis,
we absorb the block-diagonal factor into the MO frame by defining
\(
\widetilde{\mathbf{O}}'_0
\equiv
\widetilde{\mathbf{O}}_0\widetilde{\mathbf{L}}
\).
This gives
\begin{equation}
\begin{aligned}
\widetilde{\mathbf{P}}_0
=
\widetilde{\mathbf{O}}'_0
\widetilde{\mathbf{P}}_\mathrm{ref}
(\widetilde{\mathbf{O}}'_0)^{\mathsf{T}},
\quad
\widetilde{\mathbf{P}}_1
=
\bigl(
\widetilde{\mathbf{O}}'_0
\widetilde{\mathbf{G}}
(\widetilde{\mathbf{O}}'_0)^{\mathsf{T}}
\bigr)
\widetilde{\mathbf{P}}_0
\bigl(
\widetilde{\mathbf{O}}'_0
\widetilde{\mathbf{G}}^{\mathsf{T}}
(\widetilde{\mathbf{O}}'_0)^{\mathsf{T}}
\bigr).
\end{aligned}
\label{eq:active_transport_oao_basis}
\end{equation}

As with 2D rotations, we express the rotation in the MO basis as the matrix exponential
\(\widetilde{\mathbf{G}}=\exp(\widetilde{\mathbf{K}})\)
of the skew-symmetric matrix
\begin{equation}
\widetilde{\mathbf{K}}
=
\begin{bmatrix}
\mathbf{0}&\bm{\Theta}&\mathbf{0}\\
-\bm{\Theta}&\mathbf{0}&\mathbf{0}\\
\mathbf{0}&\mathbf{0}&\mathbf{0}
\end{bmatrix},
\label{eq:principal_generator}
\end{equation}
which, in the orthonormal basis, yields the generator
\(\mathbf{X}
=\widetilde{\mathbf{O}}'_0
\widetilde{\mathbf{K}}
(\widetilde{\mathbf{O}}'_0)^{\mathsf{T}}\)
such that the principal rotation is
\begin{equation}
\exp(\mathbf{X})
\widetilde{\mathbf{P}}_0
\exp(-\mathbf{X})
=\widetilde{\mathbf{P}}_1.
\end{equation}

It remains to verify the second defining anticommutation condition of
\(\mathfrak{p}\).
In the MO basis, the reflection operator similarly becomes diagonal and takes the form
\(
\widetilde{\mathbf{R}}_\mathrm{ref}
=\operatorname{diag}\!\left(
\mathbf{I}_{O},
-\mathbf{I}_{O},
-\mathbf{I}_{V-O}
\right)
\).
Hence, the block forms of \(\widetilde{\mathbf{K}}\) and
\(\widetilde{\mathbf{R}}_\mathrm{ref}\) imply
\(\widetilde{\mathbf{K}}\widetilde{\mathbf{R}}_\mathrm{ref}
+\widetilde{\mathbf{R}}_\mathrm{ref}\widetilde{\mathbf{K}}
=\mathbf{0}\).
Since commutation relations are basis-independent, this statement carries over into the orthonormal basis such that
\begin{equation}
\mathbf{X}\widetilde{\mathbf{R}}_0
+\widetilde{\mathbf{R}}_0\mathbf{X}
=\mathbf{0}.
\end{equation}
Finally, since \(\mathbf{X}\) is real and skew-symmetric,
\(\mathbf{U}
=\exp(\mathbf{X})\in\mathrm{SO}(B)\),
which proves the claim.
\end{proof}

For the same density pair, any other transporting rotation differs from a pure mixing rotation only by a block-diagonal orbital rotation that leaves the density unchanged~\citep{edelmanGeometryAlgorithmsOrthogonality1998}.
\begin{proposition}[Decomposition of a general transporting rotation]
\label{prop:general_rotation_decomposition}
\propGeneralRotationBody
\end{proposition}

\begin{proof}
Define
$\mathbf{V}
=
\mathbf{U}^{\mathsf{T}}\widetilde{\mathbf{W}}$.
It lies in
$\mathrm{SO}(B)$
and gives
$\widetilde{\mathbf{W}}=\mathbf{U}\mathbf{V}$.
Both rotations transport the initial reflection to the target reflection,
so
\begin{equation}
\mathbf{V}
\widetilde{\mathbf{R}}_0
\mathbf{V}^{\mathsf{T}}
=
\mathbf{U}^{\mathsf{T}}
\widetilde{\mathbf{R}}_1
\mathbf{U}
=
\widetilde{\mathbf{R}}_0.
\end{equation}
Since \(\mathbf{V}\) is orthogonal, this is equivalent to
\(\mathbf{V}\widetilde{\mathbf{R}}_0
=
\widetilde{\mathbf{R}}_0\mathbf{V}\).
Thus \(\mathbf{V}\) is a block-diagonal orbital rotation that leaves the density unchanged.
\end{proof}

Rather than direct density prediction, \ours{}\textsubscript{rot} builds on equivariant matrix decoders~\citep{yuEfficientEquivariantGraph2023,febrerGraph2matUniversalGraph2025} to construct only an occupied-virtual rotation that moves a computationally cheap density toward the target density.
For every decoder output $\hat{\mathbf{M}}\in\mathbb{R}^{B\times B}$, we construct
\begin{equation}
    \hat{\mathbf{X}}_\mathrm{\ours}\left(\hat{\mathbf{M}}\right)
    =
    \mathbf{P}_0
    \hat{\mathbf{M}}
    \mathbf{\Pi}_0
    -
    \mathbf{\Pi}_0
    \hat{\mathbf{M}}^{\mathsf{T}}
    \mathbf{P}_0
\label{eq:groot_map}
\end{equation}
which lies in $\mathfrak{p}$, as shown in \appref{app:readout_expressivity}, Proposition~\ref{prop:groot_validity}.
The \ours{}\textsubscript{rot} map $\hat{\mathbf{M}}\mapsto\hat{\mathbf{X}}_\mathrm{\ours}(\hat{\mathbf{M}})$ is surjective onto $\mathfrak{p}$ because every valid $\mathbf{X}\in\mathfrak{p}$ is a fixed point,
i.e. choosing $\hat{\mathbf{M}}=\mathbf{X}$ yields $\hat{\mathbf{X}}_\mathrm{\ours}(\hat{\mathbf{M}})=\mathbf{X}$, as shown in \appref{app:readout_expressivity}, Proposition~\ref{prop:readout_expressivity}.

\begin{appendixproposition}[Validity of the \ours{} map]
\label{prop:groot_validity}
\propGrootValidityBody
\end{appendixproposition}

\begin{proof}
Let
$\widetilde{\mathbf{A}}
=
\widetilde{\mathbf{P}}_0
\hat{\mathbf{M}}
\widetilde{\mathbf{\Pi}}_0$.
Since the projectors are symmetric,
$\hat{\mathbf{X}}_\mathrm{\ours}(\hat{\mathbf{M}})
=
\widetilde{\mathbf{A}}-\widetilde{\mathbf{A}}^{\mathsf{T}}$
is skew-symmetric. Moreover, since
\(\widetilde{\mathbf{R}}_0\widetilde{\mathbf{P}}_0
=\widetilde{\mathbf{P}}_0\) and
\(\widetilde{\mathbf{R}}_0\widetilde{\mathbf{\Pi}}_0
=-\widetilde{\mathbf{\Pi}}_0\), together with the corresponding
right-multiplication identities, we have
\(\widetilde{\mathbf{R}}_0\widetilde{\mathbf{A}}\widetilde{\mathbf{R}}_0
=-
\widetilde{\mathbf{A}}\) and
\(\widetilde{\mathbf{R}}_0\widetilde{\mathbf{A}}^{\mathsf{T}}\widetilde{\mathbf{R}}_0
=-
\widetilde{\mathbf{A}}^{\mathsf{T}}\). Therefore,
\begin{equation}
\widetilde{\mathbf{R}}_0
\hat{\mathbf{X}}_\mathrm{\ours}(\hat{\mathbf{M}})
\widetilde{\mathbf{R}}_0
=
-\hat{\mathbf{X}}_\mathrm{\ours}(\hat{\mathbf{M}}).
\end{equation}
Since \(\widetilde{\mathbf{R}}_0^2=\mathbf{I}\), this is equivalent to
the required anticommutation relation. Thus
\(\hat{\mathbf{X}}_\mathrm{\ours}(\hat{\mathbf{M}})\) lies in
\(\mathfrak{p}\).
\end{proof}

\begin{appendixproposition}[Surjectivity of the \ours{} map]
\label{prop:readout_expressivity}
\propReadoutExpressivityBody
\end{appendixproposition}

\begin{proof}
Let $\mathbf{X}
\in
\mathfrak{p}$.
Choose \(\hat{\mathbf{M}}=\mathbf{X}\). Since
\(\mathbf{X}^{\mathsf{T}}=-\mathbf{X}\), the
readout gives
\begin{equation}
\hat{\mathbf{X}}_\mathrm{\ours}(\mathbf{X})
=
\widetilde{\mathbf{P}}_0
\mathbf{X}
\widetilde{\mathbf{\Pi}}_0
+
\widetilde{\mathbf{\Pi}}_0
\mathbf{X}
\widetilde{\mathbf{P}}_0
=
\mathbf{X}.
\end{equation}
Thus every valid generator has a preimage,
and the map is surjective.
\end{proof}

The first two propositions show that any target occupied subspace is geometrically reachable and that non-minimal transporting rotations differ only by block-diagonal orbital rotations that leave the endpoint densities unchanged.
\Cref{prop:groot_validity,prop:readout_expressivity} show that the \ours{} map reaches the entire space of minimal geodesic generators while producing a valid generator for every possible neural-network output.

The \ours{}\textsubscript{rot} readout must also preserve the $SO(3)$-equivariance of the underlying matrix decoder.
A spatial rotation induces a corresponding $\mathbf{U}\in SO(B)$ in the orthonormal basis, transforming both $\mathbf{P}_0$ and $\hat{\mathbf{M}}$ as $\mathbf{Y}\rightarrow\mathbf{U}\mathbf{Y}\mathbf{U}^{\mathsf{T}}$, so \cref{eq:groot_map} gives $\hat{\mathbf{X}}_\mathrm{\ours}\rightarrow\mathbf{U}\hat{\mathbf{X}}_\mathrm{\ours}\mathbf{U}^{\mathsf{T}}$.
Hence, \ours{}\textsubscript{rot}'s covariant generator can learn to produce a rotation to any valid target density $\mathbf{\hat{P}}\approx \mathbf{P}^*$ from $\mathbf{P}_0$ (Proposition~\ref{prop:decomposition})
\begin{equation}
    \mathbf{\hat{P}}=\exp\left(\hat{\mathbf{X}}_\mathrm{\ours}(\hat{\mathbf{M}})\right)\mathbf{P}_0\exp\left(-\hat{\mathbf{X}}_\mathrm{\ours}(\hat{\mathbf{M}})\right).
    \label{eq:groot_density}
\end{equation}

\section{Observable Bounds for the Orthonormal Loss}
\label{app:loewdin-observable-bounds}

We derive the observable-error bounds summarized in~\cref{tab:loewdin-observable-scaling} at a fixed molecular geometry.
Density errors give linear bounds for one-body observables, including the dipole and real-space density.
For valid density projectors, stationarity makes the energy error quadratic.
Hamiltonian predictions inherit these bounds through their occupied eigenspaces, with an additional dependence on the HOMO--LUMO gap.

We use real matrices and the closed-shell, spin-free convention $\widetilde{\mathbf{P}}^*=\widetilde{\mathbf{C}}_{\mathrm{o}}^*\widetilde{\mathbf{C}}_{\mathrm{o}}^{*\mathsf{T}}$, where $\operatorname{tr}(\widetilde{\mathbf{P}}^*)=O$ and $O$ is the number of occupied orbitals.
The physical electron density includes an occupation factor of two.
We write $\mathbf{D}=\hat{\widetilde{\mathbf{P}}}-\widetilde{\mathbf{P}}^*$ for the density error and assume that the overlap matrix $\mathbf{S}$ is positive definite.
For the energy bounds, we assume that $E_m$ is three times continuously differentiable near the reference, with locally bounded third derivative, and stationary there under rotations between occupied and virtual orbitals.
All asymptotic statements concern small errors for a fixed sample $m$.
The bounds use unnormalized Frobenius norms, so substitution of $\|\Delta\widetilde{\mathbf{X}}\|_F=B\mathcal{L}_{\mathrm{ON}}$ gives the corresponding bounds for the loss in~\cref{eq:loss}.

\subsection{AO and Orthonormal Frobenius Errors}

The overlap eigenvalues determine how the AO and orthonormal norms weight each error direction.
Write
\begin{align}
    \mathbf{S}&=\mathbf{U}\operatorname{diag}(\lambda_1,\ldots,\lambda_B)\mathbf{U}^{\mathsf{T}},
    &0<\lambda_{\min}\leq\lambda_i\leq\lambda_{\max}.
    \label{eq:app-overlap-spectrum}
\end{align}
In this eigenbasis, let $\Delta\bar{\mathbf{X}}=\mathbf{U}^{\mathsf{T}}\Delta\mathbf{X}\mathbf{U}$.
The density transformation in~\cref{eq:lowdin_transform} gives $\|\Delta\widetilde{\mathbf{P}}\|_F^2=\sum_{i,j}\lambda_i\lambda_j|\Delta\bar{P}_{ij}|^2$, whereas the Hamiltonian transformation gives $\|\Delta\widetilde{\mathbf{H}}\|_F^2=\sum_{i,j}|\Delta\bar{H}_{ij}|^2/(\lambda_i\lambda_j)$.
Bounding each weight by its extreme values yields
\begin{align}
    \lambda_{\min}\|\Delta\mathbf{P}\|_F
    &\leq\|\Delta\widetilde{\mathbf{P}}\|_F
    \leq\lambda_{\max}\|\Delta\mathbf{P}\|_F,
    \label{eq:app-density-ao-bound}\\
    \lambda_{\max}^{-1}\|\Delta\mathbf{H}\|_F
    &\leq\|\Delta\widetilde{\mathbf{H}}\|_F
    \leq\lambda_{\min}^{-1}\|\Delta\mathbf{H}\|_F.
    \label{eq:app-hamiltonian-ao-bound}
\end{align}
Errors of the form $\mathbf{u}_i\mathbf{u}_i^{\mathsf{T}}$ attain these endpoints when $\mathbf{u}_i$ is an extreme overlap eigenvector.
The constants are therefore sharp for unrestricted symmetric errors.
Restrictions to projector tangent directions or occupied--virtual Hamiltonian blocks can lower the constants for a particular sample.

Relative to the orthonormal norm, the plain AO density norm gives more weight to directions with small overlap eigenvalues and less weight to directions with large eigenvalues.
For Hamiltonian errors, the weighting reverses.
\Cref{eq:app-density-ao-bound,eq:app-hamiltonian-ao-bound} supply the worst-case AO conversion factors in~\cref{tab:loewdin-observable-scaling}.
They do not establish that the orthonormal bound is tighter for every observable and error direction.

\subsection{Choice of Orthonormal Basis}

Every full-rank orthonormalization gives the same Frobenius loss.
To see this, take any square matrix $\mathbf{W}$ satisfying $\mathbf{W}^{\mathsf{T}}\mathbf{S}\mathbf{W}=\mathbf{I}$ and set $\mathbf{Q}=\mathbf{S}^{1/2}\mathbf{W}$.
Then $\mathbf{Q}^{\mathsf{T}}\mathbf{Q}=\mathbf{I}$ and $\mathbf{W}=\mathbf{S}^{-1/2}\mathbf{Q}$.
The resulting Hamiltonian and density representations are $\mathbf{Q}^{\mathsf{T}}\widetilde{\mathbf{H}}\mathbf{Q}$ and $\mathbf{Q}^{\mathsf{T}}\widetilde{\mathbf{P}}\mathbf{Q}$, respectively.
Orthogonal invariance gives $\|\mathbf{Q}^{\mathsf{T}}\Delta\widetilde{\mathbf{X}}\mathbf{Q}\|_F=\|\Delta\widetilde{\mathbf{X}}\|_F$ for either matrix.
The symmetric L\"owdin choice additionally minimizes the sum of squared distances between the original basis functions and their orthonormal counterparts~\citep{lowdinNonOrthogonalityProblemConnected1950,carlsonOrthogonalizationProcedures1957}.

Observable-specific bounds also depend on the observable operator.
For example, an AO one-body operator $\mathbf{M}=\mathbf{S}^{1/2}\widetilde{\mathbf{M}}\mathbf{S}^{1/2}$ satisfies $|\operatorname{tr}(\mathbf{M}\Delta\mathbf{P})|\leq\|\mathbf{M}\|_F\|\Delta\mathbf{P}\|_F$, with
\begin{align*}
    \lambda_{\min}\|\widetilde{\mathbf{M}}\|_F
    \leq\|\mathbf{M}\|_F
    \leq\lambda_{\max}\|\widetilde{\mathbf{M}}\|_F.
\end{align*}
Thus, a direct AO bound can have a smaller or larger constant than its orthonormal counterpart.

\subsection{General Density Errors}

An unconstrained density prediction can incur a first-order energy error even at a stationary reference.
Under our occupation convention, $\nabla E_m[\widetilde{\mathbf{P}}^*]=2\widetilde{\mathbf{H}}^*$, so Taylor expansion gives
\begin{align*}
    E_m[\widetilde{\mathbf{P}}^*+\mathbf{D}]-E_m[\widetilde{\mathbf{P}}^*]
    &=2\langle\widetilde{\mathbf{H}}^*,\mathbf{D}\rangle_F
    +\mathcal{O}(\|\mathbf{D}\|_F^2).
\end{align*}
Cauchy--Schwarz bounds the magnitude of the linear term by $2\|\widetilde{\mathbf{H}}^*\|_F\|\mathbf{D}\|_F$.
Stationarity makes $\widetilde{\mathbf{H}}^*$ block diagonal in the occupied--virtual basis, so only the occupied--occupied and virtual--virtual blocks of $\mathbf{D}$ contribute to this term.
An unconstrained predictor can change these blocks at first order through occupation or charge errors.

One-body observables depend linearly on the density matrix.
For an observable with orthonormal-basis matrix $\widetilde{\mathbf{M}}$, including its fixed charge and occupation factors, Cauchy--Schwarz gives
\begin{align*}
    |\Delta M|
    =|\operatorname{tr}(\widetilde{\mathbf{M}}\mathbf{D})|
    \leq\|\widetilde{\mathbf{M}}\|_F\|\mathbf{D}\|_F.
\end{align*}
For the dipole, we apply this inequality to each Cartesian component, square the three inequalities, and add them.
With $C_{\mu,m}^2=\sum_{\alpha\in\{x,y,z\}}\|\widetilde{\boldsymbol{\mu}}_\alpha\|_F^2$, the result is $\|\Delta\boldsymbol{\mu}\|_2\leq C_{\mu,m}\|\mathbf{D}\|_F$.
This proves the dipole row of~\cref{tab:loewdin-observable-scaling}.
The nuclear contribution cancels because we compare predictions at the same geometry.

The real-space electron density is another linear function of the density matrix.
For real orthonormal basis functions $\widetilde{\chi}_\mu$, define $\Phi_{\mu\nu}(\mathbf{r})=2\widetilde{\chi}_\mu(\mathbf{r})\widetilde{\chi}_\nu(\mathbf{r})$ and $\Delta\rho(\mathbf{r})=\sum_{\mu,\nu}D_{\mu\nu}\Phi_{\mu\nu}(\mathbf{r})$.
Its $L_1$ norm measures the integrated absolute density error and satisfies
\begin{align*}
    \|\Delta\rho\|_{L_1}
    \leq 2\sum_{j=1}^{B}|d_j|
    \leq 2\sqrt{B}\|\mathbf{D}\|_F,
\end{align*}
where $d_j$ are the eigenvalues of $\mathbf{D}$.
To obtain the first inequality, diagonalize $\mathbf{D}$ and rotate the basis functions by the same orthogonal matrix.
Then $\Delta\rho=2\sum_j d_j\psi_j^2$ with $\int\psi_j^2=1$, so the triangle inequality applies term by term.
The second inequality is Cauchy--Schwarz.
Taking $\mathbf{D}=\mathbf{I}$ attains both inequalities, so the induced Frobenius-to-$L_1$ constant for unrestricted symmetric errors is $C_{1,m}=2\sqrt{B}$.

The $L_2$ density error depends on the spatial overlap of the basis-function products.
Assuming these products are square-integrable, define $\Gamma_{(\mu\nu),(\kappa\ell)}=\int\Phi_{\mu\nu}(\mathbf{r})\Phi_{\kappa\ell}(\mathbf{r})\,\mathrm{d}\mathbf{r}$.
Expanding the squared density error gives
\begin{align*}
    \|\Delta\rho\|_{L_2}^2
    &=\operatorname{vec}(\mathbf{D})^{\mathsf{T}}\boldsymbol{\Gamma}\operatorname{vec}(\mathbf{D})
    \leq\lambda_{\max}(\boldsymbol{\Gamma})\|\mathbf{D}\|_F^2.
\end{align*}
Thus $\|\Delta\rho\|_{L_2}\leq C_{2,m}\|\mathbf{D}\|_F$, where $C_{2,m}=\sqrt{\lambda_{\max}(\boldsymbol{\Gamma})}$.
The $L_2$ row in~\cref{tab:loewdin-observable-scaling} refers to the norm itself, whose bound is linear in $\|\mathbf{D}\|_F$.
Its square has a quadratic bound.
Substituting~\eqref{eq:app-density-ao-bound} into these observable bounds supplies one factor of $\lambda_{\max}$ for each power of the density error.

\subsection{Grassmann-Constrained Density Errors}

For a valid rank-$O$ projector, the leading density error only mixes occupied and virtual orbitals.
This is the constraint that removes the first-order energy term.
In the occupied--virtual basis, write $\widetilde{\mathbf{P}}^*=\operatorname{diag}(\mathbf{I}_O,\mathbf{0})$ and parameterize a nearby projector by
\begin{align*}
    \mathbf{K}(\boldsymbol{X})
    &=\begin{pmatrix}\mathbf{0}&\boldsymbol{X}^{\mathsf{T}}\\-\boldsymbol{X}&\mathbf{0}\end{pmatrix},
    &\widetilde{\mathbf{P}}(\boldsymbol{X})
    &=e^{\mathbf{K}(\boldsymbol{X})}\widetilde{\mathbf{P}}^*e^{-\mathbf{K}(\boldsymbol{X})}.
\end{align*}
Here $\boldsymbol{X}\in\mathbb{R}^{(B-O)\times O}$ contains the occupied--virtual rotation parameters.
Because $\mathbf{K}$ is skew-symmetric, this construction preserves rank and idempotency.
Expansion around $\boldsymbol{X}=\mathbf{0}$ gives
\begin{align}
    \mathbf{D}
    &=\begin{pmatrix}\mathbf{0}&-\boldsymbol{X}^{\mathsf{T}}\\-\boldsymbol{X}&\mathbf{0}\end{pmatrix}
    +\begin{pmatrix}-\boldsymbol{X}^{\mathsf{T}}\boldsymbol{X}&\mathbf{0}\\\mathbf{0}&\boldsymbol{X}\boldsymbol{X}^{\mathsf{T}}\end{pmatrix}
    +\mathcal{O}(\|\boldsymbol{X}\|_F^3).
    \label{eq:app-grassmann-expansion}
\end{align}
The diagonal blocks begin at second order.
The squared error satisfies $\|\mathbf{D}\|_F^2=2\|\boldsymbol{X}\|_F^2+\mathcal{O}(\|\boldsymbol{X}\|_F^4)$.

To make the energy constant precise, let $\mathbf{x}=\operatorname{vec}(\boldsymbol{X})$ and let $\mathbf{A}_m$ be the Hessian of $f_m(\mathbf{x})=E_m[\widetilde{\mathbf{P}}(\boldsymbol{X})]$ at $\mathbf{x}=\mathbf{0}$.
Stationarity gives $\nabla f_m(\mathbf{0})=\mathbf{0}$, hence $f_m(\mathbf{x})-f_m(\mathbf{0})=\mathbf{x}^{\mathsf{T}}\mathbf{A}_m\mathbf{x}/2+\mathcal{O}(\|\mathbf{x}\|_2^3)$.
Define $h_m=\|\mathbf{A}_m\|_2/2$, where $\|\cdot\|_2$ denotes the spectral norm.
The factor of two accounts for the two off-diagonal blocks in the Frobenius norm of a tangent perturbation.
We obtain
\begin{align}
    \big|E_m[\hat{\widetilde{\mathbf{P}}}]-E_m[\widetilde{\mathbf{P}}^*]\big|
    &\leq\frac{h_m}{2}\|\mathbf{D}\|_F^2
    +\mathcal{O}(\|\mathbf{D}\|_F^3).
    \label{eq:app-groot-energy-bound}
\end{align}
This local bound holds for nearby rank-$O$ projectors, including predictions from both GROOT constructions.
The dipole and real-space density bounds remain linear because their maps are exactly linear in $\mathbf{D}$.

The projector error also has an exact geometric interpretation.
For principal angles $\theta_i$ between the predicted and reference occupied subspaces,
\begin{align}
    \|\mathbf{D}\|_F^2
    =2\sum_{i=1}^{O}\sin^2\theta_i
    =2\bigl(O-\operatorname{tr}(\hat{\widetilde{\mathbf{P}}}\widetilde{\mathbf{P}}^*)\bigr).
    \label{eq:app-principal-angle-distance}
\end{align}
\cref{eq:app-principal-angle-distance} connects the Frobenius error to the orbital-projection score without choosing individual occupied orbitals.

\subsection{Hamiltonian-Induced Density Errors}

A Hamiltonian prediction determines a valid density projector through its lowest $O$ eigenvectors.
Write $\hat{\widetilde{\mathbf{H}}}=\widetilde{\mathbf{H}}^*+\mathbf{V}$ and let $g=\Delta\epsilon_{\mathrm{HL}}=\epsilon_{\mathrm{LUMO}}^*-\epsilon_{\mathrm{HOMO}}^*>0$.
We assume that $\widetilde{\mathbf{P}}^*$ projects onto the lowest $O$ eigenvectors of $\widetilde{\mathbf{H}}^*$ and that $\|\mathbf{V}\|_2<g/2$, which keeps the occupied and virtual eigenvalues separated.
In the reference eigenbasis, first-order expansion of the projector and eigenvalue equations gives
\begin{align*}
    \mathbf{D}
    &=\begin{pmatrix}\mathbf{0}&\boldsymbol{T}^{\mathsf{T}}\\\boldsymbol{T}&\mathbf{0}\end{pmatrix}
    +\mathcal{O}\bigl(\|\mathbf{V}\|_F^2/g^2\bigr),
    &T_{ai}&=-\frac{V_{ai}}{\epsilon_a^*-\epsilon_i^*}.
\end{align*}
Indices $i$ and $a$ denote occupied and virtual orbitals, respectively, and the matrix remainder is in Frobenius norm.
The first-order diagonal blocks vanish by idempotency, and the occupied--virtual eigenvalue equation is $(\epsilon_a^*-\epsilon_i^*)T_{ai}+V_{ai}=0$.
Thus
\begin{align*}
    \|\mathbf{D}\|_F^2
    &=2\sum_{i\in\mathrm{o},\,a\in\mathrm{v}}\frac{|V_{ai}|^2}{(\epsilon_a^*-\epsilon_i^*)^2}
    +\mathcal{O}\bigl(\|\mathbf{V}\|_F^3/g^3\bigr)\\
    &\leq\frac{\|\mathbf{V}\|_F^2}{g^2}
    +\mathcal{O}\bigl(\|\mathbf{V}\|_F^3/g^3\bigr).
\end{align*}
The inequality uses $\epsilon_a^*-\epsilon_i^*\geq g$ and $2\sum_{i,a}|V_{ai}|^2\leq\|\mathbf{V}\|_F^2$ for symmetric $\mathbf{V}$.
Equivalently, $\|\mathbf{D}\|_F\leq\|\mathbf{V}\|_F/g+\mathcal{O}(\|\mathbf{V}\|_F^2/g^2)$, which gives the inverse-gap factor for the dipole and density bounds.

For the energy rebuilt from this induced density, substitution into~\eqref{eq:app-groot-energy-bound} gives
\begin{align}
    \big|E_m[\hat{\widetilde{\mathbf{P}}}]-E_m[\widetilde{\mathbf{P}}^*]\big|
    &\leq\frac{h_m}{2g^2}\|\mathbf{V}\|_F^2
    +\mathcal{O}\bigl(\|\mathbf{V}\|_F^3/g^3\bigr).
    \label{eq:app-hamiltonian-energy-bound}
\end{align}
The curvature $h_m$ and the remainder constants depend on the sample and energy functional.
\cref{eq:app-hamiltonian-energy-bound} gives the inverse-gap-squared factor in~\cref{tab:loewdin-observable-scaling} without assuming how these constants vary across molecules.
Using~\eqref{eq:app-hamiltonian-ao-bound} adds a factor of $\lambda_{\min}^{-1}$ to linear bounds and $\lambda_{\min}^{-2}$ to quadratic bounds in the AO Hamiltonian error.

For the fixed independent-particle functional $E_{0,m}[\mathbf{P}]=\operatorname{tr}(\widetilde{\mathbf{H}}^*\mathbf{P})$, we can retain each occupied--virtual gap to obtain a sharper energy bound.
Substituting $\boldsymbol{X}=-\boldsymbol{T}$ to first order into~\eqref{eq:app-grassmann-expansion} gives
\begin{align}
    E_{0,m}[\hat{\widetilde{\mathbf{P}}}]-E_{0,m}[\widetilde{\mathbf{P}}^*]
    &=\sum_{i\in\mathrm{o},\,a\in\mathrm{v}}\frac{|V_{ai}|^2}{\epsilon_a^*-\epsilon_i^*}
    +\mathcal{O}\bigl(\|\mathbf{V}\|_F^3/g^2\bigr)\nonumber\\
    &\leq\frac{\|\mathbf{V}\|_F^2}{2g}
    +\mathcal{O}\bigl(\|\mathbf{V}\|_F^3/g^2\bigr).
    \label{eq:app-independent-particle-energy}
\end{align}
For this functional, $h_{0,m}=\max_{i,a}(\epsilon_a^*-\epsilon_i^*)$.
The generic curvature bound would have a quadratic coefficient larger by $h_{0,m}/g$.
The inverse-gap bound in~\eqref{eq:app-independent-particle-energy} follows from keeping the individual gaps and does not require $h_{0,m}$ to scale with $g$.
An explicit occupation factor of two doubles the energy and its bound.
Evaluating energy directly from the predicted Hamiltonian eigenvalues can instead introduce a linear error through their first-order shifts.

\subsection{Rotation Covariance and the Entrywise Term}

Both AO and orthonormal Frobenius losses are invariant under physical rotations for the orthonormal spherical-harmonic shell convention used here.
Such rotations act through an orthogonal matrix $\mathbf{R}$, so $\|\mathbf{R}^{\mathsf{T}}\Delta\mathbf{X}\mathbf{R}\|_F=\|\Delta\mathbf{X}\|_F$.
The symmetric L\"owdin transformation also preserves the same rotation law for the matrix representations because
\begin{align}
    (\mathbf{R}^{\mathsf{T}}\mathbf{S}\mathbf{R})^{\pm1/2}
    =\mathbf{R}^{\mathsf{T}}\mathbf{S}^{\pm1/2}\mathbf{R}.
    \label{eq:app-lowdin-rotation-covariance}
\end{align}
Order-dependent orthogonalizers, such as Cholesky or Gram--Schmidt, need not satisfy the covariance law in~\eqref{eq:app-lowdin-rotation-covariance}, although they give the same invariant Frobenius loss.
General nonorthogonal AO coordinate changes can alter the raw AO Frobenius norm.

An entrywise matrix $L_1$ term does not share this orthogonal invariance.
Writing $\|\mathbf{A}\|_{1,1}=\sum_{i,j}|A_{ij}|$ and $\|\mathbf{A}\|_1=\max_j\sum_i|A_{ij}|$, we obtain
\begin{align*}
    \|\Delta\widetilde{\mathbf{P}}\|_{1,1}
    &\leq\|\mathbf{S}^{1/2}\|_1^2\|\Delta\mathbf{P}\|_{1,1},
    &\|\Delta\widetilde{\mathbf{H}}\|_{1,1}
    &\leq\|\mathbf{S}^{-1/2}\|_1^2\|\Delta\mathbf{H}\|_{1,1}.
\end{align*}
These inequalities follow from vectorization and $\|\mathbf{A}\otimes\mathbf{A}\|_1=\|\mathbf{A}\|_1^2$.
The constants depend on the overlap eigenvectors as well as the eigenvalues.
The overlap bounds and orthogonal invariance proved above therefore apply to the Frobenius loss in~\cref{eq:loss}.
Adding an entrywise term does not preserve those guarantees for the full loss.

\section{Directional Sensitivity in the Overlap Spectrum}
\label{app:loewdin-sensitivity}

This appendix formalizes the relative sensitivity analysis of \cref{sec:experiments}. We formally introduce the sensitivity of a variable as its absolute positive directional derivative. We construct the eigenbasis from the overlap matrix and analytically derive the derivatives of $\mathcal{L}_\mathrm{NON}$ and $\mathcal{L}_\mathrm{ON}$.

\subsection{Sensitivity as One-Sided Directional Derivatives}
\label{app:absolute-directional-derivative} %
Let \(\mathcal{O}\) be an observable evaluated at the converged reference ground state
\(\mathbf{X}^*\in\{\mathbf{H}^*,\mathbf{P}^*\}\), and let
\(\boldsymbol{\Pi}\in\mathbb{R}^{B\times B}\) be a symmetric perturbation direction with
\(\left\|\boldsymbol{\Pi}\right\|_F=1\).
We measure the sensitivity of \(\mathcal{O}\) along \(\boldsymbol{\Pi}\) by the one-sided directional
derivative
\begin{equation}
    D_{\boldsymbol{\Pi}}\mathcal{O}(\mathbf{X}^*)
    :=
    \lim_{\epsilon\rightarrow0^+}
    \frac{\mathcal{O}(\mathbf{X}^*+\epsilon\boldsymbol{\Pi})-\mathcal{O}(\mathbf{X}^*)}{\epsilon}.
\label{eq:app-directional-derivative}
\end{equation}
The one-sided limit is required because the quantities we compare are not all differentiable at the reference. The \(L^2\) real-space density error and the Frobenius matrix norms
entering both losses vanish at \(\mathbf{X}^*\) and grow linearly in \(\epsilon\) in either direction, so their two-sided derivatives do not exist while \cref{eq:app-directional-derivative} is well defined and positive.
The \emph{sensitivity} of \(\mathcal{O}\) along \(\boldsymbol{\Pi}\) and the \emph{relative
directional sensitivity} of a loss \(\mathcal{L}\) with respect to \(\mathcal{O}\) are then
\begin{equation}
    \mathcal{S}_{\boldsymbol{\Pi}}(\mathcal{O})
    :=
    \left|D_{\boldsymbol{\Pi}}\mathcal{O}(\mathbf{X}^*)\right|,
    \qquad
    \relsens{\boldsymbol{\Pi}}(\mathcal{L},\mathcal{O})
    :=
    \frac{\mathcal{S}_{\boldsymbol{\Pi}}(\mathcal{L})}{\mathcal{S}_{\boldsymbol{\Pi}}(\mathcal{O})}.
\label{eq:app-sensitivity-definition}
\end{equation}

\subsection{The Overlap-Spectrum Basis of Directions}
Due to its importance in deriving the Lipschitz %
continuity relationships %
between the NON loss and the physical observables we probe the sensitivity in a basis derived from the overlap $\mathbf{S}$.
We write the overlap matrix in its eigenbasis,
\begin{equation}
    \mathbf{S}
    =
    \sum_{k=1}^{B}\lambda_k\mathbf{u}_k\mathbf{u}_k^{\mathsf{T}}
    =
    \mathbf{U}\operatorname{diag}(\lambda_1,\ldots,\lambda_B)\mathbf{U}^{\mathsf{T}},
    \qquad
    0<\lambda_{\min}\leq\lambda_k\leq\lambda_{\max},
\label{eq:app-sens-overlap-spectrum}
\end{equation}
consistent with \cref{eq:app-overlap-spectrum}.
Because the matrix-error bounds of \appref{app:loewdin-observable-bounds} are
mediated by \(\mathbf{S}\), we resolve the analysis in the directions built from its
eigenvectors,
\begin{equation}
    \boldsymbol{\Pi}_{ij}
    :=
    \begin{cases}
        \mathbf{u}_i\mathbf{u}_i^{\mathsf{T}}, & i=j,\\[4pt]
        \dfrac{1}{\sqrt{2}}
        \left(
            \mathbf{u}_i\mathbf{u}_j^{\mathsf{T}}
            +
            \mathbf{u}_j\mathbf{u}_i^{\mathsf{T}}
        \right), & i<j,
    \end{cases}
\label{eq:app-spectrum-basis}
\end{equation}
where the prefactors normalize each direction to \(\left\|\boldsymbol{\Pi}_{ij}\right\|_F=1\).
The \(B(B+1)/2\) matrices \(\{\boldsymbol{\Pi}_{ij}\}_{i\leq j}\) are mutually orthogonal
under the Frobenius trace product
\(\operatorname{tr}(\mathbf{A}^{\mathsf{T}}\mathbf{B})\)
and therefore form a complete orthonormal basis of the space of real symmetric
\(B\times B\) matrices, which is exactly the space in which a symmetric matrix error
\(\Delta\mathbf{X}\) lives.

In addition to probing orthonormal perturbations the chosen basis decomposes the total Frobenius matrix error in its components. Since the total Frobenius %
error is itself invariant under orthonormal basis transformations we write
\begin{equation}
    ||\Delta \mathbf{X}||_F^2 = \mathrm{tr}(\Delta \mathbf{X}^2) = \mathrm{tr}((\mathbf{U}^{\mathsf{T}}\Delta \mathbf{X}\mathbf{U})^2) = \mathrm{tr}(\Delta \bar{\mathbf{X}}^2) = ||\Delta \bar{\mathbf{X}} ||_F^2.
\end{equation}
The total Frobenius error is then the sum of the squared components in the overlap basis, which we project out using the previously defined projectors,
\begin{align}
    ||\Delta \bar{\mathbf{X}}||_F^2
    &=
    \sum_{i}\left(\Delta\bar{\mathbf{X}}\right)_{ii}^2
    +
    2\sum_{i<j}\left(\Delta\bar{\mathbf{X}}\right)_{ij}^2
    \nonumber\\
    &=
    \sum_{i}\left(\mathbf{u}_i^{\mathsf{T}}\Delta\mathbf{X}\,\mathbf{u}_i\right)^2
    +
    2\sum_{i<j}\left(\mathbf{u}_i^{\mathsf{T}}\Delta\mathbf{X}\,\mathbf{u}_j\right)^2
    \nonumber\\
    &=
    \sum_{i\leq j}\mathrm{tr}\!\left(\boldsymbol{\Pi}_{ij}\Delta\mathbf{X}\right)^2
    \equiv
    \sum_{i\leq j}\Delta\bar{x}_{ij}^2 .
\label{eq:app-frobenius-components}
\end{align}
This decomposition leads us to define the binned matrix error along a direction as the
root mean square of its elements.
Indexing the logarithmically %
spaced bins in \(\sqrt{\lambda_i\lambda_j}\) with \(b\), the binned error is given by
\begin{equation}
    \Delta\bar{x}_b
    \equiv
    \sqrt{\left\langle \Delta\bar{x}_{ij}^2 \right\rangle_b}
    =
    \left(
        \frac{1}{N_b}
        \sum_{(m,i,j)\in b}
        \left(\Delta\bar{x}^{(m)}_{ij}\right)^2
    \right)^{1/2},
    \quad
    N_b
    \equiv
    \left|\left\{(m,i,j)\,:\,\sqrt{\lambda^{(m)}_i\lambda^{(m)}_j}\in b\right\}\right|.
\label{eq:app-binned-matrix-error}
\end{equation}
Here the superscript \(m\) indexes molecules.
Squaring makes the bins additive and connects the binned errors back to total prediction error over the dataset
\begin{equation}
    \sum_b N_b\,\Delta\bar{x}_b^2
    =
    \sum_m\left\|\Delta\mathbf{X}^{(m)}\right\|_F^2 .
\label{eq:app-binned-budget}
\end{equation}
After picking the basis $\boldsymbol{\Pi}_{ij}$ in \cref{eq:app-spectrum-basis}, we use the shorthands $D_{ij} \equiv D_{\boldsymbol{\Pi}_{ij}}$, $\mathcal{S}_{ij} \equiv \mathcal{S}_{\boldsymbol{\Pi}_{ij}}$ and $\relsens{ij} \equiv \relsens{\boldsymbol{\Pi}_{ij}}$.

\subsection{Loss Sensitivities}
The reason \cref{eq:app-spectrum-basis} is the natural basis for this analysis is that it diagonalizes the overlap matrix $\mathbf{S}\mathbf{u}_i = \lambda_i \mathbf{u}_i$ and therefore the transformation of both the non-orthonormal %
and the orthonormal loss functions $\mathcal{L}_\mathrm{NON}$, $\mathcal{L}_\mathrm{ON}$. Applying the orthonormalizing basis transformations on the perturbation %
basis yields
\begin{equation}
    \mathbf{S}^{1/2}\boldsymbol{\Pi}_{ij}\mathbf{S}^{1/2}
    =
    \sqrt{\lambda_i\lambda_j}\,\boldsymbol{\Pi}_{ij},
    \qquad
    \mathbf{S}^{-1/2}\boldsymbol{\Pi}_{ij}\mathbf{S}^{-1/2}
    =
    \frac{1}{\sqrt{\lambda_i\lambda_j}}\,\boldsymbol{\Pi}_{ij}.
\label{eq:app-congruence-eigenvalue}
\end{equation}
Both losses are the normalized Frobenius loss \(\|\Delta\mathbf{X}\|_F/B\) of
\cref{eq:loss}, evaluated in the non-orthonormal and the orthonormal basis respectively.
Perturbing along a single direction leaves \(\Delta\mathbf{X}=\epsilon\boldsymbol{\Pi}_{ij}\),
so \cref{eq:app-directional-derivative} together with \cref{eq:app-congruence-eigenvalue}
gives
\begin{equation}
    \mathcal{S}_{ij}(\mathcal{L}_\mathrm{NON})
    =
    \frac{\left\|\boldsymbol{\Pi}_{ij}\right\|_F}{B}
    =
    \frac{1}{B},
    \qquad
    \mathcal{S}_{ij}(\mathcal{L}_\mathrm{ON})
    =
    \frac{\left\|\mathbf{S}^{s/2}\boldsymbol{\Pi}_{ij}\mathbf{S}^{s/2}\right\|_F}{B}
    =
    \frac{\left(\lambda_i\lambda_j\right)^{s/2}}{B},
\label{eq:app-sens-losses}
\end{equation}
with \(s=+1\) for \(\mathbf{X}=\mathbf{P}\) and \(s=-1\) for \(\mathbf{X}=\mathbf{H}\)
(\cref{eq:lowdin_transform}).
The non-orthonormal loss weights every direction equally, while the orthonormal loss
carries the factor $(\lambda_i\lambda_j)^{s/2}$ with which the direction enters the
orthonormal representation, the same factor that mediates the observable bounds of
\appref{app:loewdin-observable-bounds}.
\subsection{Observable Sensitivities}

\textbf{Energy} The derivative of the energy $E$ with respect to the density at the ground state is per definition the Hamiltonian matrix
\begin{equation}
\frac{\partial E}{\partial \mathbf{P}} = \mathbf{H}.
\end{equation}
Hence the sensitivity, i.e. the directional derivative of energy with respect to a perturbation, %
is given by
\begin{equation}
    \mathcal{S}_{ij}(E)
    =
    \left|\mathrm{tr}\!\left(\mathbf{H}^*\boldsymbol{\Pi}_{ij}\right)\right|
    =
    \sqrt{2-\delta_{ij}}
    \left|\left(\mathbf{U}^{\mathsf{T}}\mathbf{H}^*\mathbf{U}\right)_{ij}\right|
\label{eq:app-sens-energy}
\end{equation}
so it reduces to rotating $\mathbf{H}^*$ into the overlap eigenbasis with $\mathbf{U}$ and reading off its $(i,j)$ entry. Note that this is the unconstrained linear response with the Hamiltonian matrix fully pointing out of the Grassmannian manifold.

\textbf{Dipole} Since the dipole is linear in the density, $\mu_\alpha(\mathbf{P}) = \mathrm{tr}(\mathbf{P}\,\boldsymbol{\mu}_\alpha)$ with the dipole-integral matrices $\boldsymbol{\mu}_\alpha$, $\alpha\in\{x,y,z\}$, so is the dipole error with respect to a perturbation %
$\epsilon \boldsymbol{\Pi}_{ij}$
\begin{equation}
    \Delta\mu_\alpha = \epsilon \mathrm{tr}(\boldsymbol{\mu}_\alpha \boldsymbol{\Pi}_{ij}) = \epsilon\sqrt{2-\delta_{ij}}
    \left(\mathbf{U}^{\mathsf{T}}\boldsymbol{\mu}_\alpha\mathbf{U}\right)_{ij}
\end{equation}

The reported observable is the error magnitude $\|\Delta\boldsymbol{\mu}\|_2$. Making use of the componentwise error linearity we obtain
\begin{equation}
    \mathcal{S}_{ij}\!\left(\left\|\Delta\boldsymbol{\mu}\right\|_2\right)=
    \sqrt{2-\delta_{ij}}
    \left(
        \sum_{\alpha\in\{x,y,z\}}
        \left(\mathbf{U}^{\mathsf{T}}\boldsymbol{\mu}_\alpha\mathbf{U}\right)_{ij}^2
    \right)^{1/2}.
\label{eq:app-sens-dipole}
\end{equation}

\textbf{$L^1$ density error} The real-space density is linear in the density matrix, $\rho(\mathbf{r}) = \sum_{\mu\nu}P_{\mu\nu}\chi_\mu(\mathbf{r})\chi_\nu(\mathbf{r})$, so with the eigenvector orbitals $\bar{u}_i(\mathbf{r}) := \sum_\mu (\mathbf{u}_i)_\mu \chi_\mu(\mathbf{r})$ the perturbation is again exact in $\epsilon$, $\Delta\rho(\epsilon,\mathbf{r}) = \epsilon\sqrt{2-\delta_{ij}}\,\bar{u}_i(\mathbf{r})\bar{u}_j(\mathbf{r})$, and the one-sided limit gives
\begin{equation}
    \mathcal{S}_{ij}\!\left(\left\|\Delta\rho\right\|_{L^1}\right)
    =
    \sqrt{2-\delta_{ij}}
    \left\|\bar{u}_i\bar{u}_j\right\|_{L^1}
    =
    \sqrt{2-\delta_{ij}}
    \int
    \left|
        \bar{u}_i(\mathbf{r})
        \bar{u}_j(\mathbf{r})
    \right|
    \,\mathrm{d}\mathbf{r},
\label{eq:app-sens-density-l1}
\end{equation}
the absolute charge displaced along the direction.
On the diagonal the prefactor is $\sqrt{2-\delta_{ii}}=1$ and the integrand $\bar{u}_i^2$ is already non-negative, so the absolute value drops and the integral collapses onto the overlap itself,
\begin{equation}
    \mathcal{S}_{ii}\!\left(\left\|\Delta\rho\right\|_{L^1}\right)
    =
    \int\bar{u}_i(\mathbf{r})^2\,\mathrm{d}\mathbf{r}
    =
    \mathbf{u}_i^{\mathsf{T}}\mathbf{S}\mathbf{u}_i
    =
    \lambda_i .
\label{eq:app-sens-density-l1-diagonal}
\end{equation}

\textbf{$L^2$ density error} Measuring the same perturbation in the $L^2$ norm instead gives
\begin{equation}
    \mathcal{S}_{ij}\!\left(\left\|\Delta\rho\right\|_{L^2}\right)
    =
    \sqrt{2-\delta_{ij}}
    \left\|\bar{u}_i\bar{u}_j\right\|_{L^2}
    =
    \sqrt{2-\delta_{ij}}
    \left(
        \int
        \bar{u}_i(\mathbf{r})^2
        \bar{u}_j(\mathbf{r})^2
        \,\mathrm{d}\mathbf{r}
    \right)^{1/2},
\label{eq:app-sens-density}
\end{equation}
the spatial overlap of the two orbital densities that span the direction.

\section{A Unifying Perspective on Hamiltonian and Density Learning}
\label{app:unifying_H_P_learning}

Naively, examining the inverse HOMO-LUMO gap scaling in \cref{tab:loewdin-observable-scaling} invites the question whether artificially increasing the gap could be utilized to reduce observable errors to zero.
\textbf{Combining \ourLoss{} with the level shifts formally unifies Hamiltonian and density matrix learning problems through a continuous interpolation parameter $\xi$.}
Assume we learn a level-shifted Hamiltonian
\begin{equation}
    \mathbf{\widetilde{H}}_\xi
    = \mathbf{\widetilde{H}} + \xi(\mathbf{I}-\mathbf{\widetilde{P}})
    = \xi \mathbf{I} + (\mathbf{\widetilde{H}} - \xi \mathbf{\widetilde{P}})\:.
    \label{eq:level_shifted_hamiltonian}
\end{equation}
Since any rescaling of the matrix target is absorbed by data normalization and $\mathcal{L}_\textrm{\ourLoss}$ is invariant under matrix shifts $\xi \mathbf{I}$ we formalize an equivalence relation for Hamiltonian learning targets
\begin{equation}
    \left[\mathbf{\widetilde{H}}_1\right] = \left[\mathbf{\widetilde{H}_2}\right] \Leftrightarrow \left( \exists \, c \in \mathbb{R}\,, \,\sigma > 0:\, \mathbf{\widetilde{H}}_1 = \sigma\big(\mathbf{\widetilde{H}_2} + c\mathbf{I}\big) \right)\:.
\end{equation}
Such that we can rewrite \cref{eq:level_shifted_hamiltonian} to obtain
$
[\mathbf{\widetilde{H}}_\xi] = [\mathbf{\widetilde{H}} - \xi \mathbf{\widetilde{P}}].
$
Dividing by $\xi$ and taking the limit $\xi \to \infty$ gives
$
   [\mathbf{\widetilde{H}}_\xi] = [-\mathbf{\widetilde{P}}],
$
meaning that in the limit of large level shifts, Hamiltonian learning reduces to density learning.
This says nothing about the optimality of the level shift, but it makes the tradeoff between learning a smooth but higher-dimensional Hamiltonian target vs learning a discontinuous but lower-dimensional density target explicit.

\section{Relation of the Orthonormal Loss to the Wavefunction Alignment Loss}
\label{app:waloss}

\citet{liEnhancingScalabilityApplicability2025} introduce
Wavefunction Alignment Loss (WALoss) for Hamiltonian prediction.
Their construction is closely related to our orthonormal Hamiltonian
supervision, but differs in several important respects.
We distinguish the mathematical loss presented in the paper from the
loss in the released implementation.\footnote{
Paper: \url{https://arxiv.org/pdf/2502.19227}, Sec.~3, Eqs.~(2)--(3).
Code:
\url{https://github.com/microsoft/SPHNet/blob/main/src/training/module.py\#L409-L421}
and
\url{https://github.com/microsoft/SPHNet/blob/main/src/training/module.py\#L441-L455}.
}
The generalized eigenvectors of \eqref{eq:generalized_eigenvalue_problem} satisfy
\begin{equation}
    \mathbf H^\ast \mathbf C^\ast
    =
    \mathbf S \mathbf C^\ast \boldsymbol{\epsilon}^\ast,
    \qquad
    {\mathbf C^\ast}^{\mathsf T}
    \mathbf S
    \mathbf C^\ast
    =
    \mathbf I.
\end{equation}
We further
define
$\Delta\mathbf H=\widehat{\mathbf H}-\mathbf H^\ast$.
\textbf{The initial WALoss formulation in Eq.~(2) of
\citet{liEnhancingScalabilityApplicability2025} is}
\begin{equation}
    \mathcal L_{\mathrm{WA}}^{(2)}
    := \left\|
        {\mathbf C_i^\ast}^{\mathsf T}
    \widehat{\mathbf H}_i
    \mathbf C_i^\ast
    -
    \boldsymbol{\varepsilon}_i^\ast
    \right\|_F^2
    \equiv \left\|
        {\mathbf C_i^\ast}^{\mathsf T}
    \widehat{\mathbf H}_i
    \mathbf C_i^\ast
    -
    {\mathbf C_i^\ast}^{\mathsf T}
    \mathbf H_i^\ast
    \mathbf C_i^\ast
    \right\|_F^2
    =
    \left\|
        {\mathbf C^\ast}^{\mathsf T}
        \Delta\mathbf H
        \mathbf C^\ast
    \right\|_F^2 \:.
    \label{eq:waloss-full}
\end{equation}
This expression has a direct relation to \oursOrthonormal{}.
With
$\widetilde{\mathbf H}
=
\mathbf S^{-1/2}\mathbf H\mathbf S^{-1/2}$
and
$\widetilde{\mathbf C}^\ast
=
\mathbf S^{1/2}\mathbf C^\ast$,
the matrix $\widetilde{\mathbf C}^\ast$ is orthogonal and
$
    {\mathbf C^\ast}^{\mathsf T}
    \Delta\mathbf H
    \mathbf C^\ast
    =
    {\widetilde{\mathbf C}^{\ast \mathsf T}}
    \Delta\widetilde{\mathbf H}
    \widetilde{\mathbf C}^\ast \:,
$
s.t. the invariance of the norm gives
\begin{equation}
    \left\|
        {\mathbf C^\ast}^{\mathsf T}
        \Delta\mathbf H
        \mathbf C^\ast
    \right\|_F
    =
    \left\|
        \Delta\widetilde{\mathbf H}
    \right\|_F .
    \label{eq:waloss-lowdin-equivalence}
\end{equation}
Hence, if all entries are weighted uniformly, \cref{eq:waloss-full}
uses exactly the same matrix geometry as our Löwdin formulation, up to
normalization and the use of a squared rather than unsquared norm.
The reference MO basis is simply a different choice of an orthonormal basis.
\textbf{However, the final WALoss equation presented in the paper is different.}
Eq.~(3) of \citet{liEnhancingScalabilityApplicability2025} assigns weights $\rho\gg\xi$ to the occupied orbitals and
LUMO, but evaluates each state separately,
\begin{equation}
    \mathcal L_{\mathrm{WA}}^{(3)}
    =
    \rho
    \sum_{j\in\mathcal I}
    \left(
        {\mathbf c_j^\ast}^{\mathsf T}
        \Delta\mathbf H
        \mathbf c_j^\ast
    \right)^2
    +
    \xi
    \sum_{j\notin\mathcal I}
    \left(
        {\mathbf c_j^\ast}^{\mathsf T}
        \Delta\mathbf H
        \mathbf c_j^\ast
    \right)^2 .
    \label{eq:waloss-paper}
\end{equation}
Thus \cref{eq:waloss-paper} contains only the diagonal of
${\mathbf C^\ast}^{\mathsf T}\Delta\mathbf H\mathbf C^\ast$.
These terms are the first-order orbital-energy perturbations.
In particular, the off-diagonal occupied--virtual couplings that drive
first-order orbital rotations are absent.
This contrasts with the preceding Eq.~(2), whose full Frobenius norm
does contain them.

\textbf{The released implementation differs again.}
It explicitly forms both the diagonal and strict off-diagonal entries
and, for its mean-squared-error (MSE) variant, uses
\footnote{
The implementation constructs the transformed residual at
\url{https://github.com/microsoft/SPHNet/blob/main/src/training/module.py\#L409-L421}
and combines the two sectors at
\url{https://github.com/microsoft/SPHNet/blob/main/src/training/module.py\#L441-L455}.
The default orbital-energy metric is MSE in
\url{https://github.com/microsoft/SPHNet/blob/main/src/utility/hydra_config.py}.
}

\begin{equation}
    \mathcal L_{\mathrm{WA}}^{\mathrm{code}}
    =
    \frac{1}{B}
    \sum_i T_{ii}^2
    +
    \frac{1}{B(B-1)}
    \sum_{i\neq j}T_{ij}^2,
    \qquad
    \mathbf T
    =
    {\mathbf C^\ast}^{\mathsf T}
    \Delta\mathbf H
    \mathbf C^\ast .
    \label{eq:waloss-code}
\end{equation}

\cref{eq:waloss-code} is therefore not the uniformly weighted
Frobenius norm in \cref{eq:waloss-lowdin-equivalence}.
The separate means give each diagonal entry $B-1$ times the weight of
each off-diagonal entry.
Moreover, the released code does not apply the occupied--LUMO
$\rho/\xi$ weighting of \cref{eq:waloss-paper}.

\textbf{A further difference is that none of these flavors of WALoss is used as a standalone orthonormal loss, marking a significant gap in the literature w.r.t. the understanding of the proposed objective.}
The final training objective in App.~K, Eq.~(19) of
\citet{liEnhancingScalabilityApplicability2025} combines it with
elementwise MSE and MAE of the Hamiltonian in the original
non-orthogonal representation.\footnote{
See \url{https://arxiv.org/pdf/2502.19227}, App.~K, Eq.~(19).
The released implementation analogously adds \texttt{HamiltonianError}
and \texttt{OrbitalEnergyError} as separate training losses:
\url{https://github.com/microsoft/SPHNet/blob/main/src/training/module.py\#L555-L559}.
}
The AO Frobenius term retains the overlap-dependent anisotropy removed
by \cref{eq:waloss-lowdin-equivalence}, while the entrywise MAE is
not invariant even under a general orthogonal change of basis.
Finally, WALoss does not account for the Hamiltonian gauge freedom.
For any $c\in\mathbb R$,

\begin{equation}
    \mathbf H
    \mapsto
    \mathbf H+c\mathbf S,
    \qquad
    \widetilde{\mathbf H}
    \mapsto
    \widetilde{\mathbf H}+c\mathbf I,
\end{equation}

leaves all generalized eigenvectors and therefore the density unchanged.
Both the paper WALoss and its implementation nevertheless penalize this
physically irrelevant direction through their diagonal terms.
ROCKET instead quotients it out explicitly,

\begin{equation}
    \mathcal L_{\mathrm{ROCKET}}
    =
    \min_{c\in\mathbb R}
    \mathcal L_{\mathrm{L\ddot{o}wdin}}
    \left(
        \Delta\widetilde{\mathbf H}-c\mathbf I
    \right),
\end{equation}

before introducing any further conditioning or occupied-subspace
structure. Thus, the occupied--LUMO emphasis of WALoss and our later
ROCKET constructions should not be conflated: WALoss reweights a fixed
reference Hamiltonian, whereas ROCKET first identifies Hamiltonians
that represent the same physical ground state.

\newpage
\section{AO versus L\"owdin Spatial Decay}
\label{app:matrix-localization}
\begin{figure}[h]
  \vspace{-0.6cm}
  \centering
  \includegraphics[width=0.97\linewidth]{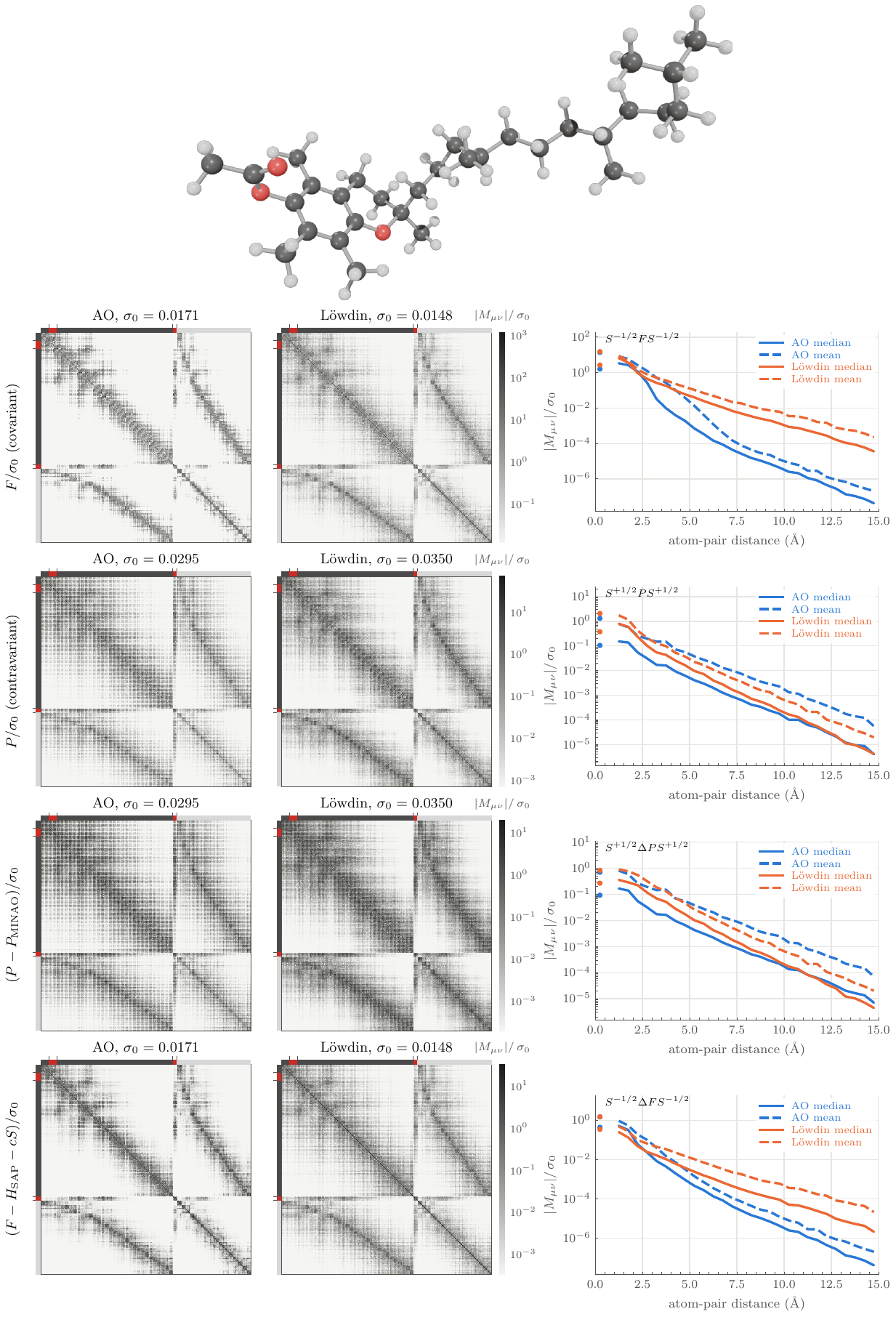}
  \vspace{-0.1cm}
  \caption{\textbf{AO versus L\"owdin spatial decay for QM40 molecule 68311 at B3LYP/def2-SVP.}
  Columns are $\mathbf{H}$, $\mathbf{P}$, $\mathbf{P}-\mathbf{P}_{\textrm{MINAO}}$, and $\mathbf{H}-\mathbf{H}_{\textrm{SAP}}-c\mathbf{S}$,
  each in the AO (a.k.a. NON) basis and after L\"owdin orthogonalization, scaled by the dataset root-mean-square (RMS) $\sigma_0$.
  AO entries decay with atom-pair distance.
  L\"owdin entries remain large after the AO entries have dropped.}
\end{figure}

\newpage
\section{Datasets}
\label{app:experimental-configuration}

We evaluate on five datasets.
QM9 is the training set of the first regime, and QM40 tests the size extrapolation of the QM9-trained models.
QMugs is the training set of the second regime and tests size extrapolation up to 660 electrons.
3BPA and Transition1x test the transfer of the QMugs-trained models to conformational and reactive chemistry.
\Cref{tab:train-dataset,tab:eval-dataset,tab:qmugs-dataset,tab:3bpa-dataset,tab:transition1x-dataset} provide one table per dataset.
The upper block of each table describes the release and our copy of it.
Each lower block covers one split, the rule that produced it, and the compute stages of \cref{tab:compute-stages} we ran on it.

\newcommand{\splithead}[1]{\midrule\multicolumn{3}{@{}l}{\textbf{#1}}\\\addlinespace[2pt]}

\paragraph{Shared dataset conventions.}
From every release we take only the atomic numbers and Cartesian coordinates.
Released energies, forces and other properties do not enter any loss or metric.
We recompute every electronic-structure quantity with the pipeline of \cref{tab:dft-reference}.
All molecules are neutral and closed-shell.
A checkpoint's element embedding is fixed by its training set.
QM9-trained checkpoints know H, C, N, O and F, so we evaluate them only on datasets restricted to these elements.
QMugs-trained checkpoints know H, C, N, O, F, P, S and Cl, so we can evaluate them on every dataset without Br or I.
Each evaluation subset is a frozen list of sample identifiers, drawn with NumPy's \texttt{default\_rng} under the seed given in the tables.
Where a table stratifies by size, we assign each molecule to the electron bin $\lfloor N_\mathrm{elec}/20\rfloor\cdot 20$ and draw the same number of molecules from every bin, so each size carries equal weight.

\paragraph{DFT reference.}
\label{app:dft-reference}
All reference labels use restricted Kohn--Sham DFT at B3LYP/def2-SVP.
We run it with GPU4PySCF~\citep{liIntroducingGpuAcceleration2025}, the GPU backend of PySCF~\citep{sunRecentDevelopmentsPyscf2020}, with the settings in \cref{tab:dft-reference}.

\begin{table}[h]
\caption{\textbf{DFT reference pipeline (GPU4PySCF 1.8.1, B3LYP/def2-SVP, RI-JK).}
This pipeline produces the evaluation references for all five datasets, with the sample counts of each split given in \cref{tab:train-dataset,tab:eval-dataset,tab:qmugs-dataset,tab:3bpa-dataset,tab:transition1x-dataset}.
It also produces the training labels of the QMugs-trained checkpoints.
The QM9 training labels come from a separate store computed with identical settings except for a looser convergence threshold (\cref{tab:train-dataset}).
Transition1x self-consistency training uses no labels.}
\label{tab:dft-reference}
\centering
\footnotesize
\begin{tabularx}{\linewidth}{@{}llX@{}}
\toprule
\textbf{Parameter} & \textbf{Value} & \textbf{Notes} \\
\midrule
Basis set & def2-SVP & Ahlrichs split-valence double-zeta basis with polarization functions~\citep{weigendBalancedBasisSets2005}. \\
XC functional & B3LYP & Three-parameter hybrid combining exact exchange~\citep{beckeDensityFunctionalThermochemistry1993} with LYP~\citep{leeDevelopmentColleSalvetti1988} and VWN local correlation~\citep{voskoAccurateSpinDependent1980} in the parameterization of~\citet{stephensAbInitioCalculation1994}; we use the VWN-RPA form, as Gaussian does. \\
Density fitting & RI-J, RI-K & Density fitting of both the Coulomb~\citep{vahtrasIntegralApproximationsLcao1993} and the exact-exchange term~\citep{weigendFullyDirectRIHF2002}. \\
Auxiliary basis & def2-universal-JKFIT & One auxiliary basis serves both the RI-J and the RI-K contractions~\citep{weigendHartreeFockExchange2008}. \\
Spin treatment & Restricted & Closed-shell restricted Kohn--Sham. \\
Quadrature grid & Level 3 & PySCF grid level for the XC energy and potential~\citep{beckeMulticenterNumericalIntegration1988,sunRecentDevelopmentsPyscf2020}. \\
Initial guess & MINAO & Superposition of atomic density matrices in a minimal ANO basis with spherically averaged occupations, projected onto def2-SVP~\citep{almlofPrinciplesDirectScf1982,vanlentheStartingScfCalculations2006,lehtolaAssessmentInitialGuesses2019}. MINAO starts every reference SCF and is the base of the $\Delta\mathbf{P}$ readouts. \\
Hamiltonian base & SAP & $\mathbf{H}_\mathrm{SAP}=\mathbf{T}+\mathbf{V}_\mathrm{SAP}$: kinetic matrix plus a superposition of pretabulated screened atomic potentials, integrated on the level-3 grid~\citep{lehtolaAssessmentInitialGuesses2019,lehtolaEfficientImplementationSuperposition2020}. It is the base of the $\Delta\mathbf{H}$ readouts and does not start the reference SCF. \\
SCF convergence & $10^{-10}\,E_h$ & Converged when the energy changes by less than $10^{-10}\,E_h$ between cycles and the orbital-gradient norm is below $10^{-5}$ (GPU4PySCF default $\sqrt{10^{-10}}$); at most 100 cycles. \\
\bottomrule
\end{tabularx}
\end{table}

\paragraph{Compute stages.}
We evaluate each prediction in up to three stages of increasing cost (\cref{tab:compute-stages}).
The cheap stage holds every metric available at the cost of one Fock build, the budget at which we compare all models.
The SCF and force stages cost considerably more on large molecules.
On the larger test sets we therefore run them on a nested subset only, drawn equally from every electron bin so that it keeps the size distribution of the full set.
The stage rows of each test set in the dataset tables give the number of molecules per bin on which each stage ran.

\begin{table}[h]
\caption{\textbf{Compute stages, costs, and evaluated quantities for model predictions.}
Each split in \cref{tab:train-dataset,tab:eval-dataset,tab:qmugs-dataset,tab:3bpa-dataset,tab:transition1x-dataset} states the sample count evaluated at each stage.
All errors are measured relative to the reference calculations of \cref{tab:dft-reference}.}
\label{tab:compute-stages}
\centering
\footnotesize
\begin{tabularx}{\linewidth}{@{}llX@{}}
\toprule
\textbf{Stage} & \textbf{Cost} & \textbf{Computed quantities} \\
\midrule
Cheap & One Fock build & The predicted density matrix $\hat{\mathbf{P}}$ is the model output for density readouts and the density matrix after diagonalization of the predicted $\hat{\mathbf{H}}$ (\cref{eq:generalized_eigenvalue_problem}) for Hamiltonian readouts. From $\hat{\mathbf{P}}$ we build $\mathbf{F}[\hat{\mathbf{P}}]$ once. Every quantity that needs no further Fock build and no nuclear gradient counts as cheap. \\
SCF & Up to $2\times$ reference & A full SCF started from $\hat{\mathbf{P}}$ with the settings of \cref{tab:dft-reference}. A run is counted as not converged after twice the cycle count of the MINAO-started reference SCF. \\
Forces & One nuclear gradient & Nuclear forces evaluated at the density matrix $\mathbf{P}(\mathbf{F}[\hat{\mathbf{P}}])$, compared with the reference forces $-\partial E/\partial\mathbf{R}$. \\
\bottomrule
\end{tabularx}
\end{table}

\begin{table}[p]
\caption{\textbf{QM9 training splits and evaluation sample counts~\citep{ramakrishnanQuantumChemistryStructures2014}.}
The split is a deterministic size cut on the total atom count including hydrogens and involves no random draw.
Compute stages are defined in \cref{tab:compute-stages}.
The electron-count distribution is shown in \cref{fig:qm9-electron-histogram}.}
\label{tab:train-dataset}
\centering
\footnotesize
\begin{tabularx}{\linewidth}{@{}llX@{}}
\toprule
\textbf{Parameter} & \textbf{Value} & \textbf{Notes} \\
\midrule
Release & $133{,}885$ molecules & Small organic molecules with up to 9 heavy atoms (C, N, O, F), enumerated from GDB-17, with equilibrium geometries relaxed at B3LYP/6-31G(2df,p). \\
Our copy & $130{,}831$ molecules & We drop the $3{,}054$ molecules that the authors flag as uncharacterized. \\
Geometries & 0\,K minima & Only the released thermochemistry refers to 298.15\,K. \\
Elements & H, C, N, O, F & $1{,}923$ molecules contain F. \\
Size range & & 3--29 atoms, 1--9 heavy atoms, 10--74 electrons, 24--226 basis functions. \\
Training labels & $10^{-9}\,E_h$ & The QM9 checkpoints are trained on labels converged to $10^{-9}\,E_h$ and otherwise identical to \cref{tab:dft-reference}. The total energies agree with the $10^{-10}\,E_h$ evaluation reference to a median of $2.5\times10^{-10}\,E_h$, with a maximum of $1.5\times10^{-5}\,E_h$. \\
\splithead{Training split}
Rule & $N_\mathrm{atoms}\le 20$ & Molecules with at most 20 atoms. \\
Size & $103{,}790$ & All $1{,}923$ F-containing molecules have 5--20 atoms and fall into this split, so the validation and test splits contain no F. \\
\splithead{Validation split}
Rule & $N_\mathrm{atoms}\in\{21,22\}$ & Molecules with 21 or 22 atoms. \\
Elements & H, C, N, O & No molecule contains F. \\
Size & $17{,}594$ & The validation loss during training uses a fixed random subset of $5{,}000$ molecules (seed 0), which also serves for checkpoint selection. \\
\splithead{Test split}
Rule & $N_\mathrm{atoms}\ge 23$ & Uniform draw without replacement from the $9{,}447$ molecules with 23--29 atoms (seed 0). \\
Elements & H, C, N, O & No molecule contains F. \\
Size & $2{,}000$ & $68\%$ of the draw has 23 atoms; $1{,}960$ molecules have 9 heavy atoms, 37 have 8 and 3 have 7. Used in \cref{tab:ablation_qm9_test2000}. \\
Cheap & All & \\
SCF & All & \\
Forces & All & \\
\bottomrule
\end{tabularx}
\end{table}

\begin{figure}[h]
  \centering
  \includegraphics[width=0.5\linewidth]{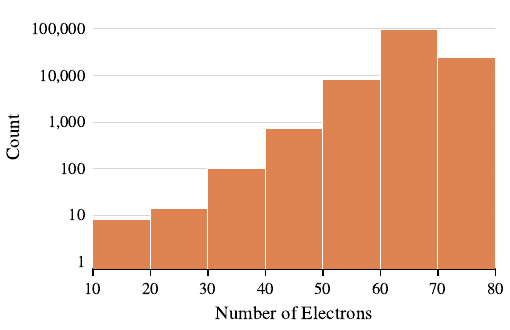}
  \caption{\textbf{Electron-count histogram of QM9.}
  The vertical axis is the molecule count on a logarithmic scale.
  QM9 contains only H, C, N, O, and F.}
  \label{fig:qm9-electron-histogram}
\end{figure}

\begin{table}[p]
\caption{\textbf{QM40 size-extrapolation dataset splits and evaluation sample counts~\citep{madushankaQm40RealisticQuantum2024}.}
These checkpoints never see QM40 during training.
To match their element vocabulary, we remove every molecule containing S or Cl.
Compute stages are defined in \cref{tab:compute-stages}.
The electron-count distribution is shown in \cref{fig:qm40-electron-histogram}.}
\label{tab:eval-dataset}
\centering
\footnotesize
\begin{tabularx}{\linewidth}{@{}llX@{}}
\toprule
\textbf{Parameter} & \textbf{Value} & \textbf{Notes} \\
\midrule
Release & $162{,}954$ molecules & Drug-like molecules from ZINC14 with 10--40 heavy atoms (C, N, O, S, F, Cl), covering the size range of 88\% of FDA-approved drugs. Geometries are pre-optimized with GFN2-xTB and relaxed at B3LYP/6-31G(2df,p), the level of theory of QM9; molecules with imaginary frequencies are removed. \\
Geometries & 0\,K minima & \\
Element filter & $99{,}303$ molecules & H, C, N, O, F only, 78--272 electrons. No molecule with 20 or 21 heavy atoms passes the filter, since all of them contain S or Cl. \\
Electron bins & $[80,280)$ & Ten bins of 20 electrons. The two molecules with 78 electrons fall below the first bin. The sparsest bin, 80--100 electrons, holds 311 molecules. \\
\splithead{Validation (hyperparameter selection)}
Rule & 50 regular, 25 hard & Used to select learning rate and weight decay (\appref{app:training-optimizer}). We remove them from the pool before drawing the test sets. \\
Elements & H, C, N, O, F & Drawn from the element-filtered pool. \\
Size & 75 & \\
Cheap & All & Energy, dipole and $L_1$ error in \cref{tab:prod_observables}, alongside the validation loss. \\
\splithead{Test split}
Rule & 100 per bin & Nested in a parent draw of 200 molecules per bin from the remaining $99{,}228$ molecules (seed 0, bins in ascending order): a first draw of 50 per bin, extended by 50 more per bin (seed 0). \\
Elements & H, C, N, O, F & Drawn from the element-filtered pool. \\
Size & $1{,}000$ & Used in \cref{tab:ablation_qm40_r100} and all QM40 figures. \\
Cheap & 100 per bin & All $1{,}000$ molecules. \\
SCF & 20 per bin & 200 molecules: 10 per bin from each of the two draws of 50 (seed 0). \\
Forces & 20 per bin & The same 200 molecules. \\
\bottomrule
\end{tabularx}
\end{table}

\begin{figure}[h]
  \centering
  \includegraphics[width=0.65\linewidth]{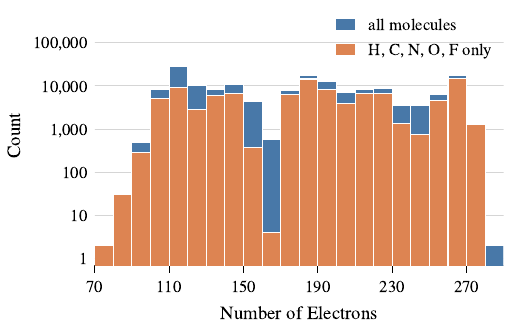}
  \caption{\textbf{Electron-count histogram of QM40.}
  The vertical axis is the molecule count on a logarithmic scale.
  The two series are the full dataset and its H, C, N, O, F subset.}
  \label{fig:qm40-electron-histogram}
\end{figure}

\begin{table}[p]
\caption{\textbf{QMugs dataset splits and evaluation sample counts for pre-training and size extrapolation~\citep{isertQmugsQuantumMechanical2021}.}
Every sample is a single conformer, and every split is disjoint by molecule.
Sets with one conformer per molecule fix that choice in a single draw (seed 0).
Compute stages are defined in \cref{tab:compute-stages}.
The electron-count distribution is shown in \cref{fig:qmugs-electron-histogram}.}
\label{tab:qmugs-dataset}
\centering
\footnotesize
\begin{tabularx}{\linewidth}{@{}llX@{}}
\toprule
\textbf{Parameter} & \textbf{Value} & \textbf{Notes} \\
\midrule
Release & \makecell[tl]{$1{,}992{,}984$ conformers\\of $665{,}911$ molecules} & Neutral, closed-shell drug-like molecules from ChEMBL~27 with 3--100 heavy atoms (H, C, N, O, F, P, S, Cl, Br, I). Each molecule has three conformers, selected by clustering GFN2-xTB metadynamics snapshots and then optimized with GFN2-xTB. \\
Geometries & 0\,K minima & GFN2-xTB local minima; 300\,K is only the temperature of the conformer search. \\
Our copy & $657{,}142$ molecules & All conformers via OpenQDC, 22--850 electrons. Molecules are identified by SMILES, which merges some ChEMBL entries. \\
\splithead{Training split}
Rule & $N_\mathrm{elec}\le 150$ & Every conformer with at most 150 electrons ($184{,}583$ conformers of $60{,}367$ molecules; 3--22 heavy atoms, 4--54 atoms) that is not in the validation split; all conformers kept. \\
Elements & H, C, N, O, F, P, S, Cl & Molecules with Br or I are excluded. \\
Size & $175{,}256$ conformers & $57{,}349$ molecules. \\
\splithead{Validation split}
Rule & Largest 5\% & The largest 5\% of the molecules with at most 150 electrons, ranked by basis-set size, then electron count: 130--150 electrons, 16--22 heavy atoms, 383--436 basis functions. \\
Elements & H, C, N, O, F, P, S, Cl & Molecules with Br or I are excluded. \\
Size & $9{,}327$ conformers & $3{,}018$ molecules; used for checkpoint selection. \\
\splithead{Test split, QMugs-trained checkpoints}
Rule & $N_\mathrm{elec}\in[160,660)$ & 25 bins of 20 electrons, one conformer per molecule; 100 molecules per bin, a first draw of 50 per bin extended by 50 more (seed 0). Disjoint from training because all conformers of a molecule share the electron count. \\
Elements & H, C, N, O, F, P, S, Cl & Molecules with Br or I are excluded. \\
Size & $2{,}500$ & 21--90 heavy atoms, 32--196 atoms. Used in \cref{tab:ablation_qmugs_ood160_r100}. \\
Cheap & 100 per bin & All $2{,}500$ molecules. \\
SCF & 20 per bin & 500 molecules: 10 per bin from each of the two draws of 50 (seed 0). \\
Forces & 20 per bin & The same 500 molecules. \\
\splithead{Test split, QM9-trained checkpoints}
Rule & $N_\mathrm{elec}\in[280,660)$ & 19 bins of 20 electrons, one conformer per molecule; 50 molecules per bin, nested in a parent draw of 200 per bin (seed 0). \\
Elements & H, C, N, O, F & The element vocabulary of the QM9-trained checkpoints. \\
Size & 950 & Used in \cref{tab:ablation_qmugs_r50}. \\
Cheap & 50 per bin & All 950 molecules. \\
SCF & 10 per bin & 190 molecules (seed 0). \\
Forces & 10 per bin & The same 190 molecules. \\
\bottomrule
\end{tabularx}
\end{table}

\begin{figure}[h]
  \centering
  \includegraphics[width=\linewidth]{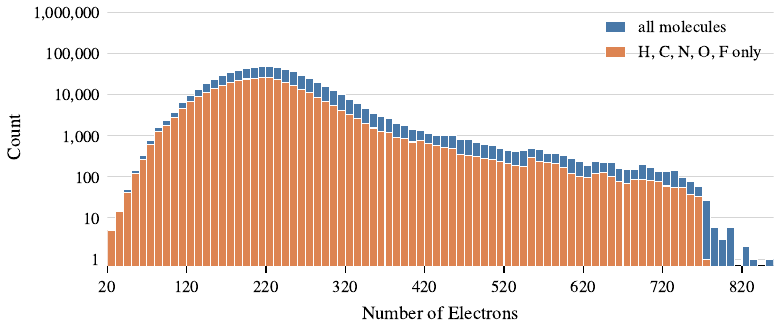}
  \caption{\textbf{Electron-count histogram of QMugs.}
  The vertical axis is the molecule count on a logarithmic scale.
  The two series are the full set and the H, C, N, O, F subset~\citep{isertQmugsQuantumMechanical2021}.}
  \label{fig:qmugs-electron-histogram}
\end{figure}

\begin{table}[p]
\caption{\textbf{3BPA conformational transfer dataset and scan sample counts~\citep{kovacsLinearAtomicCluster2021}.}
These checkpoints never see 3BPA during training.
We score each fixed-dihedral scan in full, without subsampling, because a subset of a potential-energy cut would be a different curve.
Compute stages are defined in \cref{tab:compute-stages}.}
\label{tab:3bpa-dataset}
\centering
\footnotesize
\begin{tabularx}{\linewidth}{@{}llX@{}}
\toprule
\textbf{Parameter} & \textbf{Value} & \textbf{Notes} \\
\midrule
Release & $13{,}993$ frames & A single flexible molecule, 3-(benzyloxy)pyridin-2-amine, with three consecutive rotatable bonds. The MD frames come from 25\,ps Langevin trajectories at 300, 600 and 1200\,K (1\,fs time step) started in five conformational pockets; the dihedral scans are constrained DFT geometry optimizations. The released labels are $\omega$B97X/6-31G(d), computed with ORCA. \\
Molecule & C$_{12}$H$_{12}$N$_2$O & 27 atoms, 106 electrons, 270 basis functions. \\
Used subsets & Dihedral scans & $7{,}047$ frames. We do not use the two training sets (500 frames each) or the MD test sets ($1{,}669$ / $2{,}138$ / $2{,}139$ frames at 300 / 600 / 1200\,K). \\
Dihedrals & $\alpha,\beta,\gamma$ & Three consecutive rotatable bonds of the benzyl-ether linker: $\alpha$ about the pyridine C--O bond, $\beta$ about the O--CH$_2$ bond, and $\gamma$ about the CH$_2$--phenyl bond. A fourth constraint fixes the amino group. \\
Grid & $47\times 50$ & $\alpha=55^\circ,\dots,239^\circ$ in steps of $4^\circ$; $\gamma=5^\circ,\dots,299^\circ$ in steps of $6^\circ$. \\
Geometries & 0\,K & Constrained minima at each $(\alpha,\gamma)$ with $\beta$ fixed. \\
\splithead{Test split}
Rule & Three full scans & $\beta=120^\circ$: $2{,}347$ frames. The release lacks the three grid points $(\alpha,\gamma)=(239^\circ, 155^\circ/161^\circ/167^\circ)$. $\beta=150^\circ$ and $\beta=180^\circ$: complete grids of $2{,}350$ frames. Each scan is scored as its own set. Used in \cref{tab:ablation_3bpa_dih} and \cref{fig:3bpa_and_transition1x}. \\
Size & $7{,}047$ frames & \\
Cheap & All & Since every frame is the same molecule, we remove a constant energy offset per scan and model, the median signed error. \\
Forces & All & \\
\bottomrule
\end{tabularx}
\end{table}

\begin{table}[p]
\caption{\textbf{Transition1x reactive chemistry dataset splits and sample counts for self-consistency training~\citep{schreinerTransition1xDatasetBuilding2022}.}
We inherit the release split, which assigns each molecular formula to one split, so our training, validation and test sets share no formula and no reaction.
Within each reaction we keep the reactant (R), transition state (TS) and product (P) and a window of path images around the TS.
We sign each image's Kabsch root-mean-square deviation (RMSD) to the TS by the side of the TS it lies on.
We divide $[-0.5,0.5]$\,\AA{} of this signed distance into 20 bins of 0.05\,\AA{} and keep the image nearest each bin center, dropping empty bins and geometric duplicates.
Compute stages are defined in \cref{tab:compute-stages}.}
\label{tab:transition1x-dataset}
\centering
\footnotesize
\begin{tabularx}{\linewidth}{@{}llX@{}}
\toprule
\textbf{Parameter} & \textbf{Value} & \textbf{Notes} \\
\midrule
Release & $9{,}654{,}813$ frames & Configurations on and around the reaction paths of $10{,}073$ reactions from \citet{grambowReactantsProductsTransition2020}, each with up to seven heavy atoms (C, N, O). The frames are images of nudged-elastic-band and climbing-image NEB runs, labeled at $\omega$B97x/6-31G(d). \\
Elements & H, C, N, O & 4--23 atoms. \\
Geometries & Reaction paths & Non-equilibrium configurations: many lie on or near reaction barriers. \\
Release split & \makecell[tl]{$9{,}561$ / 225 / 287\\reactions} & Training / validation / test reactions of the release. \\
Reaction coordinate & $\xi\in[0,1]$ & $\xi=\tfrac12(1-d_\mathrm{TS}/d_\mathrm{R,TS})$ on the reactant side and $\xi=\tfrac12(1+d_\mathrm{TS}/d_\mathrm{P,TS})$ on the product side, where $d$ is the Kabsch RMSD; R, TS and P sit at $\xi=0$, $0.5$ and $1$. \\
\splithead{Training split (label-free)}
Rule & All reactions & R, TS, P and the TS window for every training reaction. \\
Size & $164{,}187$ frames & 8--23 frames per reaction (median 17), 56--178 basis functions. \\
\splithead{Validation split}
Rule & 100 reactions & One reaction per equal-size stratum of the 225 validation reactions, ranked by atom count (seed 0). Per reaction we keep R, TS, P and the two window images nearest $\pm0.25$\,\AA. \\
Size & 500 frames & Used to select epoch and learning rate by the self-consistency error alone, which needs no reference. \\
\splithead{Test split}
Rule & 100 reactions & The validation draw rule applied to the 287 test reactions, with the full TS window of the training split. \\
Size & $1{,}715$ frames & 5--14 atoms, 61--124 basis functions. \\
Cheap & All & \\
SCF & All & \\
Forces & All & \\
\bottomrule
\end{tabularx}
\end{table}

\begin{table}[h]
    \centering
    \caption{\textbf{Output-scale normalization constants for different model types.}
    $\sigma_0$ is the element-wise root-mean-square of the supervision target in the representation the readout emits into, using the L\"owdin frame for the L\"owdin-basis heads and the AO basis for the AO-basis heads.
    We fit all constants on the full training split ($103{,}790$ molecules) and verify them end-to-end through the training data pipeline on $20$ random training molecules with agreement within $2.7\%$.
    The readout applies the constant internally.
    The additive heads compose $\mathrm{init} + \sigma_0\,\mathrm{sym}(\hat{\mathbf{M}})$, while the rotation head scales the decoder output before generator construction ($\hat{\mathbf{X}}$ is linear in $\hat{\mathbf{M}}$), keeping the network output at unit scale while losses and gauge constants retain their physical units.
    $\bar c = 0.404129$\,Ha is the training-set constant gauge.
    $\bar c_{\mathrm{AO}} = 0.382862$\,Ha is the corresponding constant optimal in the AO Frobenius metric.
    The two metrics give different constants with spreads of $37.3$ and $41.0$\,mHa.
    Because $\sigma_0$ anticorrelates with system size ($r=-0.86$ against the basis dimension), constants must be fitted on size-representative samples.
    For the rotation head, the decoder output $\hat{\mathbf{M}}$ enters through the projector sandwich $\hat{\mathbf{X}}(\hat{\mathbf{M}})$ (\cref{eq:groot_map}), which pins only the occupied--virtual block of $\hat{\mathbf{M}}$ at unit weight ($\mathbf{P}_0\mathbf{S}$ is idempotent) and annihilates the rest.
    The value $0.0110$ is the first-order pinned occupied--virtual block rms, measured against the purified \textsc{minao} projector.
    The true Grassmann-log value runs about $7\%$ higher ($\approx 0.0117$), and we use the first-order convention because the fit and the configuration must share one convention.
    Under the Grassmann interpretation, $\sigma_0$ is a rotation-angle scale rather than a matrix-element scale.
    The gap-shifted target admits a closed-form constant because the shift lies in the span of $\{\mathbf{I}, \boldsymbol{\Pi}^{\mathrm{ref}}\}$, giving $\sigma_0^2(\xi) = \sigma_{0,\Delta H}^2 + 2 b_1 \xi + b_2 \xi^2$ with $b_1 = \mathbb{E}\!\left[\mathrm{tr}\,\Delta - \mathrm{tr}(\Delta\Pi)\right]/B^2 = 2.03\times10^{-3}$ and geometric $b_2 = \mathbb{E}\!\left[B - N\right]/B^2 = 4.70\times10^{-3}$.
    No re-measurement is needed for a new shift ($\sigma_0 = 0.0378 / 0.0409 / 0.0471$ at $\xi = 0.05 / 0.1 / 0.2\,\mathrm{Ha}$).}
    \label{tab:delta_normalization}
    \footnotesize
    \begin{tabularx}{\textwidth}{lXrr}
        \toprule
        model type & supervision target & $\sigma_0 [\text{e}/\text{Ha}/-]$ & $1/\sigma_0$ \\
        \midrule
        $\Delta \mathbf{P}_\textrm{MINAO}$ & $\mathbf{S}^{1/2}\!\left(\mathbf{P}^{\mathrm{ref}} - \mathbf{P}^{\textsc{minao}}\right)\!\mathbf{S}^{1/2}$ & $0.0350$ & $28.6$ \\
        $\Delta \mathbf{P}_\textrm{MINAO}$ (AO) & $\mathbf{P}^{\mathrm{ref}} - \mathbf{P}^{\textsc{minao}}$ & $0.0295$ & $33.9$ \\
        $\Delta \mathbf{H}_\textrm{SAP}$ & $\mathbf{S}^{-1/2}\!\left(\mathbf{H}^{\mathrm{ref}} - \mathbf{H}^{\textsc{sap}}\right)\!\mathbf{S}^{-1/2}$ & $0.0349$ & $28.7$ \\
        $\Delta \mathbf{H}_\textrm{SAP}$ (AO) & $\mathbf{H}^{\mathrm{ref}} - \mathbf{H}^{\textsc{sap}}$ & $0.0471$ & $21.2$ \\
        $\Delta \mathbf{H}_\textrm{SAP}$ fitted gauge & $\mathbf{S}^{-1/2}\!\left(\mathbf{H}^{\mathrm{ref}} - \mathbf{H}^{\textsc{sap}} - \bar c\,\mathbf{S}\right)\!\mathbf{S}^{-1/2}$ & $0.0148$ & $67.6$ \\
        $\Delta \mathbf{H}_\textrm{SAP}$ fitted gauge (AO) & $\mathbf{H}^{\mathrm{ref}} - \mathbf{H}^{\textsc{sap}} - \bar c_{\mathrm{AO}}\,\mathbf{S}$ & $0.0171$ & $58.5$ \\
        \ours{} & $\hat{\mathbf{M}}:\; \hat{\mathbf{X}}(\hat{\mathbf{M}}) = \mathbf{P}_0\hat{\mathbf{M}}\boldsymbol{\Pi}_0 - \boldsymbol{\Pi}_0\hat{\mathbf{M}}^{\mathsf T}\mathbf{P}_0$ & $0.0110$ & $90.9$ \\
        \bottomrule
    \end{tabularx}
\end{table}
\clearpage

\newpage
\keepwithfloat[0.45]
\section{Model Hyperparameters}
\label{app:model-architectures}
\vspace{-0.2cm}

\keepwithfloat[0.45]
\subsection{Ground-State Descriptor Models}
\vspace{-0.2cm}
    \begin{table}[h]
    \caption{\textbf{QHFlow2 backbone \citep{kimMachineLearningHamiltonians2026}.} SO(2)-equivariant edgewise--nodewise updates
    in the style of eSEN \citep{fuLearningSmoothExpressive2025}, with the generative flow replaced by the \ours{} readout.}
    \label{tab:qhflow-backbone}
    \centering
    \footnotesize
    \begin{tabularx}{\textwidth}{llX}
    \toprule
    \textbf{Hyperparameter} & \textbf{Value} & \textbf{Notes} \\
    \midrule
    Number of layers & 4 & \\
    Sphere channels & 192 & Shared by backbone and decoder. \\
    Maximum degree $\ell_{\max}$ & 4 & def2-SVP off-diagonal $2\otimes2$ requires $\ell\le4$. \\
    Edge channels & 128 & \\
    Node hidden channels & 128 & \\
    Radial basis & 128 &fixed Gaussians on $[0,5]$\,\AA \\
    Cutoff radius & 5.0\,\AA & Envelope only. \\
    Decoder degree $\ell_{\max}$ & 4 & Same irreps as the backbone. \\
    Decoder channels & $192 \to 64$ & Self and pair layers at 192 per degree, then an equivariant linear to 64 per degree before the AO-block expansion. \\
    Parameter count & 81.8\,M & Identical across all model types. \\
    Pair-stream blocks & 2 & Off-diagonal head. \\
    Decoder self/pair layers & 2 each & \\
    $S^2$ grid resolution & $12\times25$ & \\
    Decoder pair radial basis & 32 & exp.-Bernstein on $[0,15]$\,\AA \\
    AO block irreps & $3\times0e+2\times1e+1\times2e$ & def2-SVP on H, C, N, O, F. \\
    \bottomrule
    \end{tabularx}
    \end{table}

\keepwithfloat[0.3]
\subsection{Machine-Learned Interatomic Potential (MLIP) Reference Models}
\vspace{-0.2cm}

\begin{table}[h]
    \centering
    \caption{\textbf{Architecture of the PaiNN, NequIP and MACE baselines in the heat maps.}
    All three models sum per-atom energies over the molecule; the PaiNN and MACE MLPs use SiLU.
    MACE hidden features are $256{\times}(0e{+}1o{+}2e)$, and its edge spherical harmonics reach $l=3$.
    The NequIP per-element energy scale is the per-element force RMS of the QM9 training split for the E and E+F models, and 1 for QMugs, which has no forces.
    The NequIP per-element shift is a least-squares fit to the training energies.
    MACE uses linear read-outs after interaction layers 1 and 2 and an MLP after layer 3, summed over layers.}
    \label{tab:mlip-architecture}
    \footnotesize
    \begin{tabularx}{\textwidth}{l>{\centering\arraybackslash}X>{\centering\arraybackslash}X>{\centering\arraybackslash}X}
    \toprule
     & PaiNN & NequIP & MACE \\
    \midrule
    Interaction layers & 3 & 4 & 3 \\
    Features / channels & 128 & 128 per $l$ & 256 \\
    Max.\ angular order & $l=1$ (vectors) & $l_{\max}=3$ & $L_{\max}=2$ \\
    Correlation order & -- & -- & 3 (body order 4) \\
    Parity & -- & off & -- \\
    Radial basis & Gaussian, 20 & Bessel, 8 & Bessel, 8 \\
    Cutoff [\AA], function & 5.0, cosine & 5.0, polynomial ($p=6$) & 5.0, polynomial ($p=5$) \\
    Read-out & MLP 128--64--1 & linear, scale, shift & linear; MLP 256--128--1 \\
    Trainable parameters & 0.6\,M & 33.5\,M &
      \begin{tabular}[t]{@{}c@{}}18.5\,M (E, E+F)\\ 20.4\,M (QMugs)\end{tabular} \\
    \bottomrule
    \end{tabularx}
    \end{table}

\newpage
\section{Model Implementation Details}
\label{app:model-implementation}

Following \citet{kimMachineLearningHamiltonians2026}, we combine the eSCNMD GNN backbone by \citet{woodUMAFamilyUniversal2025}, which builds on eSEN and eSCN~\citep{fuLearningSmoothExpressive2025,passaroReducingSO3Convolutions2023}, with equivariant matrix encoders for the initial-guess and overlap matrices and a QHNet readout~\citep{yuEfficientEquivariantGraph2023}.
From the atomic numbers $Z_a$, coordinates $\mathbf{R}_a$, and initial-guess sub-matrices $\mathbf{X}_a$ centered at each atom $a$, the backbone constructs spherical-harmonic node features through $\mathrm{SO}(2)$-equivariant message-passing layers.
The QHNet readout constructs the target matrix $\mathbf{X}$ blockwise.
Each atom pair $(i,j)$ contributes a block $\mathbf{X}_{ij}\in\mathbb{R}^{B_{Z_i}\times B_{Z_j}}$, formed by a self-tensor product of the node features for $i=j$ and a filtered tensor product across the two nodes for $i\ne j$.
Depending on the ansatz, $\mathbf{X}$ is the Hamiltonian $\mathbf{H}$, the density $\mathbf{P}$, or the \ours{}\textsubscript{rot} input $\hat{\mathbf{M}}$.
For the Hamiltonian ansatz, we obtain the density by solving \cref{eq:generalized_eigenvalue_problem}.
For the Hamiltonian and density ansätze, we symmetrize the decoded block matrix as $\mathbf{X}\leftarrow(\mathbf{X}+\mathbf{X}^{\mathsf{T}})/2$.
 The \ours{}\textsubscript{rot} readout instead uses the raw $\hat{\mathbf{M}}$.
Since we parametrize the generator in the same ON basis as the Hamiltonian and density matrices, these readout heads remain interchangeable.

Our JAX port~\citep{babuschkinDeepmindJaxEcosystem2020} corrects two bugs in the official reference implementation at commit \texttt{d82fa3a1} from 2026-08-09.
The first bug occurs in \href{https://github.com/seongsukim-ml/QHFlow2/blob/d82fa3a137d92a45a0918bdc24d30d85cb949f82/src/models/modules/escn_backbone_v4_no_t.py#L576-L584}{\texttt{escn\_backbone\_v4\_no\_t.py}, lines 576 to 584}, where the model adds contracted overlap and initial-guess blocks to the node representations.
The contraction yields one scalar feature vector per atom, but the reference code reshapes this atom-indexed tensor as if its leading dimension counted graphs in \href{https://github.com/seongsukim-ml/QHFlow2/blob/d82fa3a137d92a45a0918bdc24d30d85cb949f82/src/models/QHFlow_so2_v5_1_no_t.py#L347-L348}{\texttt{QHFlow\_so2\_v5\_1\_no\_t.py}, lines 347 to 348} and then gathers from it with the graph assignment vector of the batch.
Consequently, every atom of the $g$-th molecule in a batch receives the feature of the $g$-th atom of the batch instead of its own, so for a single molecule all atoms receive the feature of the first atom and the injected features depend on the atom ordering.
Our port keeps the atom index and, going beyond the reference, contracts each block to features of degree $\ell=0,1,2$ that it adds to the matching channels of its own node.
In our tests, this correction reduces the relative rotation-equivariance error from $10^{-3}$ to $10^{-6}$.
The second bug appears where the backbone hands off features to the tensor-product decoder at \href{https://github.com/seongsukim-ml/QHFlow2/blob/d82fa3a137d92a45a0918bdc24d30d85cb949f82/src/models/QHFlow_so2_v5_1_no_t.py#L356}{\texttt{QHFlow\_so2\_v5\_1\_no\_t.py}, line 356} for node features and \href{https://github.com/seongsukim-ml/QHFlow2/blob/d82fa3a137d92a45a0918bdc24d30d85cb949f82/src/models/QHFlow_so2_v5_1_no_t.py#L365}{line 365} for pair features.
The backbone stores features as $(N,(\ell_{\max}+1)^2,C)$ tensors, in which the magnetic index precedes the channel index within each degree $\ell$, whereas the \texttt{e3nn} irreps of the decoder hold $C$ copies of each degree with the channel index first.
The reference code flattens the tensor with a raw reshape that omits the required blockwise transposition, so every $\ell\ge1$ block reaches the decoder with its components misordered.
Our port converts between the two memory layouts explicitly.
With identical parameters and the same activation-normalization constant, our JAX decoder matches the PyTorch decoder to double-precision round-off.

\newpage
\section{Training Costs}
\label{app:training-costs}

\begin{figure}[h]
  \centering
  \includegraphics[width=\linewidth]{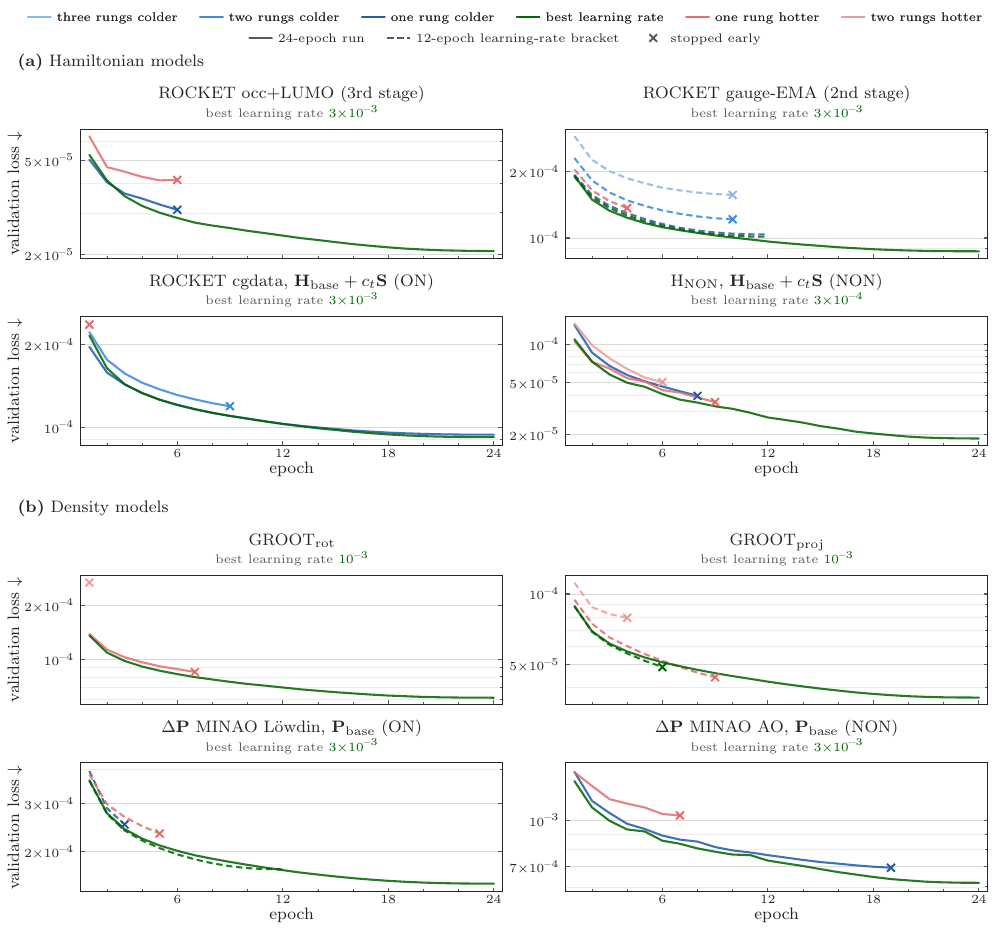}
  \caption{\textbf{Validation-loss trajectories for learning-rate selection on QMugs.}
  \textbf{(a)} Hamiltonian models.
  \textbf{(b)} Density models.
  Each panel reports the selected learning rate above the axes and colors the remaining rates by their position on the sweep.
  Solid lines show 24-epoch runs, dashed lines show 12-epoch brackets, and crosses mark runs that stopped early.
  Completed 24-epoch runs take 35--39 h on one NVIDIA H200 GPU, or about 1.5 h per epoch.}
  \label{fig:qmugs-training-curves}
\end{figure}

\newpage
\keepwithfloat[0.75]
\section{Training and Optimization}
\label{app:training-optimizer}
\vspace{-0.2cm}

In the tables of this appendix, bold and underlined entries are the lowest and second-lowest values in each column, respectively, and red marks an error above the thresholds of \cref{tab:prod_observables} (bold red marks a catastrophic error).

\begin{table}[h]

    \caption{\textbf{Training and optimization hyperparameters.}
    Muon optimizer configuration.
    The optimizer configuration is shared by all reported runs.
    Family-specific deviations are stated in the notes.
    Model-init and training-PRNG seeds are drawn from OS entropy per run and logged.}

    \label{tab:training-hyperparameters}

    \centering
    \footnotesize
    \begin{tabularx}{\textwidth}{ll>{\raggedright\arraybackslash}X}

    \toprule
    \textbf{Hyperparameter} & \textbf{Value} & \textbf{Notes} \\
    \midrule
    Batch size
    & $16$
    & \\

    Optimizer
    & Muon
    & Muon on the 2D hidden weight matrices, AdamW on all remaining
    parameters.
    Both optimizer branches use
    Nesterov momentum, following the Optax (contrib) defaults~\citep{babuschkinDeepmindJaxEcosystem2020}. We also tried the same momentum as used for standard Adam but observed no improvements.\\
    Muon momentum & $0.95$ & Default~\citep{jordanMuonOptimizerHidden2024}. \\
    Muon Newton--Schulz steps& $5$ & Default~\citep{jordanMuonOptimizerHidden2024}. \\

    Adam $(\beta_1,\beta_2)$ & $(0.9,\,0.999)$ & Default~\citep{kingmaAdamMethodStochastic2015}.\\

    Muon consistent RMS
    & $0.2$
    & Each orthogonalized $m\times n$ Muon update is scaled by
    $0.2\sqrt{\max(m,n)}$, giving approx. shape-independent
    update RMS~\citep{liuMuonScalableLlm2025}. \\

    Weight decay $\lambda_{\mathrm{wd}}$
    &
    & Tried \{3e-4, 1e-3, 3e-3, 1e-2, 3e-2\}.
    Same weight decay on the Muon and AdamW branches.
    Decay is masked and applied only to the 2D weight matrices and tensor-product path weights.
    Biases, normalizations and embeddings are undecayed. \\

    Base learning rate (lr)
    & -
    & retuned for each model family\\

    Warmup schedule & Linear & \\
    Warmup steps & $2{,}000$ & \\

    Decay schedule & Cosine & With minimum learning rate of 1e-5 $\times$ base lr\\
    Decay steps
    & $101{,}776$
    & Total $103{,}776$ steps $=16$ epochs at $6{,}486$ steps/epoch,
    batch size $16$. \\

    Training epochs
    & $16$
    & Tried $\{12,\,16,\,20\}$ tuned in favor of the baseline methods.
    Longer runs mainly improve in distribution, with no gains on the size extrapolation val sets.\\

    Parameter EMA decay
    & $0.995$
    & Tried $\{0.99,\,0.995\}$ \\

    Consecutive NaN skips
    & $10$
    & Maximum consecutive non-finite updates skipped before the run aborts. \\

    \bottomrule
    \end{tabularx}
\end{table}

\begin{table}[h]
    \caption{\textbf{\ourLoss{} restraint hyperparameters.}
    The restraint weight $\alpha_{\textrm{\ourLoss}}$ is the $3\sigma$
    weight $\alpha/9$ with $\alpha=\mathbb{E}[\mathcal{L}_{\mathrm{min}}]/\sigma_c^2$
    fitted on the full training split.
    A three-standard-deviation gauge offset then costs as much as the mean
    shift-minimized loss.
    See \cref{sec:rocket}.}
    \label{tab:rocket-hyperparameters}
    \centering
    \footnotesize
    \begin{tabularx}{\textwidth}{ll>{\raggedright\arraybackslash}X}
    \toprule
    \textbf{Hyperparameter} & \textbf{Value} & \textbf{Notes} \\
    \midrule
    Optimal shift std $\sigma_c$
    & $37.3\,\mathrm{mHa}$
    & Standard deviation of the $L_2$-optimal gauge shifts over the
    training split. \\

    Restraint weight $\alpha$
    & $1.161$
    & Calibrated as
    $\alpha_{\textrm{\ourLoss}}=
    \mathbb{E}[\mathcal{L}_{\mathrm{min}}
    (\Delta\widetilde{\mathbf H})]/\sigma_c^2$
    over the full training split ($103{,}790$ molecules). \\

    EMA decay
    & $0.99$
    & Decay of the EMA of the optimal overlap shift center.
    Initialized to training split optimal $0.404$. \\
    \bottomrule
    \end{tabularx}
\end{table}

\keepwithfloat[0.52]
\subsection{Density Models}
\label{app:training-optimizer-density}
\vspace{-0.2cm}

\keepwithfloat[0.52]
\subsubsection{NON Loss}
\vspace{-0.2cm}

\begin{table}[h]
  \centering
  \caption{\textbf{Validation and probe loss of density models trained with the NON loss across learning rates, weight decays, and seeds.}
  Masked decay with $\sigma_0$ output norm.
  Target is $\Delta$-MINAO except for the absolute-readout block.
  Validation and QM40 loss are the AO objective (masked Frobenius root-mean-square error, RMSE) on the
  QM9 validation split at the final epoch and on the fixed probe molecules (reg.\,50 / hard\,25).
  $\|W\|$ columns are per-entry RMS of the final EMA weights by optimizer branch.
  \oursOrthonormal{} models are in \cref{tab:new_generation_fb1_density}.
  }
  \label{tab:new_generation_fb1_density_ao}
  \footnotesize
  \begin{tabular*}{\textwidth}{@{\extracolsep{\fill}}cccrrrrrr@{}}
    \toprule
    & & & & &
    \multicolumn{2}{c}{QM40 loss} &
    \multicolumn{2}{c}{$\|W\|_{\mathrm{rms}}$} \\
    \cmidrule(lr){6-7}\cmidrule(lr){8-9}
    ep & lr & wd & seed & val loss & reg.\,50 & hard\,25 & Muon & Adam \\
    \midrule
    \multicolumn{9}{l}{\textit{abs $P$-learning} (no-rescale)} \\
    16 & 3e-4 & 1e-2 & 995756 & 5.73 & 16.3 & 29.6 & 0.16 & 0.55 \\
    16 & 1e-3 & 1e-2 & 995756 & 4.82 & 14.8 & \underline{27.8} & 0.17 & 0.44 \\
    \textbf{16} & \textbf{3e-3} & \textbf{1e-2} & 995756 & 4.48 & \textbf{13.8} & \textbf{26.0} & 0.26 & 0.40 \\
    " & " & " & 684851 & \underline{4.40} & \underline{14.1} & 36.9 & 0.25 & 0.41 \\
    " & " & " & 240706 & \textbf{4.38} & 18.0 & \bad{384} & 0.26 & 0.41 \\
    16 & 3e-3 & 3e-3 & 995756 & 4.66 & 14.6 & 52.9 & 0.42 & 0.50 \\
    16 & 1e-2 & 1e-2 & 995756 & 5.31 & 15.0 & \textbf{26.0} & 0.32 & 1.03 \\
    \multicolumn{9}{l}{\textit{abs $P$-learning}} \\
    \textbf{16} & \textbf{3e-3} & \textbf{1e-2} & 995756 & 4.49 & 14.8 & \bad{77.9} & 0.25 & 0.42 \\
    " & " & " & 684851 & \textbf{4.41} & \underline{14.1} & \underline{27.2} & 0.25 & 0.42 \\
    " & " & " & 240706 & \underline{4.43} & \textbf{13.9} & \textbf{26.5} & 0.25 & 0.42 \\
    \addlinespace
    \multicolumn{9}{l}{\textit{NON loss, $\Delta$-Learning with $P_\mathrm{MINAO}$}} \\
    16 & 1e-3 & 3e-3 & 982933 & 5.18 & 15.0 & 25.8 & 0.22 & 0.55 \\
    16 & 3e-3 & 3e-3 & 982933 & 4.84 & 14.5 & \underline{25.6} & 0.41 & 0.50 \\
    \textbf{16} & \textbf{3e-3} & \textbf{1e-2} & 982933 & 4.67 & 14.3 & \textbf{24.6} & 0.25 & 0.41 \\
    " & " & " & 709626 & \underline{4.51} & \underline{13.8} & 34.8 & 0.26 & 0.42 \\
    " & " & " & 744226 & \textbf{4.50} & \textbf{13.0} & 25.7 & 0.26 & 0.41 \\
    16 & 1e-2 & 3e-3 & 982933 & 7.80 & 16.9 & 28.6 & 0.79 & 1.36 \\
    \bottomrule
  \end{tabular*}
\end{table}

\keepwithfloat[0.45]
\subsubsection{\oursOrthonormal{}}
\vspace{-0.2cm}
\begin{table}[h]
  \centering
  \caption{\textbf{Validation loss of density models trained with the L\"owdin-frame \oursOrthonormal{}.}
  Masked decay with $\sigma_0$ output norm.
  Target is $\Delta$-MINAO.
  Validation is the training objective (masked Frobenius RMSE, Löwdin picture) on the
  QM9 validation split at the final epoch.
  The QM40-loss columns are the \emph{same
  objective} evaluated on the fixed probe molecules (reg.\,50 / hard\,25, same units).
  The density family ranks models (Spearman $+0.87/+0.79$ vs $E$).
  $\|W\|$ columns are per-entry RMS of the final EMA weights by optimizer branch.
  NON-loss models are in \cref{tab:new_generation_fb1_density_ao}.
  \ours{} models are in \cref{tab:new_generation_fb1_density_groot}.
  Prediction-frame and auxiliary-supervision models are in \cref{tab:new_generation_fb1_density_lowdin_variants}.
  }
  \label{tab:new_generation_fb1_density}
  \footnotesize
  \begin{tabular*}{\textwidth}{@{\extracolsep{\fill}}cccrrrrrr@{}}
    \toprule
    & & & & &
    \multicolumn{2}{c}{QM40 loss} &
    \multicolumn{2}{c}{$\|W\|_{\mathrm{rms}}$} \\
    \cmidrule(lr){6-7}\cmidrule(lr){8-9}
    ep & lr & wd & seed & val loss & reg.\,50 & hard\,25 & Muon & Adam \\
    \midrule
    \multicolumn{9}{l}{\textit{\oursOrthonormal{}}} \\
    12 & 1e-3 & 1e-2 & 0 & 1.46 & 7.54 & 9.06 & 0.17 & 0.48 \\
    " & " & " & 749811 & 1.47 & 7.10 & 9.12 & 0.17 & 0.47 \\
    12 & 1e-3 & 3e-2 & 0 & 1.40 & 7.99 & 9.51 & 0.10 & 0.35 \\
    12 & 3e-3 & 1e-3 & 0 & 1.43 & 6.03 & \underline{8.60} & 0.42 & 0.59 \\
    12 & 3e-3 & 3e-3 & 0 & 1.43 & 6.39 & 8.90 & 0.37 & 0.51 \\
    12 & 3e-3 & 1e-2 & 0 & 1.39 & 6.67 & 9.09 & 0.24 & 0.39 \\
    12 & 3e-3 & 3e-2 & 0 & 1.35 & 6.93 & 9.24 & 0.11 & 0.44 \\
    12 & 1e-2 & 1e-2 & 0 & 1.40 & 7.34 & 9.43 & 0.33 & 0.73 \\
    \addlinespace
    16 & 3e-3 & 3e-4 & 882901 & 1.33 & 6.13 & 9.47 & 0.50 & 0.62 \\
    16 & 3e-3 & 1e-3 & 118790 & 1.35 & 6.89 & 10.7 & 0.47 & 0.58 \\
    \textbf{16} & \textbf{3e-3} & \textbf{3e-3} & 154091 & \underline{1.30} & \underline{5.93} & 8.97 & 0.40 & 0.48 \\
    " & " & " & 406157 & 1.32 & \textbf{5.75} & \textbf{8.28} & 0.38 & 0.49 \\
    " & " & " & 145343 & 1.35 & 6.73 & 9.25 & 0.40 & 0.48 \\
    16 & 3e-3 & 1e-2 & 0 & \textbf{1.29} & 6.36 & 9.17 & 0.24 & 0.40 \\
    " & " & " & 261042 & \textbf{1.29} & 5.96 & 9.23 & 0.24 & 0.40 \\
    \bottomrule
  \end{tabular*}
\end{table}

\begin{table}[!h]
  \centering
  \caption{\textbf{Validation and probe loss of \ours{} density models trained with \oursOrthonormal{}.}
  Masked decay with $\sigma_0$ output norm and a $\Delta$-MINAO target.
  Validation is the training objective on the QM9 validation split at the final epoch.
  The QM40-loss columns evaluate that objective on the fixed probe molecules (reg.\,50 / hard\,25).
  $\|W\|$ columns are per-entry RMS of the final EMA weights by optimizer branch.}
  \label{tab:new_generation_fb1_density_groot}
  \footnotesize
  \begin{tabular*}{\textwidth}{@{\extracolsep{\fill}}cccrrrrrr@{}}
    \toprule
    & & & & &
    \multicolumn{2}{c}{QM40 loss} &
    \multicolumn{2}{c}{$\|W\|_{\mathrm{rms}}$} \\
    \cmidrule(lr){6-7}\cmidrule(lr){8-9}
    ep & lr & wd & seed & val loss & reg.\,50 & hard\,25 & Muon & Adam \\
    \midrule
    \multicolumn{9}{l}{\textit{+ \ours{}$_{\mathrm{proj}}$}} \\
    16 & 1e-3 & 3e-2 & 193360 & \textbf{0.38} & 1.24 & \underline{2.48} & 0.09 & 0.32 \\
    16 & 3e-3 & 3e-3 & 193360 & \underline{0.39} & 2.33 & 2.52 & 0.36 & 0.47 \\
    \textbf{16} & \textbf{3e-3} & \textbf{1e-2} & 193360 & \underline{0.39} & 1.17 & 2.97 & 0.22 & 0.36 \\
    " & " & " & 141727 & \underline{0.39} & \textbf{1.06} & \textbf{2.44} & 0.21 & 0.35 \\
    " & " & " & 899526 & \underline{0.39} & \underline{1.12} & 2.72 & 0.22 & 0.35 \\
    16 & 3e-3 & 3e-2 & 193360 & \underline{0.39} & 1.39 & 2.69 & 0.10 & 0.41 \\
    16 & 1e-2 & 3e-2 & 193360 & 0.44 & 1.18 & 2.60 & 0.11 & 0.83 \\
    \addlinespace
    \multicolumn{9}{l}{\textit{+ \ours{}$_{\mathrm{rot}}$}} \\
    16 & 1e-3 & 3e-3 & 628859 & 0.57 & 1.59 & 9.99 & 0.22 & 0.56 \\
    16 & 3e-3 & 3e-3 & 628859 & 0.56 & 1.54 & 11.8 & 0.39 & 0.48 \\
    16 & 3e-3 & 1e-2 & 628859 & \textbf{0.54} & 1.54 & 11.7 & 0.24 & 0.38 \\
    16 & 3e-3 & 3e-2 & 628859 & \underline{0.55} & 1.63 & 10.1 & 0.10 & 0.46 \\
    \textbf{16} & \textbf{1e-2} & \textbf{1e-2} & 628859 & 0.57 & 1.54 & \underline{9.82} & 0.30 & 0.77 \\
    " & " & " & 506493 & 0.60 & 1.53 & 9.90 & 0.28 & 0.74 \\
    " & " & " & 900833 & 0.58 & \underline{1.50} & \textbf{9.74} & 0.29 & 0.75 \\
    16 & 1e-2 & 3e-2 & 628859 & 0.60 & \textbf{1.47} & 9.89 & 0.11 & 1.02 \\
    " & " & " & 798434 & \bad{NaN} & \bad{NaN} & \bad{NaN} & \bad{NaN} & \bad{NaN} \\
    " & " & " & 400920 & \bad{NaN} & \bad{NaN} & \bad{NaN} & \bad{NaN} & \bad{NaN} \\
    \bottomrule
  \end{tabular*}
\end{table}

\keepwithfloat[0.3]
\subsubsection{Prediction Frame and Auxiliary Supervision}
\vspace{-0.2cm}

\begin{table}[h]
  \centering
  \caption{\textbf{Prediction-frame and auxiliary-supervision ablations for density and Hamiltonian models under \oursOrthonormal{}.}
  Density target is $\Delta$-MINAO and Hamiltonian target is $\Delta$-SAP with the training-set optimal constant gauge.
  Parenthesized validation is the model's own training objective when that objective differs from the table metric.
  The QM40-loss columns use the Löwdin objective on the fixed probe molecules (reg.\,50 / hard\,25).
  The Hamiltonian AO-frame run was killed before the QM40 probe.}
  \label{tab:new_generation_fb1_density_lowdin_variants}
  \footnotesize
  \begin{tabular*}{\textwidth}{@{\extracolsep{\fill}}cccrrrrrr@{}}
    \toprule
    & & & & &
    \multicolumn{2}{c}{QM40 loss} &
    \multicolumn{2}{c}{$\|W\|_{\mathrm{rms}}$} \\
    \cmidrule(lr){6-7}\cmidrule(lr){8-9}
    ep & lr & wd & seed & val loss & reg.\,50 & hard\,25 & Muon & Adam \\
    \midrule
    \multicolumn{9}{l}{\textit{Density \ours{}$_{\mathrm{proj}}$, pre/post supervision}} \\
    16 & 3e-3 & 1e-2 & 193360 & (2.17) & 3.75 & 5.77 & 0.23 & 0.38 \\
    \addlinespace
    \multicolumn{9}{l}{\textit{Density \ours{}$_{\mathrm{proj}}$, AO-frame prediction}} \\
    16 & 3e-3 & 1e-2 & 193360 & 0.58 & 1.39 & 3.03 & 0.21 & 0.34 \\
    \addlinespace
    \multicolumn{9}{l}{\textit{Hamiltonian \ourLoss{}, AO-frame prediction}} \\
    16 & 3e-4 & 3e-3 & 443905 & (25.9) & -- & -- & -- & -- \\
    \bottomrule
  \end{tabular*}
\end{table}

\keepwithfloat[0.68]
\subsection{Hamiltonian Models}
\label{app:training-optimizer-hamiltonian}
\vspace{-0.2cm}

\keepwithfloat[0.68]
\subsubsection{Non-Orthonormal Loss}
\vspace{-0.2cm}
\begin{table}[h]
  \centering
  \caption{\textbf{Validation loss of Hamiltonian models trained with the NON loss, with and without a constant gauge shift.}
  Masked decay with $\sigma_0$ output norm.
  Target is $\Delta$-SAP, with a training-set constant gauge in the second block only.
  Validation is the NON training objective on the
  QM9 validation split at the final epoch.
  The QM40-loss columns evaluate that objective on the fixed probe molecules (reg.\,50 / hard\,25).
  The covariant loss is gauge-dominated and nearly flat across
  models, quoted for completeness, not for ranking.
  $\|W\|$ columns are per-entry RMS of the final EMA weights by optimizer branch.
  \oursOrthonormal{} models are in \cref{tab:new_generation_fb1_hamiltonian_lowdin}.
}
  \label{tab:new_generation_fb1_hamiltonian}
  \footnotesize
  \begin{tabular*}{\textwidth}{@{\extracolsep{\fill}}cccrrrrrr@{}}
    \toprule
    & & & & &
    \multicolumn{2}{c}{QM40 loss} &
    \multicolumn{2}{c}{$\|W\|_{\mathrm{rms}}$} \\
    \cmidrule(lr){6-7}\cmidrule(lr){8-9}
    ep & lr & wd & seed & QM9 val & reg.\,50 & hard\,25 & Muon & Adam \\
    \midrule
    \multicolumn{9}{l}{\textit{abs $H$-learning}} \\
    16 & 1e-4 & 1e-3 & 684464 & 0.78 & 3.95 & 12.2 & 0.18 & 0.62 \\
    \addlinespace
    \multicolumn{9}{l}{\textit{$\Delta H_\mathrm{SAP}$-learning}} \\
    16 & 3e-6 & 1e-3 & 684464 & 2.12 & 8.7 & 18.4 & 0.18 & 0.63 \\
    16 & 1e-5 & 1e-3 & 684464 & 1.12 & 5.3 & 15.1 & 0.18 & 0.63 \\
    \textbf{16} & \textbf{3e-5} & \textbf{1e-3} & 684464 & 0.66 &  \textbf{4.0} & \textbf{13.7} & 0.18 & 0.63 \\
    16 & 1e-4 & 3e-4 & 684464 & 0.39 & \textbf{4.0} & \underline{13.8} & 0.18 & 0.62 \\
    16 & 1e-4 & 1e-3 & 684464 & 0.39 &  \textbf{4.0} & \underline{13.8} & 0.18 & 0.62 \\
    16 & 1e-4 & 3e-3 & 684464 & 0.39 & \textbf{4.0} & 14.0 & 0.18 & 0.62 \\
    16 & 3e-4 & 1e-3 & 684464 & \underline{0.27} & \bad{50.2} & \bad{119} & 0.18 & 0.62 \\
    16 & 1e-3 & 1e-3 & 684464 & \textbf{0.26} & \bad{33.8} & \bad{81.2} & 0.20 & 0.60 \\
    \addlinespace
    \multicolumn{9}{l}{\textit{$\Delta (H_\mathrm{SAP} + c_\textrm{train-set-fitted} S)$-learning}} \\
    16 & 1e-6 & 3e-3 & 684464 & 6.94 & 11.9 & 25.0 & 0.18 & 0.63 \\
    16 & 3e-6 & 1e-3 & 684464 & 1.65 & 4.8 & 15.3 & 0.18 & 0.63 \\
    16 & 3e-6 & 3e-3 & 684464 & 1.65 & 4.8 & 15.3 & 0.18 & 0.63 \\
    \textbf{16} & \textbf{1e-5} & \textbf{1e-3} & 684464 & 0.78 & \textbf{3.4} & \textbf{13.4} & 0.18 & 0.63 \\
    " & " & " & 896521 & 0.78 & \underline{3.5} & 13.7 & 0.18 & 0.63 \\
    " & " & " & 405602 & 0.92 & 5.0 & 15.2 & 0.18 & 0.63 \\
    " & " & " & 748187 & 0.79 & \underline{3.5} & \underline{13.6} & 0.18 & 0.63 \\
    16 & 3e-5 & 1e-3 & 684464 & 0.45 & 3.6 & 13.9 & 0.18 & 0.63 \\
    16 & 1e-4 & 1e-3 & 684464 & 0.27 & 4.0 & 15.3 & 0.18 & 0.63 \\
    16 & 1e-4 & 3e-3 & 684464 & 0.27 & 4.0 & 15.4 & 0.18 & 0.62 \\
    16 & 3e-4 & 1e-3 & 684464 & \underline{0.20} & \bad{29.0} & \bad{65.8} & 0.18 & 0.62 \\
    16 & 3e-4 & 3e-3 & 684464 & \underline{0.20} & \bad{28.1} & \bad{64.4} & 0.17 & 0.60 \\
    16 & 1e-3 & 1e-3 & 684464 & \textbf{0.19} & \bad{26.1} & \bad{259} & 0.20 & 0.60 \\
    16 & 3e-3 & 1e-3 & 684464 & 0.21 & \bad{2\,608} & \bad{35\,112} & 0.36 & 0.58 \\
    \bottomrule
  \end{tabular*}
\end{table}

\keepwithfloat[0.36]
\subsubsection{WALoss}
\vspace{-0.2cm}

\begin{table}[h]
  \centering
  \caption{\textbf{Validation and probe loss of Hamiltonian models trained with WALoss.}
  Target is $\Delta$-SAP, with a training-set constant gauge in the second block only.
  Validation is the WALoss training objective on the
  QM9 validation split at the final epoch.
  The QM40-loss columns evaluate each run's own WALoss objective on the fixed probe molecules (reg.\,50 / hard\,25).
  $\|W\|$ columns are per-entry RMS of the final EMA weights by optimizer branch.}
  \label{tab:new_generation_fb1_hamiltonian_waloss}
  \footnotesize
  \begin{tabular*}{\textwidth}{@{\extracolsep{\fill}}cccrrrrrr@{}}
    \toprule
    & & & & &
    \multicolumn{2}{c}{QM40 loss} &
    \multicolumn{2}{c}{$\|W\|_{\mathrm{rms}}$} \\
    \cmidrule(lr){6-7}\cmidrule(lr){8-9}
    ep & lr & wd & seed & QM9 val & reg.\,50 & hard\,25 & Muon & Adam \\
    \midrule
    \multicolumn{9}{l}{\textit{$\Delta H_\mathrm{SAP}$-learning}} \\
    \textbf{16} & \textbf{1e-4} & \textbf{1e-3} & 684464 & 11.8 & \bad{12.1} & \bad{56} & 0.18 & 0.63 \\
    \addlinespace
    \multicolumn{9}{l}{\textit{$\Delta (H_\mathrm{SAP} + c_\textrm{train-set-fitted} S)$-learning}} \\
    16 & 3e-6 & 1e-3 & 684464 & 36.0 & \bad{31.6} & \bad{87} & 0.18 & 0.63 \\
    16 & 1e-5 & 1e-3 & 684464 & 9.17 & \bad{12.1} & \bad{110} & 0.18 & 0.63 \\
    16 & 3e-5 & 1e-3 & 684464 & 3.59 & 7.18 & \bad{339} & 0.18 & 0.63 \\
    \textbf{16} & \textbf{1e-4} & \textbf{1e-3} & 684464 & \underline{2.37} & \textbf{5.01} & \textbf{20} & 0.18 & 0.63 \\
    " & " & " & 485210 & 2.77 & 5.82 & \bad{56.6} & 0.18 & 0.63 \\
    " & " & " & 710469 & 3.10 & 5.55 & 24.3 & 0.18 & 0.63 \\
    16 & 3e-4 & 1e-3 & 684464 & \textbf{1.91} & \underline{5.19} & \underline{22} & 0.18 & 0.64 \\
    16 & 1e-3 & 1e-3 & 684464 & 3.93 & 10.0 & \bad{62} & 0.21 & 0.73 \\
    \bottomrule
  \end{tabular*}
\end{table}

\keepwithfloat[0.55]
\subsubsection{Orthonormal Loss}
\vspace{-0.2cm}

\begin{table}[h]
  \centering
  \caption{\textbf{Validation and probe loss of Hamiltonian models trained with \oursOrthonormal{}.}
  Target is $\Delta$-SAP, with a training-set constant gauge in the
  $\Delta (H_\mathrm{SAP}+cS)$ block only.
  Validation is the training objective on the QM9 validation split at the final epoch.
  The QM40-loss columns evaluate that objective on the fixed probe molecules (reg.\,50 / hard\,25).
  $\|W\|$ columns are per-entry RMS of the final EMA weights by optimizer branch.}
  \label{tab:new_generation_fb1_hamiltonian_lowdin}
  \footnotesize
  \begin{tabular*}{\textwidth}{@{\extracolsep{\fill}}cccrrrrrr@{}}
    \toprule
    & & & & &
    \multicolumn{2}{c}{QM40 loss} &
    \multicolumn{2}{c}{$\|W\|_{\mathrm{rms}}$} \\
    \cmidrule(lr){6-7}\cmidrule(lr){8-9}
    ep & lr & wd & seed & QM9 val & reg.\,50 & hard\,25 & Muon & Adam \\
    \midrule
    \addlinespace
    \multicolumn{9}{l}{\textit{$\Delta H_\mathrm{SAP}$-learning}} \\
    \textbf{16} & \textbf{3e-5} & \textbf{3e-3} & 670079 & 1.60 & \textbf{3.6} & \textbf{9.7} & 0.18 & 0.63 \\
    " & " & " & 435126 & 1.61 & \underline{3.7} & 9.9 & 0.18 & 0.62 \\
    " & " & " & 301732 & 1.63 & \underline{3.7} & 10.0 & 0.18 & 0.63 \\
    \addlinespace
    \multicolumn{9}{l}{\textit{$\Delta (H_\mathrm{SAP} + c_\textrm{train-set-fitted} S)$-learning}} \\
    12 & 3e-5 & 1e-2 & 0 & 1.68 & 3.8 & 9.9 & 0.18 & 0.62 \\
    12 & 1e-4 & 1e-2 & 0 & 1.13 & 4.0 & 10.7 & 0.17 & 0.61 \\
    12 & 3e-4 & 1e-2 & 0 & 0.87 & 4.5 & 13.3 & 0.16 & 0.57 \\
    12 & 1e-3 & 1e-2 & 0 & 0.78 & 13.7 & 33.2 & 0.17 & 0.47 \\
    12 & 3e-3 & 1e-2 & 0 & \underline{0.73} & 17.5 & 33.4 & 0.23 & 0.37 \\
    \addlinespace
    16 & 1e-5 & 3e-3 & 670079 & 2.20 & 4.2 & 9.9 & 0.18 & 0.63 \\
    \textbf{16} & \textbf{3e-5} & \textbf{3e-3} & 670079 & 1.51 & 3.8 & \underline{9.7} & 0.18 & 0.62 \\
    " & " & " & 435126 & 1.52 & \underline{3.7} & \underline{9.7} & 0.18 & 0.62 \\
    " & " & " & 301732 & 1.53 & \textbf{3.6} & \textbf{9.6} & 0.18 & 0.63 \\
    16 & 3e-5 & 1e-2 & 670079 & 1.51 & 3.8 & 9.8 & 0.18 & 0.62 \\
    16 & 1e-4 & 3e-3 & 670079 & 1.04 & 3.8 & 10.1 & 0.18 & 0.62 \\
    16 & 1e-4 & 1e-2 & 670079 & 1.03 & 3.8 & 10.3 & 0.17 & 0.60 \\
    16 & 1e-4 & 3e-2 & 670079 & 1.02 & 4.0 & 10.9 & 0.15 & 0.56 \\
    16 & 3e-4 & 1e-2 & 670079 & 0.80 & 4.5 & 13.4 & 0.16 & 0.56 \\
    16 & 1e-3 & 3e-3 & 940789 & \textbf{0.71} & 22.0 & 48.3 & 0.21 & 0.55 \\
    16 & 1e-3 & 1e-2 & 0 & \underline{0.73} & 17.3 & 40.4 & 0.17 & 0.44 \\
    \bottomrule
  \end{tabular*}
\end{table}

\keepwithfloat[0.78]
\subsection{\ourLoss-Loss Ablations}
\label{app:rocket-ablations}
\vspace{-0.2cm}

\begin{table}[h]
  \centering
  \caption{\textbf{Validation and probe loss of Hamiltonian models trained with the \ourLoss{} loss.}
  Target is $\Delta$-SAP with a training-set constant gauge.
  Validation is the model's training objective on the QM9 validation split at the final epoch.
  The QM40-loss columns evaluate that objective on the fixed probe molecules (reg.\,50 / hard\,25).
  $\|W\|$ columns are per-entry RMS of the final EMA weights by optimizer branch.}
  \label{tab:rocket-ablations}
  \footnotesize
  \begin{tabular*}{\textwidth}{@{\extracolsep{\fill}}cccrrrrrr@{}}
    \toprule
    & & & & &
    \multicolumn{2}{c}{QM40 loss} &
    \multicolumn{2}{c}{$\|W\|_{\mathrm{rms}}$} \\
    \cmidrule(lr){6-7}\cmidrule(lr){8-9}
    ep & lr & wd & seed & QM9 val & reg.\,50 & hard\,25 & Muon & Adam \\
    \midrule
    \multicolumn{9}{l}{\textit{1st stage (unrestrained gauge)}} \\
    16 & 1e-4 & 1e-2 & 443905 & 0.96 & 41.0 & 32.5 & 0.17 & 0.60 \\
    \textbf{16} & \textbf{3e-4} & \textbf{1e-2} & 443905 & \underline{0.75} & \textbf{14.2} & \textbf{24.5} & 0.16 & 0.56 \\
    16 & 1e-3 & 1e-2 & 443905 & \textbf{0.67} & \underline{24.6} & \underline{26.8} & 0.17 & 0.44 \\
    \addlinespace
    \multicolumn{9}{l}{\textit{2nd stage (restrained gauge)}} \\
    16 & 1e-4 & 1e-2 & 586405 & 1.38 & 3.6 & 7.8 & 0.17 & 0.60 \\
    16 & 1e-4 & 3e-3 & 586405 & 1.36 & 3.5 & 7.7 & 0.18 & 0.62 \\
    \textbf{16} & \textbf{3e-4} & \textbf{3e-3} & 443905 & 0.76 & \underline{2.4} & \underline{7.5} & 0.18 & 0.60 \\
    " & " & " & 452826 & 0.77 & \textbf{2.3} & \textbf{7.1} & 0.18 & 0.60 \\
    " & " & " & 862718 & 0.77 & 2.5 & 8.2 & 0.18 & 0.61 \\
    16 & 3e-4 & 1e-2 & 443905 & 0.76 & 2.5 & 8.0 & 0.16 & 0.56 \\
    16 & 1e-3 & 1e-2 & 779430 & \underline{0.73} & 2.5 & 8.9 & 0.17 & 0.44 \\
    16 & 3e-3 & 1e-2 & 958391 & \textbf{0.70} & 3.5 & 14.9 & 0.24 & 0.37 \\
    \addlinespace
    \multicolumn{9}{l}{\textit{2nd stage (10 fold rel restrained gauge)}} \\
    \textbf{16} & \textbf{3e-4} & \textbf{3e-3} & 443905 & 0.87 & 2.1 & 6.31 & 0.18 & 0.60 \\
    \addlinespace
    \multicolumn{9}{l}{\textit{3rd stage (supervise occ+LUMO only; $k=1$)}} \\
    16 & 3e-4 & 3e-3 & 443905 & 0.46 & 2.2 & 3.8 & 0.18 & 0.61 \\
    16 & 3e-4 & 1e-2 & 443905 & 0.48 & 2.3 & 4.2 & 0.16 & 0.56 \\
    \textbf{16} & \textbf{1e-3} & \textbf{3e-3} & 443905 & 0.27 & \underline{1.4} & \underline{3.3} & 0.21 & 0.56 \\
    " & " & " & 617814 & 0.43 & 2.2 & 3.9 & 0.21 & 0.56 \\
    " & " & " & 307265 & \textbf{0.25} & \textbf{1.3} & \textbf{3.2} & 0.21 & 0.56 \\
    16 & 1e-3 & 1e-3 & 443905 & 0.33 & 1.7 & 3.5 & 0.23 & 0.60 \\
    16 & 3e-3 & 3e-3 & 443905 & \underline{0.26} & \underline{1.4} & 4.1 & 0.36 & 0.47 \\
    \addlinespace
    \multicolumn{9}{l}{\textit{3rd stage (supervise occ+LUMO+1; $k=2$)}} \\
    16 & 1e-3 & 3e-3 & 443905 & 0.25 & 1.1 & 3.3 & 0.21 & 0.56 \\
    \addlinespace
    \multicolumn{9}{l}{\textit{3rd stage (supervise occ+LUMO+3; $k=4$)}} \\
    16 & 1e-3 & 3e-3 & 443905 & 0.34 & 0.8 & 3.1 & 0.21 & 0.56 \\
    \addlinespace
    \multicolumn{9}{l}{\textit{2nd stage + gap shifts (50 mHa, 100 mHa, 200 mHa, 800 mHa)}} \\
    \textbf{16} & \textbf{3e-4} & \textbf{3e-3} & 443905 & \textbf{0.83} & \textbf{2.4} & \underline{7.9} & 0.18 & 0.60 \\
    \textbf{16} & \textbf{3e-4} & \textbf{3e-3} & 443905 & \underline{0.92} & \underline{2.6} & \textbf{7.7} & 0.18 & 0.60 \\
    \textbf{16} & \textbf{3e-4} & \textbf{3e-3} & 443905 & 0.99 & 3.0 & 8.0 & 0.18 & 0.60 \\
    \textbf{16} & \textbf{3e-4} & \textbf{3e-3} & 443905 & 3.08 & 7.1 & 9.6 & 0.18 & 0.60 \\
    \addlinespace
    \multicolumn{9}{l}{\textit{3rd stage + gap shift 100 mHa}} \\
    \textbf{16} & \textbf{3e-4} & \textbf{3e-3} & 443905 & 0.86 & 3.1 & 4.5 & 0.18 & 0.61 \\
    \bottomrule
  \end{tabular*}
\end{table}

\newpage
\subsection{Self-Consistency-Training Learning-Rate Selection}
\label{app:sct-lr-selection}

Self-consistency training needs no reference labels, so we also select its learning rate without
them.
\Cref{tab:transition1x_sct_val_selection} reports the training objective on the Transition1x
validation split for every learning rate and epoch.
At each scored epoch we take the learning rate with the lowest residual; the resulting checkpoints
are evaluated on the test split in \cref{tab:ablation_transition1x_sct_test}.

\begin{table}[h]
\centering
\footnotesize
\setlength{\tabcolsep}{3pt}
\renewcommand{\arraystretch}{1.05}
\caption{\textbf{Validation loss across learning rates and epochs for Transition1x self-consistency training (\cref{tab:transition1x-dataset}).}
Entries report the mean training objective residual over the Transition1x validation split in units of $10^{-4}$, used to select checkpoints for \cref{tab:ablation_transition1x_sct_test}.
For ROCKET 3rd stage, values omit the gauge-restraint term.
LR denotes the peak learning rate of the warmup-cosine schedule.
Underlined values mark the lowest validation loss per configuration and epoch column, identifying the selected checkpoint.
$^\dagger$ Epochs evaluated in the Transition1x test evaluation table (\cref{tab:ablation_transition1x_sct_test}).
$^{\ast}$ Runs stopped before configured completion.
Each block reports a distinct model residual and cannot be compared across blocks.}
\label{tab:transition1x_sct_val_selection}
\begin{tabular*}{\linewidth}{@{\extracolsep{\fill}}lrrrrrrrrrrrr}
\toprule
LR & \multicolumn{12}{c}{Validation self-consistency residual after epoch} \\
 & 1$^\dagger$ & 2$^\dagger$ & 3 & 4$^\dagger$ & 5 & 6 & 7 & 8 & 9 & 10 & 11 & 12$^\dagger$ \\
\midrule
\multicolumn{13}{@{}l@{}}{\makebox[0pt][l]{\textbf{GROOT\textsubscript{rot}} -- density residual, $10^{-4}$ (dimensionless)}} \\
$3{\times}10^{-5}$$^{\ast}$ & \phantom{00}12.4\phantom{0} & \phantom{000}9.10 & 8.47 & 7.59 & 7.00 & 6.98 &  &  &  &  &  &  \\
$10^{-4}$ & \phantom{000}9.08 & \phantom{000}\underline{7.41} & 7.00 & 6.65 & 6.20 & 5.33 & 5.26 & 4.98 & 4.88 & 4.84 & 4.81 & 4.81 \\
$3{\times}10^{-4}$ & \phantom{000}\underline{8.13} & \phantom{000}7.42 & \underline{6.71} & \underline{5.97} & \underline{5.86} & \underline{5.25} & \underline{5.01} & \underline{4.66} & \underline{4.34} & 4.28 & 4.23 & 4.19 \\
$10^{-3}$ & \phantom{00}10.6\phantom{0} & \phantom{000}8.75 & 8.45 & 7.38 & 6.36 & 5.50 & 5.19 & 4.71 & 4.41 & \underline{4.15} & \underline{4.07} & \underline{4.04} \\
$3{\times}10^{-3}$ & \phantom{00}13.1\phantom{0} & \phantom{00}10.4\phantom{0} & 8.87 & 8.01 & 7.02 & 6.85 & 5.78 & 5.43 & 5.07 & 4.75 & 4.58 & 4.52 \\
\midrule
\multicolumn{13}{@{}l@{}}{\makebox[0pt][l]{\textbf{GROOT\textsubscript{proj}} -- density residual, $10^{-4}$ (dimensionless)}} \\
$3{\times}10^{-5}$$^{\ast}$ & \phantom{000}9.83 & \phantom{000}7.86 & 7.30 & 6.77 &  &  &  &  &  &  &  &  \\
$10^{-4}$ & \phantom{000}\underline{7.32} & \phantom{000}\underline{6.11} & \underline{5.72} & 5.52 & \underline{5.08} & 4.86 & 4.81 & 4.72 & 4.62 & 4.58 & 4.55 & 4.54 \\
$3{\times}10^{-4}$ & \phantom{000}7.61 & \phantom{000}6.46 & 6.06 & \underline{5.42} & 5.18 & \underline{4.77} & \underline{4.66} & \underline{4.27} & \underline{4.12} & 4.11 & 4.01 & 4.02 \\
$10^{-3}$ & \phantom{000}8.45 & \phantom{000}7.06 & 7.04 & 6.11 & 5.80 & 4.85 & 4.79 & 4.37 & 4.17 & \underline{3.99} & \underline{3.92} & \underline{3.91} \\
$3{\times}10^{-3}$ & \phantom{000}9.93 & \phantom{000}8.06 & 7.57 & 6.93 & 6.58 & 5.85 & 5.58 & 4.82 & 4.61 & 4.41 & 4.33 & 4.26 \\
\midrule
\multicolumn{13}{@{}l@{}}{\makebox[0pt][l]{\textbf{$\mathbf{P}_\mathrm{base}$ (ON)} -- raw-readout density residual, $10^{-4}$ (dimensionless)}} \\
$10^{-4}$$^{\ast}$ & \phantom{000}8.52 & \phantom{000}6.99 & 6.27 & 5.84 & 5.58 & 5.40 &  &  &  &  &  &  \\
$3{\times}10^{-4}$$^{\ast}$ & \phantom{000}6.68 & \phantom{000}5.63 & 5.10 & 4.80 & 4.61 & 4.40 & 4.29 & 4.19 & 4.14 & 4.11 &  &  \\
$10^{-3}$ & \phantom{000}\underline{6.52} & \phantom{000}\underline{5.57} & \underline{5.03} & \underline{4.72} & \underline{4.43} & \underline{4.16} & \underline{3.94} & \underline{3.77} & \underline{3.66} & \underline{3.60} & \underline{3.56} & \underline{3.55} \\
$3{\times}10^{-3}$$^{\ast}$ & \phantom{000}8.86 & \phantom{000}7.04 & 7.49 &  &  &  &  &  &  &  &  &  \\
$10^{-2}$$^{\ast}$ & 1377\phantom{.}\phantom{00} & 2037\phantom{.}\phantom{00} &  &  &  &  &  &  &  &  &  &  \\
\midrule
\multicolumn{13}{@{}l@{}}{\makebox[0pt][l]{\textbf{ROCKET 3rd stage} -- gauge-free occ+LUMO Hamiltonian residual, $10^{-4}$ Ha}} \\
$10^{-4}$$^{\ast}$ & \phantom{000}4.08 & \phantom{000}3.24 & 2.85 & 2.60 &  &  &  &  &  &  &  &  \\
$3{\times}10^{-4}$ & \phantom{000}2.97 & \phantom{000}\underline{2.54} & \underline{2.29} & \underline{2.11} & \underline{1.99} & 1.90 & 1.85 & 1.75 & 1.72 & 1.70 & 1.69 & 1.69 \\
$10^{-3}$ & \phantom{000}\underline{2.97} & \phantom{000}2.61 & 2.38 & 2.14 & 2.00 & \underline{1.80} & \underline{1.68} & \underline{1.60} & \underline{1.52} & \underline{1.49} & \underline{1.46} & \underline{1.45} \\
$3{\times}10^{-3}$$^{\ast}$ & \phantom{000}3.68 & \phantom{000}2.96 & 2.78 & 2.48 &  &  &  &  &  &  &  &  \\
$10^{-2}$$^{\ast}$ & \phantom{000}5.06 & \phantom{000}8.35 &  &  &  &  &  &  &  &  &  &  \\
\bottomrule
\end{tabular*}
\end{table}

\newpage
\section{Self-Consistency Training}\label{app:sct}
Self-consistency training (SCT) fine-tunes a pretrained \model{} on molecular geometries alone~\citep{zhangSelfConsistencyTraining2024}.
Instead of evaluating errors against converged electronic reference labels, each training step performs a single Fock build and penalizes deviation from self-consistency.
We compare two such objectives on Transition1x~(\cref{tab:ablation_transition1x_sct_test}).

\paragraph{Hamiltonian and density matrix residuals.}
\citet{zhangSelfConsistencyTraining2024} introduce SCT for Hamiltonian prediction, penalizing the Frobenius difference between the predicted Hamiltonian and the Hamiltonian reconstructed from its own eigenvectors.
We extend this residual to density \model s and express both targets as
\begin{equation}
    \mathcal{L}_\mathrm{res}(\hat{\mathbf{X}}) := \frac{\|\hat{\mathbf{X}}-\mathrm{SCF}(\hat{\mathbf{X}})\|_F}{B}, \qquad \mathrm{SCF}(\hat{\mathbf{P}}) := \mathbf{P}\big(\mathbf{H}(\hat{\mathbf{P}})\big), \qquad \mathrm{SCF}(\hat{\mathbf{H}}) := \mathbf{H}\big(\mathbf{P}(\hat{\mathbf{H}})\big),
\end{equation}
where $\mathbf{H}(\cdot)$ denotes one Fock build, and $\mathbf{P}(\cdot)$ forms the occupied density matrix from the lowest $O$ eigenvectors of \cref{eq:generalized_eigenvalue_problem} (\cref{eq:P_from_C}).
For density models, this residual is the self-consistency error $\Delta_\mathrm{SC}\mathbf{P}$ of \cref{sec:experiments}.
Following \citet{zhangSelfConsistencyTraining2024}, we evaluate the residual in the non-orthogonal (NON) basis and backpropagate through the Fock build.
This residual weights all matrix entries equally and inherits the overlap-dependent spectral distortion of the NON metric~(\cref{sec:methods}, \appref{app:loewdin-sensitivity}).
The targets are also scaled asymmetrically because $\mathcal{L}_\mathrm{res}(\hat{\mathbf{H}})$ carries energy units while $\mathcal{L}_\mathrm{res}(\hat{\mathbf{P}})$ is dimensionless, and differentiating through the eigensolver to form $\mathbf{P}(\cdot)$ can destabilize training near orbital degeneracies.

\paragraph{Occupied--virtual gradient objective.}
Instead of evaluating matrix differences, we supervise \model s with the occupied--virtual (OV) gradient of the electronic energy from solver-aligned initialization learning (SAIL)~\citep{eberhardTransferableScfAcceleration2026}.
Under a budget of one Fock build, the multi-step trajectory of SAIL reduces to a single step
\begin{equation}
    \mathcal{L}_\mathrm{OV}(\hat{\mathbf{X}}) := \Big(\frac{1}{OV}\sum_{i=1}^{O}\sum_{a=1}^{V} G_{ia}^2\Big)^{1/2}, \qquad \mathbf{G} := \widetilde{\mathbf{C}}_O^{\top}\, \widetilde{\mathbf{H}}(\hat{\mathbf{P}})\, \widetilde{\mathbf{C}}_V.
\end{equation}
Here $O$ and $V$ denote the occupied and virtual orbital counts ($O + V = B$), and $\widetilde{\mathbf{H}}(\hat{\mathbf{P}}) = \mathbf{S}^{-1/2}\mathbf{H}(\hat{\mathbf{P}})\mathbf{S}^{-1/2}$ is the Fock matrix in the L\"owdin orthonormal basis.
For density models, $\widetilde{\mathbf{C}}_O$ and $\widetilde{\mathbf{C}}_V$ collect the eigenvectors of the predicted L\"owdin density $\hat{\widetilde{\mathbf{P}}}$ with the $O$ largest and $V$ smallest eigenvalues.
For Hamiltonian models, they collect the eigenvectors of $\hat{\widetilde{\mathbf{H}}}$, with $\hat{\mathbf{P}} = \mathbf{P}(\hat{\mathbf{H}})$.
The matrix $\mathbf{G} \in \mathbb{R}^{O \times V}$ measures the gradient of the electronic energy with respect to occupied--virtual rotations.
It carries units of energy, vanishes exactly at self-consistency, and remains invariant under rotations within the occupied and virtual subspaces.
The OV gradient provides a unified, physical objective for both density and Hamiltonian \model s without requiring an ad hoc matrix norm.

\paragraph{Connection to the self-consistency rejection rule.}
The OV gradient directly minimizes the commutator rejection statistic used at inference~(\cref{sec:experiments}).
In the orbital basis $\widetilde{\mathbf{C}} = [\widetilde{\mathbf{C}}_O, \widetilde{\mathbf{C}}_V]$, any idempotent density matrix satisfies $\widetilde{\mathbf{P}} = \operatorname{diag}(\mathbf{I}_O, \mathbf{0}_V)$.
The diagonal blocks of the commutator between the L\"owdin Fock matrix and the density matrix vanish identically, leaving only the off-diagonal blocks $\mathbf{G}$ and $-\mathbf{G}^\top$.
The Frobenius norm of this commutator gives the DIIS residual of the self-consistency rejection rule
\begin{equation}
    r = \sqrt{8}\,\big\|\widetilde{\mathbf{H}}\widetilde{\mathbf{P}}-\widetilde{\mathbf{P}}\widetilde{\mathbf{H}}\big\|_F = 4\,\|\mathbf{G}\|_F = 4\sqrt{OV}\,\mathcal{L}_\mathrm{OV}(\hat{\mathbf{X}}).
\end{equation}
For \ours{} and all Hamiltonian models, the prediction is idempotent by construction, so $\mathcal{L}_\mathrm{OV}$ directly minimizes the rejection statistic itself.
This matches the drop in Transition1x rejection rates in \cref{tab:ablation_transition1x_sct_test}, where rejection drops from $13.5\%$ under the matrix residual to $2.6\%$ under $\mathcal{L}_\mathrm{OV}$ for \ours{}\textsubscript{rot}, and from $1.2\%$ to $0.2\%$ for \ours{}\textsubscript{proj}.
For unconstrained $\mathbf{P}_\mathrm{base}$, predictions have fractional eigenvalues off the Grassmannian, producing commutator contributions that $\mathcal{L}_\mathrm{OV}$ does not penalize and leaving rejection rates elevated at $25.0\%$.

\paragraph{Training protocol.}
We warm-start all models from the QMugs parent checkpoints of \cref{tab:ablation_qmugs_ood160_r100} and train on the $164{,}187$ unlabeled frames of the Transition1x training split~(\cref{tab:transition1x-dataset}) under a cosine learning-rate schedule without weight decay or replay of pretraining data.
We select the peak learning rate of each model by the validation residual alone~(\appref{app:sct-lr-selection}, \cref{tab:transition1x_sct_val_selection}), so that no training or selection step uses reference labels.
Each epoch costs one Fock build and its backward pass per geometry, avoiding the costly iterations required for converged references.

\newpage
\keepwithfloat[0.75]
\section{Loss vs Observables}
\label{app:loss-vs-observables}
\vspace{-0.2cm}

\begin{table}[h]
  \centering
  \caption{\textbf{QM40 probe observables of the production runs.}
  Every model whose epoch, learning rate, and weight decay are set to its production
  value, aggregated over its entropy-seed replicates as the mean over seeds.
  The loss column evaluates each family's own training objective on the fixed
  probe molecules (reg.\,50, $\times 10^{4}$) and is comparable only within a block.
  The cata.\ column sums the reg.\,50 and hard\,25 catastrophe counts.
  $E$ is the one-Fock-build electronic-energy MAE in mEh, and a catastrophe is
  $E > 1\,000$~mEh.
  Bold and underlined entries are the lowest and second-lowest values in each observable column, respectively.
  Red marks an error above $3$/$10$~mEh for regular/hard $E$, $1$/$4$~D for the dipole,
  $0.2$/$1\,e$ for $L_1$, or cata.\ $>1$.}
  \label{tab:prod_observables}
  \footnotesize
  \setlength{\tabcolsep}{3pt}
  \begin{tabular*}{\textwidth}{@{\extracolsep{\fill}}lccc*{8}{r}@{}}
    \toprule
    & & & & &
    \multicolumn{2}{c}{$E$ [mEh]} &
    \multicolumn{2}{c}{dipole [D]} &
    \multicolumn{2}{c}{$L_1$ [$e$]} &
    \\
    \cmidrule(lr){6-7}\cmidrule(lr){8-9}\cmidrule(lr){10-11}
    & $n$ & lr & wd & loss & reg. & hard & reg. & hard & reg. & hard & cata. \\
    \midrule
    \multicolumn{12}{l}{\textit{NON losses}} \\
    \multicolumn{12}{l}{\quad \textit{Density}} \\
    \qquad abs. $P$ & 3 & 3e-3 & 1e-2 & $14.3$ & $\superbad{259}$ & $\superbad{1\,655}$ & $0.76$ & $\bad{31.9}$ & $\bad{0.50}$ & $\bad{4.13}$ & $\superbad{11.7}$ \\
    \qquad $\Delta P_{\mathrm{MINAO}}$ & 3 & 3e-3 & 1e-2 & $13.7$ & $\superbad{76}$ & $\superbad{263}$ & $0.51$ & $\bad{4.50}$ & $\bad{0.38}$ & $\bad{1.4}$ & $\bad{1.3}$ \\
    \addlinespace
    \multicolumn{12}{l}{\quad \textit{Hamiltonian}} \\
    \qquad abs. $H$ & 1 & 1e-4 & 1e-3 & $4.0$ & $\superbad{8\,626}$ & $\superbad{19\,188}$ & $\superbad{319}$ & $\superbad{441}$ & $\superbad{24.7}$ & $\superbad{44.4}$ & $\superbad{64}$ \\
    \qquad $\Delta H_{\mathrm{SAP}}$ & 1 & 3e-5 & 1e-3 & $4.0$ & $\superbad{712}$ & $\superbad{2\,989}$ & $\superbad{50.8}$ & $\superbad{120}$ & $\superbad{3.57}$ & $\superbad{11.1}$ & $\superbad{24}$ \\
    \qquad $\Delta (H_{\mathrm{SAP}}+cS)$ & 4 & 1e-5 & 1e-3 & $3.9$ & $\superbad{1\,396}$ & $\superbad{3\,532}$ & $\superbad{69.5}$ & $\superbad{125}$ & $\superbad{5.44}$ & $\superbad{11.2}$ & $\superbad{35}$ \\
    \addlinespace
    \multicolumn{12}{l}{\textit{WALoss}} \\
    \multicolumn{12}{l}{\quad \textit{Hamiltonian}} \\
    \qquad $\Delta H_{\mathrm{SAP}}$ & 1 & 1e-4 & 1e-3 & $12.1$ & $\superbad{2\,032}$ & $\superbad{4\,338}$ & $\bad{1.46}$ & $\bad{25.0}$ & $\superbad{4.64}$ & $\bad{9.12}$ & $\superbad{68}$ \\
    \qquad $\Delta (H_{\mathrm{SAP}}+cS)$ & 3 & 1e-4 & 1e-3 & $5.5$ & $\superbad{614}$ & $\superbad{3\,333}$ & $\bad{7.07}$ & $\superbad{51.7}$ & $\bad{1.76}$ & $\bad{7.54}$ & $\superbad{19.7}$ \\
    \addlinespace
    \multicolumn{12}{l}{\textit{\oursOrthonormal{}}} \\
    \multicolumn{12}{l}{\quad \textit{Density}} \\
    \qquad $\Delta P_{\mathrm{MINAO}}$ & 3 & 3e-3 & 3e-3 & $6.1$ & $\superbad{35.0}$ & $\bad{41.1}$ & $\underline{0.09}$ & $\textbf{0.60}$ & $\underline{0.08}$ & $\textbf{0.24}$ & $0$ \\
    \qquad \ours{}$_{\mathrm{proj}}$ & 3 & 3e-3 & 1e-2 & $1.1$ & $0.87$ & $9.24$ & $\textbf{0.07}$ & $1.89$ & $\textbf{0.05}$ & $\underline{0.46}$ & $0$ \\
    \qquad \ours{}$_{\mathrm{rot}}$ & 3 & 1e-2 & 1e-2 & $1.5$ & $1.68$ & $\superbad{121}$ & $\textbf{0.07}$ & $\superbad{51.4}$ & $\textbf{0.05}$ & $\superbad{3.57}$ & $\bad{2}$\\
    \addlinespace
    \multicolumn{12}{l}{\quad \textit{Hamiltonian}} \\
    \qquad $\Delta H_{\mathrm{SAP}}$ & 3 & 3e-5 & 3e-3 & $3.7$ & $2.74$ & $\bad{19.8}$ & $0.20$ & $\bad{4.43}$ & $\bad{0.22}$ & $\bad{1.25}$ & $0$ \\
    \qquad $\Delta (H_{\mathrm{SAP}}+cS)$ & 3 & 3e-5 & 3e-3 & $3.7$ & $2.22$ & $\bad{15.5}$ & $0.23$ & $\bad{4.26}$ & $0.19$ & $\bad{1.11}$ & $0$ \\
    \addlinespace
    \multicolumn{12}{l}{\textit{\ourLoss{}} (\textit{Hamiltonian})} \\
    \qquad 1st stage & 1 & 3e-4 & 1e-2 & $14.2$ & $\bad{3.83}$ & $\bad{69.7}$ & $0.46$ & $1.87$ & $\bad{0.38}$ & $\superbad{2.30}$ & $0$ \\
    \qquad 2nd stage & 3 & 3e-4 & 3e-3 & $2.4$ & $1.22$ & $8.38$ & $0.12$ & $3.63$ & $0.11$ & $0.83$ & $0$ \\
    \qquad + gap shift 50 mHa & 1 & 3e-4 & 3e-3 & $2.4$ & $1.22$ & $8.24$ & $0.11$ & $3.70$ & $0.11$ & $0.84$ & $0$ \\
    \qquad + gap shift 100 mHa & 1 & 3e-4 & 3e-3 & $2.6$ & $1.25$ & $8.60$ & $0.11$ & $3.49$ & $0.11$ & $0.83$ & $0$ \\
    \qquad + gap shift 200 mHa & 1 & 3e-4 & 3e-3 & $3.0$ & $1.44$ & $7.78$ & $0.12$ & $3.52$ & $0.12$ & $0.78$ & $0$ \\
    \qquad + gap shift 800 mHa & 1 & 3e-4 & 3e-3 & $7.1$ & $2.83$ & $9.42$ & $0.13$ & $2.08$ & $0.13$ & $0.56$ & $0$ \\
    \qquad 3rd stage $k=1$ & 3 & 1e-3 & 3e-3 & $1.6$ & $\textbf{0.42}$ & $7.69$ & $\underline{0.09}$ & $1.58$ & $0.09$ & $0.72$ & $0$ \\
    \qquad + gap shift 100 mHa & 1 & 3e-4 & 3e-3 & $3.1$ & $0.48$ & $\textbf{6.34}$ & $\underline{0.09}$ & $2.54$ & $0.09$ & $0.73$ & $0$ \\
    \qquad 3rd stage $k=2$ & 1 & 1e-3 & 3e-3 & $1.1$ & $\underline{0.45}$ & $\underline{6.50}$ & $0.10$ & $\underline{1.39}$ & $0.09$ & $0.67$ & $0$ \\
    \qquad 3rd stage $k=4$ & 1 & 1e-3 & 3e-3 & $0.8$ & $0.51$ & $8.57$ & $\underline{0.09}$ & $2.50$ & $0.10$ & $0.78$ & $0$ \\
    \bottomrule
  \end{tabular*}
\end{table}

\clearpage
\subsection{Frobenius Loss versus Observable Errors}
\label{app:frobenius-vs-observables}

\Cref{tab:prod_observables} compares training losses by downstream observable accuracy.
Here we evaluate whether the matrix-level training loss directly tracks observable errors across molecules, and whether evaluating that loss in the L\"owdin basis strengthens the relationship.
\Cref{fig:frobenius-energy-scatter,fig:frobenius-dipole-scatter,fig:frobenius-density-l1-scatter,fig:frobenius-force-scatter} plot observable errors against Frobenius losses in both bases across QM40 molecules.
\Cref{tab:frobenius_panel_stats} reports the corresponding Pearson correlation $r$ of $\log_{10}$ loss with $\log_{10}$ observable error, alongside mutual information $I$.
The L\"owdin loss predicts observable errors more reliably than the atomic-orbital loss in 22 of 24 panel pairs for $r$ and in all 24 for $I$.
This improvement is most pronounced for models trained with $\mathcal{L}_\mathrm{ON}$, where the atomic-orbital loss decorrelates from observable errors.

\begin{table}[h]
\centering
\footnotesize
\setlength{\tabcolsep}{5pt}
\renewcommand{\arraystretch}{1.05}
\caption{\textbf{Correlation and mutual information between Frobenius loss and observable errors across QM40 plate panels.}
We report the Pearson correlation $r$ between $\log_{10}$ Frobenius loss and $\log_{10}$ observable error, alongside mutual information $I$ in bits estimated with $k$-nearest neighbors ($k=5$).
Statistics are evaluated for each model pooled across seeds, corresponding to the scatter points in each plate panel.
NON denotes the raw congruence $\|\Delta\mathbf{X}\|_F$, and ON denotes the L\"owdin congruence $\|\mathbf{S}^{+1/2}\Delta\mathbf{X}\mathbf{S}^{+1/2}\|_F$ for densities or $\|\mathbf{S}^{-1/2}\Delta\mathbf{X}\mathbf{S}^{-1/2}\|_F$ for Hamiltonians.
Bold text marks the superior value in each pair, where larger values indicate stronger predictive alignment.
The evaluation uses 1000 QM40 molecules per seed without exclusions, with force rows restricted to the 200 molecules per seed that possess reference forces.}
\label{tab:frobenius_panel_stats}
\begin{tabular*}{\linewidth}{@{\extracolsep{\fill}}lcrrrr}
\toprule
 &  & \multicolumn{2}{c}{$r$} & \multicolumn{2}{c}{$I$ (base 2)} \\
\cmidrule(lr){3-4}\cmidrule(lr){5-6}
Model & Seeds & NON & ON & NON & ON \\
\midrule
\multicolumn{6}{l}{$|\Delta E|$ [m$E_h$]} \\
$\mathbf{P}_\mathrm{base}$ (NON) & 3 & $0.595$ & $\mathbf{0.669}$ & $0.38$ & $\mathbf{0.56}$ \\
$\mathbf{P}_\mathrm{base}$ (ON) & 3 & $0.778$ & $\mathbf{0.909}$ & $0.87$ & $\mathbf{1.58}$ \\
$\mathbf{H}_\mathrm{base}$ (NON) & 1 & $0.924$ & $\mathbf{0.932}$ & $1.17$ & $\mathbf{1.32}$ \\
$\mathbf{H}_\mathrm{base}$ (ON) & 3 & $0.555$ & $\mathbf{0.802}$ & $0.30$ & $\mathbf{1.06}$ \\
$\mathbf{H}_\mathrm{base} + c_t \mathbf{S}$ (NON) & 4 & $0.911$ & $\mathbf{0.927}$ & $1.10$ & $\mathbf{1.24}$ \\
$\mathbf{H}_\mathrm{base} + c_t \mathbf{S}$ (ON) & 3 & $0.556$ & $\mathbf{0.760}$ & $0.29$ & $\mathbf{0.84}$ \\
\midrule
\multicolumn{6}{l}{$\|\Delta\boldsymbol\mu\|_2$ [D]} \\
$\mathbf{P}_\mathrm{base}$ (NON) & 3 & $0.695$ & $\mathbf{0.732}$ & $0.51$ & $\mathbf{0.60}$ \\
$\mathbf{P}_\mathrm{base}$ (ON) & 3 & $0.544$ & $\mathbf{0.616}$ & $0.34$ & $\mathbf{0.41}$ \\
$\mathbf{H}_\mathrm{base}$ (NON) & 1 & $\mathbf{0.883}$ & $0.880$ & $1.09$ & $\mathbf{1.30}$ \\
$\mathbf{H}_\mathrm{base}$ (ON) & 3 & $0.529$ & $\mathbf{0.602}$ & $0.25$ & $\mathbf{0.46}$ \\
$\mathbf{H}_\mathrm{base} + c_t \mathbf{S}$ (NON) & 4 & $\mathbf{0.869}$ & $0.867$ & $1.04$ & $\mathbf{1.10}$ \\
$\mathbf{H}_\mathrm{base} + c_t \mathbf{S}$ (ON) & 3 & $0.524$ & $\mathbf{0.567}$ & $0.29$ & $\mathbf{0.35}$ \\
\midrule
\multicolumn{6}{l}{$\|\Delta\rho\|_{L^1}$ [$e$]} \\
$\mathbf{P}_\mathrm{base}$ (NON) & 3 & $0.784$ & $\mathbf{0.857}$ & $0.78$ & $\mathbf{1.06}$ \\
$\mathbf{P}_\mathrm{base}$ (ON) & 3 & $0.737$ & $\mathbf{0.837}$ & $0.67$ & $\mathbf{0.97}$ \\
$\mathbf{H}_\mathrm{base}$ (NON) & 1 & $0.907$ & $\mathbf{0.913}$ & $1.18$ & $\mathbf{1.44}$ \\
$\mathbf{H}_\mathrm{base}$ (ON) & 3 & $0.577$ & $\mathbf{0.727}$ & $0.38$ & $\mathbf{0.88}$ \\
$\mathbf{H}_\mathrm{base} + c_t \mathbf{S}$ (NON) & 4 & $0.905$ & $\mathbf{0.912}$ & $1.14$ & $\mathbf{1.31}$ \\
$\mathbf{H}_\mathrm{base} + c_t \mathbf{S}$ (ON) & 3 & $0.580$ & $\mathbf{0.675}$ & $0.37$ & $\mathbf{0.67}$ \\
\midrule
\multicolumn{6}{l}{$\langle\|\Delta\mathbf{F}_a\|_2\rangle_a$ [m$E_h$/\AA]} \\
$\mathbf{P}_\mathrm{base}$ (NON) & 3 & $0.772$ & $\mathbf{0.800}$ & $0.66$ & $\mathbf{0.72}$ \\
$\mathbf{P}_\mathrm{base}$ (ON) & 3 & $0.618$ & $\mathbf{0.717}$ & $0.51$ & $\mathbf{0.70}$ \\
$\mathbf{H}_\mathrm{base}$ (NON) & 1 & $0.905$ & $\mathbf{0.911}$ & $1.05$ & $\mathbf{1.36}$ \\
$\mathbf{H}_\mathrm{base}$ (ON) & 3 & $0.599$ & $\mathbf{0.791}$ & $0.45$ & $\mathbf{0.93}$ \\
$\mathbf{H}_\mathrm{base} + c_t \mathbf{S}$ (NON) & 4 & $0.891$ & $\mathbf{0.910}$ & $1.19$ & $\mathbf{1.38}$ \\
$\mathbf{H}_\mathrm{base} + c_t \mathbf{S}$ (ON) & 3 & $0.604$ & $\mathbf{0.757}$ & $0.38$ & $\mathbf{0.73}$ \\
\bottomrule
\end{tabular*}
\end{table}

\clearpage

\begin{figure}[!h]
  \centering
  \includegraphics[width=\linewidth]{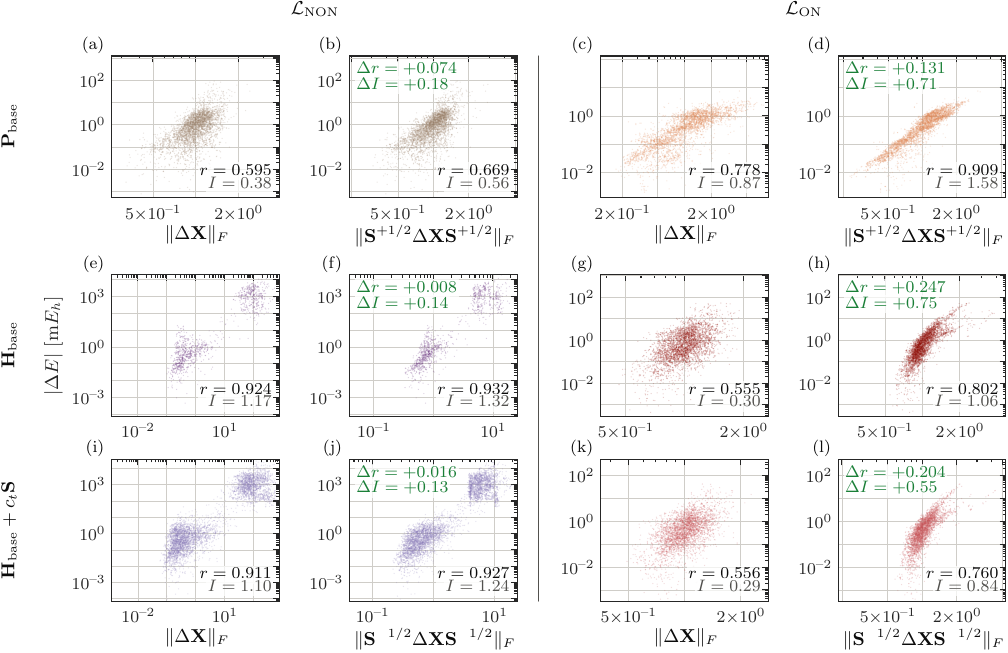}
  \caption{\textbf{Energy error $E$ against Frobenius loss.}
  Each point represents one QM40 test molecule and seed on logarithmic axes.
  Rows correspond to model families $P_\mathrm{base}$, $H_\mathrm{base}$, and $H_\mathrm{base}+c_t\mathbf{S}$.
  Column blocks indicate the training objective ($\mathcal{L}_\mathrm{NON}$ followed by $\mathcal{L}_\mathrm{ON}$).
  Within each block, the left panel evaluates the raw loss $\|\Delta\mathbf{X}\|_F$ and the right panel evaluates the L\"owdin loss (\cref{eq:lowdin_transform}).
  Individual panels report $r$ and $I$, with L\"owdin panels displaying the gain over the raw counterpart.
  \Cref{tab:frobenius_panel_stats} compiles summary statistics across all four observables.}
  \label{fig:frobenius-energy-scatter}
\end{figure}

\begin{figure}[!h]
  \centering
  \includegraphics[width=\linewidth]{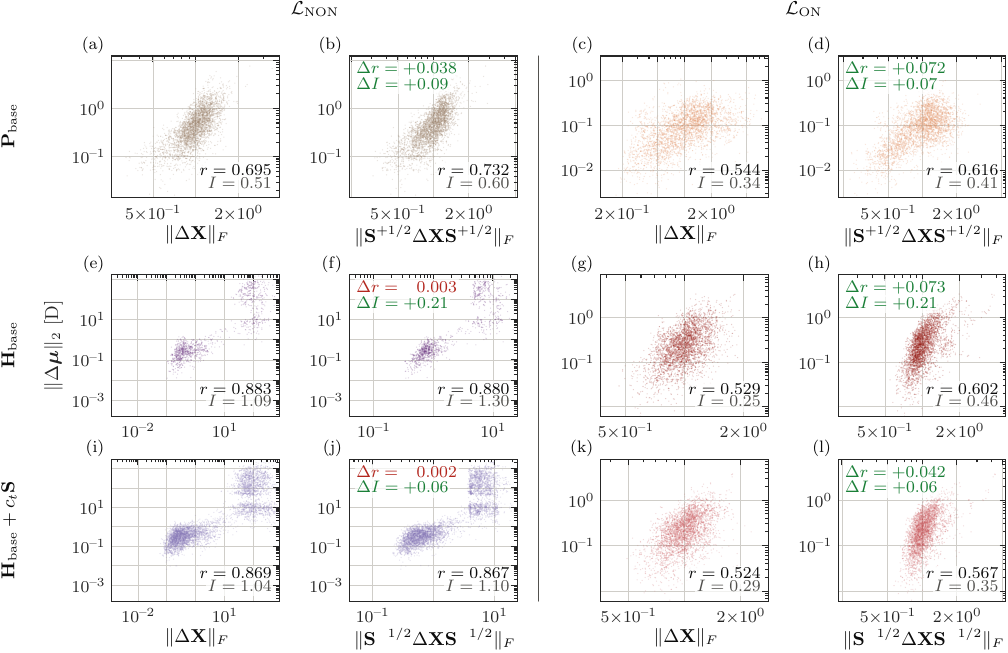}
  \caption{\textbf{Dipole error against Frobenius loss.}
  Panels share the models, layout, pooling, and annotations of \cref{fig:frobenius-energy-scatter}.
  The vertical axis displays the dipole error $\|\Delta\boldsymbol\mu\|$ in Debye.}
  \label{fig:frobenius-dipole-scatter}
\end{figure}

\clearpage

\begin{figure}[!h]
  \centering
  \includegraphics[width=\linewidth]{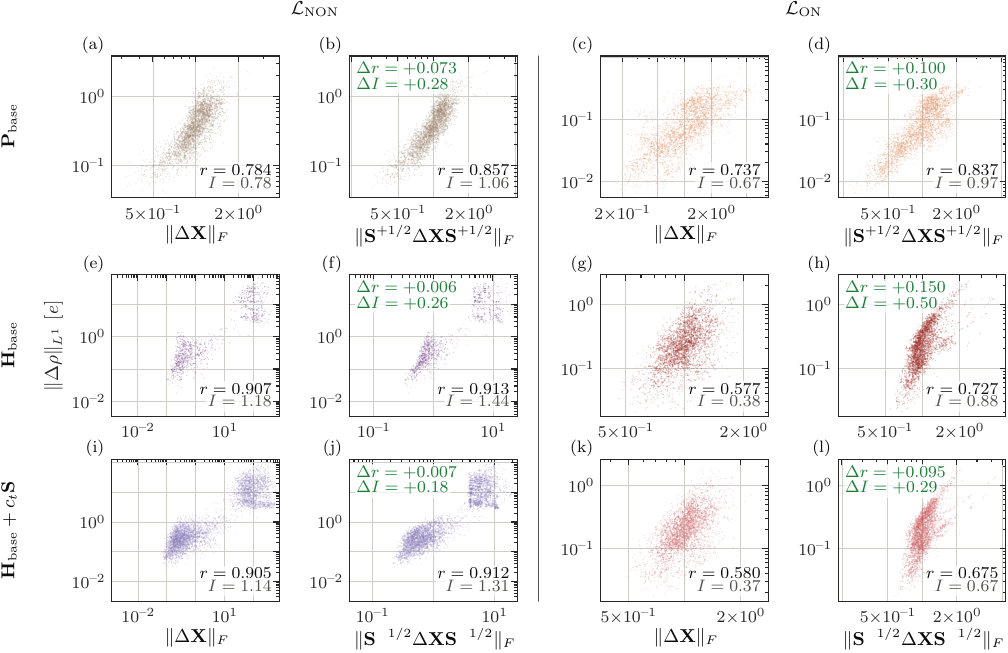}
  \caption{\textbf{Density $L^1$ error against Frobenius loss.}
  Panels share the models, layout, pooling, and annotations of \cref{fig:frobenius-energy-scatter}.
  The vertical axis displays the real-space density error $\|\Delta\rho\|_{L^1}$ in electrons.}
  \label{fig:frobenius-density-l1-scatter}
\end{figure}

\begin{figure}[!h]
  \centering
  \includegraphics[width=\linewidth]{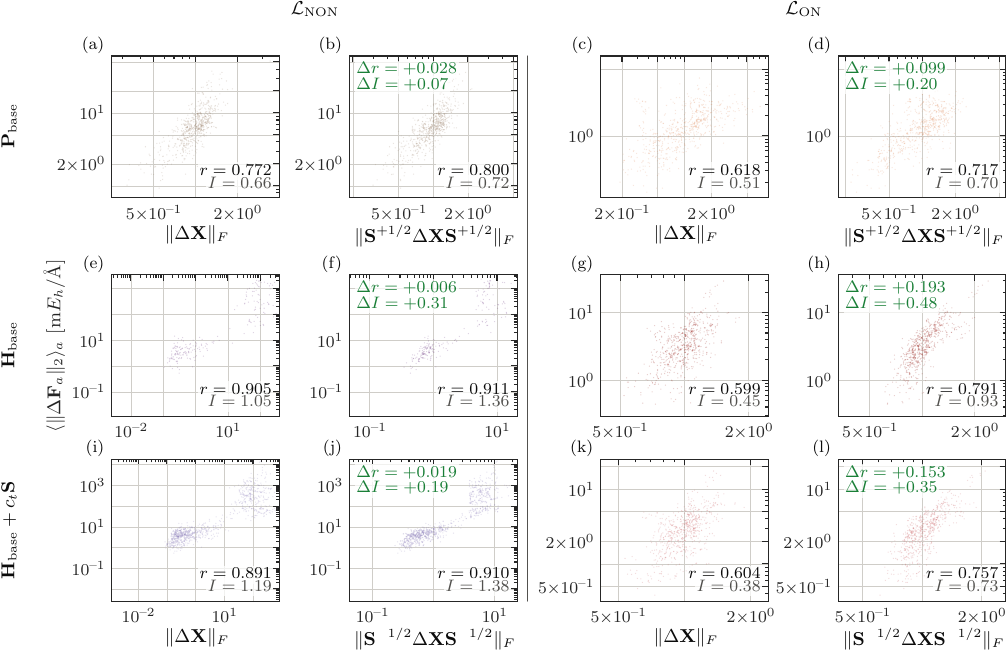}
  \caption{\textbf{Force error against Frobenius loss.}
  Panels share the models, layout, pooling, and annotations of \cref{fig:frobenius-energy-scatter}.
  The vertical axis displays the per-atom mean force error $\langle\|\Delta\mathbf{F}_a\|\rangle$.
  Because only 200 molecules per seed carry reference forces, these panels contain fewer samples than the preceding figures.}
  \label{fig:frobenius-force-scatter}
\end{figure}

\newpage
\section{Model Evaluation Data}
\label{app:model-ablation-data}

\Cref{tab:ablation_qm9_test2000} evaluates observable errors for the models of \cref{tab:prod_observables} on the QM9 test split, providing the in-distribution baseline for size extrapolation.
\Cref{fig:qm40-nelec-density-ablation,fig:qmugs-nelec-density-ablation} resolve density observable errors over electron count $N_\mathrm{elec}$, separating errors present at training size from errors induced by extrapolation.
\Cref{fig:qm40-nelec-hamiltonian-ablation,fig:qmugs-nelec-hamiltonian-ablation} report the corresponding resolution for Hamiltonian models.
Each figure pair shows QM40 before QMugs, extending extrapolation to approximately twice the electron counts of QM40.
\Cref{tab:ablation_qm40_r100,tab:ablation_qmugs_r50} aggregate observable errors for these models over the full QM40 and QMugs test sets.

\keepwithfloat
\subsection{Training on QM9}
\label{app:training-on-qm9}
\vspace{-0.2cm}

\begin{table}[h]
\centering
\footnotesize
\setlength{\tabcolsep}{3pt}
\renewcommand{\arraystretch}{1.05}
\caption{\textbf{Observable errors on the QM9 test split for QM9-trained models (\cref{tab:train-dataset}).}
Each cell reports the mean over seeds of the per-molecule mean absolute error, with the seed range in square brackets.
Values in round brackets average all molecules of the lane before outlier exclusion, identical across models.
Per column, the best bracketed mean and the best single seed are underlined.
Predictions with DIIS residual $r_\mathrm{DIIS}/\sqrt{N_\mathrm{elec}} > 30$~mHa are excluded as outliers, where $r_\mathrm{DIIS} = \sqrt{8}\,\|\mathbf{S}^{-1/2}(\mathbf{F}\mathbf{P}\mathbf{S} - \mathbf{S}\mathbf{P}\mathbf{F})\mathbf{S}^{-1/2}\|_F$ for scored density $\mathbf{P}$ and Fock matrix $\mathbf{F} = \mathbf{F}[\mathbf{P}]$.
Unbracketed values average kept molecules only, and Excluded reports the lost molecule percentage.
Models losing all molecules display a dash before the bracketed value.
$\hat E_\mathrm{KS}$ is the KS-DFT energy evaluated at the predicted density (\cref{sec:experiments}), and $E^*$ is the converged reference energy.
All observables evaluate 2000 molecules before outlier exclusion.
ERIC reports SCF iterations relative to the MINAO guess, averaged over the full lane with unconverged runs capped at 2.
LR and WD denote peak learning rate and weight decay from the training configuration.
$^{\dagger}$ Single seed without range.}
\label{tab:ablation_qm9_test2000}
\resizebox{\linewidth}{!}{%
\begin{tabular}{lrrrrrrrrr}
\toprule
Model & LR & WD & Excluded & Self-consistent & Dipole & $L_1$ & Forces & HOMO--LUMO & ERIC \\
 &  &  & {\scriptsize (\%)} & energy~{\scriptsize (mHa)} & {\scriptsize (D)} & {\scriptsize (e)} & {\scriptsize (mHa/\AA)} & {\scriptsize (mHa)} & {\scriptsize (ratio)} \\
 &  &  &  & {\scriptsize $|\hat E_\mathrm{KS} - E^*|$} & {\scriptsize $\|\Delta\boldsymbol\mu\|_2$} & {\scriptsize $\|\Delta\rho\|_{L^1}$} & {\scriptsize $\langle\|\Delta\mathbf{F}_a\|_2\rangle_a$} & {\scriptsize $|\Delta\varepsilon_\mathrm{gap}|$} &  \\
\midrule
\multicolumn{10}{l}{\textit{Density-space models}} \\
$\mathbf{P}_\mathrm{abs.}$ (NON) & $3{\times}10^{-3}$ & $10^{-2}$ & \begin{tabular}[t]{@{}r@{}}\phantom{0}62.9\phantom{000}\\{\scriptsize[62, 66]}\end{tabular} & \begin{tabular}[t]{@{}r@{}}\phantom{0}9.65\phantom{000}\,\phantom{0}(12.1)\phantom{0000}\\{\scriptsize[9.1, 10]}\\{\scriptsize(12, 13)}\end{tabular} & \begin{tabular}[t]{@{}r@{}}\phantom{0}0.0419\phantom{0}\,(0.0553)\phantom{0}\\{\scriptsize[0.040, 0.043]}\\{\scriptsize(0.055, 0.056)}\end{tabular} & \begin{tabular}[t]{@{}r@{}}\phantom{0}0.0333\phantom{0}\,(0.0428)\phantom{0}\\{\scriptsize[0.032, 0.034]}\\{\scriptsize(0.043, 0.043)}\end{tabular} & \begin{tabular}[t]{@{}r@{}}\phantom{0}0.974\phantom{0}\,\phantom{0}(1.26)\phantom{00}\\{\scriptsize[0.94, 1.0]}\\{\scriptsize(1.3, 1.3)}\end{tabular} & \begin{tabular}[t]{@{}r@{}}\phantom{0}0.578\phantom{0}\,\phantom{0}(0.774)\phantom{0}\\{\scriptsize[0.56, 0.62]}\\{\scriptsize(0.77, 0.78)}\end{tabular} & \begin{tabular}[t]{@{}r@{}}\phantom{0}0.745\,(0.754)\\{\scriptsize[0.74, 0.75]}\\{\scriptsize(0.75, 0.75)}\end{tabular} \\
$\mathbf{P}_\mathrm{abs.}$ (NON), no rescale & $3{\times}10^{-3}$ & $10^{-2}$ & \begin{tabular}[t]{@{}r@{}}\phantom{0}61.8\phantom{000}\\{\scriptsize[60, 64]}\end{tabular} & \begin{tabular}[t]{@{}r@{}}10.4\phantom{0000}\,\phantom{0}(12.2)\phantom{0000}\\{\scriptsize[10, 11]}\\{\scriptsize(12, 13)}\end{tabular} & \begin{tabular}[t]{@{}r@{}}\phantom{0}0.0421\phantom{0}\,(0.0545)\phantom{0}\\{\scriptsize[0.042, 0.042]}\\{\scriptsize(0.054, 0.055)}\end{tabular} & \begin{tabular}[t]{@{}r@{}}\phantom{0}0.0337\phantom{0}\,(0.0427)\phantom{0}\\{\scriptsize[0.034, 0.034]}\\{\scriptsize(0.042, 0.043)}\end{tabular} & \begin{tabular}[t]{@{}r@{}}\phantom{0}1.00\phantom{00}\,\phantom{0}(1.27)\phantom{00}\\{\scriptsize[1.00, 1.0]}\\{\scriptsize(1.3, 1.3)}\end{tabular} & \begin{tabular}[t]{@{}r@{}}\phantom{0}0.570\phantom{0}\,\phantom{0}(0.759)\phantom{0}\\{\scriptsize[0.56, 0.59]}\\{\scriptsize(0.75, 0.78)}\end{tabular} & \begin{tabular}[t]{@{}r@{}}\phantom{0}0.745\,(0.754)\\{\scriptsize[0.74, 0.75]}\\{\scriptsize(0.75, 0.75)}\end{tabular} \\
$\mathbf{P}_\mathrm{base}$ (NON) & $3{\times}10^{-3}$ & $10^{-2}$ & \begin{tabular}[t]{@{}r@{}}\phantom{0}21.2\phantom{000}\\{\scriptsize[19, 25]}\end{tabular} & \begin{tabular}[t]{@{}r@{}}\phantom{0}8.47\phantom{000}\,\phantom{00}(9.32)\phantom{000}\\{\scriptsize[8.4, 8.5]}\\{\scriptsize(9.2, 9.6)}\end{tabular} & \begin{tabular}[t]{@{}r@{}}\phantom{0}0.0512\phantom{0}\,(0.0563)\phantom{0}\\{\scriptsize[0.051, 0.052]}\\{\scriptsize(0.055, 0.057)}\end{tabular} & \begin{tabular}[t]{@{}r@{}}\phantom{0}0.0399\phantom{0}\,(0.0439)\phantom{0}\\{\scriptsize[0.039, 0.040]}\\{\scriptsize(0.043, 0.045)}\end{tabular} & \begin{tabular}[t]{@{}r@{}}\phantom{0}1.19\phantom{00}\,\phantom{0}(1.31)\phantom{00}\\{\scriptsize[1.2, 1.2]}\\{\scriptsize(1.3, 1.3)}\end{tabular} & \begin{tabular}[t]{@{}r@{}}\phantom{0}0.741\phantom{0}\,\phantom{0}(0.813)\phantom{0}\\{\scriptsize[0.73, 0.77]}\\{\scriptsize(0.79, 0.86)}\end{tabular} & \begin{tabular}[t]{@{}r@{}}\phantom{0}0.751\,(0.755)\\{\scriptsize[0.75, 0.75]}\\{\scriptsize(0.75, 0.76)}\end{tabular} \\
$\mathbf{P}_\mathrm{base}$ (ON) & $3{\times}10^{-3}$ & $3{\times}10^{-3}$ & \begin{tabular}[t]{@{}r@{}}\phantom{00}0\phantom{.}\phantom{0000}\\{\scriptsize[0, 0]}\end{tabular} & \begin{tabular}[t]{@{}r@{}}\phantom{0}1.22\phantom{000}\,\phantom{00}(1.22)\phantom{000}\\{\scriptsize[1.2, 1.3]}\\{\scriptsize(1.2, 1.3)}\end{tabular} & \begin{tabular}[t]{@{}r@{}}\phantom{0}0.00634\,(0.00634)\\{\scriptsize[0.0062, 0.0064]}\\{\scriptsize(0.0062, 0.0064)}\end{tabular} & \begin{tabular}[t]{@{}r@{}}\phantom{0}0.00460\,(0.00460)\\{\scriptsize[0.0046, 0.0046]}\\{\scriptsize(0.0046, 0.0046)}\end{tabular} & \begin{tabular}[t]{@{}r@{}}\phantom{0}0.139\phantom{0}\,\phantom{0}(0.139)\phantom{0}\\{\scriptsize[0.14, 0.14]}\\{\scriptsize(0.14, 0.14)}\end{tabular} & \begin{tabular}[t]{@{}r@{}}\phantom{0}0.0798\,\phantom{0}(0.0798)\\{\scriptsize[0.079, 0.082]}\\{\scriptsize(0.079, 0.082)}\end{tabular} & \begin{tabular}[t]{@{}r@{}}\phantom{0}0.660\,(0.660)\\{\scriptsize[0.66, 0.66]}\\{\scriptsize(0.66, 0.66)}\end{tabular} \\
GROOT\textsubscript{rot} & $10^{-2}$ & $10^{-2}$ & \begin{tabular}[t]{@{}r@{}}\phantom{00}0\phantom{.}\phantom{0000}\\{\scriptsize[0, 0]}\end{tabular} & \begin{tabular}[t]{@{}r@{}}\phantom{0}0.0499\phantom{0}\,\phantom{00}(0.0499)\phantom{0}\\{\scriptsize[0.048, 0.052]}\\{\scriptsize(0.048, 0.052)}\end{tabular} & \begin{tabular}[t]{@{}r@{}}\phantom{0}0.00385\,(0.00385)\\{\scriptsize[0.0037, 0.0041]}\\{\scriptsize(0.0037, 0.0041)}\end{tabular} & \begin{tabular}[t]{@{}r@{}}\phantom{0}0.00368\,(0.00368)\\{\scriptsize[0.0036, 0.0039]}\\{\scriptsize(0.0036, 0.0039)}\end{tabular} & \begin{tabular}[t]{@{}r@{}}\phantom{0}0.0798\,\phantom{0}(0.0798)\\{\scriptsize[0.077, 0.084]}\\{\scriptsize(0.077, 0.084)}\end{tabular} & \begin{tabular}[t]{@{}r@{}}\phantom{0}0.0559\,\phantom{0}(0.0559)\\{\scriptsize[0.055, 0.056]}\\{\scriptsize(0.055, 0.056)}\end{tabular} & \begin{tabular}[t]{@{}r@{}}\phantom{0}0.623\,(0.623)\\{\scriptsize[0.62, 0.63]}\\{\scriptsize(0.62, 0.63)}\end{tabular} \\
GROOT\textsubscript{proj} & $3{\times}10^{-3}$ & $10^{-2}$ & \begin{tabular}[t]{@{}r@{}}\phantom{00}0\phantom{.}\phantom{0000}\\{\scriptsize[0, 0]}\end{tabular} & \begin{tabular}[t]{@{}r@{}}\phantom{0}0.0254\phantom{0}\,\phantom{00}(0.0254)\phantom{0}\\{\scriptsize[0.025, 0.026]}\\{\scriptsize(0.025, 0.026)}\end{tabular} & \begin{tabular}[t]{@{}r@{}}\phantom{0}0.00307\,(\underline{0.00307})\\{\scriptsize[0.0031, 0.0031]}\\{\scriptsize(\underline{0.0031}, 0.0031)}\end{tabular} & \begin{tabular}[t]{@{}r@{}}\phantom{0}0.00284\,(\underline{0.00284})\\{\scriptsize[0.0028, 0.0028]}\\{\scriptsize(\underline{0.0028}, 0.0028)}\end{tabular} & \begin{tabular}[t]{@{}r@{}}\phantom{0}0.0622\,\phantom{0}(\underline{0.0622})\\{\scriptsize[0.062, 0.062]}\\{\scriptsize(\underline{0.062}, 0.062)}\end{tabular} & \begin{tabular}[t]{@{}r@{}}\phantom{0}0.0442\,\phantom{0}(\underline{0.0442})\\{\scriptsize[0.044, 0.044]}\\{\scriptsize(\underline{0.044}, 0.044)}\end{tabular} & \begin{tabular}[t]{@{}r@{}}\phantom{0}0.605\,(\underline{0.605})\\{\scriptsize[0.60, 0.61]}\\{\scriptsize(\underline{0.60}, 0.61)}\end{tabular} \\
\midrule
\multicolumn{10}{l}{\textit{Hamiltonian-space models}} \\
$\mathbf{H}_\mathrm{abs.}$ (NON){}$^{\dagger}$ & $10^{-4}$ & $10^{-3}$ & \phantom{00}3.85\phantom{00} & \phantom{0}0.995\phantom{00}\,\phantom{0}(30.3)\phantom{0000} & \phantom{0}0.0208\phantom{0}\,(0.484)\phantom{00} & \phantom{0}0.0168\phantom{0}\,(0.0832)\phantom{0} & \phantom{0}0.381\phantom{0}\,\phantom{0}(3.18)\phantom{00} & \phantom{0}0.263\phantom{0}\,\phantom{0}(2.32)\phantom{00} & \phantom{0}0.703\,(0.714) \\
$\mathbf{H}_\mathrm{base}$ (NON){}$^{\dagger}$ & $3{\times}10^{-5}$ & $10^{-3}$ & \phantom{00}1.15\phantom{00} & \phantom{0}0.635\phantom{00}\,\phantom{00}(3.92)\phantom{000} & \phantom{0}0.0226\phantom{0}\,(0.0337)\phantom{0} & \phantom{0}0.0190\phantom{0}\,(0.0249)\phantom{0} & \phantom{0}0.427\phantom{0}\,\phantom{0}(0.632)\phantom{0} & \phantom{0}0.331\phantom{0}\,\phantom{0}(0.540)\phantom{0} & \phantom{0}0.709\,(0.711) \\
$\mathbf{H}_\mathrm{base}$ (NON), QM9 best{}$^{\dagger}$ & $10^{-3}$ & $10^{-3}$ & \phantom{00}0.0500 & \phantom{0}0.107\phantom{00}\,\phantom{00}(0.112)\phantom{00} & \phantom{0}0.00894\,(0.00897) & \phantom{0}0.00826\,(0.00829) & \phantom{0}0.179\phantom{0}\,\phantom{0}(0.180)\phantom{0} & \phantom{0}0.127\phantom{0}\,\phantom{0}(0.128)\phantom{0} & \phantom{0}0.667\,(0.667) \\
$\mathbf{H}_\mathrm{base}$ (NON), WALoss{}$^{\dagger}$ & $10^{-4}$ & $10^{-3}$ & 100\phantom{.}\phantom{0000} & --\phantom{.}\phantom{00000}\,(852)\phantom{.}\phantom{00000} & --\phantom{.}\phantom{00000}\,(0.461)\phantom{00} & --\phantom{.}\phantom{00000}\,(2.14)\phantom{000} & --\phantom{.}\phantom{0000}\,(44.8)\phantom{000} & --\phantom{.}\phantom{0000}\,(26.4)\phantom{000} & --\phantom{.}\phantom{000}\,(1.06)\phantom{0} \\
$\mathbf{H}_\mathrm{base}$ (ON) & $3{\times}10^{-5}$ & $3{\times}10^{-3}$ & \begin{tabular}[t]{@{}r@{}}\phantom{00}0\phantom{.}\phantom{0000}\\{\scriptsize[0, 0]}\end{tabular} & \begin{tabular}[t]{@{}r@{}}\phantom{0}0.146\phantom{00}\,\phantom{00}(0.146)\phantom{00}\\{\scriptsize[0.14, 0.15]}\\{\scriptsize(0.14, 0.15)}\end{tabular} & \begin{tabular}[t]{@{}r@{}}\phantom{0}0.0157\phantom{0}\,(0.0157)\phantom{0}\\{\scriptsize[0.015, 0.016]}\\{\scriptsize(0.015, 0.016)}\end{tabular} & \begin{tabular}[t]{@{}r@{}}\phantom{0}0.0154\phantom{0}\,(0.0154)\phantom{0}\\{\scriptsize[0.015, 0.016]}\\{\scriptsize(0.015, 0.016)}\end{tabular} & \begin{tabular}[t]{@{}r@{}}\phantom{0}0.326\phantom{0}\,\phantom{0}(0.326)\phantom{0}\\{\scriptsize[0.32, 0.33]}\\{\scriptsize(0.32, 0.33)}\end{tabular} & \begin{tabular}[t]{@{}r@{}}\phantom{0}0.238\phantom{0}\,\phantom{0}(0.238)\phantom{0}\\{\scriptsize[0.23, 0.24]}\\{\scriptsize(0.23, 0.24)}\end{tabular} & \begin{tabular}[t]{@{}r@{}}\phantom{0}0.690\,(0.690)\\{\scriptsize[0.69, 0.69]}\\{\scriptsize(0.69, 0.69)}\end{tabular} \\
$\mathbf{H}_\mathrm{base} + c_t \mathbf{S}$ (NON) & $10^{-5}$ & $10^{-3}$ & \begin{tabular}[t]{@{}r@{}}\phantom{00}5.76\phantom{00}\\{\scriptsize[2.4, 10]}\end{tabular} & \begin{tabular}[t]{@{}r@{}}\phantom{0}1.05\phantom{000}\,\phantom{0}(52.0)\phantom{0000}\\{\scriptsize[0.86, 1.4]}\\{\scriptsize(22, 88)}\end{tabular} & \begin{tabular}[t]{@{}r@{}}\phantom{0}0.0286\phantom{0}\,(0.213)\phantom{00}\\{\scriptsize[0.028, 0.031]}\\{\scriptsize(0.074, 0.46)}\end{tabular} & \begin{tabular}[t]{@{}r@{}}\phantom{0}0.0242\phantom{0}\,(0.109)\phantom{00}\\{\scriptsize[0.023, 0.026]}\\{\scriptsize(0.057, 0.18)}\end{tabular} & \begin{tabular}[t]{@{}r@{}}\phantom{0}0.543\phantom{0}\,\phantom{0}(3.21)\phantom{00}\\{\scriptsize[0.52, 0.59]}\\{\scriptsize(1.7, 5.6)}\end{tabular} & \begin{tabular}[t]{@{}r@{}}\phantom{0}0.385\phantom{0}\,\phantom{0}(2.73)\phantom{00}\\{\scriptsize[0.37, 0.40]}\\{\scriptsize(1.2, 4.6)}\end{tabular} & \begin{tabular}[t]{@{}r@{}}\phantom{0}0.725\,(0.742)\\{\scriptsize[0.72, 0.73]}\\{\scriptsize(0.73, 0.76)}\end{tabular} \\
$\mathbf{H}_\mathrm{base} + c_t \mathbf{S}$ (NON), QM9 best{}$^{\dagger}$ & $10^{-3}$ & $10^{-3}$ & \phantom{00}0.350\phantom{0} & \phantom{0}0.0558\phantom{0}\,\phantom{00}(5.45)\phantom{000} & \phantom{0}0.00666\,(0.0258)\phantom{0} & \phantom{0}0.00599\,(0.0162)\phantom{0} & \phantom{0}0.133\phantom{0}\,\phantom{0}(0.376)\phantom{0} & \phantom{0}0.106\phantom{0}\,\phantom{0}(0.448)\phantom{0} & \phantom{0}0.651\,(0.653) \\
$\mathbf{H}_\mathrm{base} + c_t \mathbf{S}$ (NON), WALoss & $10^{-4}$ & $10^{-3}$ & \begin{tabular}[t]{@{}r@{}}100\phantom{.}\phantom{0000}\\{\scriptsize[100, 100]}\end{tabular} & \begin{tabular}[t]{@{}r@{}}--\phantom{.}\phantom{00000}\,(164)\phantom{.}\phantom{00000}\\{\scriptsize(130, 203)}\end{tabular} & \begin{tabular}[t]{@{}r@{}}--\phantom{.}\phantom{00000}\,(0.152)\phantom{00}\\{\scriptsize(0.12, 0.19)}\end{tabular} & \begin{tabular}[t]{@{}r@{}}--\phantom{.}\phantom{00000}\,(0.487)\phantom{00}\\{\scriptsize(0.39, 0.61)}\end{tabular} & \begin{tabular}[t]{@{}r@{}}--\phantom{.}\phantom{0000}\,\phantom{0}(8.54)\phantom{00}\\{\scriptsize(6.6, 11)}\end{tabular} & \begin{tabular}[t]{@{}r@{}}--\phantom{.}\phantom{0000}\,\phantom{0}(9.11)\phantom{00}\\{\scriptsize(7.4, 11)}\end{tabular} & \begin{tabular}[t]{@{}r@{}}--\phantom{.}\phantom{000}\,(0.930)\\{\scriptsize(0.91, 0.95)}\end{tabular} \\
$\mathbf{H}_\mathrm{base} + c_t \mathbf{S}$ (ON) & $3{\times}10^{-5}$ & $3{\times}10^{-3}$ & \begin{tabular}[t]{@{}r@{}}\phantom{00}0\phantom{.}\phantom{0000}\\{\scriptsize[0, 0]}\end{tabular} & \begin{tabular}[t]{@{}r@{}}\phantom{0}0.132\phantom{00}\,\phantom{00}(0.132)\phantom{00}\\{\scriptsize[0.13, 0.13]}\\{\scriptsize(0.13, 0.13)}\end{tabular} & \begin{tabular}[t]{@{}r@{}}\phantom{0}0.0145\phantom{0}\,(0.0145)\phantom{0}\\{\scriptsize[0.014, 0.015]}\\{\scriptsize(0.014, 0.015)}\end{tabular} & \begin{tabular}[t]{@{}r@{}}\phantom{0}0.0136\phantom{0}\,(0.0136)\phantom{0}\\{\scriptsize[0.013, 0.014]}\\{\scriptsize(0.013, 0.014)}\end{tabular} & \begin{tabular}[t]{@{}r@{}}\phantom{0}0.291\phantom{0}\,\phantom{0}(0.291)\phantom{0}\\{\scriptsize[0.29, 0.29]}\\{\scriptsize(0.29, 0.29)}\end{tabular} & \begin{tabular}[t]{@{}r@{}}\phantom{0}0.212\phantom{0}\,\phantom{0}(0.212)\phantom{0}\\{\scriptsize[0.20, 0.22]}\\{\scriptsize(0.20, 0.22)}\end{tabular} & \begin{tabular}[t]{@{}r@{}}\phantom{0}0.684\,(0.684)\\{\scriptsize[0.68, 0.68]}\\{\scriptsize(0.68, 0.68)}\end{tabular} \\
ROCKET 2nd stage & $3{\times}10^{-4}$ & $3{\times}10^{-3}$ & \begin{tabular}[t]{@{}r@{}}\phantom{00}0\phantom{.}\phantom{0000}\\{\scriptsize[0, 0]}\end{tabular} & \begin{tabular}[t]{@{}r@{}}\phantom{0}0.0395\phantom{0}\,\phantom{00}(0.0395)\phantom{0}\\{\scriptsize[0.039, 0.040]}\\{\scriptsize(0.039, 0.040)}\end{tabular} & \begin{tabular}[t]{@{}r@{}}\phantom{0}0.00612\,(0.00612)\\{\scriptsize[0.0059, 0.0063]}\\{\scriptsize(0.0059, 0.0063)}\end{tabular} & \begin{tabular}[t]{@{}r@{}}\phantom{0}0.00605\,(0.00605)\\{\scriptsize[0.0060, 0.0061]}\\{\scriptsize(0.0060, 0.0061)}\end{tabular} & \begin{tabular}[t]{@{}r@{}}\phantom{0}0.128\phantom{0}\,\phantom{0}(0.128)\phantom{0}\\{\scriptsize[0.13, 0.13]}\\{\scriptsize(0.13, 0.13)}\end{tabular} & \begin{tabular}[t]{@{}r@{}}\phantom{0}0.0957\,\phantom{0}(0.0957)\\{\scriptsize[0.088, 0.10]}\\{\scriptsize(0.088, 0.10)}\end{tabular} & \begin{tabular}[t]{@{}r@{}}\phantom{0}0.653\,(0.653)\\{\scriptsize[0.65, 0.65]}\\{\scriptsize(0.65, 0.65)}\end{tabular} \\
ROCKET 3rd stage & $10^{-3}$ & $3{\times}10^{-3}$ & \begin{tabular}[t]{@{}r@{}}\phantom{00}0\phantom{.}\phantom{0000}\\{\scriptsize[0, 0]}\end{tabular} & \begin{tabular}[t]{@{}r@{}}\phantom{0}0.00958\,\phantom{00}(\underline{0.00958})\\{\scriptsize[0.0092, 0.0099]}\\{\scriptsize(\underline{0.0092}, 0.0099)}\end{tabular} & \begin{tabular}[t]{@{}r@{}}\phantom{0}0.00402\,(0.00402)\\{\scriptsize[0.0040, 0.0041]}\\{\scriptsize(0.0040, 0.0041)}\end{tabular} & \begin{tabular}[t]{@{}r@{}}\phantom{0}0.00428\,(0.00428)\\{\scriptsize[0.0041, 0.0044]}\\{\scriptsize(0.0041, 0.0044)}\end{tabular} & \begin{tabular}[t]{@{}r@{}}\phantom{0}0.0898\,\phantom{0}(0.0898)\\{\scriptsize[0.088, 0.091]}\\{\scriptsize(0.088, 0.091)}\end{tabular} & \begin{tabular}[t]{@{}r@{}}\phantom{0}0.0690\,\phantom{0}(0.0690)\\{\scriptsize[0.067, 0.071]}\\{\scriptsize(0.067, 0.071)}\end{tabular} & \begin{tabular}[t]{@{}r@{}}\phantom{0}0.618\,(0.618)\\{\scriptsize[0.62, 0.62]}\\{\scriptsize(0.62, 0.62)}\end{tabular} \\
\bottomrule
\end{tabular}}
\end{table}

\clearpage

\subsubsection{Density-Prediction Plots}

\begin{figure}[h]
  \centering
  \includegraphics[width=\linewidth]{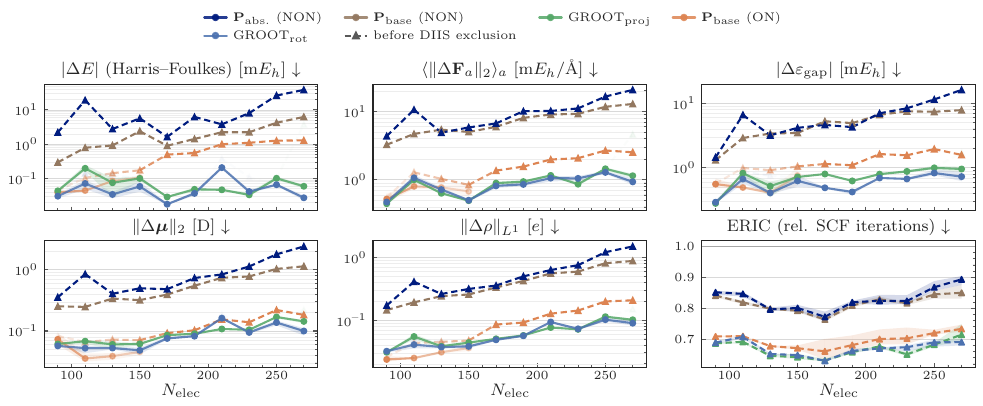}
  \caption{\textbf{Density-readout models on QM40, resolved over electron count.}
  All models use the QM9-trained checkpoint of \cref{tab:train-dataset} evaluated on the
  out-of-distribution QM40 set of \cref{tab:eval-dataset}, binned by the molecular electron
  count $N_\mathrm{elec}$ and given a one-Fock-build budget.
  We compare the absolute density readout, its unrescaled variant, the
  $\Delta\mathbf{P}_\mathrm{MINAO}$ readout in the AO basis, the same $\Delta$-readout in the
  L\"owdin basis (\oursOrthonormal{}), and the two manifold-restricted variants \ours{}$_\mathrm{proj}$
  and \ours{}$_\mathrm{rot}$ (\cref{sec:methods}).
  Lines are means over the entropy-seed replicates stated in the legend and the shaded band
  spans the seed spread, so single-seed models carry no band; the vertical axes of all error
  panels are logarithmic.
  ERIC is the effective relative iteration count~\citep{eberhardTransferableScfAcceleration2026},
  where values below one mean fewer SCF iterations than the reference initial guess.}
  \label{fig:qm40-nelec-density-ablation}
\end{figure}

\begin{figure}[!h]
  \centering
  \includegraphics[width=\linewidth]{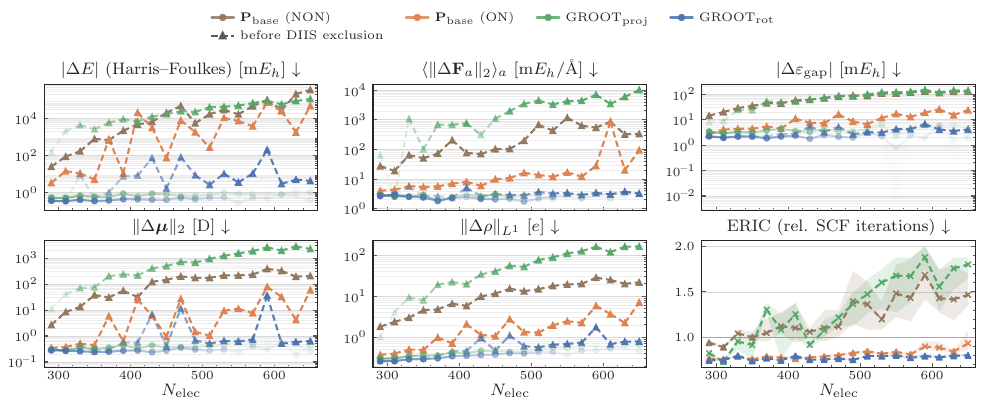}
  \caption{\textbf{Density-readout models on QMugs, resolved over electron count.}
  Same QM9-trained checkpoints, evaluation protocol, binning and aggregation as
  \cref{fig:qm40-nelec-density-ablation}, on the larger QMugs molecules
  (\cref{fig:qmugs-electron-histogram}), which reach $N_\mathrm{elec}\approx650$ against
  $\approx270$ on QM40.
  \ours{}$_\mathrm{rot}$ is the only model whose error panels stay flat over the whole range and
  it leads every error panel, by up to three orders of magnitude on the energies; together with the L\"owdin
  $\Delta$-readout it is also the only model that stays below the reference guess in ERIC, whereas
  \ours{}$_\mathrm{proj}$ rises above it from $N_\mathrm{elec}\approx320$ on and climbs towards
  the cap on the largest molecules.
  The evaluation set is the 950-molecule \texttt{qmugs\_ebins200\_cheap50} subset of \cref{tab:ablation_qmugs_r50}.}
  \label{fig:qmugs-nelec-density-ablation}
\end{figure}

\clearpage

\subsubsection{Hamiltonian-Prediction Plots}

\begin{figure}[h]
  \centering
  \includegraphics[width=\linewidth]{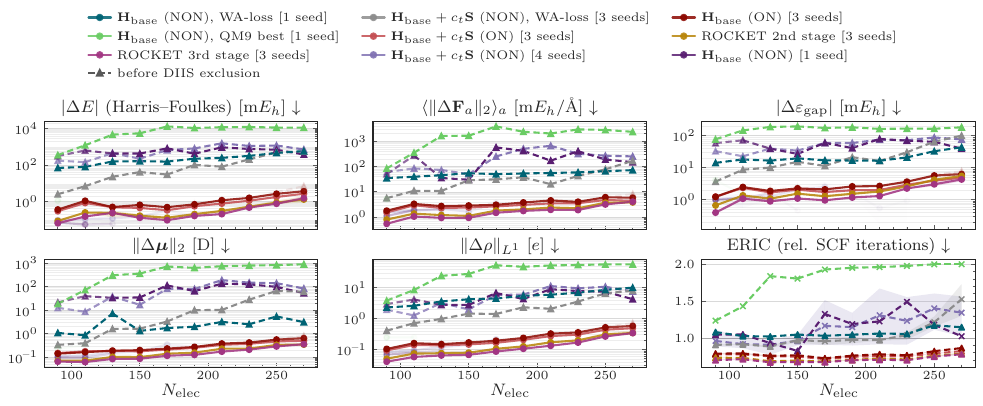}
  \caption{\textbf{Hamiltonian-readout models on QM40, resolved over electron count.}
  Same evaluation protocol, binning and aggregation as
  \cref{fig:qm40-nelec-density-ablation}, for the $\Delta\mathbf{H}_\mathrm{SAP}$ readouts.
  The models span the representation (AO versus L\"owdin basis, \cref{eq:lowdin_transform}),
  the loss (WALoss versus the \ourLoss{} objective of \cref{eq:rocket_stage_1}) and the gauge
  treatment (none, a training-set fitted constant shift, and the EMA-centered restraint of
  \cref{eq:rocket_stage_2}), with the occupied-plus-LUMO masked variant as the final model.
  The L\"owdin-basis \ourLoss{} models separate from the AO-basis models by two to three orders of
  magnitude on every energy panel, and are the only models that stay below the reference guess
  in ERIC across the whole range.}
  \label{fig:qm40-nelec-hamiltonian-ablation}
\end{figure}

\begin{figure}[!h]
  \centering
  \includegraphics[width=\linewidth]{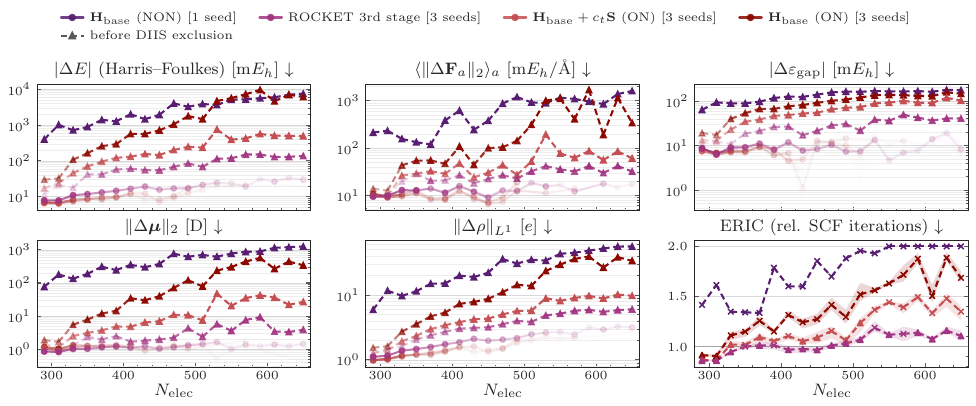}
  \caption{\textbf{Hamiltonian-readout models on QMugs, resolved over electron count.}
  Same protocol as \cref{fig:qmugs-nelec-density-ablation}, for the
  $\Delta\mathbf{H}_\mathrm{SAP}$ readouts.
  The AO-basis model sits one to two orders of magnitude above the L\"owdin-basis \ourLoss{} models
  on every panel and is pinned at the ERIC cap across the range.
  Among the \ourLoss{} models the gauge-EMA, occupied-plus-LUMO masked variant leads every panel,
  but unlike on QM40 (\cref{fig:qm40-nelec-hamiltonian-ablation}) every
  Hamiltonian model exceeds the reference guess in ERIC beyond $N_\mathrm{elec}\approx500$.}
  \label{fig:qmugs-nelec-hamiltonian-ablation}
\end{figure}

\clearpage

\keepwithfloat
\subsubsection{QM40 Extrapolation Data Tables}
\vspace{-0.2cm}

\begin{table}[h]
\centering
\footnotesize
\setlength{\tabcolsep}{3pt}
\renewcommand{\arraystretch}{1.05}
\caption{\textbf{Observable errors under size extrapolation to QM40 for QM9-trained models (\cref{tab:eval-dataset}).}
Each cell reports the mean over seeds of the per-molecule mean absolute error, with the seed range in square brackets.
Values in round brackets average all molecules of the lane before outlier exclusion, identical across models.
Per column, the best bracketed mean and the best single seed are underlined.
Predictions with DIIS residual $r_\mathrm{DIIS}/\sqrt{N_\mathrm{elec}} > 30$~mHa are excluded as outliers, where $r_\mathrm{DIIS} = \sqrt{8}\,\|\mathbf{S}^{-1/2}(\mathbf{F}\mathbf{P}\mathbf{S} - \mathbf{S}\mathbf{P}\mathbf{F})\mathbf{S}^{-1/2}\|_F$ for scored density $\mathbf{P}$ and Fock matrix $\mathbf{F} = \mathbf{F}[\mathbf{P}]$.
Unbracketed values average kept molecules only, and Excluded reports the lost molecule percentage.
Models losing all molecules display a dash before each bracketed value.
$\hat E_\mathrm{KS}$ is the KS-DFT energy evaluated at the predicted density, $\hat E$ adds the first-order self-consistency correction (\cref{sec:experiments}), and $E^*$ is the converged reference energy.
Before outlier exclusion, 1000 molecules enter Excluded, $\hat E_\mathrm{KS}$, $\hat E$, Dipole, $L_1$, and HOMO--LUMO, while 200 molecules enter Forces and ERIC.
ERIC reports SCF iterations relative to the MINAO guess, averaged over the full lane with unconverged runs capped at 2.
LR and WD denote peak learning rate and weight decay from the training configuration.
$^{\dagger}$ Single seed without range.}
\label{tab:ablation_qm40_r100}
\resizebox{\linewidth}{!}{%
\begin{tabular}{lrrrrrrrrrr}
\toprule
Model & LR & WD & Excluded & Self-consistent & First-order corrected & Dipole & $L_1$ & Forces & HOMO--LUMO & ERIC \\
 &  &  & {\scriptsize (\%)} & energy~{\scriptsize (mHa)} & energy~{\scriptsize (mHa)} & {\scriptsize (D)} & {\scriptsize (e)} & {\scriptsize (mHa/\AA)} & {\scriptsize (mHa)} & {\scriptsize (ratio)} \\
 &  &  &  & {\scriptsize $|\hat E_\mathrm{KS} - E^*|$} & {\scriptsize $|\hat E - E^*|$} & {\scriptsize $\|\Delta\boldsymbol\mu\|_2$} & {\scriptsize $\|\Delta\rho\|_{L^1}$} & {\scriptsize $\langle\|\Delta\mathbf{F}_a\|_2\rangle_a$} & {\scriptsize $|\Delta\varepsilon_\mathrm{gap}|$} &  \\
\midrule
\multicolumn{11}{l}{\textit{Density-space models}} \\
$\mathbf{P}_\mathrm{abs.}$ (NON) & $3{\times}10^{-3}$ & $10^{-2}$ & \begin{tabular}[t]{@{}r@{}}100\phantom{.}\phantom{0000}\\{\scriptsize[100, 100]}\end{tabular} & \begin{tabular}[t]{@{}r@{}}--\phantom{.}\phantom{000}\,\phantom{0}(338)\phantom{.}\phantom{000}\\{\scriptsize(245, 431)}\end{tabular} & \begin{tabular}[t]{@{}r@{}}--\phantom{.}\phantom{0000}\,\phantom{00}(11.5)\phantom{00}\\{\scriptsize(8.0, 16)}\end{tabular} & \begin{tabular}[t]{@{}r@{}}--\phantom{.}\phantom{0000}\,\phantom{00}(0.952)\phantom{0}\\{\scriptsize(0.76, 1.1)}\end{tabular} & \begin{tabular}[t]{@{}r@{}}--\phantom{.}\phantom{0000}\,\phantom{0}(0.616)\phantom{0}\\{\scriptsize(0.55, 0.67)}\end{tabular} & \begin{tabular}[t]{@{}r@{}}--\phantom{.}\phantom{000}\,\phantom{00}(10.1)\phantom{00}\\{\scriptsize(8.5, 12)}\end{tabular} & \begin{tabular}[t]{@{}r@{}}--\phantom{.}\phantom{000}\,\phantom{00}(6.72)\phantom{0}\\{\scriptsize(5.8, 7.7)}\end{tabular} & \begin{tabular}[t]{@{}r@{}}--\phantom{.}\phantom{000}\,(0.830)\\{\scriptsize(0.82, 0.84)}\end{tabular} \\
$\mathbf{P}_\mathrm{abs.}$ (NON), no rescale{}$^{\dagger}$ & $3{\times}10^{-3}$ & $10^{-2}$ & 100\phantom{.}\phantom{0000} & --\phantom{.}\phantom{000}\,\phantom{0}(175)\phantom{.}\phantom{000} & --\phantom{.}\phantom{0000}\,\phantom{000}(4.07)\phantom{0} & --\phantom{.}\phantom{0000}\,\phantom{00}(0.594)\phantom{0} & --\phantom{.}\phantom{0000}\,\phantom{0}(0.477)\phantom{0} & --\phantom{.}\phantom{000}\,\phantom{000}(8.45)\phantom{0} & --\phantom{.}\phantom{000}\,\phantom{00}(5.36)\phantom{0} & --\phantom{.}\phantom{000}\,(0.821) \\
$\mathbf{P}_\mathrm{base}$ (NON) & $3{\times}10^{-3}$ & $10^{-2}$ & \begin{tabular}[t]{@{}r@{}}100\phantom{.}\phantom{0000}\\{\scriptsize[100, 100]}\end{tabular} & \begin{tabular}[t]{@{}r@{}}--\phantom{.}\phantom{000}\,\phantom{00}(75.7)\phantom{00}\\{\scriptsize(61, 98)}\end{tabular} & \begin{tabular}[t]{@{}r@{}}--\phantom{.}\phantom{0000}\,\phantom{000}(2.17)\phantom{0}\\{\scriptsize(1.3, 3.8)}\end{tabular} & \begin{tabular}[t]{@{}r@{}}--\phantom{.}\phantom{0000}\,\phantom{00}(0.581)\phantom{0}\\{\scriptsize(0.50, 0.72)}\end{tabular} & \begin{tabular}[t]{@{}r@{}}--\phantom{.}\phantom{0000}\,\phantom{0}(0.446)\phantom{0}\\{\scriptsize(0.41, 0.52)}\end{tabular} & \begin{tabular}[t]{@{}r@{}}--\phantom{.}\phantom{000}\,\phantom{000}(7.51)\phantom{0}\\{\scriptsize(6.7, 9.0)}\end{tabular} & \begin{tabular}[t]{@{}r@{}}--\phantom{.}\phantom{000}\,\phantom{00}(5.06)\phantom{0}\\{\scriptsize(4.9, 5.3)}\end{tabular} & \begin{tabular}[t]{@{}r@{}}--\phantom{.}\phantom{000}\,(0.816)\\{\scriptsize(0.81, 0.82)}\end{tabular} \\
$\mathbf{P}_\mathrm{base}$ (NON), Aufbau-purified & $3{\times}10^{-3}$ & $10^{-2}$ & \begin{tabular}[t]{@{}r@{}}\phantom{0}96.2\phantom{000}\\{\scriptsize[96, 96]}\end{tabular} & \begin{tabular}[t]{@{}r@{}}\phantom{0}2.90\phantom{0}\,\phantom{00}(52.5)\phantom{00}\\{\scriptsize[2.9, 3.0]}\\{\scriptsize(47, 61)}\end{tabular} & \begin{tabular}[t]{@{}r@{}}\phantom{0}0.0374\,\phantom{000}(0.981)\\{\scriptsize[0.036, 0.040]}\\{\scriptsize(0.63, 1.5)}\end{tabular} & \begin{tabular}[t]{@{}r@{}}\phantom{0}0.0305\,\phantom{00}(0.255)\phantom{0}\\{\scriptsize[0.030, 0.032]}\\{\scriptsize(0.19, 0.38)}\end{tabular} & \begin{tabular}[t]{@{}r@{}}\phantom{0}0.0342\,\phantom{0}(0.217)\phantom{0}\\{\scriptsize[0.033, 0.037]}\\{\scriptsize(0.20, 0.24)}\end{tabular} & \begin{tabular}[t]{@{}r@{}}\phantom{0}0.578\,\phantom{000}(2.53)\phantom{0}\\{\scriptsize[0.56, 0.59]}\\{\scriptsize(2.3, 2.8)}\end{tabular} & \begin{tabular}[t]{@{}r@{}}\phantom{0}0.473\,\phantom{00}(1.89)\phantom{0}\\{\scriptsize[0.38, 0.56]}\\{\scriptsize(1.8, 1.9)}\end{tabular} & \begin{tabular}[t]{@{}r@{}}\phantom{0}0.722\,(0.758)\\{\scriptsize[0.71, 0.73]}\\{\scriptsize(0.75, 0.76)}\end{tabular} \\
$\mathbf{P}_\mathrm{base}$ (ON) & $3{\times}10^{-3}$ & $3{\times}10^{-3}$ & \begin{tabular}[t]{@{}r@{}}\phantom{0}75.3\phantom{000}\\{\scriptsize[72, 80]}\end{tabular} & \begin{tabular}[t]{@{}r@{}}\phantom{0}6.92\phantom{0}\,\phantom{00}(45.5)\phantom{00}\\{\scriptsize[6.0, 7.6]}\\{\scriptsize(18, 92)}\end{tabular} & \begin{tabular}[t]{@{}r@{}}\phantom{0}0.0581\,\phantom{000}(0.617)\\{\scriptsize[0.050, 0.064]}\\{\scriptsize(0.42, 0.88)}\end{tabular} & \begin{tabular}[t]{@{}r@{}}\phantom{0}0.0518\,\phantom{00}(0.118)\phantom{0}\\{\scriptsize[0.045, 0.061]}\\{\scriptsize(0.12, 0.12)}\end{tabular} & \begin{tabular}[t]{@{}r@{}}\phantom{0}0.0267\,\phantom{0}(0.102)\phantom{0}\\{\scriptsize[0.025, 0.028]}\\{\scriptsize(0.083, 0.13)}\end{tabular} & \begin{tabular}[t]{@{}r@{}}\phantom{0}0.663\,\phantom{000}(1.59)\phantom{0}\\{\scriptsize[0.61, 0.70]}\\{\scriptsize(1.3, 1.7)}\end{tabular} & \begin{tabular}[t]{@{}r@{}}\phantom{0}0.524\,\phantom{00}(1.25)\phantom{0}\\{\scriptsize[0.51, 0.55]}\\{\scriptsize(1.1, 1.6)}\end{tabular} & \begin{tabular}[t]{@{}r@{}}\phantom{0}0.694\,(0.697)\\{\scriptsize[0.69, 0.70]}\\{\scriptsize(0.69, 0.71)}\end{tabular} \\
$\mathbf{P}_\mathrm{base}$ (ON), Aufbau-purified & $3{\times}10^{-3}$ & $3{\times}10^{-3}$ & \begin{tabular}[t]{@{}r@{}}\phantom{0}66.0\phantom{000}\\{\scriptsize[63, 70]}\end{tabular} & \begin{tabular}[t]{@{}r@{}}\phantom{0}2.87\phantom{0}\,\phantom{00}(24.5)\phantom{00}\\{\scriptsize[2.8, 2.9]}\\{\scriptsize(19, 33)}\end{tabular} & \begin{tabular}[t]{@{}r@{}}\phantom{0}0.0472\,\phantom{000}(0.214)\\{\scriptsize[0.041, 0.057]}\\{\scriptsize(0.18, 0.24)}\end{tabular} & \begin{tabular}[t]{@{}r@{}}\phantom{0}0.0559\,\phantom{00}(0.113)\phantom{0}\\{\scriptsize[0.053, 0.061]}\\{\scriptsize(0.11, 0.12)}\end{tabular} & \begin{tabular}[t]{@{}r@{}}\phantom{0}0.0364\,\phantom{0}(0.113)\phantom{0}\\{\scriptsize[0.035, 0.039]}\\{\scriptsize(0.091, 0.14)}\end{tabular} & \begin{tabular}[t]{@{}r@{}}\phantom{0}0.720\,\phantom{000}(1.36)\phantom{0}\\{\scriptsize[0.58, 0.81]}\\{\scriptsize(1.2, 1.6)}\end{tabular} & \begin{tabular}[t]{@{}r@{}}\phantom{0}0.455\,\phantom{00}(1.09)\phantom{0}\\{\scriptsize[0.42, 0.50]}\\{\scriptsize(1.0, 1.2)}\end{tabular} & \begin{tabular}[t]{@{}r@{}}\phantom{0}0.691\,(0.702)\\{\scriptsize[0.69, 0.69]}\\{\scriptsize(0.69, 0.72)}\end{tabular} \\
GROOT\textsubscript{rot} & $10^{-2}$ & $10^{-2}$ & \begin{tabular}[t]{@{}r@{}}\phantom{00}1.97\phantom{00}\\{\scriptsize[1.8, 2.1]}\end{tabular} & \begin{tabular}[t]{@{}r@{}}\phantom{0}1.97\phantom{0}\,\phantom{000}(2.09)\phantom{0}\\{\scriptsize[1.9, 2.1]}\\{\scriptsize(2.0, 2.2)}\end{tabular} & \begin{tabular}[t]{@{}r@{}}\phantom{0}0.0586\,\phantom{000}(\underline{0.104})\\{\scriptsize[0.054, 0.063]}\\{\scriptsize(0.098, 0.11)}\end{tabular} & \begin{tabular}[t]{@{}r@{}}\phantom{0}0.0866\,\phantom{00}(\underline{0.0941})\\{\scriptsize[0.083, 0.088]}\\{\scriptsize(\underline{0.091}, 0.096)}\end{tabular} & \begin{tabular}[t]{@{}r@{}}\phantom{0}0.0612\,\phantom{0}(\underline{0.0652})\\{\scriptsize[0.059, 0.063]}\\{\scriptsize(0.063, 0.067)}\end{tabular} & \begin{tabular}[t]{@{}r@{}}\phantom{0}0.875\,\phantom{000}(\underline{0.889})\\{\scriptsize[0.87, 0.89]}\\{\scriptsize(\underline{0.87}, 0.91)}\end{tabular} & \begin{tabular}[t]{@{}r@{}}\phantom{0}0.585\,\phantom{00}(\underline{0.670})\\{\scriptsize[0.58, 0.60]}\\{\scriptsize(\underline{0.66}, 0.69)}\end{tabular} & \begin{tabular}[t]{@{}r@{}}\phantom{0}0.671\,(0.672)\\{\scriptsize[0.67, 0.67]}\\{\scriptsize(0.67, 0.67)}\end{tabular} \\
GROOT\textsubscript{proj} & $3{\times}10^{-3}$ & $10^{-2}$ & \begin{tabular}[t]{@{}r@{}}\phantom{00}0.633\phantom{0}\\{\scriptsize[0.30, 1.0]}\end{tabular} & \begin{tabular}[t]{@{}r@{}}\phantom{0}1.14\phantom{0}\,\phantom{000}(2.88)\phantom{0}\\{\scriptsize[1.1, 1.2]}\\{\scriptsize(1.3, 5.6)}\end{tabular} & \begin{tabular}[t]{@{}r@{}}\phantom{0}0.0730\,\phantom{000}(0.453)\\{\scriptsize[0.063, 0.082]}\\{\scriptsize(\underline{0.075}, 1.2)}\end{tabular} & \begin{tabular}[t]{@{}r@{}}\phantom{0}0.0951\,\phantom{00}(0.104)\phantom{0}\\{\scriptsize[0.092, 0.097]}\\{\scriptsize(0.094, 0.12)}\end{tabular} & \begin{tabular}[t]{@{}r@{}}\phantom{0}0.0636\,\phantom{0}(0.0692)\\{\scriptsize[0.062, 0.065]}\\{\scriptsize(\underline{0.063}, 0.078)}\end{tabular} & \begin{tabular}[t]{@{}r@{}}\phantom{0}0.902\,\phantom{000}(1.25)\phantom{0}\\{\scriptsize[0.88, 0.92]}\\{\scriptsize(0.88, 1.9)}\end{tabular} & \begin{tabular}[t]{@{}r@{}}\phantom{0}0.742\,\phantom{00}(0.831)\\{\scriptsize[0.70, 0.78]}\\{\scriptsize(0.74, 1.0)}\end{tabular} & \begin{tabular}[t]{@{}r@{}}\phantom{0}0.666\,(\underline{0.668})\\{\scriptsize[0.66, 0.67]}\\{\scriptsize(\underline{0.66}, 0.67)}\end{tabular} \\
\midrule
\multicolumn{11}{l}{\textit{Hamiltonian-space models}} \\
$\mathbf{H}_\mathrm{base}$ (NON){}$^{\dagger}$ & $3{\times}10^{-5}$ & $10^{-3}$ & \phantom{0}92.4\phantom{000} & \phantom{0}4.18\phantom{0}\,(1034)\phantom{.}\phantom{000} & \phantom{0}0.0607\,\phantom{0}(584)\phantom{.}\phantom{000} & \phantom{0}0.100\phantom{0}\,\phantom{0}(71.9)\phantom{000} & \phantom{0}0.0748\,\phantom{0}(5.23)\phantom{00} & \phantom{0}1.50\phantom{0}\,\phantom{0}(231)\phantom{.}\phantom{000} & \phantom{0}1.07\phantom{0}\,\phantom{0}(55.7)\phantom{00} & \phantom{0}0.779\,(1.13)\phantom{0} \\
$\mathbf{H}_\mathrm{base}$ (NON), QM9 best{}$^{\dagger}$ & $10^{-3}$ & $10^{-3}$ & \phantom{0}99.0\phantom{000} & \phantom{0}4.91\phantom{0}\,(9227)\phantom{.}\phantom{000} & \phantom{0}1.04\phantom{00}\,(8390)\phantom{.}\phantom{000} & \phantom{0}0.251\phantom{0}\,(537)\phantom{.}\phantom{0000} & \phantom{0}0.246\phantom{0}\,(38.8)\phantom{000} & \phantom{0}5.24\phantom{0}\,(2020)\phantom{.}\phantom{000} & \phantom{0}3.56\phantom{0}\,(166)\phantom{.}\phantom{000} & \phantom{0}0.881\,(1.81)\phantom{0} \\
$\mathbf{H}_\mathrm{base}$ (NON), WALoss{}$^{\dagger}$ & $10^{-4}$ & $10^{-3}$ & 100\phantom{.}\phantom{0000} & --\phantom{.}\phantom{000}\,(2404)\phantom{.}\phantom{000} & --\phantom{.}\phantom{0000}\,\phantom{0}(246)\phantom{.}\phantom{000} & --\phantom{.}\phantom{0000}\,\phantom{00}(2.85)\phantom{00} & --\phantom{.}\phantom{0000}\,\phantom{0}(5.43)\phantom{00} & --\phantom{.}\phantom{000}\,\phantom{00}(53.5)\phantom{00} & --\phantom{.}\phantom{000}\,\phantom{0}(21.8)\phantom{00} & --\phantom{.}\phantom{000}\,(1.07)\phantom{0} \\
$\mathbf{H}_\mathrm{base}$ (ON) & $3{\times}10^{-5}$ & $3{\times}10^{-3}$ & \begin{tabular}[t]{@{}r@{}}\phantom{00}2.40\phantom{00}\\{\scriptsize[1.9, 2.8]}\end{tabular} & \begin{tabular}[t]{@{}r@{}}\phantom{0}3.35\phantom{0}\,\phantom{000}(3.82)\phantom{0}\\{\scriptsize[3.2, 3.5]}\\{\scriptsize(3.6, 4.0)}\end{tabular} & \begin{tabular}[t]{@{}r@{}}\phantom{0}1.22\phantom{00}\,\phantom{000}(1.73)\phantom{0}\\{\scriptsize[1.1, 1.3]}\\{\scriptsize(1.5, 1.8)}\end{tabular} & \begin{tabular}[t]{@{}r@{}}\phantom{0}0.316\phantom{0}\,\phantom{00}(0.346)\phantom{0}\\{\scriptsize[0.30, 0.33]}\\{\scriptsize(0.33, 0.37)}\end{tabular} & \begin{tabular}[t]{@{}r@{}}\phantom{0}0.275\phantom{0}\,\phantom{0}(0.300)\phantom{0}\\{\scriptsize[0.27, 0.28]}\\{\scriptsize(0.29, 0.31)}\end{tabular} & \begin{tabular}[t]{@{}r@{}}\phantom{0}3.82\phantom{0}\,\phantom{000}(4.09)\phantom{0}\\{\scriptsize[3.7, 3.9]}\\{\scriptsize(3.9, 4.2)}\end{tabular} & \begin{tabular}[t]{@{}r@{}}\phantom{0}3.06\phantom{0}\,\phantom{00}(3.44)\phantom{0}\\{\scriptsize[2.9, 3.2]}\\{\scriptsize(3.3, 3.6)}\end{tabular} & \begin{tabular}[t]{@{}r@{}}\phantom{0}0.781\,(0.784)\\{\scriptsize[0.78, 0.78]}\\{\scriptsize(0.78, 0.78)}\end{tabular} \\
$\mathbf{H}_\mathrm{base} + c_t \mathbf{S}$ (NON) & $10^{-5}$ & $10^{-3}$ & \begin{tabular}[t]{@{}r@{}}\phantom{0}87.7\phantom{000}\\{\scriptsize[79, 99]}\end{tabular} & \begin{tabular}[t]{@{}r@{}}\phantom{0}5.21\phantom{0}\,(1513)\phantom{.}\phantom{000}\\{\scriptsize[4.9, 5.7]}\\{\scriptsize(641, 2551)}\end{tabular} & \begin{tabular}[t]{@{}r@{}}\phantom{0}0.123\phantom{0}\,\phantom{0}(707)\phantom{.}\phantom{000}\\{\scriptsize[0.072, 0.18]}\\{\scriptsize(257, 1141)}\end{tabular} & \begin{tabular}[t]{@{}r@{}}\phantom{0}0.145\phantom{0}\,\phantom{0}(80.0)\phantom{000}\\{\scriptsize[0.082, 0.20]}\\{\scriptsize(32, 128)}\end{tabular} & \begin{tabular}[t]{@{}r@{}}\phantom{0}0.110\phantom{0}\,\phantom{0}(5.96)\phantom{00}\\{\scriptsize[0.072, 0.15]}\\{\scriptsize(2.7, 9.4)}\end{tabular} & \begin{tabular}[t]{@{}r@{}}\phantom{0}1.83\phantom{0}\,\phantom{0}(254)\phantom{.}\phantom{000}\\{\scriptsize[1.3, 2.4]}\\{\scriptsize(94, 515)}\end{tabular} & \begin{tabular}[t]{@{}r@{}}\phantom{0}1.30\phantom{0}\,\phantom{0}(56.7)\phantom{00}\\{\scriptsize[0.80, 1.5]}\\{\scriptsize(30, 84)}\end{tabular} & \begin{tabular}[t]{@{}r@{}}\phantom{0}0.779\,(1.14)\phantom{0}\\{\scriptsize[0.76, 0.81]}\\{\scriptsize(0.98, 1.3)}\end{tabular} \\
$\mathbf{H}_\mathrm{base} + c_t \mathbf{S}$ (NON), WALoss & $10^{-4}$ & $10^{-3}$ & \begin{tabular}[t]{@{}r@{}}100\phantom{.}\phantom{0000}\\{\scriptsize[100, 100]}\end{tabular} & \begin{tabular}[t]{@{}r@{}}--\phantom{.}\phantom{000}\,\phantom{0}(974)\phantom{.}\phantom{000}\\{\scriptsize(829, 1108)}\end{tabular} & \begin{tabular}[t]{@{}r@{}}--\phantom{.}\phantom{0000}\,\phantom{0}(154)\phantom{.}\phantom{000}\\{\scriptsize(102, 181)}\end{tabular} & \begin{tabular}[t]{@{}r@{}}--\phantom{.}\phantom{0000}\,\phantom{0}(18.2)\phantom{000}\\{\scriptsize(9.8, 23)}\end{tabular} & \begin{tabular}[t]{@{}r@{}}--\phantom{.}\phantom{0000}\,\phantom{0}(2.76)\phantom{00}\\{\scriptsize(2.4, 3.1)}\end{tabular} & \begin{tabular}[t]{@{}r@{}}--\phantom{.}\phantom{000}\,\phantom{00}(43.7)\phantom{00}\\{\scriptsize(37, 52)}\end{tabular} & \begin{tabular}[t]{@{}r@{}}--\phantom{.}\phantom{000}\,\phantom{0}(28.6)\phantom{00}\\{\scriptsize(23, 34)}\end{tabular} & \begin{tabular}[t]{@{}r@{}}--\phantom{.}\phantom{000}\,(1.04)\phantom{0}\\{\scriptsize(1.0, 1.0)}\end{tabular} \\
$\mathbf{H}_\mathrm{base} + c_t \mathbf{S}$ (ON) & $3{\times}10^{-5}$ & $3{\times}10^{-3}$ & \begin{tabular}[t]{@{}r@{}}\phantom{00}0.967\phantom{0}\\{\scriptsize[0.60, 1.5]}\end{tabular} & \begin{tabular}[t]{@{}r@{}}\phantom{0}2.91\phantom{0}\,\phantom{000}(3.07)\phantom{0}\\{\scriptsize[2.9, 2.9]}\\{\scriptsize(3.0, 3.2)}\end{tabular} & \begin{tabular}[t]{@{}r@{}}\phantom{0}0.900\phantom{0}\,\phantom{000}(1.09)\phantom{0}\\{\scriptsize[0.84, 0.94]}\\{\scriptsize(0.93, 1.3)}\end{tabular} & \begin{tabular}[t]{@{}r@{}}\phantom{0}0.283\phantom{0}\,\phantom{00}(0.294)\phantom{0}\\{\scriptsize[0.27, 0.29]}\\{\scriptsize(0.28, 0.31)}\end{tabular} & \begin{tabular}[t]{@{}r@{}}\phantom{0}0.238\phantom{0}\,\phantom{0}(0.247)\phantom{0}\\{\scriptsize[0.23, 0.25]}\\{\scriptsize(0.24, 0.25)}\end{tabular} & \begin{tabular}[t]{@{}r@{}}\phantom{0}3.13\phantom{0}\,\phantom{000}(3.26)\phantom{0}\\{\scriptsize[3.0, 3.2]}\\{\scriptsize(3.1, 3.4)}\end{tabular} & \begin{tabular}[t]{@{}r@{}}\phantom{0}2.49\phantom{0}\,\phantom{00}(2.70)\phantom{0}\\{\scriptsize[2.4, 2.6]}\\{\scriptsize(2.5, 2.9)}\end{tabular} & \begin{tabular}[t]{@{}r@{}}\phantom{0}0.763\,(0.765)\\{\scriptsize[0.76, 0.76]}\\{\scriptsize(0.76, 0.77)}\end{tabular} \\
ROCKET 2nd stage & $3{\times}10^{-4}$ & $3{\times}10^{-3}$ & \begin{tabular}[t]{@{}r@{}}\phantom{00}0.0667\\{\scriptsize[0, 0.10]}\end{tabular} & \begin{tabular}[t]{@{}r@{}}\phantom{0}1.70\phantom{0}\,\phantom{000}(1.71)\phantom{0}\\{\scriptsize[1.6, 1.8]}\\{\scriptsize(1.6, 1.8)}\end{tabular} & \begin{tabular}[t]{@{}r@{}}\phantom{0}0.395\phantom{0}\,\phantom{000}(0.408)\\{\scriptsize[0.31, 0.51]}\\{\scriptsize(0.33, 0.53)}\end{tabular} & \begin{tabular}[t]{@{}r@{}}\phantom{0}0.179\phantom{0}\,\phantom{00}(0.179)\phantom{0}\\{\scriptsize[0.18, 0.18]}\\{\scriptsize(0.18, 0.18)}\end{tabular} & \begin{tabular}[t]{@{}r@{}}\phantom{0}0.153\phantom{0}\,\phantom{0}(0.154)\phantom{0}\\{\scriptsize[0.14, 0.17]}\\{\scriptsize(0.14, 0.17)}\end{tabular} & \begin{tabular}[t]{@{}r@{}}\phantom{0}2.07\phantom{0}\,\phantom{000}(2.07)\phantom{0}\\{\scriptsize[1.9, 2.3]}\\{\scriptsize(1.9, 2.3)}\end{tabular} & \begin{tabular}[t]{@{}r@{}}\phantom{0}2.17\phantom{0}\,\phantom{00}(2.19)\phantom{0}\\{\scriptsize[2.0, 2.4]}\\{\scriptsize(2.0, 2.4)}\end{tabular} & \begin{tabular}[t]{@{}r@{}}\phantom{0}0.722\,(0.722)\\{\scriptsize[0.72, 0.73]}\\{\scriptsize(0.72, 0.73)}\end{tabular} \\
ROCKET 3rd stage & $10^{-3}$ & $3{\times}10^{-3}$ & \begin{tabular}[t]{@{}r@{}}\phantom{00}0.0333\\{\scriptsize[0, 0.10]}\end{tabular} & \begin{tabular}[t]{@{}r@{}}\phantom{0}0.665\,\phantom{000}(0.671)\\{\scriptsize[0.62, 0.72]}\\{\scriptsize(\underline{0.62}, 0.73)}\end{tabular} & \begin{tabular}[t]{@{}r@{}}\phantom{0}0.369\phantom{0}\,\phantom{000}(0.377)\\{\scriptsize[0.32, 0.43]}\\{\scriptsize(0.32, 0.46)}\end{tabular} & \begin{tabular}[t]{@{}r@{}}\phantom{0}0.157\phantom{0}\,\phantom{00}(0.157)\phantom{0}\\{\scriptsize[0.15, 0.17]}\\{\scriptsize(0.15, 0.17)}\end{tabular} & \begin{tabular}[t]{@{}r@{}}\phantom{0}0.133\phantom{0}\,\phantom{0}(0.133)\phantom{0}\\{\scriptsize[0.13, 0.14]}\\{\scriptsize(0.13, 0.14)}\end{tabular} & \begin{tabular}[t]{@{}r@{}}\phantom{0}1.81\phantom{0}\,\phantom{000}(1.81)\phantom{0}\\{\scriptsize[1.7, 1.9]}\\{\scriptsize(1.7, 1.9)}\end{tabular} & \begin{tabular}[t]{@{}r@{}}\phantom{0}1.66\phantom{0}\,\phantom{00}(1.68)\phantom{0}\\{\scriptsize[1.6, 1.8]}\\{\scriptsize(1.6, 1.9)}\end{tabular} & \begin{tabular}[t]{@{}r@{}}\phantom{0}0.709\,(0.709)\\{\scriptsize[0.71, 0.71]}\\{\scriptsize(0.71, 0.71)}\end{tabular} \\
ROCKET 3rd stage, $k{=}2${}$^{\dagger}$ & $10^{-3}$ & $3{\times}10^{-3}$ & \phantom{00}0\phantom{.}\phantom{0000} & \phantom{0}0.667\,\phantom{000}(\underline{0.667}) & \phantom{0}0.298\phantom{0}\,\phantom{000}(0.298) & \phantom{0}0.148\phantom{0}\,\phantom{00}(0.148)\phantom{0} & \phantom{0}0.130\phantom{0}\,\phantom{0}(0.130)\phantom{0} & \phantom{0}1.75\phantom{0}\,\phantom{000}(1.75)\phantom{0} & \phantom{0}1.59\phantom{0}\,\phantom{00}(1.59)\phantom{0} & \phantom{0}0.706\,(0.706) \\
\bottomrule
\end{tabular}}
\end{table}

\clearpage

\keepwithfloat
\subsubsection{QMugs Extrapolation Data Tables}
\vspace{-0.2cm}

\begin{table}[h]
\centering
\footnotesize
\setlength{\tabcolsep}{3pt}
\renewcommand{\arraystretch}{1.05}
\caption{\textbf{Observable errors under size extrapolation to QMugs for QM9-trained models (\cref{tab:qmugs-dataset}).}
Each cell reports the mean over seeds of the per-molecule mean absolute error, with the seed range in square brackets.
Values in round brackets average all molecules of the lane before outlier exclusion, identical across models.
Per column, the best bracketed mean and the best single seed are underlined.
Predictions with DIIS residual $r_\mathrm{DIIS} = \sqrt{8}\,\|\mathbf{S}^{-1/2}(\mathbf{F}\mathbf{P}\mathbf{S} - \mathbf{S}\mathbf{P}\mathbf{F})\mathbf{S}^{-1/2}\|_F$ exceeding $30\sqrt{N_\mathrm{elec}}$~mHa are excluded as outliers, judged per checkpoint and molecule.
Unbracketed values average kept molecules only, and Excluded reports the lost molecule percentage.
Models losing all molecules display a dash before each bracketed value.
$\hat E_\mathrm{KS}$ is the KS-DFT energy evaluated at the predicted density (\cref{sec:experiments}), and $E^*$ is the converged reference energy.
Before outlier exclusion, 950 molecules enter Excluded, $\hat E_\mathrm{KS}$, Dipole, $L_1$, and HOMO--LUMO, while 190 molecules enter Forces and ERIC.
ERIC reports SCF iterations relative to the MINAO guess, averaged over the full lane with unconverged runs capped at 2.
LR and WD denote peak learning rate and weight decay from the training configuration.
$^{\dagger}$ Single seed without range.}
\label{tab:ablation_qmugs_r50}
\resizebox{\linewidth}{!}{%
\begin{tabular}{lrrrrrrrrr}
\toprule
Model & LR & WD & Excluded & Self-consistent & Dipole & $L_1$ & Forces & HOMO--LUMO & ERIC \\
 &  &  & {\scriptsize (\%)} & energy~{\scriptsize (mHa)} & {\scriptsize (D)} & {\scriptsize (e)} & {\scriptsize (mHa/\AA)} & {\scriptsize (mHa)} & {\scriptsize (ratio)} \\
 &  &  &  & {\scriptsize $|\hat E_\mathrm{KS} - E^*|$} & {\scriptsize $\|\Delta\boldsymbol\mu\|_2$} & {\scriptsize $\|\Delta\rho\|_{L^1}$} & {\scriptsize $\langle\|\Delta\mathbf{F}_a\|_2\rangle_a$} & {\scriptsize $|\Delta\varepsilon_\mathrm{gap}|$} &  \\
\midrule
\multicolumn{10}{l}{\textit{Density-space models}} \\
$\mathbf{P}_\mathrm{base}$ (NON) & $3{\times}10^{-3}$ & $10^{-2}$ & \begin{tabular}[t]{@{}r@{}}100\phantom{.}\phantom{0}\\{\scriptsize[100, 100]}\end{tabular} & \begin{tabular}[t]{@{}r@{}}--\phantom{.}\phantom{00}\,(46920)\phantom{.}\phantom{0}\\{\scriptsize(1214, 137230)}\end{tabular} & \begin{tabular}[t]{@{}r@{}}--\phantom{.}\phantom{000}\,(145)\phantom{.}\phantom{00}\\{\scriptsize(8.3, 369)}\end{tabular} & \begin{tabular}[t]{@{}r@{}}--\phantom{.}\phantom{000}\,(12.9)\phantom{00}\\{\scriptsize(5.5, 25)}\end{tabular} & \begin{tabular}[t]{@{}r@{}}--\phantom{.}\phantom{00}\,\phantom{0}(314)\phantom{.}\phantom{00}\\{\scriptsize(69, 756)}\end{tabular} & \begin{tabular}[t]{@{}r@{}}--\phantom{.}\phantom{00}\,\phantom{0}(76.3)\phantom{0}\\{\scriptsize(59, 104)}\end{tabular} & \begin{tabular}[t]{@{}r@{}}--\phantom{.}\phantom{000}\,(1.23)\phantom{0}\\{\scriptsize(1.1, 1.4)}\end{tabular} \\
$\mathbf{P}_\mathrm{base}$ (ON) & $3{\times}10^{-3}$ & $3{\times}10^{-3}$ & \begin{tabular}[t]{@{}r@{}}100\phantom{.}\phantom{0}\\{\scriptsize[100, 100]}\end{tabular} & \begin{tabular}[t]{@{}r@{}}--\phantom{.}\phantom{00}\,(13720)\phantom{.}\phantom{0}\\{\scriptsize(887, 31790)}\end{tabular} & \begin{tabular}[t]{@{}r@{}}--\phantom{.}\phantom{000}\,\phantom{0}(14.2)\phantom{0}\\{\scriptsize(\underline{1.0}, 31)}\end{tabular} & \begin{tabular}[t]{@{}r@{}}--\phantom{.}\phantom{000}\,\phantom{0}(2.03)\phantom{0}\\{\scriptsize(1.3, 3.2)}\end{tabular} & \begin{tabular}[t]{@{}r@{}}--\phantom{.}\phantom{00}\,\phantom{00}(60.1)\phantom{0}\\{\scriptsize(8.7, 154)}\end{tabular} & \begin{tabular}[t]{@{}r@{}}--\phantom{.}\phantom{00}\,\phantom{0}(10.8)\phantom{0}\\{\scriptsize(8.5, 14)}\end{tabular} & \begin{tabular}[t]{@{}r@{}}--\phantom{.}\phantom{000}\,(0.817)\\{\scriptsize(0.81, 0.83)}\end{tabular} \\
GROOT\textsubscript{rot} & $10^{-2}$ & $10^{-2}$ & \begin{tabular}[t]{@{}r@{}}\phantom{0}61.5\\{\scriptsize[53, 66]}\end{tabular} & \begin{tabular}[t]{@{}r@{}}10.4\phantom{0}\,\phantom{000}(\underline{36.9})\\{\scriptsize[10, 11]}\\{\scriptsize(\underline{28}, 44)}\end{tabular} & \begin{tabular}[t]{@{}r@{}}\phantom{0}0.279\,\phantom{00}(3.29)\\{\scriptsize[0.27, 0.30]}\\{\scriptsize(3.3, 3.3)}\end{tabular} & \begin{tabular}[t]{@{}r@{}}\phantom{0}0.329\,\phantom{0}(\underline{0.647})\\{\scriptsize[0.32, 0.34]}\\{\scriptsize(\underline{0.62}, 0.69)}\end{tabular} & \begin{tabular}[t]{@{}r@{}}\phantom{0}2.41\,\phantom{000}(\underline{3.27})\\{\scriptsize[2.4, 2.4]}\\{\scriptsize(\underline{3.2}, 3.4)}\end{tabular} & \begin{tabular}[t]{@{}r@{}}\phantom{0}2.17\,\phantom{00}(\underline{3.48})\\{\scriptsize[2.1, 2.3]}\\{\scriptsize(\underline{3.5}, 3.5)}\end{tabular} & \begin{tabular}[t]{@{}r@{}}\phantom{0}0.743\,(\underline{0.776})\\{\scriptsize[0.74, 0.74]}\\{\scriptsize(\underline{0.77}, 0.79)}\end{tabular} \\
GROOT\textsubscript{proj} & $3{\times}10^{-3}$ & $10^{-2}$ & \begin{tabular}[t]{@{}r@{}}\phantom{0}68.8\\{\scriptsize[59, 74]}\end{tabular} & \begin{tabular}[t]{@{}r@{}}\phantom{0}8.28\,(26318)\phantom{.}\phantom{0}\\{\scriptsize[8.0, 8.6]}\\{\scriptsize(8755, 36267)}\end{tabular} & \begin{tabular}[t]{@{}r@{}}\phantom{0}0.360\,(957)\phantom{.}\phantom{00}\\{\scriptsize[0.34, 0.37]}\\{\scriptsize(422, 1317)}\end{tabular} & \begin{tabular}[t]{@{}r@{}}\phantom{0}0.392\,(66.7)\phantom{00}\\{\scriptsize[0.39, 0.40]}\\{\scriptsize(31, 87)}\end{tabular} & \begin{tabular}[t]{@{}r@{}}\phantom{0}2.65\,(2794)\phantom{.}\phantom{00}\\{\scriptsize[2.6, 2.7]}\\{\scriptsize(1006, 3982)}\end{tabular} & \begin{tabular}[t]{@{}r@{}}\phantom{0}3.55\,\phantom{0}(79.7)\phantom{0}\\{\scriptsize[3.4, 3.6]}\\{\scriptsize(57, 92)}\end{tabular} & \begin{tabular}[t]{@{}r@{}}\phantom{0}0.743\,(1.32)\phantom{0}\\{\scriptsize[0.74, 0.75]}\\{\scriptsize(1.1, 1.4)}\end{tabular} \\
\midrule
\multicolumn{10}{l}{\textit{Hamiltonian-space models}} \\
$\mathbf{H}_\mathrm{base}$ (NON){}$^{\dagger}$ & $3{\times}10^{-5}$ & $10^{-3}$ & 100\phantom{.}\phantom{0} & --\phantom{.}\phantom{00}\,\phantom{0}(4417)\phantom{.}\phantom{0} & --\phantom{.}\phantom{000}\,(585)\phantom{.}\phantom{00} & --\phantom{.}\phantom{000}\,(29.9)\phantom{00} & --\phantom{.}\phantom{00}\,\phantom{0}(709)\phantom{.}\phantom{00} & --\phantom{.}\phantom{00}\,(137)\phantom{.}\phantom{00} & --\phantom{.}\phantom{000}\,(1.76)\phantom{0} \\
$\mathbf{H}_\mathrm{base}$ (ON) & $3{\times}10^{-5}$ & $3{\times}10^{-3}$ & \begin{tabular}[t]{@{}r@{}}\phantom{0}87.3\\{\scriptsize[86, 88]}\end{tabular} & \begin{tabular}[t]{@{}r@{}}14.7\phantom{0}\,\phantom{00}(730)\phantom{.}\phantom{0}\\{\scriptsize[14, 15]}\\{\scriptsize(529, 1000)}\end{tabular} & \begin{tabular}[t]{@{}r@{}}\phantom{0}1.16\phantom{0}\,(160)\phantom{.}\phantom{00}\\{\scriptsize[1.1, 1.2]}\\{\scriptsize(114, 228)}\end{tabular} & \begin{tabular}[t]{@{}r@{}}\phantom{0}1.18\phantom{0}\,(16.3)\phantom{00}\\{\scriptsize[1.1, 1.2]}\\{\scriptsize(13, 21)}\end{tabular} & \begin{tabular}[t]{@{}r@{}}10.3\phantom{0}\,\phantom{0}(366)\phantom{.}\phantom{00}\\{\scriptsize[10, 10]}\\{\scriptsize(236, 551)}\end{tabular} & \begin{tabular}[t]{@{}r@{}}\phantom{0}8.02\,\phantom{0}(95.9)\phantom{0}\\{\scriptsize[7.8, 8.2]}\\{\scriptsize(91, 103)}\end{tabular} & \begin{tabular}[t]{@{}r@{}}\phantom{0}0.883\,(1.39)\phantom{0}\\{\scriptsize[0.88, 0.89]}\\{\scriptsize(1.4, 1.4)}\end{tabular} \\
$\mathbf{H}_\mathrm{base} + c_t \mathbf{S}$ (ON) & $3{\times}10^{-5}$ & $3{\times}10^{-3}$ & \begin{tabular}[t]{@{}r@{}}\phantom{0}82.8\\{\scriptsize[81, 84]}\end{tabular} & \begin{tabular}[t]{@{}r@{}}14.6\phantom{0}\,\phantom{00}(164)\phantom{.}\phantom{0}\\{\scriptsize[14, 15]}\\{\scriptsize(128, 193)}\end{tabular} & \begin{tabular}[t]{@{}r@{}}\phantom{0}1.12\phantom{0}\,\phantom{0}(15.5)\phantom{0}\\{\scriptsize[1.1, 1.1]}\\{\scriptsize(9.9, 21)}\end{tabular} & \begin{tabular}[t]{@{}r@{}}\phantom{0}1.23\phantom{0}\,\phantom{0}(5.63)\phantom{0}\\{\scriptsize[1.2, 1.3]}\\{\scriptsize(4.8, 6.4)}\end{tabular} & \begin{tabular}[t]{@{}r@{}}\phantom{0}9.69\,\phantom{00}(51.8)\phantom{0}\\{\scriptsize[9.4, 10]}\\{\scriptsize(39, 69)}\end{tabular} & \begin{tabular}[t]{@{}r@{}}\phantom{0}7.62\,\phantom{0}(64.5)\phantom{0}\\{\scriptsize[7.5, 7.8]}\\{\scriptsize(55, 72)}\end{tabular} & \begin{tabular}[t]{@{}r@{}}\phantom{0}0.869\,(1.18)\phantom{0}\\{\scriptsize[0.86, 0.88]}\\{\scriptsize(1.1, 1.2)}\end{tabular} \\
ROCKET 2nd stage & $3{\times}10^{-4}$ & $3{\times}10^{-3}$ & \begin{tabular}[t]{@{}r@{}}\phantom{0}69.5\\{\scriptsize[61, 79]}\end{tabular} & \begin{tabular}[t]{@{}r@{}}16.9\phantom{0}\,\phantom{00}(103)\phantom{.}\phantom{0}\\{\scriptsize[15, 18]}\\{\scriptsize(58, 176)}\end{tabular} & \begin{tabular}[t]{@{}r@{}}\phantom{0}1.04\phantom{0}\,\phantom{00}(3.43)\\{\scriptsize[0.99, 1.1]}\\{\scriptsize(2.3, 4.5)}\end{tabular} & \begin{tabular}[t]{@{}r@{}}\phantom{0}1.49\phantom{0}\,\phantom{0}(3.65)\phantom{0}\\{\scriptsize[1.4, 1.6]}\\{\scriptsize(3.0, 4.2)}\end{tabular} & \begin{tabular}[t]{@{}r@{}}10.9\phantom{0}\,\phantom{00}(26.5)\phantom{0}\\{\scriptsize[11, 11]}\\{\scriptsize(22, 31)}\end{tabular} & \begin{tabular}[t]{@{}r@{}}\phantom{0}9.30\,\phantom{0}(33.7)\phantom{0}\\{\scriptsize[8.2, 10.0]}\\{\scriptsize(26, 40)}\end{tabular} & \begin{tabular}[t]{@{}r@{}}\phantom{0}0.879\,(1.03)\phantom{0}\\{\scriptsize[0.87, 0.89]}\\{\scriptsize(0.98, 1.1)}\end{tabular} \\
ROCKET 3rd stage & $10^{-3}$ & $3{\times}10^{-3}$ & \begin{tabular}[t]{@{}r@{}}\phantom{0}64.9\\{\scriptsize[57, 69]}\end{tabular} & \begin{tabular}[t]{@{}r@{}}15.0\phantom{0}\,\phantom{00}(152)\phantom{.}\phantom{0}\\{\scriptsize[14, 16]}\\{\scriptsize(63, 207)}\end{tabular} & \begin{tabular}[t]{@{}r@{}}\phantom{0}1.08\phantom{0}\,\phantom{00}(\underline{3.22})\\{\scriptsize[1.0, 1.1]}\\{\scriptsize(2.4, 4.2)}\end{tabular} & \begin{tabular}[t]{@{}r@{}}\phantom{0}1.71\phantom{0}\,\phantom{0}(3.76)\phantom{0}\\{\scriptsize[1.6, 1.8]}\\{\scriptsize(3.4, 4.0)}\end{tabular} & \begin{tabular}[t]{@{}r@{}}12.1\phantom{0}\,\phantom{00}(27.0)\phantom{0}\\{\scriptsize[12, 12]}\\{\scriptsize(25, 29)}\end{tabular} & \begin{tabular}[t]{@{}r@{}}\phantom{0}8.96\,\phantom{0}(28.1)\phantom{0}\\{\scriptsize[8.8, 9.1]}\\{\scriptsize(24, 33)}\end{tabular} & \begin{tabular}[t]{@{}r@{}}\phantom{0}0.891\,(1.04)\phantom{0}\\{\scriptsize[0.89, 0.89]}\\{\scriptsize(1.0, 1.1)}\end{tabular} \\
ROCKET 3rd stage, $k{=}2${}$^{\dagger}$ & $10^{-3}$ & $3{\times}10^{-3}$ & \phantom{0}64.0 & 15.3\phantom{0}\,\phantom{000}(85.9) & \phantom{0}1.06\phantom{0}\,\phantom{00}(3.28) & \phantom{0}1.75\phantom{0}\,\phantom{0}(3.76)\phantom{0} & 12.2\phantom{0}\,\phantom{00}(27.1)\phantom{0} & \phantom{0}9.10\,\phantom{0}(28.4)\phantom{0} & \phantom{0}0.890\,(1.02)\phantom{0} \\
\bottomrule
\end{tabular}}
\end{table}

\clearpage

\subsection{Training on QMugs}

\begin{figure}[h]
  \centering
  \includegraphics[width=\linewidth]{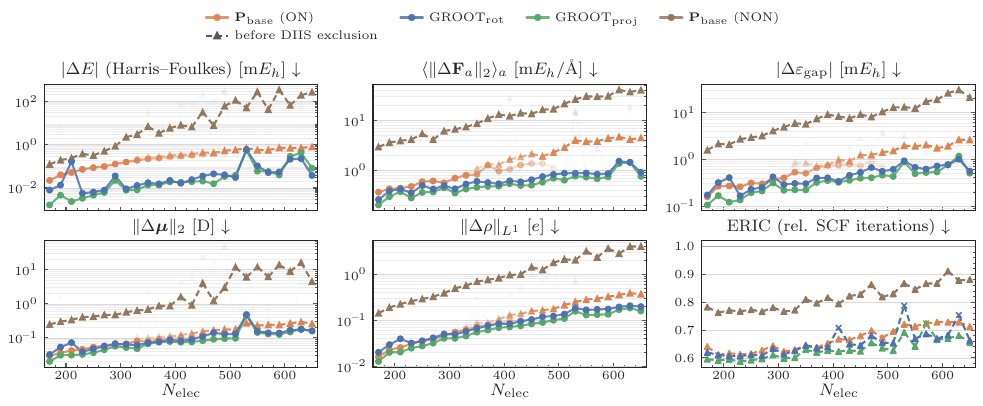}
  \caption{\textbf{Density-readout models on the QMugs out-of-distribution (OOD) test set, resolved over electron count.}
  QMugs-trained 24-epoch checkpoints evaluated under the dataset and protocol of \cref{tab:ablation_qmugs_ood160_r100}, using 100 molecules per 20-electron bin over $N_\mathrm{elec}\in[160,660)$, with forces and ERIC evaluated on the 500-molecule force subset.
  The models evaluate the $\Delta\mathbf{P}_\mathrm{MINAO}$ readout under $\mathcal{L}_\mathrm{NON}$ and
  under \oursOrthonormal{}, and the two manifold-restricted variants \ours{}$_\mathrm{rot}$ and
  \ours{}$_\mathrm{proj}$ (\cref{sec:methods}).
  Markers are the per-bin mean absolute error of each model's single seed and every error panel uses a
  logarithmic vertical axis; \cref{tab:ablation_qmugs_ood160_r100} reports the corresponding
  dataset-level means.
  ERIC is the effective relative iteration count~\citep{eberhardTransferableScfAcceleration2026},
  where values below one mean fewer SCF iterations than the reference initial guess.}
  \label{fig:qmugs-ood160-nelec-density-ablation}
\end{figure}

\begin{figure}[h]
  \centering
  \includegraphics[width=\linewidth]{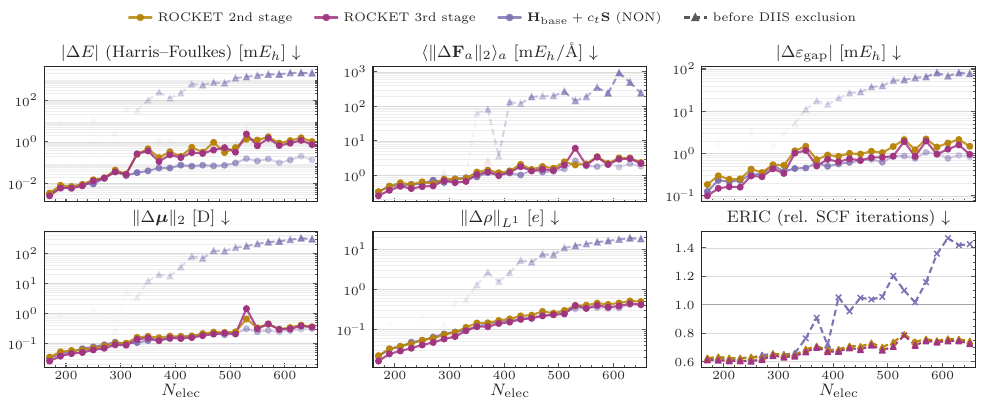}
  \caption{\textbf{Hamiltonian-readout models on the QMugs OOD test set, resolved over electron count.}
  Same checkpoints, dataset, protocol, binning, and aggregation as \cref{fig:qmugs-ood160-nelec-density-ablation} for the $\Delta\mathbf{H}_\mathrm{SAP}$ readouts, comprising the AO-basis readout with a training-set fitted constant gauge under $\mathcal{L}_\mathrm{NON}$ and the two \ourLoss{} stages of \cref{eq:rocket_stage_2,eq:rocket_stage_3}.
  \Cref{tab:ablation_qmugs_ood160_r100} reports the corresponding dataset-level means.}
  \label{fig:qmugs-ood160-nelec-hamiltonian-ablation}
\end{figure}

\keepwithfloat[0.52]
\subsubsection{QMugs Size-Extrapolation Data Tables}
\vspace{-0.2cm}

\begin{table}[h]
\centering
\footnotesize
\setlength{\tabcolsep}{3pt}
\renewcommand{\arraystretch}{1.05}
\caption{\textbf{Observable errors under size extrapolation to QMugs for QMugs-trained models (\cref{tab:qmugs-dataset}).}
Each cell reports the per-molecule mean absolute error of a single seed.
Values in round brackets average all molecules of the lane before outlier exclusion, identical across models.
Per column, the best bracketed value within each training group is underlined.
Predictions with DIIS residual $r_\mathrm{DIIS} = \sqrt{8}\,\|\mathbf{S}^{-1/2}(\mathbf{F}\mathbf{P}\mathbf{S} - \mathbf{S}\mathbf{P}\mathbf{F})\mathbf{S}^{-1/2}\|_F$ exceeding $30\sqrt{N_\mathrm{elec}}$~mHa are excluded as outliers, judged per checkpoint and molecule.
Unbracketed values average kept molecules only, and Excluded reports the lost molecule percentage.
$\hat E_\mathrm{KS}$ is the KS-DFT energy evaluated at the predicted density, $\hat E$ adds the first-order self-consistency correction (\cref{sec:experiments}), and $E^*$ is the converged reference energy.
Before outlier exclusion, 2500 molecules enter Excluded, $\hat E_\mathrm{KS}$, $\hat E$, Dipole, $L_1$, and HOMO--LUMO, while 500 molecules enter Forces and ERIC.
ERIC reports SCF iterations relative to the MINAO guess, averaged over the full lane with unconverged runs capped at 2.
LR and WD denote peak learning rate and weight decay from the training configuration.}
\label{tab:ablation_qmugs_ood160_r100}
\resizebox{\linewidth}{!}{%
\begin{tabular}{lrrrrrrrrrr}
\toprule
Model & LR & WD & Excluded & Self-consistent & First-order corrected & Dipole & $L_1$ & Forces & HOMO--LUMO & ERIC \\
 &  &  & {\scriptsize (\%)} & energy~{\scriptsize (mHa)} & energy~{\scriptsize (mHa)} & {\scriptsize (D)} & {\scriptsize (e)} & {\scriptsize (mHa/\AA)} & {\scriptsize (mHa)} & {\scriptsize (ratio)} \\
 &  &  &  & {\scriptsize $|\hat E_\mathrm{KS} - E^*|$} & {\scriptsize $|\hat E - E^*|$} & {\scriptsize $\|\Delta\boldsymbol\mu\|_2$} & {\scriptsize $\|\Delta\rho\|_{L^1}$} & {\scriptsize $\langle\|\Delta\mathbf{F}_a\|_2\rangle_a$} & {\scriptsize $|\Delta\varepsilon_\mathrm{gap}|$} &  \\
\midrule
\multicolumn{11}{l}{\textit{Supervised, 24 epochs}} \\
$\mathbf{P}_\mathrm{base}$ (NON) & $3{\times}10^{-3}$ & $10^{-2}$ & ${>}99.9$\phantom{00} & \phantom{0}4.01\,(245)\phantom{.}\phantom{00} & 0.0142\,\phantom{0}(59.0)\phantom{00} & 0.127\phantom{0}\,\phantom{0}(3.92)\phantom{0} & 0.0343\,(1.43)\phantom{00} & --\phantom{.}\phantom{000}\,\phantom{0}(16.2)\phantom{00} & 0.123\,\phantom{0}(9.80)\phantom{0} & --\phantom{.}\phantom{000}\,(0.818) \\
$\mathbf{P}_\mathrm{base}$ (ON) & $3{\times}10^{-3}$ & $3{\times}10^{-3}$ & \phantom{0}53.2\phantom{00} & 10.1\phantom{0}\,\phantom{0}(31.4)\phantom{0} & 0.140\phantom{0}\,\phantom{00}(0.365) & 0.0638\,\phantom{0}(\underline{0.143}) & 0.0495\,(0.155)\phantom{0} & \phantom{0}0.696\,\phantom{00}(2.01)\phantom{0} & 0.406\,\phantom{0}(1.10)\phantom{0} & \phantom{0}0.628\,(0.668) \\
GROOT\textsubscript{rot} & $10^{-3}$ & $10^{-2}$ & \phantom{00}1.12\phantom{0} & \phantom{0}3.22\,\phantom{00}(9.29) & 0.0712\,\phantom{0}(36.8)\phantom{00} & 0.114\phantom{0}\,\phantom{0}(4.69)\phantom{0} & 0.101\phantom{0}\,(0.313)\phantom{0} & \phantom{0}0.678\,\phantom{00}(3.11)\phantom{0} & 0.481\,\phantom{0}(1.18)\phantom{0} & \phantom{0}0.643\,(0.654) \\
GROOT\textsubscript{proj} & $10^{-3}$ & $10^{-2}$ & \phantom{00}0.360 & \phantom{0}1.73\,\phantom{00}(1.80) & 0.0699\,\phantom{00}(\underline{0.209}) & 0.0980\,\phantom{0}(0.165) & 0.0818\,(\underline{0.0884}) & \phantom{0}0.573\,\phantom{00}(\underline{0.879}) & 0.391\,\phantom{0}(\underline{0.491}) & \phantom{0}0.626\,(\underline{0.631}) \\
$\mathbf{H}_\mathrm{base} + c_t \mathbf{S}$ (NON) & $3{\times}10^{-4}$ & $10^{-3}$ & \phantom{0}29.6\phantom{00} & \phantom{0}5.33\,(963)\phantom{.}\phantom{00} & 0.0527\,(655)\phantom{.}\phantom{000} & 0.136\phantom{0}\,(96.9)\phantom{00} & 0.143\phantom{0}\,(6.20)\phantom{00} & \phantom{0}1.02\phantom{0}\,(156)\phantom{.}\phantom{000} & 0.499\,(30.2)\phantom{00} & \phantom{0}0.667\,(0.937) \\
$\mathbf{H}_\mathrm{base} + c_t \mathbf{S}$ (ON) & $3{\times}10^{-3}$ & $3{\times}10^{-3}$ & \phantom{00}2.96\phantom{0} & \phantom{0}4.85\,\phantom{0}(10.7)\phantom{0} & 1.57\phantom{00}\,\phantom{00}(8.25)\phantom{0} & 0.307\phantom{0}\,\phantom{0}(0.491) & 0.340\phantom{0}\,(0.424)\phantom{0} & \phantom{0}2.45\phantom{0}\,\phantom{00}(3.35)\phantom{0} & 1.24\phantom{0}\,\phantom{0}(2.40)\phantom{0} & \phantom{0}0.719\,(0.734) \\
ROCKET 2nd stage & $3{\times}10^{-3}$ & $3{\times}10^{-3}$ & \phantom{00}0.600 & \phantom{0}3.42\,\phantom{00}(3.70) & 0.529\phantom{0}\,\phantom{00}(0.941) & 0.212\phantom{0}\,\phantom{0}(0.256) & 0.223\phantom{0}\,(0.237)\phantom{0} & \phantom{0}1.57\phantom{0}\,\phantom{00}(1.79)\phantom{0} & 1.04\phantom{0}\,\phantom{0}(1.12)\phantom{0} & \phantom{0}0.692\,(0.695) \\
ROCKET 3rd stage & $3{\times}10^{-3}$ & $3{\times}10^{-3}$ & \phantom{00}0.200 & \phantom{0}1.10\,\phantom{00}(\underline{1.18}) & 0.445\phantom{0}\,\phantom{00}(0.557) & 0.226\phantom{0}\,\phantom{0}(0.229) & 0.189\phantom{0}\,(0.193)\phantom{0} & \phantom{0}1.56\phantom{0}\,\phantom{00}(1.61)\phantom{0} & 0.782\,\phantom{0}(0.793) & \phantom{0}0.678\,(0.679) \\
\midrule
\multicolumn{11}{l}{\textit{Supervised, continued to 72 epochs}} \\
GROOT\textsubscript{rot} & $3.3{\times}10^{-4}$ & $10^{-2}$ & \phantom{00}0.960 & \phantom{0}2.77\,\phantom{00}(8.81) & 0.0796\,\phantom{0}(36.7)\phantom{00} & 0.116\phantom{0}\,\phantom{0}(4.67)\phantom{0} & 0.108\phantom{0}\,(0.317)\phantom{0} & \phantom{0}0.697\,\phantom{00}(3.18)\phantom{0} & 0.494\,\phantom{0}(1.20)\phantom{0} & \phantom{0}0.645\,(0.658) \\
GROOT\textsubscript{proj} & $3.3{\times}10^{-4}$ & $10^{-2}$ & \phantom{00}0.200 & \phantom{0}1.52\,\phantom{00}(\underline{1.59}) & 0.0715\,\phantom{00}(\underline{0.217}) & 0.0933\,\phantom{0}(\underline{0.160}) & 0.0783\,(\underline{0.0848}) & \phantom{0}0.547\,\phantom{00}(\underline{0.853}) & 0.390\,\phantom{0}(\underline{0.487}) & \phantom{0}0.623\,(\underline{0.628}) \\
\bottomrule
\end{tabular}}
\end{table}

\clearpage

\keepwithfloat
\subsubsection{3BPA Dihedral-Scan Data Tables}
\vspace{-0.2cm}

\begin{table}[h]
\centering
\footnotesize
\setlength{\tabcolsep}{3pt}
\renewcommand{\arraystretch}{1.05}
\caption{\textbf{Observable errors under conformational transfer across three 3BPA fixed-dihedral scans for QMugs-trained models (\cref{tab:3bpa-dataset}).}
Because 3BPA samples a single molecule across a dense conformational grid rather than distinct chemical formulas, this test isolates conformational generalization from compositional variation.
Each scan constitutes a distinct block of 2347 frames ($\beta = 120^\circ$) or 2350 frames ($\beta = 150^\circ, 180^\circ$), evaluated separately without pooling.
Each cell reports the per-molecule mean absolute error of a single seed.
Values in round brackets average all frames before outlier exclusion, identical across models.
Per column and block, the best bracketed value is underlined.
$\hat E_\mathrm{KS}$ is the KS-DFT energy evaluated at the predicted density, $\hat E$ adds the first-order self-consistency correction (\cref{sec:experiments}), and $E^*$ is the converged reference energy.
Energy columns subtract the per-frame signed error median separately for each model and block.
Predictions with DIIS residual $r_\mathrm{DIIS} = \sqrt{8}\,\|\mathbf{S}^{-1/2}(\mathbf{F}\mathbf{P}\mathbf{S} - \mathbf{S}\mathbf{P}\mathbf{F})\mathbf{S}^{-1/2}\|_F$ exceeding $30\sqrt{N_\mathrm{elec}}$~mHa are excluded as outliers, judged per checkpoint and molecule.
Unbracketed values average kept molecules only, and Excluded reports the lost molecule percentage.
Models losing all molecules display a dash before each bracketed value.
LR and WD denote peak learning rate and weight decay from the training configuration.
Baseline potentials PaiNN, NequIP, and MACE train on QMugs energies alone and predict energy directly, with forces obtained by negative gradients.}
\label{tab:ablation_3bpa_dih}
\resizebox{\linewidth}{!}{%
\begin{tabular}{lrrrrrrrrr}
\toprule
Model & LR & WD & Excluded & First-order corrected & Self-consistent & Dipole & $L_1$ & Forces & HOMO--LUMO \\
 &  &  & {\scriptsize (\%)} & energy~{\scriptsize (mHa)} & energy~{\scriptsize (mHa)} & {\scriptsize (D)} & {\scriptsize (e)} & {\scriptsize (mHa/\AA)} & {\scriptsize (mHa)} \\
 &  &  &  & {\scriptsize $|\hat E - E^*|$} & {\scriptsize $|\hat E_\mathrm{KS} - E^*|$} & {\scriptsize $\|\Delta\boldsymbol\mu\|_2$} & {\scriptsize $\|\Delta\rho\|_{L^1}$} & {\scriptsize $\langle\|\Delta\mathbf{F}_a\|_2\rangle_a$} & {\scriptsize $|\Delta\varepsilon_\mathrm{gap}|$} \\
\midrule
\multicolumn{10}{l}{\textbf{$\beta = 120^\circ$}} \\
$\mathbf{P}_\mathrm{base}$ (NON) & $3{\times}10^{-3}$ & $10^{-2}$ & 100\phantom{.}\phantom{0000} & --\phantom{.}\phantom{00000}\,(0.204)\phantom{00} & --\phantom{.}\phantom{0000}\,(13.5)\phantom{000} & --\phantom{.}\phantom{0000}\,\phantom{0}(0.187)\phantom{0} & --\phantom{.}\phantom{0000}\,\phantom{0}(0.148)\phantom{0} & --\phantom{.}\phantom{00}\,\phantom{0}(4.12) & --\phantom{.}\phantom{000}\,\phantom{0}(3.32)\phantom{0} \\
$\mathbf{P}_\mathrm{base}$ (ON) & $3{\times}10^{-3}$ & $3{\times}10^{-3}$ & \phantom{00}0.0426 & \phantom{0}0.00634\,(\underline{0.00643}) & \phantom{0}4.34\phantom{00}\,\phantom{0}(4.35)\phantom{00} & \phantom{0}0.108\phantom{0}\,\phantom{0}(\underline{0.108})\phantom{0} & \phantom{0}0.0572\,\phantom{0}(\underline{0.0572}) & \phantom{0}1.27\,\phantom{0}(\underline{1.27}) & \phantom{0}0.995\,\phantom{0}(0.995) \\
$\mathbf{H}_\mathrm{base} + c_t \mathbf{S}$ (NON) & $3{\times}10^{-4}$ & $10^{-3}$ & \phantom{00}0.213\phantom{0} & \phantom{0}0.104\phantom{00}\,(0.604)\phantom{00} & \phantom{0}0.157\phantom{0}\,\phantom{0}(1.26)\phantom{00} & \phantom{0}0.235\phantom{0}\,\phantom{0}(0.285)\phantom{0} & \phantom{0}0.183\phantom{0}\,\phantom{0}(0.188)\phantom{0} & \phantom{0}4.35\,\phantom{0}(4.54) & \phantom{0}1.19\phantom{0}\,\phantom{0}(1.29)\phantom{0} \\
GROOT\textsubscript{rot} & $10^{-3}$ & $10^{-2}$ & \phantom{00}0\phantom{.}\phantom{0000} & \phantom{0}0.0159\phantom{0}\,(0.0159)\phantom{0} & \phantom{0}0.0278\,\phantom{0}(0.0278) & \phantom{0}0.139\phantom{0}\,\phantom{0}(0.139)\phantom{0} & \phantom{0}0.107\phantom{0}\,\phantom{0}(0.107)\phantom{0} & \phantom{0}2.33\,\phantom{0}(2.33) & \phantom{0}0.635\,\phantom{0}(0.635) \\
GROOT\textsubscript{proj} & $10^{-3}$ & $10^{-2}$ & \phantom{00}0\phantom{.}\phantom{0000} & \phantom{0}0.0316\phantom{0}\,(0.0316)\phantom{0} & \phantom{0}0.0209\,\phantom{0}(\underline{0.0209}) & \phantom{0}0.112\phantom{0}\,\phantom{0}(0.112)\phantom{0} & \phantom{0}0.0968\,\phantom{0}(0.0968) & \phantom{0}1.95\,\phantom{0}(1.95) & \phantom{0}0.621\,\phantom{0}(\underline{0.621}) \\
ROCKET 2nd stage & $3{\times}10^{-3}$ & $3{\times}10^{-3}$ & \phantom{00}0\phantom{.}\phantom{0000} & \phantom{0}0.125\phantom{00}\,(0.125)\phantom{00} & \phantom{0}0.0732\,\phantom{0}(0.0732) & \phantom{0}0.252\phantom{0}\,\phantom{0}(0.252)\phantom{0} & \phantom{0}0.229\phantom{0}\,\phantom{0}(0.229)\phantom{0} & \phantom{0}4.91\,\phantom{0}(4.91) & \phantom{0}1.42\phantom{0}\,\phantom{0}(1.42)\phantom{0} \\
ROCKET 3rd stage & $3{\times}10^{-3}$ & $3{\times}10^{-3}$ & \phantom{00}0\phantom{.}\phantom{0000} & \phantom{0}0.107\phantom{00}\,(0.107)\phantom{00} & \phantom{0}0.0595\,\phantom{0}(0.0595) & \phantom{0}0.222\phantom{0}\,\phantom{0}(0.222)\phantom{0} & \phantom{0}0.220\phantom{0}\,\phantom{0}(0.220)\phantom{0} & \phantom{0}4.65\,\phantom{0}(4.65) & \phantom{0}1.28\phantom{0}\,\phantom{0}(1.28)\phantom{0} \\
PaiNN & $5{\times}10^{-4}$ & $10^{-2}$ &  & \phantom{\phantom{0}1.14\phantom{000}}\,(1.14)\phantom{000} &  &  &  & \phantom{26.7\phantom{0}}\,(26.7)\phantom{0} &  \\
NequIP & $3{\times}10^{-3}$ & $5{\times}10^{-7}$ &  & \phantom{\phantom{0}1.16\phantom{000}}\,(1.16)\phantom{000} &  &  &  & \phantom{25.6\phantom{0}}\,(25.6)\phantom{0} &  \\
MACE & $3{\times}10^{-3}$ & $5{\times}10^{-7}$ &  & \phantom{\phantom{0}1.17\phantom{000}}\,(1.17)\phantom{000} &  &  &  & \phantom{33.1\phantom{0}}\,(33.1)\phantom{0} &  \\
\midrule
\multicolumn{10}{l}{\textbf{$\beta = 150^\circ$}} \\
$\mathbf{P}_\mathrm{base}$ (NON) & $3{\times}10^{-3}$ & $10^{-2}$ & 100\phantom{.}\phantom{0000} & --\phantom{.}\phantom{00000}\,(0.225)\phantom{00} & --\phantom{.}\phantom{0000}\,\phantom{0}(7.60)\phantom{00} & --\phantom{.}\phantom{0000}\,\phantom{0}(0.193)\phantom{0} & --\phantom{.}\phantom{0000}\,\phantom{0}(0.131)\phantom{0} & --\phantom{.}\phantom{00}\,\phantom{0}(3.83) & --\phantom{.}\phantom{000}\,\phantom{0}(2.92)\phantom{0} \\
$\mathbf{P}_\mathrm{base}$ (ON) & $3{\times}10^{-3}$ & $3{\times}10^{-3}$ & \phantom{00}0\phantom{.}\phantom{0000} & \phantom{0}0.00833\,(\underline{0.00833}) & \phantom{0}3.55\phantom{00}\,\phantom{0}(3.55)\phantom{00} & \phantom{0}0.0921\,\phantom{0}(\underline{0.0921}) & \phantom{0}0.0548\,\phantom{0}(\underline{0.0548}) & \phantom{0}1.18\,\phantom{0}(\underline{1.18}) & \phantom{0}0.970\,\phantom{0}(0.970) \\
$\mathbf{H}_\mathrm{base} + c_t \mathbf{S}$ (NON) & $3{\times}10^{-4}$ & $10^{-3}$ & \phantom{00}0\phantom{.}\phantom{0000} & \phantom{0}0.136\phantom{00}\,(0.136)\phantom{00} & \phantom{0}0.213\phantom{0}\,\phantom{0}(0.213)\phantom{0} & \phantom{0}0.213\phantom{0}\,\phantom{0}(0.213)\phantom{0} & \phantom{0}0.184\phantom{0}\,\phantom{0}(0.184)\phantom{0} & \phantom{0}4.30\,\phantom{0}(4.30) & \phantom{0}1.39\phantom{0}\,\phantom{0}(1.39)\phantom{0} \\
GROOT\textsubscript{rot} & $10^{-3}$ & $10^{-2}$ & \phantom{00}0\phantom{.}\phantom{0000} & \phantom{0}0.0200\phantom{0}\,(0.0200)\phantom{0} & \phantom{0}0.0395\,\phantom{0}(0.0395) & \phantom{0}0.158\phantom{0}\,\phantom{0}(0.158)\phantom{0} & \phantom{0}0.105\phantom{0}\,\phantom{0}(0.105)\phantom{0} & \phantom{0}2.43\,\phantom{0}(2.43) & \phantom{0}1.03\phantom{0}\,\phantom{0}(1.03)\phantom{0} \\
GROOT\textsubscript{proj} & $10^{-3}$ & $10^{-2}$ & \phantom{00}0\phantom{.}\phantom{0000} & \phantom{0}0.0208\phantom{0}\,(0.0208)\phantom{0} & \phantom{0}0.0223\,\phantom{0}(\underline{0.0223}) & \phantom{0}0.107\phantom{0}\,\phantom{0}(0.107)\phantom{0} & \phantom{0}0.0939\,\phantom{0}(0.0939) & \phantom{0}1.89\,\phantom{0}(1.89) & \phantom{0}0.855\,\phantom{0}(\underline{0.855}) \\
ROCKET 2nd stage & $3{\times}10^{-3}$ & $3{\times}10^{-3}$ & \phantom{00}0\phantom{.}\phantom{0000} & \phantom{0}0.162\phantom{00}\,(0.162)\phantom{00} & \phantom{0}0.0973\,\phantom{0}(0.0973) & \phantom{0}0.236\phantom{0}\,\phantom{0}(0.236)\phantom{0} & \phantom{0}0.236\phantom{0}\,\phantom{0}(0.236)\phantom{0} & \phantom{0}5.06\,\phantom{0}(5.06) & \phantom{0}1.64\phantom{0}\,\phantom{0}(1.64)\phantom{0} \\
ROCKET 3rd stage & $3{\times}10^{-3}$ & $3{\times}10^{-3}$ & \phantom{00}0\phantom{.}\phantom{0000} & \phantom{0}0.140\phantom{00}\,(0.140)\phantom{00} & \phantom{0}0.0820\,\phantom{0}(0.0820) & \phantom{0}0.215\phantom{0}\,\phantom{0}(0.215)\phantom{0} & \phantom{0}0.227\phantom{0}\,\phantom{0}(0.227)\phantom{0} & \phantom{0}4.79\,\phantom{0}(4.79) & \phantom{0}1.46\phantom{0}\,\phantom{0}(1.46)\phantom{0} \\
PaiNN & $5{\times}10^{-4}$ & $10^{-2}$ &  & \phantom{\phantom{0}1.37\phantom{000}}\,(1.37)\phantom{000} &  &  &  & \phantom{24.1\phantom{0}}\,(24.1)\phantom{0} &  \\
NequIP & $3{\times}10^{-3}$ & $5{\times}10^{-7}$ &  & \phantom{\phantom{0}1.63\phantom{000}}\,(1.63)\phantom{000} &  &  &  & \phantom{26.5\phantom{0}}\,(26.5)\phantom{0} &  \\
MACE & $3{\times}10^{-3}$ & $5{\times}10^{-7}$ &  & \phantom{\phantom{0}1.62\phantom{000}}\,(1.62)\phantom{000} &  &  &  & \phantom{33.5\phantom{0}}\,(33.5)\phantom{0} &  \\
\midrule
\multicolumn{10}{l}{\textbf{$\beta = 180^\circ$}} \\
$\mathbf{P}_\mathrm{base}$ (NON) & $3{\times}10^{-3}$ & $10^{-2}$ & 100\phantom{.}\phantom{0000} & --\phantom{.}\phantom{00000}\,(0.243)\phantom{00} & --\phantom{.}\phantom{0000}\,\phantom{0}(8.96)\phantom{00} & --\phantom{.}\phantom{0000}\,\phantom{0}(0.207)\phantom{0} & --\phantom{.}\phantom{0000}\,\phantom{0}(0.137)\phantom{0} & --\phantom{.}\phantom{00}\,\phantom{0}(4.06) & --\phantom{.}\phantom{000}\,\phantom{0}(3.17)\phantom{0} \\
$\mathbf{P}_\mathrm{base}$ (ON) & $3{\times}10^{-3}$ & $3{\times}10^{-3}$ & \phantom{00}0\phantom{.}\phantom{0000} & \phantom{0}0.00802\,(\underline{0.00802}) & \phantom{0}4.76\phantom{00}\,\phantom{0}(4.76)\phantom{00} & \phantom{0}0.0868\,\phantom{0}(0.0868) & \phantom{0}0.0538\,\phantom{0}(\underline{0.0538}) & \phantom{0}1.14\,\phantom{0}(\underline{1.14}) & \phantom{0}0.970\,\phantom{0}(0.970) \\
$\mathbf{H}_\mathrm{base} + c_t \mathbf{S}$ (NON) & $3{\times}10^{-4}$ & $10^{-3}$ & \phantom{00}0\phantom{.}\phantom{0000} & \phantom{0}0.151\phantom{00}\,(0.151)\phantom{00} & \phantom{0}0.196\phantom{0}\,\phantom{0}(0.196)\phantom{0} & \phantom{0}0.188\phantom{0}\,\phantom{0}(0.188)\phantom{0} & \phantom{0}0.181\phantom{0}\,\phantom{0}(0.181)\phantom{0} & \phantom{0}4.18\,\phantom{0}(4.18) & \phantom{0}1.37\phantom{0}\,\phantom{0}(1.37)\phantom{0} \\
GROOT\textsubscript{rot} & $10^{-3}$ & $10^{-2}$ & \phantom{00}0\phantom{.}\phantom{0000} & \phantom{0}0.0157\phantom{0}\,(0.0157)\phantom{0} & \phantom{0}0.0481\,\phantom{0}(0.0481) & \phantom{0}0.128\phantom{0}\,\phantom{0}(0.128)\phantom{0} & \phantom{0}0.0965\,\phantom{0}(0.0965) & \phantom{0}2.27\,\phantom{0}(2.27) & \phantom{0}1.05\phantom{0}\,\phantom{0}(1.05)\phantom{0} \\
GROOT\textsubscript{proj} & $10^{-3}$ & $10^{-2}$ & \phantom{00}0\phantom{.}\phantom{0000} & \phantom{0}0.0106\phantom{0}\,(0.0106)\phantom{0} & \phantom{0}0.0250\,\phantom{0}(\underline{0.0250}) & \phantom{0}0.0825\,\phantom{0}(\underline{0.0825}) & \phantom{0}0.0851\,\phantom{0}(0.0851) & \phantom{0}1.66\,\phantom{0}(1.66) & \phantom{0}0.838\,\phantom{0}(\underline{0.838}) \\
ROCKET 2nd stage & $3{\times}10^{-3}$ & $3{\times}10^{-3}$ & \phantom{00}0\phantom{.}\phantom{0000} & \phantom{0}0.154\phantom{00}\,(0.154)\phantom{00} & \phantom{0}0.0940\,\phantom{0}(0.0940) & \phantom{0}0.195\phantom{0}\,\phantom{0}(0.195)\phantom{0} & \phantom{0}0.230\phantom{0}\,\phantom{0}(0.230)\phantom{0} & \phantom{0}4.86\,\phantom{0}(4.86) & \phantom{0}1.59\phantom{0}\,\phantom{0}(1.59)\phantom{0} \\
ROCKET 3rd stage & $3{\times}10^{-3}$ & $3{\times}10^{-3}$ & \phantom{00}0\phantom{.}\phantom{0000} & \phantom{0}0.153\phantom{00}\,(0.153)\phantom{00} & \phantom{0}0.0887\,\phantom{0}(0.0887) & \phantom{0}0.189\phantom{0}\,\phantom{0}(0.189)\phantom{0} & \phantom{0}0.224\phantom{0}\,\phantom{0}(0.224)\phantom{0} & \phantom{0}4.70\,\phantom{0}(4.70) & \phantom{0}1.45\phantom{0}\,\phantom{0}(1.45)\phantom{0} \\
PaiNN & $5{\times}10^{-4}$ & $10^{-2}$ &  & \phantom{\phantom{0}1.19\phantom{000}}\,(1.19)\phantom{000} &  &  &  & \phantom{24.3\phantom{0}}\,(24.3)\phantom{0} &  \\
NequIP & $3{\times}10^{-3}$ & $5{\times}10^{-7}$ &  & \phantom{\phantom{0}1.88\phantom{000}}\,(1.88)\phantom{000} &  &  &  & \phantom{25.7\phantom{0}}\,(25.7)\phantom{0} &  \\
MACE & $3{\times}10^{-3}$ & $5{\times}10^{-7}$ &  & \phantom{\phantom{0}1.66\phantom{000}}\,(1.66)\phantom{000} &  &  &  & \phantom{33.2\phantom{0}}\,(33.2)\phantom{0} &  \\
\bottomrule
\end{tabular}}
\end{table}

\clearpage

\keepwithfloat
\subsection{Self-Consistency-Fine-Tuning on Transition1x}
\vspace{-0.2cm}

\begin{table}[h]
\centering
\footnotesize
\setlength{\tabcolsep}{3pt}
\renewcommand{\arraystretch}{1.05}
\caption{\textbf{Observable errors on the Transition1x test split across self-consistency training epochs (\cref{tab:transition1x-dataset}).}
These reactive geometries represent chemistry that the QMugs-trained parent checkpoints of \cref{tab:ablation_qmugs_ood160_r100} never encountered during pre-training.
The first block evaluates the 24-epoch parent checkpoints before self-consistency training.
Subsequent blocks evaluate models warm-started from each parent after 2, 4, and 6 epochs of occupied-virtual gradient training under a cosine schedule.
The final block evaluates models trained for 6 epochs with the residual of \citet{zhangSelfConsistencyTraining2024}.
No run uses weight decay or experience replay.
LR denotes the peak learning rate selected on the validation split by the self-consistency residual alone (\cref{tab:transition1x_sct_val_selection}).
Each cell reports the per-molecule mean absolute error of a single seed.
Values in round brackets average all 1715 frames before outlier exclusion, and the best bracketed value per column and block is underlined.
$\hat E_\mathrm{KS}$ is the KS-DFT energy at the predicted density, $\hat E$ adds the first-order self-consistency correction (\cref{sec:experiments}), and $E^*$ is the converged reference energy.
Predictions with DIIS residual $r_\mathrm{DIIS} = \sqrt{8}\,\|\mathbf{S}^{-1/2}(\mathbf{F}\mathbf{P}\mathbf{S} - \mathbf{S}\mathbf{P}\mathbf{F})\mathbf{S}^{-1/2}\|_F$ exceeding $30\sqrt{N_\mathrm{elec}}$~mHa are excluded as outliers, judged per checkpoint and molecule.
Unbracketed values average kept molecules only, and Excluded reports the lost molecule percentage.
Models losing all molecules display a dash before each unbracketed value.
ERIC reports SCF iterations relative to the MINAO guess, averaged over all frames with unconverged runs capped at 2.}
\label{tab:ablation_transition1x_sct_test}
\resizebox{\linewidth}{!}{%
\begin{tabular}{lrrrrrrrrr}
\toprule
Model & LR & Excluded & First-order corrected & Self-consistent & Dipole & $L_1$ & Forces & HOMO--LUMO & ERIC \\
 &  & {\scriptsize (\%)} & energy~{\scriptsize (mHa)} & energy~{\scriptsize (mHa)} & {\scriptsize (D)} & {\scriptsize (e)} & {\scriptsize (mHa/\AA)} & {\scriptsize (mHa)} & {\scriptsize (ratio)} \\
 &  &  & {\scriptsize $|\hat E - E^*|$} & {\scriptsize $|\hat E_\mathrm{KS} - E^*|$} & {\scriptsize $\|\Delta\boldsymbol\mu\|_2$} & {\scriptsize $\|\Delta\rho\|_{L^1}$} & {\scriptsize $\langle\|\Delta\mathbf{F}_a\|_2\rangle_a$} & {\scriptsize $|\Delta\varepsilon_\mathrm{gap}|$} &  \\
\midrule
\multicolumn{10}{l}{\textbf{Before self-consistency training (the QMugs-trained parents)}} \\
$\mathbf{P}_\mathrm{base}$ (NON) & $3{\times}10^{-3}$ & \phantom{0}95.9\phantom{00} & \phantom{0}0.0363\,(33.4)\phantom{000} & 13.3\phantom{00}\,(142)\phantom{.}\phantom{000} & \phantom{0}0.0680\,(2.09)\phantom{00} & \phantom{0}0.0396\,(0.667)\phantom{0} & \phantom{0}2.80\,(47.0)\phantom{0} & \phantom{0}0.784\,(17.8)\phantom{00} & \phantom{0}0.769\,(0.980) \\
$\mathbf{P}_\mathrm{base}$ (ON) & $3{\times}10^{-3}$ & \phantom{0}84.7\phantom{00} & \phantom{0}0.0386\,\phantom{0}(\underline{3.58})\phantom{00} & 44.1\phantom{00}\,(116)\phantom{.}\phantom{000} & \phantom{0}0.0416\,(\underline{0.805})\phantom{0} & \phantom{0}0.0323\,(\underline{0.269})\phantom{0} & \phantom{0}2.64\,(\underline{21.8})\phantom{0} & \phantom{0}0.455\,\phantom{0}(\underline{7.99})\phantom{0} & \phantom{0}0.741\,(\underline{0.896}) \\
GROOT\textsubscript{rot} & $10^{-3}$ & \phantom{0}72.4\phantom{00} & \phantom{0}0.202\phantom{0}\,(63.7)\phantom{000} & \phantom{0}0.574\,\phantom{0}(33.7)\phantom{00} & \phantom{0}0.146\phantom{0}\,(4.14)\phantom{00} & \phantom{0}0.0713\,(1.01)\phantom{00} & \phantom{0}4.37\,(66.4)\phantom{0} & \phantom{0}1.49\phantom{0}\,(23.9)\phantom{00} & \phantom{0}0.794\,(0.999) \\
GROOT\textsubscript{proj} & $10^{-3}$ & \phantom{0}64.5\phantom{00} & \phantom{0}0.259\phantom{0}\,(43.9)\phantom{000} & \phantom{0}0.511\,\phantom{0}(21.3)\phantom{00} & \phantom{0}0.153\phantom{0}\,(3.29)\phantom{00} & \phantom{0}0.0716\,(0.830)\phantom{0} & \phantom{0}4.87\,(56.9)\phantom{0} & \phantom{0}1.75\phantom{0}\,(20.7)\phantom{00} & \phantom{0}0.791\,(0.978) \\
$\mathbf{H}_\mathrm{base} + c_t \mathbf{S}$ (ON) & $3{\times}10^{-3}$ & \phantom{0}61.7\phantom{00} & \phantom{0}0.893\phantom{0}\,(46.2)\phantom{000} & \phantom{0}0.814\,\phantom{0}(22.6)\phantom{00} & \phantom{0}0.261\phantom{0}\,(3.05)\phantom{00} & \phantom{0}0.158\phantom{0}\,(0.877)\phantom{0} & \phantom{0}8.22\,(51.8)\phantom{0} & \phantom{0}2.84\phantom{0}\,(22.4)\phantom{00} & \phantom{0}0.853\,(0.998) \\
ROCKET 2nd stage & $3{\times}10^{-3}$ & \phantom{0}59.7\phantom{00} & \phantom{0}0.876\phantom{0}\,(35.7)\phantom{000} & \phantom{0}0.831\,\phantom{0}(17.2)\phantom{00} & \phantom{0}0.255\phantom{0}\,(2.78)\phantom{00} & \phantom{0}0.155\phantom{0}\,(0.789)\phantom{0} & \phantom{0}8.17\,(47.2)\phantom{0} & \phantom{0}3.02\phantom{0}\,(21.5)\phantom{00} & \phantom{0}0.850\,(0.987) \\
ROCKET 3rd stage & $3{\times}10^{-3}$ & \phantom{0}58.9\phantom{00} & \phantom{0}0.913\phantom{0}\,(33.5)\phantom{000} & \phantom{0}0.783\,\phantom{0}(\underline{16.3})\phantom{00} & \phantom{0}0.273\phantom{0}\,(2.72)\phantom{00} & \phantom{0}0.156\phantom{0}\,(0.770)\phantom{0} & \phantom{0}8.38\,(45.8)\phantom{0} & \phantom{0}2.93\phantom{0}\,(20.3)\phantom{00} & \phantom{0}0.851\,(0.985) \\
\midrule
\multicolumn{10}{l}{\textbf{After 2 epochs of self-consistency training (occupied--virtual gradient)}} \\
$\mathbf{P}_\mathrm{base}$ (NON) & $10^{-3}$ & 100\phantom{.}\phantom{000} & --\phantom{.}\phantom{0000}\,(10.7)\phantom{000} & --\phantom{.}\phantom{000}\,(393)\phantom{.}\phantom{000} & --\phantom{.}\phantom{0000}\,(1.55)\phantom{00} & --\phantom{.}\phantom{0000}\,(0.523)\phantom{0} & --\phantom{.}\phantom{00}\,(39.0)\phantom{0} & --\phantom{.}\phantom{000}\,(14.5)\phantom{00} & --\phantom{.}\phantom{000}\,(0.986) \\
$\mathbf{P}_\mathrm{base}$ (ON) & $3{\times}10^{-4}$ & \phantom{0}52.3\phantom{00} & \phantom{0}0.239\phantom{0}\,\phantom{0}(1.26)\phantom{00} & 17.3\phantom{00}\,\phantom{0}(33.0)\phantom{00} & \phantom{0}0.153\phantom{0}\,(0.587)\phantom{0} & \phantom{0}0.0565\,(0.149)\phantom{0} & \phantom{0}5.94\,(13.4)\phantom{0} & \phantom{0}2.22\phantom{0}\,\phantom{0}(5.73)\phantom{0} & \phantom{0}0.789\,(0.853) \\
GROOT\textsubscript{rot} & $10^{-3}$ & \phantom{00}5.60\phantom{0} & \phantom{0}0.491\phantom{0}\,\phantom{0}(2.66)\phantom{00} & \phantom{0}1.39\phantom{0}\,\phantom{00}(3.21)\phantom{0} & \phantom{0}0.320\phantom{0}\,(0.670)\phantom{0} & \phantom{0}0.0891\,(0.152)\phantom{0} & \phantom{0}6.28\,(10.5)\phantom{0} & \phantom{0}2.34\phantom{0}\,\phantom{0}(2.91)\phantom{0} & \phantom{0}0.809\,(0.825) \\
GROOT\textsubscript{proj} & $10^{-3}$ & \phantom{00}0.466 & \phantom{0}0.237\phantom{0}\,\phantom{0}(\underline{0.608})\phantom{0} & \phantom{0}0.450\,\phantom{00}(0.542) & \phantom{0}0.239\phantom{0}\,(0.280)\phantom{0} & \phantom{0}0.0676\,(0.0767) & \phantom{0}4.78\,\phantom{0}(5.23) & \phantom{0}1.63\phantom{0}\,\phantom{0}(1.72)\phantom{0} & \phantom{0}0.781\,(0.783) \\
$\mathbf{H}_\mathrm{base} + c_t \mathbf{S}$ (ON) & $3{\times}10^{-3}$ & \phantom{00}0.350 & \phantom{0}0.254\phantom{0}\,\phantom{0}(0.945)\phantom{0} & \phantom{0}0.358\,\phantom{00}(0.505) & \phantom{0}0.238\phantom{0}\,(0.283)\phantom{0} & \phantom{0}0.0702\,(0.0806) & \phantom{0}4.91\,\phantom{0}(5.39) & \phantom{0}1.65\phantom{0}\,\phantom{0}(1.76)\phantom{0} & \phantom{0}0.787\,(0.788) \\
ROCKET 2nd stage & $3{\times}10^{-3}$ & \phantom{00}0.350 & \phantom{0}0.220\phantom{0}\,\phantom{0}(1.21)\phantom{00} & \phantom{0}0.319\,\phantom{00}(0.500) & \phantom{0}0.219\phantom{0}\,(0.264)\phantom{0} & \phantom{0}0.0663\,(0.0769) & \phantom{0}4.55\,\phantom{0}(5.04) & \phantom{0}1.52\phantom{0}\,\phantom{0}(1.62)\phantom{0} & \phantom{0}0.783\,(0.784) \\
ROCKET 3rd stage & $3{\times}10^{-3}$ & \phantom{00}0.292 & \phantom{0}0.180\phantom{0}\,\phantom{0}(0.899)\phantom{0} & \phantom{0}0.293\,\phantom{00}(\underline{0.439}) & \phantom{0}0.194\phantom{0}\,(\underline{0.233})\phantom{0} & \phantom{0}0.0607\,(\underline{0.0698}) & \phantom{0}4.13\,\phantom{0}(\underline{4.54}) & \phantom{0}1.43\phantom{0}\,\phantom{0}(\underline{1.51})\phantom{0} & \phantom{0}0.777\,(\underline{0.777}) \\
\midrule
\multicolumn{10}{l}{\textbf{After 4 epochs of self-consistency training (occupied--virtual gradient)}} \\
$\mathbf{P}_\mathrm{base}$ (NON) & $10^{-3}$ & 100\phantom{.}\phantom{000} & --\phantom{.}\phantom{0000}\,\phantom{0}(7.21)\phantom{00} & --\phantom{.}\phantom{000}\,(315)\phantom{.}\phantom{000} & --\phantom{.}\phantom{0000}\,(1.34)\phantom{00} & --\phantom{.}\phantom{0000}\,(0.409)\phantom{0} & --\phantom{.}\phantom{00}\,(33.2)\phantom{0} & --\phantom{.}\phantom{000}\,(12.5)\phantom{00} & --\phantom{.}\phantom{000}\,(0.961) \\
$\mathbf{P}_\mathrm{base}$ (ON) & $3{\times}10^{-4}$ & \phantom{0}36.8\phantom{00} & \phantom{0}0.232\phantom{0}\,\phantom{0}(0.803)\phantom{0} & 16.0\phantom{00}\,\phantom{0}(26.9)\phantom{00} & \phantom{0}0.166\phantom{0}\,(0.490)\phantom{0} & \phantom{0}0.0545\,(0.120)\phantom{0} & \phantom{0}5.63\,(10.9)\phantom{0} & \phantom{0}2.52\phantom{0}\,\phantom{0}(4.66)\phantom{0} & \phantom{0}0.784\,(0.833) \\
GROOT\textsubscript{rot} & $10^{-3}$ & \phantom{00}3.21\phantom{0} & \phantom{0}0.418\phantom{0}\,\phantom{0}(2.12)\phantom{00} & \phantom{0}1.12\phantom{0}\,\phantom{00}(2.50)\phantom{0} & \phantom{0}0.303\phantom{0}\,(0.577)\phantom{0} & \phantom{0}0.0784\,(0.129)\phantom{0} & \phantom{0}5.52\,\phantom{0}(8.84) & \phantom{0}1.88\phantom{0}\,\phantom{0}(2.32)\phantom{0} & \phantom{0}0.791\,(0.802) \\
GROOT\textsubscript{proj} & $10^{-3}$ & \phantom{00}0.233 & \phantom{0}0.170\phantom{0}\,\phantom{0}(\underline{0.630})\phantom{0} & \phantom{0}0.313\,\phantom{00}(0.400) & \phantom{0}0.191\phantom{0}\,(0.230)\phantom{0} & \phantom{0}0.0541\,(0.0626) & \phantom{0}3.79\,\phantom{0}(4.20) & \phantom{0}1.29\phantom{0}\,\phantom{0}(1.35)\phantom{0} & \phantom{0}0.760\,(0.762) \\
$\mathbf{H}_\mathrm{base} + c_t \mathbf{S}$ (ON) & $3{\times}10^{-3}$ & \phantom{00}0.350 & \phantom{0}0.166\phantom{0}\,\phantom{0}(0.687)\phantom{0} & \phantom{0}0.223\,\phantom{00}(\underline{0.321}) & \phantom{0}0.184\phantom{0}\,(0.238)\phantom{0} & \phantom{0}0.0544\,(0.0651) & \phantom{0}3.74\,\phantom{0}(4.30) & \phantom{0}1.29\phantom{0}\,\phantom{0}(1.35)\phantom{0} & \phantom{0}0.764\,(0.765) \\
ROCKET 2nd stage & $3{\times}10^{-3}$ & \phantom{00}0.175 & \phantom{0}0.156\phantom{0}\,\phantom{0}(0.955)\phantom{0} & \phantom{0}0.207\,\phantom{00}(0.342) & \phantom{0}0.174\phantom{0}\,(\underline{0.211})\phantom{0} & \phantom{0}0.0524\,(0.0608) & \phantom{0}3.56\,\phantom{0}(3.94) & \phantom{0}1.26\phantom{0}\,\phantom{0}(1.30)\phantom{0} & \phantom{0}0.760\,(0.760) \\
ROCKET 3rd stage & $3{\times}10^{-3}$ & \phantom{00}0.175 & \phantom{0}0.161\phantom{0}\,\phantom{0}(0.821)\phantom{0} & \phantom{0}0.202\,\phantom{00}(0.323) & \phantom{0}0.177\phantom{0}\,(0.212)\phantom{0} & \phantom{0}0.0513\,(\underline{0.0595}) & \phantom{0}3.52\,\phantom{0}(\underline{3.88}) & \phantom{0}1.21\phantom{0}\,\phantom{0}(\underline{1.26})\phantom{0} & \phantom{0}0.756\,(\underline{0.758}) \\
\midrule
\multicolumn{10}{l}{\textbf{After 6 epochs of self-consistency training (occupied--virtual gradient)}} \\
$\mathbf{P}_\mathrm{base}$ (NON) & $10^{-3}$ & 100\phantom{.}\phantom{000} & --\phantom{.}\phantom{0000}\,\phantom{0}(5.52)\phantom{00} & --\phantom{.}\phantom{000}\,(258)\phantom{.}\phantom{000} & --\phantom{.}\phantom{0000}\,(1.19)\phantom{00} & --\phantom{.}\phantom{0000}\,(0.347)\phantom{0} & --\phantom{.}\phantom{00}\,(29.5)\phantom{0} & --\phantom{.}\phantom{000}\,(11.1)\phantom{00} & --\phantom{.}\phantom{000}\,(0.942) \\
$\mathbf{P}_\mathrm{base}$ (ON) & $3{\times}10^{-4}$ & \phantom{0}25.0\phantom{00} & \phantom{0}0.216\phantom{0}\,\phantom{0}(\underline{0.520})\phantom{0} & 14.5\phantom{00}\,\phantom{0}(21.1)\phantom{00} & \phantom{0}0.180\phantom{0}\,(0.401)\phantom{0} & \phantom{0}0.0532\,(0.0964) & \phantom{0}5.28\,\phantom{0}(8.80) & \phantom{0}2.49\phantom{0}\,\phantom{0}(3.71)\phantom{0} & \phantom{0}0.780\,(0.815) \\
GROOT\textsubscript{rot} & $10^{-3}$ & \phantom{00}2.57\phantom{0} & \phantom{0}0.393\phantom{0}\,\phantom{0}(2.00)\phantom{00} & \phantom{0}0.936\,\phantom{00}(2.17)\phantom{0} & \phantom{0}0.284\phantom{0}\,(0.531)\phantom{0} & \phantom{0}0.0724\,(0.117)\phantom{0} & \phantom{0}5.13\,\phantom{0}(8.13) & \phantom{0}1.66\phantom{0}\,\phantom{0}(2.01)\phantom{0} & \phantom{0}0.782\,(0.792) \\
GROOT\textsubscript{proj} & $10^{-3}$ & \phantom{00}0.233 & \phantom{0}0.140\phantom{0}\,\phantom{0}(0.655)\phantom{0} & \phantom{0}0.253\,\phantom{00}(0.340) & \phantom{0}0.172\phantom{0}\,(0.212)\phantom{0} & \phantom{0}0.0481\,(0.0565) & \phantom{0}3.39\,\phantom{0}(3.82) & \phantom{0}1.12\phantom{0}\,\phantom{0}(1.18)\phantom{0} & \phantom{0}0.748\,(0.750) \\
$\mathbf{H}_\mathrm{base} + c_t \mathbf{S}$ (ON) & $3{\times}10^{-3}$ & \phantom{00}0.292 & \phantom{0}0.147\phantom{0}\,\phantom{0}(0.793)\phantom{0} & \phantom{0}0.177\,\phantom{00}(0.277) & \phantom{0}0.173\phantom{0}\,(0.220)\phantom{0} & \phantom{0}0.0493\,(0.0592) & \phantom{0}3.41\,\phantom{0}(3.89) & \phantom{0}1.16\phantom{0}\,\phantom{0}(1.23)\phantom{0} & \phantom{0}0.754\,(0.756) \\
ROCKET 2nd stage & $3{\times}10^{-3}$ & \phantom{00}0.233 & \phantom{0}0.107\phantom{0}\,\phantom{0}(0.674)\phantom{0} & \phantom{0}0.152\,\phantom{00}(\underline{0.244}) & \phantom{0}0.146\phantom{0}\,(\underline{0.185})\phantom{0} & \phantom{0}0.0445\,(\underline{0.0531}) & \phantom{0}3.07\,\phantom{0}(\underline{3.45}) & \phantom{0}1.00\phantom{0}\,\phantom{0}(\underline{1.07})\phantom{0} & \phantom{0}0.748\,(0.749) \\
ROCKET 3rd stage & $3{\times}10^{-3}$ & \phantom{00}0.233 & \phantom{0}0.134\phantom{0}\,\phantom{0}(0.741)\phantom{0} & \phantom{0}0.159\,\phantom{00}(0.247) & \phantom{0}0.155\phantom{0}\,(0.194)\phantom{0} & \phantom{0}0.0456\,(0.0537) & \phantom{0}3.13\,\phantom{0}(3.53) & \phantom{0}1.05\phantom{0}\,\phantom{0}(1.08)\phantom{0} & \phantom{0}0.746\,(\underline{0.747}) \\
\midrule
\multicolumn{10}{l}{\textbf{After 6 epochs of NON matrix residual self-consistency training} (residual of \citet{zhangSelfConsistencyTraining2024})} \\
$\mathbf{P}_\mathrm{base}$ (NON) & $10^{-3}$ & \phantom{0}47.1\phantom{00} & \phantom{0}0.0201\,\phantom{0}(\underline{0.0724}) & 10.8\phantom{00}\,\phantom{0}(13.6)\phantom{00} & \phantom{0}0.0453\,(0.113)\phantom{0} & \phantom{0}0.0220\,(0.0401) & \phantom{0}2.27\,\phantom{0}(3.57) & \phantom{0}0.499\,\phantom{0}(0.941) & \phantom{0}0.735\,(0.763) \\
$\mathbf{P}_\mathrm{base}$ (ON) & $3{\times}10^{-4}$ & \phantom{0}30.4\phantom{00} & \phantom{0}0.0151\,\phantom{0}(0.0777) & \phantom{0}5.55\phantom{0}\,\phantom{00}(8.53)\phantom{0} & \phantom{0}0.0384\,(\underline{0.0992}) & \phantom{0}0.0165\,(\underline{0.0326}) & \phantom{0}1.44\,\phantom{0}(\underline{2.75}) & \phantom{0}0.413\,\phantom{0}(\underline{0.861}) & \phantom{0}0.719\,(0.748) \\
GROOT\textsubscript{rot} & $10^{-3}$ & \phantom{0}13.5\phantom{00} & \phantom{0}0.0411\,\phantom{0}(2.01)\phantom{00} & \phantom{0}0.474\,\phantom{00}(1.66)\phantom{0} & \phantom{0}0.0844\,(0.408)\phantom{0} & \phantom{0}0.0330\,(0.0922) & \phantom{0}2.12\,\phantom{0}(6.08) & \phantom{0}0.729\,\phantom{0}(1.39)\phantom{0} & \phantom{0}0.735\,(0.763) \\
GROOT\textsubscript{proj} & $10^{-3}$ & \phantom{00}1.17\phantom{0} & \phantom{0}0.118\phantom{0}\,\phantom{0}(1.07)\phantom{00} & \phantom{0}0.268\,\phantom{00}(0.419) & \phantom{0}0.151\phantom{0}\,(0.214)\phantom{0} & \phantom{0}0.0455\,(0.0582) & \phantom{0}3.20\,\phantom{0}(3.84) & \phantom{0}0.966\,\phantom{0}(1.10)\phantom{0} & \phantom{0}0.744\,(0.747) \\
$\mathbf{H}_\mathrm{base} + c_t \mathbf{S}$ (ON) & $10^{-3}$ & \phantom{00}0.816 & \phantom{0}0.238\phantom{0}\,\phantom{0}(1.65)\phantom{00} & \phantom{0}0.221\,\phantom{00}(0.439) & \phantom{0}0.205\phantom{0}\,(0.290)\phantom{0} & \phantom{0}0.0620\,(0.0787) & \phantom{0}4.29\,\phantom{0}(5.25) & \phantom{0}1.42\phantom{0}\,\phantom{0}(1.59)\phantom{0} & \phantom{0}0.767\,(0.770) \\
ROCKET 2nd stage & $3{\times}10^{-3}$ & \phantom{00}0.233 & \phantom{0}0.182\phantom{0}\,\phantom{0}(1.05)\phantom{00} & \phantom{0}0.173\,\phantom{00}(0.295) & \phantom{0}0.179\phantom{0}\,(0.216)\phantom{0} & \phantom{0}0.0528\,(0.0609) & \phantom{0}3.57\,\phantom{0}(3.95) & \phantom{0}1.20\phantom{0}\,\phantom{0}(1.24)\phantom{0} & \phantom{0}0.755\,(0.757) \\
ROCKET 3rd stage & $3{\times}10^{-3}$ & \phantom{00}0.117 & \phantom{0}0.159\phantom{0}\,\phantom{0}(0.774)\phantom{0} & \phantom{0}0.120\,\phantom{00}(\underline{0.201}) & \phantom{0}0.169\phantom{0}\,(0.197)\phantom{0} & \phantom{0}0.0484\,(0.0544) & \phantom{0}3.30\,\phantom{0}(3.56) & \phantom{0}1.06\phantom{0}\,\phantom{0}(1.07)\phantom{0} & \phantom{0}0.744\,(\underline{0.745}) \\
\bottomrule
\end{tabular}}
\end{table}

\newpage
\section{Runtime}
\label{app:runtime}

\begin{figure}[h]
  \centering
  \includegraphics[width=\linewidth]{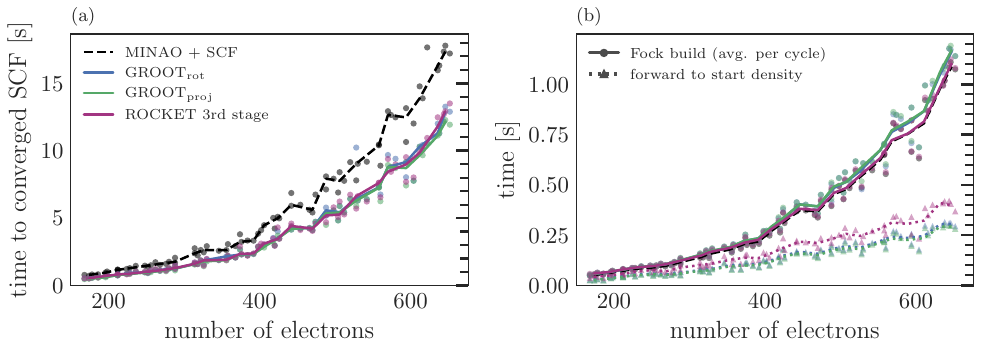}
  \caption{\textbf{Wall time to a converged SCF from ML and MINAO starts on out-of-distribution QMugs molecules.}
  \textbf{(a)} Time from the prepared molecule to a converged SCF against the electron count, for the MINAO baseline (dashed) and for SCFs started from the predictions of \ours{}\textsubscript{rot}, \ours{}\textsubscript{proj} and \ourLoss{} (third stage).
  The timed region covers the initial guess, the model input, the forward pass and the SCF; only the per-molecule system setup is excluded.
    \textbf{(b)} Solid lines: mean cost of one SCF cycle, i.e.\ SCF time divided by cycle count, dominated by the Fock build and including DIIS and diagonalization.
    Dotted lines: time from the forward pass to the starting density, including the copy to the CPU and, for \ourLoss{}, the density matrix after diagonalization of the predicted Hamiltonian; the \ours{} models start from their predicted density directly.
    Points are single molecules; lines join per-bin medians in (a) and per-bin means in (b) over 20-electron bins, without fitted curves.
    The set holds three random QMugs molecules per 20-electron bin in $[160,660)$, 75 in total, all beyond the 150-electron training range, computed with the reference settings of \cref{tab:dft-reference}.
    Each molecule runs on one NVIDIA H200.
    One discarded warm-up round absorbs compilation and the grid and density-fitting setup; the four pipelines then alternate over 10 timed rounds, and each molecule contributes its minimum.
    All runs converge, and the ML-started SCFs reach the baseline total energy to within $3\times10^{-11}\,E_h$.
    Relative to MINAO, the median time to convergence is 0.70 for both \ours{} models and 0.72 for \ourLoss{}.}
  \label{fig:runtime}
\end{figure}

\stopcontents[appendix]

\end{document}